%% file: main.tex
\documentclass[11pt]{article}

\usepackage{setup}
\DeclareFontShape{T1}{cmr}{bx}{sc}{<-> ssub * cmr/b/sc}{}
\usepackage{algorithm2e}
\usepackage{comment}
\usepackage{microtype}
\usepackage{subfigure}

\def\draft{1}
\newcommand{\mnote}[1]{\ifnum\draft=1\textcolor{red}{[\textbf{Madhu:} #1]}\fi}
\newcommand{\cnote}[1]{\ifnum\draft=1\textcolor{violet}{[\textbf{Cassandra:} #1]}\fi}
\newcommand{\anote}[1]{\ifnum\draft=1\textcolor{teal}{[\textbf{Anna:} #1]}\fi}
\newcommand{\enote}[1]{\ifnum\draft=1\textcolor{blue}{[\textbf{Elchanan:} #1]}\fi}

\input{cmds}

\renewcommand{\mathbb}{\mathbf}

\tikzset{every picture/.style={line width=0.75pt}} 

\title{Errors are Robustly Tamed in Cumulative Knowledge Processes}

\subtitle{}
\author{Anna Brandenberger{$^*$}\footnote{Supported in part by Vannevar Bush Faculty Fellowship ONR-N00014-20-1-2826, and NSF GRFP 2141064. Email: abrande@mit.edu}, Cassandra Marcussen{$^\dagger$}\footnote{Supported in part by a Simons Investigator Award and NSF Award CCF 2152413 to Madhu Sudan. Email: cmarcussen@g.harvard.edu}, Elchanan Mossel{$^*$}\footnote{Supported in part by a Simons Investigator Award, Vannevar Bush Faculty Fellowship ONR-N00014-20-1-2826, NSF award CCF 1918421, and ARO MURI W911NF1910217. Email: elmos@mit.edu}, Madhu Sudan{$^\dagger$}\footnote{Supported in part by a Simons Investigator Award and NSF Award CCF 2152413. Email: madhu@cs.harvard.edu}} 
\affiliation{{$^*$}Department of Mathematics, MIT; {$^\dagger$}School of Engineering and Applied Sciences, Harvard University}

\begin{document}

\maketitle
\thispagestyle{empty}  
\abstract{
  We study processes of societal knowledge accumulation, where the validity of a new unit of knowledge depends both on the correctness of its derivation and on the validity of the units it depends on. A fundamental question in this setting is: If a constant fraction of the new derivations is wrong, can investing a constant fraction, bounded away from one, of effort ensure that a constant fraction of knowledge in society is valid? Ben-Eliezer, Mikulincer, Mossel, and Sudan (ITCS 2023) introduced a concrete probabilistic model to analyze such questions and showed an affirmative answer to this question. Their study, however, focuses on the simple case where each new unit depends on just one existing unit, and units attach according to a {\em preferential attachment rule}. 
  
In this work, we consider much more general families of cumulative knowledge processes, where new units may attach according to varied attachment mechanisms and depend on multiple existing units. We also allow a (random) fraction of insertions of adversarial nodes. 

We give a robust affirmative answer to the above question by showing that for \textit{all} of these models, as long as many of the units follow simple heuristics for checking a bounded number of units they depend on, all errors will be eventually eliminated. Our results indicate that preserving the quality of large interdependent collections of units of knowledge is feasible, as long as careful but not too costly checks are performed when new units are derived/deposited.\footnotemark
}{}

\footnotetext{Accepted for presentation at the Conference on Learning Theory (COLT) 2024}
\newpage 
\thispagestyle{empty}
\tableofcontents

\newpage
\pagenumbering{arabic}

\section{Introduction}

A fundamental question in the study of information and knowledge and their processing is: how robust is the processing to errors? In the well-studied setting of information transmission, this leads to the fields of information and coding theory, and a foundational result here is that a (sufficiently small) constant fraction of errors can be remedied by a constant fraction of additional redundancy. Thus most channels can operate at a positive constant rate, where the effort (redundancy) needed to protect information from errors does not overwhelm the amount of new information being transmitted. In this work, continuing on the work of Ben-Eliezer, Mikulincer, Mossel, and Sudan~\cite{BMMS:23}, we explore the analogous question in the context of {\em cumulative knowledge}. 

The phrase {\em cumulative knowledge} refers to the fact that societal knowledge consists of a collection of units of knowledge that depend on each other. In such processes of knowledge accumulation, a major concern is that, as bodies of knowledge become larger, a significant fraction of nodes may become invalid (meaning they either introduce an error or rely on a faulty unit of knowledge). Over the years many papers \cite{IannidisSelfCorrect2012, Ioannidis2005, Grcar2013, LarcombeEnvironmental2018} have emphasized this problem for scientific publication networks. Because new units of knowledge rely on potentially incorrect prior units, any error in cumulative knowledge can result in future slow-downs and the spread of false information; see, for example, \cite{lehrer2010truth}. Of particular interest are corpuses where the validity of different units of knowledge is interdependent. Key examples include corpuses of scientific publications mentioned above, where the validity of results of one paper may depend on the validity of a cited paper, and code libraries where the correctness of units of libraries depends on the correctness of dependent libraries. 

It is perhaps natural to expect that in such corpuses, when a new unit of knowledge is introduced, the creators of the unit will not only check/debug their own creation/code/derivations but also some of the units they depend on. 
The recent paper \cite{BMMS:23} formalized a family of such models and showed that in the simple case where each new unit of knowledge connects to one existing unit of knowledge according to a preferential attachment distribution, errors will be eventually eliminated if a (random) fraction of the nodes checks a bounded number of nodes they depend on. 

Our main interest in this paper is to study if similar results hold for more general and more realistic models. In particular, we wish to consider networks where the attachment rules are more flexible than preferential attachment, new units may depend on more than one existing unit, and some of the nodes behave in an adversarial fashion, not only intentionally introducing incorrect units of knowledge but also strategically connecting to other units. 

The main result of this paper is that under a broad class of knowledge accumulation processes, natural error elimination heuristics will eliminate all of the errors in the process even when each new unit of knowledge checks at most a constant number of units it depends on. We also study the conditions under which the error effect survives forever.

\subsection{Cumulative Knowledge Processes}
Cumulative Knowledge Processes, with errors and checking, are characterized by three key ingredients: (1) The growth model that specifies how new units of knowledge are created and how they may depend on previous units of knowledge. (2) An error model that specifies how and when errors may happen. (For example, do errors simply affect the (in)correctness of individual pieces of reasoning, or do they also violate assumptions of the growth model?) (3) The checking process that determines how often checks are performed, how high up the chain of dependencies the checks go, and how broadly checks search for potential errors.

The aforementioned prior work in this field~\cite{BMMS:23} made one specific choice for each of these aspects and analyzed the ability of the resulting processes to cope with errors. For instance, for the growth model, they consider and analyze a preferential attachment model, and additionally, new units can only depend on one existing unit. This sheds no light on other attachment models, such as the uniform attachment setting or settings where new units connect to multiple existing units. The main quest of the current work is to go beyond the specifics of any one choice of the first two ingredients (the growth and error models) and rather suggest error-checking mechanisms that may be robust to the exact choice of these (up to reasonable limits). To do so, however, one needs to study broad classes of growth and error models without enumerating them individually. A conceptual contribution of this paper is to describe large families of growth models in a ``uniform'' way. We elaborate below.

\paragraph{Growth Models and Attachment Functions} In this work, we consider a model of cumulative knowledge where each unit of knowledge is given by a node, and edges point from a node to the nodes that directly depend on it. The growth model specifies how a new node chooses the set of parents it depends on. In this work, we capture a broad class of growth models by looking at an attachment function $\mathbf{a}:\Z_{\geq 0}\to \R_{\geq 0}$ such that a node of degree $d$ is chosen as a parent with probability proportional to $\mathbf{a}(d)$. This simple abstraction allows us to capture a whole spectrum of growth models uniformly with $\mathbf{a}(d) = 1$ capturing uniform growth and $\mathbf{a}(d) = d + 1$ capturing preferential attachment. 
Additionally, we allow new units to connect to a random number of parents $M$ called the combination factor, where each is chosen according to the attachment function.
(And our results do work very broadly in this spectrum of models.) 

\paragraph{Error models} The main consideration in our work is the impact of errors in cumulative knowledge. Errors originate because some new units of knowledge (new nodes) were incorrectly derived from the previous knowledge (parent nodes). We refer to these nodes as ``Conditionally False'' or $\CF$ nodes (as opposed to the rest of the nodes that are conditionally true). 
The fact that errors are not discovered immediately implies that these nodes continue to have children according to our growth process until some check reveals the error, at which point the nodes stop getting new children. All nodes that have a $\CF$ node as an ancestor are $\False$ nodes since their truth depends indirectly on an incorrect derivation, but discovering their falsehood is up to the checking process. Nodes that are not $\False$, called $\True$ nodes,  represent the actual knowledge accumulated by the process and the goal is to build checking processes that ensure that $\False$ nodes do not overwhelm the corpus of knowledge.

In our setting (as in previous settings), a node is deemed a $\CF/\CT$ node at the time of its birth, by an independent coin toss (with a parameter $\epsilon$ being the probability of its being $\CF$). The novelty in our error models is to allow some even more {\em adversarial} tampering.  With some positive probability $q$ at each growth step, we allow an adversary to inject a new node (which may itself be $\CF$ or $\CT$) with up to $r$ adversarially selected parents. We seek checking processes that allow such adversarial tampering with positive $q$ and bounded $r$. 

\paragraph{Checking process} In our model(s), nodes that are $\CF$ are not revealed to be $\False$ at birth and continue to have children. Incorrect nodes are discovered only when some newborn node has a given node as an ancestor and checks it and discovers the error. When the error is discovered, the entire path from the newborn node to the $\CF$ ancestor is proclaimed false $\PF$. $\PF$ nodes no longer participate in the growth process (in that they don't get any more children) --- but note that their descendant count may continue to grow due to the existence of $\CT$ and $\CF$ descendants. 

The checking process involves a number of choices and parameters: the probability with which a newly born node decides to check its parents, the sequence of ancestors that are checked, and the stopping condition. In the previous work these choices were made somewhat simpler due to the condition that each node had only one parent, so the sequence was determined naturally: check one's parent, then grandparent, and so forth up to a prescribed bound on the depth of the checking or till an error is found. In our setting, checking mechanisms become more complex and we outline several natural options including breadth-first search and exhaustive checking up to a certain depth limit. In all cases, the important condition is that the number of nodes checked is bounded independently of the size of the graph (to ensure that the effort of checking is only a constant factor larger than the effort of producing new units of knowledge) and the probability of checking is bounded away from 1 (to model the fact that not everyone may follow the rules, either out of laziness or malintent). 

The variability of the growth, the error, and the checking processes lead to an entire family of cumulative knowledge processes that we describe formally in the following section. 

\subsection{The family of processes}

We model the cumulative knowledge process described in the prior subsection as 
a sequence of states $X_0, X_1, \dots, X_t, \dots$, where each state $X_t = (\G_t, L_t)$ consists of body of knowledge represented by a directed acyclic graph $\G_t$ of size $t+1$ as well as an accompanying set of labels $L_t$ (each node is $\CT$, $\CF$ or $\PF$, where $\CT$ and $\CF$ nodes are both publicly considered $\PT$). The $i$th element of $L_t$ is the label of the node added at time $i$.

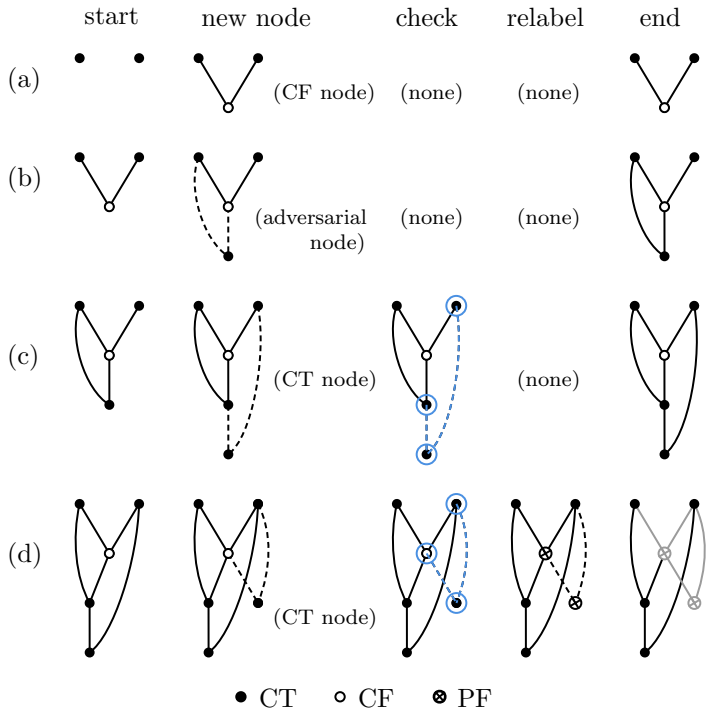
\begin{SCfigure}[.8]
    \centering
    \input{Visuals/4-step-evolution-CKP}
    \caption{Example evolution of a cumulative knowledge process under \\ (a) adding an erroneous node and running no checks, \\ (b) adversarially adding a true node (no checks), \\ (c) adding a true node and running a check of depth 1, which does not successfully find the false node, \\ (d) adding a true node and running a check, successfully finding a false node.}
    \label{fig:evolution}
\end{SCfigure}

\paragraph{Generalized Cumulative Knowledge Processes (CKP)} 
A family is specified by a set of features
\begin{equation}\label{eq:params}
    \left(\textsc{Attach} = (\mathbf{a}, M), \textsc{Error} = (\e), \textsc{Adversary} = (q, r), \textsc{Check} = (p, k) \right),
\end{equation}
where $\mathbf{a}: \mathbb{Z}_{\geq 0} \to \mathbb{R}_{\geq 0}$ is an attachment function, $M$ the combination factor is a bounded random variable $(1 \leq M < \infty)$, and the remaining are all constants: $\e \in [0,1)$ the error probability, $q \in [0,1)$ the probability of an adversarial node and $r \in \N$ its maximal number of parents, $k \in \N$ the radius of a check, and finally $p \in [0,1)$ the checking probability. 

A CKP begins with a starting state $X_0$ being a $\CT$ node (though our analysis could allow for multiple $\CT$ nodes in $X_0$). Given $X_t = (\G_t, L_t)$, the state $X_{t + 1}$ is created through the following steps. 

First, a new node is added to $\G_t$. 
With probability $1 - q$, this step evolves non-adversarially, according to a random process. That is, the node connects to a random number of parent nodes $m$ drawn according to the random variable $M$, each of which is chosen according to the attachment function $\mathbf{a}$, which is a function of the number of $\PT$ children $\deg_{\PT}$ of a node. This means that, for each selection of a parent node, the node $v$ is chosen with probability proportional to $\mathbf{a}(\deg_{\PT}(v))$. 
The new node introduces an error (is labeled \CF) with probability $\e$ and is $\CT$ otherwise. After this, with probability $p$, the new node performs a local checking procedure. 
This checking procedure examines the truth values of ancestor nodes within a constant radius $k$; if any $\PF$ or $\CF$ node is detected, a path of nodes from the new node to this node is marked as $\PF$. 
 
Otherwise (with probability $q$), an adversary controls the next step, deciding which parent nodes the new node connects to (up to $r$ parents) as well as the label of the new node. 

\paragraph{Simple CKP} We also consider a simplification where we initialize $X_0$ as a $\CF$ node, and set $\e = 0$ and $q = 0$. That is, we start with a \CF\ node, followed by all non-adversarial \CT\ nodes.

\paragraph{Preferential-tree CKP example} As mentioned earlier, setting $\mathbf{a}$ according to the preferential attachment rule $\mathbf{a}(d) = d + 1$, $M$ to be the constant $1$, and $q = 0$ yields the tree-like CKP defined and analyzed in \cite{BMMS:23}.

\subsection{Main results}

Our results indicate that natural heuristics suffice for preserving the quality of large bodies of knowledge. We are interested in eliminating errors in bodies of knowledge, defined formally below. 

\begin{defn}[Error elimination] \label{def:error-elimination}
    The error effect in a CKP
    is completely eliminated if for all $\CF$ nodes $u$, there exists a time $t'$ such that for all $t \geq t'$, the sub-DAG $\G_t^{(u)}$ rooted at $u$ at time $t$ is completely marked as $\PF$.
\end{defn}

Our error elimination result holds for a general class of \textit{regular} attachment functions. 

\begin{defn}[Regular attachment functions] \label{def:regular-attachment}
$(b_1, b_2)$-regular attachment functions $\mathbf{a}: \Z_{\geq 0} \to \R_{\geq 0}$ satisfy the following boundedness conditions: 
(i) for all $d \in \mathbb{Z}_{\geq 0}$, $b_1 \leq \mathbf{a}(d + 1) - \mathbf{a}(d) \leq b_2$, with $b_1, b_2 \geq 0$, and (ii) 
$\mathbf{a}(0) \geq 1$. Note that $b_1 \geq 0$ implies $\mathbf a$ is non-decreasing.
\end{defn}

Setting $\mathbf{a}(0) = 1$, the $(1,1)$-regular and $(0,0)$-regular attachment functions are respectively the well-known preferential and uniform attachment models.

\begin{thm}[Error elimination, informal] \label{thm:intro-error-elimination} For natural checking processes that we propose, for every $(0, b)$-regular attachment function and every combination factor $M \geq 1$, and any adversary bound $r \in \Z_{\geq 0}$, there exists $ q_0 \in [0,1)$ such that for every low enough adversarial probability $q \leq q_0$, and every error probability $\epsilon < 1$, there exist $p_0 < 1$ and $k_0 \in \N$ such that for all large enough checking parameters $p \geq p_0$ and $k \geq k_0$, error elimination occurs.
\end{thm}
For a formal version, see Theorem~\ref{thm:error-elimination-general}, which applies to local error-checking procedures given in Definition~\ref{checking-procedures}. More concretely, this result says that for arbitrarily bounded attachments and adversaries, if the adversarial steps occur with low enough probability, we can ensure the error is eliminated for any error probability, as long as the checking probability and depth are high enough. Interestingly, this result also holds for noisy checks. Specifically, if checking a $\CF$ node only reveals the error with some probability $p_e$ (independently, for each $\CF$ node), then we still have error elimination as long as the product with the checking probability $p$ satisfies the cutoff $p\cdot p_e \geq p_0$; see Remark~\ref{remark:noisy-check}. Note also that \cite[Theorems 3,7]{BMMS:23} can be obtained as a sub-case of our result, by setting $\mathbf{a}(d) = d + 1$ (preferential attachment), $M = 1$, and $q = 0$. 

We note in Section \ref{subsection:fraction-false-nodes} that this asymptotic error elimination also implies that, at any time, the proportion of undetected \False\ nodes is small; see Proposition \ref{prop:proportion-of-false-nodes}.

We also study the conditions under which error survival takes place, i.e., when we do not have error elimination. Specifically, error survival is defined as follows.

\begin{defn}[Error survival] \label{def:error-survival}
    The error effect in a CKP survives if with some positive probability, there exists a $\CF$ node $u$ such that for all time steps $t$, there exists at least one $\PT$ node in the sub-DAG $\G_t^{(u)}$ rooted at $u$.
\end{defn}

Our error survival results demonstrate the non-triviality of error elimination: without sufficiently frequent and extensive error-checking mechanisms the error effect may never be eliminated. First, a simple argument shows the following error survival result, which applies to all CKPs.
\begin{prop}\label{prop:error-survival-p-less-than-epsilon}
    For natural checking processes that we propose, for every CKP, whenever $p < \e$, the error effect survives forever with positive probability.
\end{prop}
See Lemma~\ref{lemma:minimal-false-potential-survival} for the formal version of this result. For CKPs with attachment functions $\mathbf{a}$ that are $(\mathbf{a}(0), b)$-regular and small values of $\E\{M\}$, we obtain more refined error survival results. These results depend on $b$, $\E\{M\}$, and the minimum value $\min(M)$ that $M$ can take.

\begin{thm}[Error survival, informal] \label{thm:intro-error-survival}
 For natural checking processes that we propose, for every $b \in \mathbb{R}_{\geq 0}$ and $(\mathbf{a}(0), b)$-regular attachment function $\mathbf{a}$, and every bounded random variable $M \geq 1$, 
 there exist small enough adversary parameters such that the error survives for all $p < p_0$ for some $p_0 \geq \e$, for any $k \in \N$. 
\end{thm}
For a formal version, see Theorem \ref{thm:general-ckp-error-survival}. In particular, this bound $p < p_0$ is better than the $p < \e$ bound from Proposition~\ref{prop:error-survival-p-less-than-epsilon} for several parameter settings. For instance, for attachment models with $(\mathbf{a}(0), b)$-regularity and small constant $M$, for small enough adversarial parameters $q$ and $r$, there exists $\e$ small enough such that the checking probability cutoff $p_0$ from Theorem~\ref{thm:intro-error-survival} exceeds $\e$. 

Interestingly, the error survival guarantees do not depend on the checking depth $k$, indicating that a large checking depth cannot counteract the spread of errors when the overall checking probability $p$ is too small.

\begin{rem}
    The error survival phase also justifies our choice of regular attachment functions. Indeed, for any CKP with attachment function growing as $\mathbf{a}(d) \gtrsim d^3$, the error survives forever with positive probability (see Lemma~\ref{lem:growing-a}). The same holds for attachment functions with ``holes" in their distribution, i.e., if $\mathbf{a}(d) = 0$ for any $d \in \Z_{\geq 0}$. These are ruled out by our monotonicity requirement, and indeed do not make too much sense in terms of dependency network modeling.
\end{rem}

The error elimination and survival results, while natural, require some work to show, due to the locality of the process and our inability to rely on the structure of the CKP over all probabilistic events. Indeed, even simple properties such as monotonicity in the checking parameters $p$ and $k$ (it should be easier to eliminate errors with higher checking probability or higher checking depth) are not obvious to show. In this paper, we prove that error elimination is \textit{monotonic} with respect to the checking parameters (checking probability and checking depth) for the simple CKP where each node connects to one parent node. This answers an open question presented in \cite{BMMS:23}. We conjecture that monotonicity holds for the case when nodes can connect to more than one parent node, and it remains an open question to prove this. The existence of monotonicity certifies that the family of models and the checking heuristics are natural and align with the intuition that more checking (higher probability or greater depth) can only improve the chances of eliminating all of the error effect. We prove the following result related to monotonicity.

\begin{thm}[Monotonicity with respect to checking parameters] \label{thm:intro-monotonicity}
    For all $p_2 \geq p_1$ and $k_2 \geq k_1$, if the error effect is completely eliminated in the simple CKP with features $(\textsc{Attach} = (\mathbf{a}, M=1),  \textsc{Check} = (p_1, k_1))$, then the error effect is completely eliminated in the simple CKP with features $(\textsc{Attach} = (\mathbf{a}, M=1),  \textsc{Check} = (p_2, k_2))$.
\end{thm}

\subsection{Related work}

We study general families of processes of societal knowledge accumulation. As mentioned earlier, prior results from \cite{BMMS:23} only deal with processes evolving according to preferential attachment in which each node connects to only one parent, and additionally do not account for adversarial behavior. When we focus on CKPs whose attachment function is preferential attachment and there is no adversarial behavior, our results imply the following phenomena that present contrasts between CKPs with one parent or multiple parents and answer questions raised in \cite{BMMS:23}.

First, it was proven in \cite{BMMS:23} that when each node connects to one parent node and the checks are shallow (meaning $k = 2$), for any $0 \leq p < 1$, the error effect in the simple preferential attachment tree-CKP $(M=1)$ with $\textsc{Check} = (p, k)$ survives with positive probability. On the other hand, we prove that if $\E\{1/M\} \leq 1/3$, for sufficiently large values of $p$, shallow checks eliminate the error effect completely for preferential attachment CKPs. Additionally, when $\E\{1/M\} \leq 1/2$, we prove that error elimination occurs for $k \geq 3$ and sufficiently large values of $p$, answering a question raised in \cite{BMMS:23}.

We now highlight a range of other related works.

\paragraph{Models of noisy computation} Our model can be thought of as a model of noisy computation, as it models the errors within and dependencies between units of knowledge, which may be viewed as computational units. There is a vast literature on noisy computation, see e.g.~\cite{vonNeumann:56, EvansSchulman:99, PMC:67, AMP:20}. As mentioned in~\cite{BMMS:23}, the main difference in the setting is that in noisy computation models, one is allowed to design the structure of the network, while the design choices in our setting are limited: e.g., in a scientific publication network, we cannot arbitrarily design the prior results that an author relies on.

A vein of this literature models the broadcasting of one bit of information in a noisy network, with error correction, see e.g.~\cite{information-spread,gray2001reader,MaMoPo:20}. These models do not involve any accumulation of new knowledge, rather they study the possibility of retaining some starting knowledge. 

\paragraph{Error-resilience of preferential attachment networks} Our paper considers a family of processes each evolving according to a general attachment function $\mathbf{a}: \mathbb{Z}_{\geq 0} \to \mathbb{R}_{\geq 0}$ that specifies the probability of attaching to a node as a function of its $\PT$-outdegree. A classic choice for $\mathbf{a}$ yields the preferential attachment model, a well-studied model of networks used to theoretically model the development of the web, citation networks, and other real-world network processes \cite{amaral2000classes,jeong2003measuring}. A range of previous papers study error resilience of preferential attachment networks. These papers, e.g. \cite{bollobas_riordan_2003, flaxman_frieze_vera_2007, Deijfen_2009, DeoCami2007, bollobas_riordan_2003_robustness},  focus on the retention or loss of aspects of the preferential attachment network's structure -- such as its scale-free property, degree distribution, or large components -- under various assumptions about deletions in the network. A deletion can be compared to marking a $\False$ node as $\PF$ in our paper. In contrast, our results do not rely on the retention of structural properties of preferential attachment networks and we instead conjecture that much of this structure is lost in the CKP. We also consider a model of deletion that happens as the graph is generated and which deletes nodes along \textit{paths}, contrasting the aforementioned papers. Thus the results on error-resilience of preferential attachment networks shed no light on knowledge accumulation as defined in this work. 

\paragraph{Multitype Preferential Attachment Networks}
We may consider a different model, where nodes connect in a preferential attachment fashion and the probability that a unit of knowledge is correct depends on the type (e.g. correct/incorrect) of the nodes it connects to. 
For such models, it is natural to study the fraction of correct nodes in the limit. Such models are analyzed in \cite{antunovic2016coexistence,briend2023broadcasting}. 

\subsection{Open problems}

\paragraph{Alternative attachment functions} In this paper, we considered a broad family of models of knowledge accumulation whose underlying corpuses of knowledge evolve according to an attachment function that depends on the outdegree of nodes. More generally one may ask about preserving the quality of large corpuses of knowledge where new units of knowledge attach to existing units according to other properties. For example, other properties such as the height of nodes, the time step in which a node was added, or some other notion of ``importance'' of nodes could be considered. 

Additionally, one may ask about attachment functions with alternative procedures in which parent nodes are not chosen independently and instead are correlated with each other. The attachment procedures we consider, which choose parent nodes independently, particularly make sense in contexts where the corpus of knowledge is cohesive (for example, theorems proven in a certain field of study). However, it would be compelling to consider broader settings as well. An example in which the choice of parent nodes is correlated is the geometric preferential attachment model (\cite{flaxman_frieze_vera_2004_Geo1, flaxman_frieze_vera_2007_Geo2}) in which nodes connect to existing nodes according to a combination of proximity and preferential attachment. Another example would be an attachment procedure in which, after the first parent is chosen, the remaining parent nodes are chosen to be within a fixed radius of the first parent node.

\paragraph{Monotonicity with respect to the checking parameters} We conjecture that error elimination is monotone with respect to the overall checking probability $p$ and the checking depth $k$. That is, for $p_2 \geq p_1$ and $k_2 \geq k_1$, if a CKP with $\textsc{Check} = (p_1, k_1)$ eliminates all errors, then a CKP with $\textsc{Check} = (p_2, k_2)$ (and the same parameter specifications for $\textsc{Attach}$, $\textsc{Error}$, and $\textsc{Adversary}$) eliminates all errors. We have proven such a monotonicity result on trees in the simple case, i.e. the case where $M = 1$, $\e = 0$, and $q = 0$.

When $M$ is a fixed constant, we also conjecture that error elimination is monotone with respect to the number of parents. That is, for $m_2 > m_1$, if a CKP with $\textsc{Attach} = (\mathbf{a}, M = m_1)$ eliminates all errors then a CKP with $\textsc{Attach} = (\mathbf{a}, M = m_2)$ (and the same parameter specifications for  $\textsc{Error}$, $\textsc{Adversary}$, and $\textsc{Check}$) eliminates all errors.

\begin{figure}[hbtp]
    \centering \textcolor{white}{\subfigure{\includegraphics[width=0.3\textwidth]{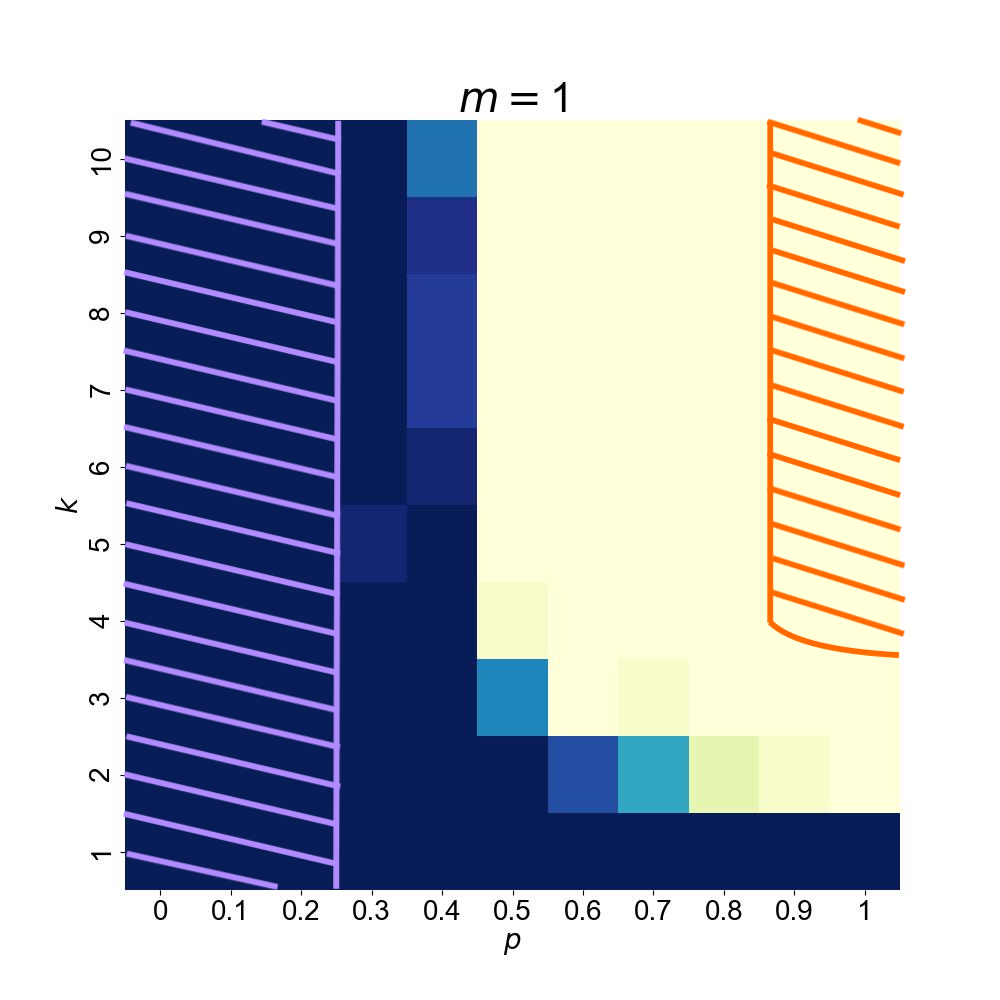}} 
    \subfigure{\includegraphics[width=0.3\textwidth]{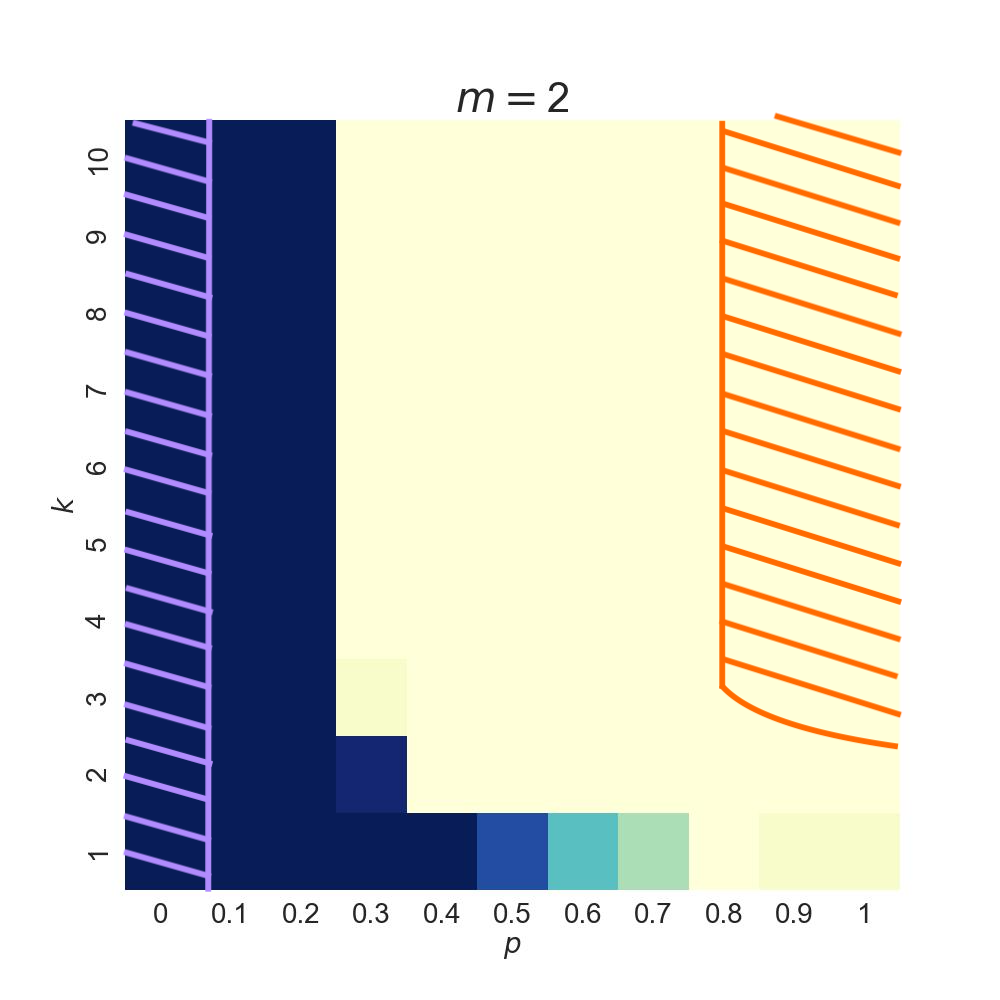}} 
    \subfigure{\includegraphics[width=0.3\textwidth]{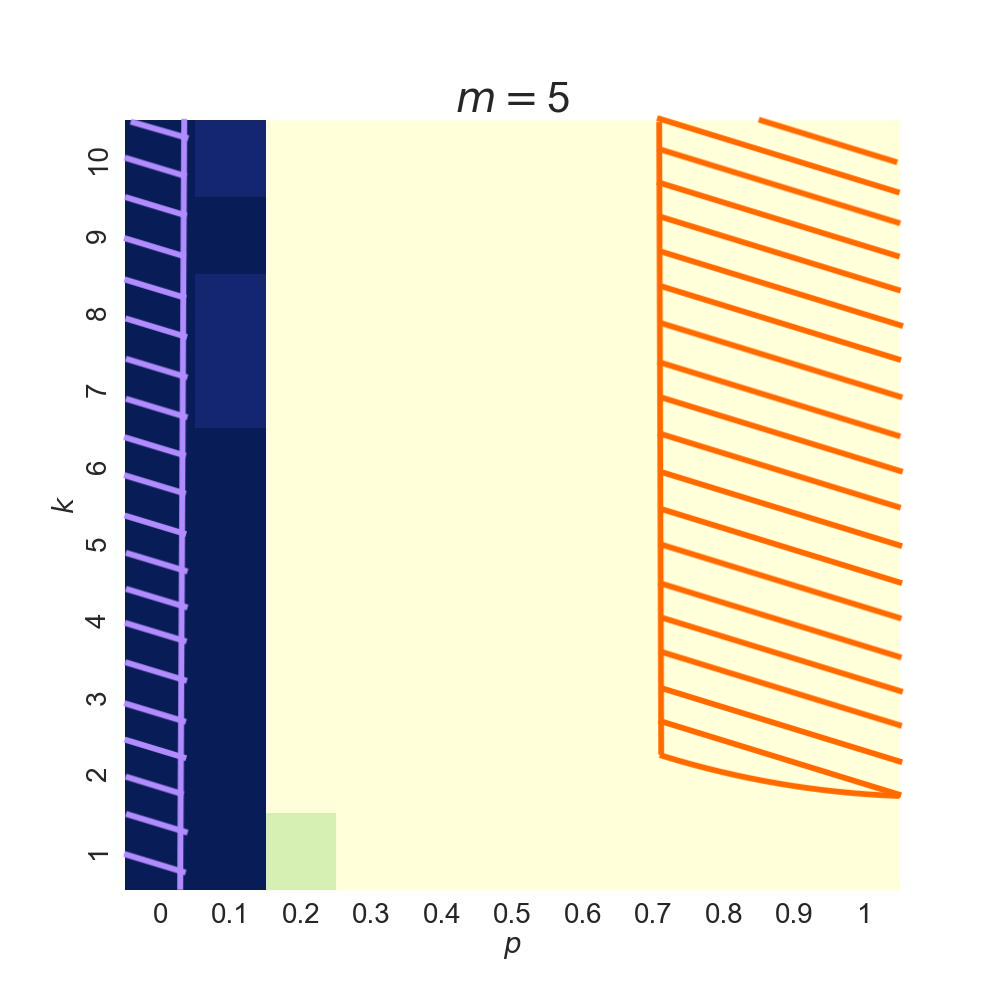}}
    \subfigure{\includegraphics[height=0.29\textwidth]{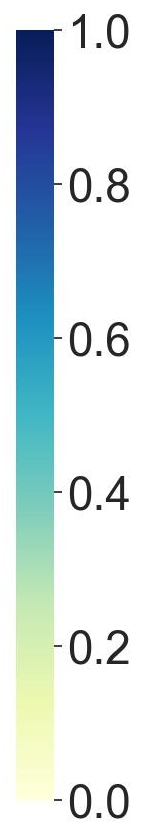}}
    }
    \vspace{-.6cm}
    \caption{ \textcolor{black}
    Simulations of the simple CKP with $\textsc{Check} = (p, k)$ evolving according to preferential attachment where new nodes connect to $m$ existing nodes, for fixed $m = 1, 2, 5$. The checking mechanism is the \textsc{Exhaustive BFS} checking mechanism. The parameter regime in which we prove that errors survive with positive probability is shaded in purple,
    while the proven error extinction region is shaded in orange.
    The heat map displays the percentage of trials that survived until time step 2000. We run 20 simulations for each $(m,p,k)$ choice, with an initialization of a chain of 25 nodes (one $\CF$ followed by 24 $\CT$ nodes) with $m$ edges between each node and its parent.}
    \label{fig:intro-plots}
\end{figure}

See Figure~\ref{fig:intro-plots} for a few simulations of CKPs with $M$ set to various constants which suggest that monotonicity holds in both contexts mentioned. More figures and information can also be found in the appendix~\ref{section:simulations}.

\paragraph{Parameters} We conjecture that there is a phase transition between error survival and elimination with respect to $p$ and $k$, for all processes in the family of knowledge accumulation processes considered. This phase transition may vary with the parameters for attachment, errors, and adversarial steps. Understanding the precise parameters that yield error survival and elimination requires further study. 

We created and ran CKP simulations for the special case where $\textsc{Attach} = (\mathbf{a}, M)$ sets $\mathbf{a}$ to be the preferential attachment function and $M = m$ to be constant. Experiments on these simulations indicate where the conjectured phase transition takes place with respect to $p$ and $k$ for different values of $m$. See Figure~\ref{fig:intro-plots} for the simple CKP, and Appendix~\ref{section:simulations} for the general CKP and further plots of both cases.

\paragraph{Structural properties of CKPs} Based on simulations we performed, we conjecture that CKPs evolving according to preferential attachment lose their preferential attachment structure when the checking probability $p$ is sufficiently large. We also conjecture that if the error survives in a CKP with preferential attachment and $p$ is sufficiently far away from $0$, the CKP becomes shallow, meaning every node is close in distance to all of its ancestor nodes. We remark that for $p$ that is very close to $0$, a CKP should retain its preferential attachment structure, due to the robustness of preferential attachment under small amounts of adversarial deletion~\cite{bollobas_riordan_2003, flaxman_frieze_vera_2007, Deijfen_2009, bollobas_riordan_2003_robustness}.

\subsection{Proof sketches}

We outline the key ideas used in the proofs of our main results.

\paragraph{Proof of Theorem \ref{thm:intro-error-elimination} (Error elimination)} The full proofs can be found in Section~\ref{section:error-elimination}.

In this proof sketch, we focus on error elimination in a simplified model where $X_0$ consists of one $\CF$ node and no additional $\CF$ nodes are added to the process $(\e = 0)$. The techniques used to prove the full error elimination results for the general CKP model are similar. To prove our error elimination results, we carefully design a potential over steps $X_t = (\G_t, L_t)$ of a CKP which on one hand we can prove forms a super-martingale converging to $0$, 
and on the other hand also upper bounds the number of false nodes in $\G_t$ (note that since the $X_0$ is \CF, the entire DAG is \False). 
We can then conclude that there exists a time $t'$ such that for all $t \geq t'$, the error effect in $\G_t$ is completely eliminated.

We define a potential that captures a natural quantity of a step of a CKP: the distances between nodes and their closest minimal false ancestor node. The closest minimal false ancestor node is the first node that would be found to be $\False$ by a check initiated at the node. Given a directed acyclic graph $\G_t$ with labels $L_t$ of the vertices of the graph, and $c > 1$, we define the \textit{minimum-distance potential} $\Phi(\G_t)$ as follows:
$$  \Phi(\G_t) = \sum_{v \in \G_t^{\PT,F}} \mathbf{a}(\deg(v)) \cdot c^{|v|},
$$
where $\G_t^{\PT,F} = \G_t^\PT$ is the sub-DAG of $\PT$ \False\ nodes in $\G_t$, $\mathbf{a}: \mathbb{Z}_{\geq 0} \to \mathbb{R}_{\geq 0}$ is the attachment function according to which the DAG evolves, and $|v|$ is the number of edges on the shortest path from $v$ to a minimal false node. Note that if $v$ is itself a minimal false node, then $|v| = 0$.

We prove error elimination under a checking mechanism that searches for the closest minimal false node (up to distance $k$ away from the new node). This has a natural correspondence with the minimum-distance potential, which measures the distance to this node. Let $m \sim M$ be drawn according to the combination factor. In the analysis, we consider the following cases.

\begin{itemize}
    \item When the next step is non-adversarial, there are three possible cases:
\begin{itemize}
    \item If the new node is $\CT$ and performs a successful check, then we can identify an entire sub-DAG of nodes whose minimum-distance values all reduce by at least one.
    \item If the new node is $\CF$ and performs a successful check, only the new node gets removed.
    \item If the new node does not perform a check or the check is unsuccessful, the new node contributes to the potential. Suppose the new node connects to parent nodes $u_1, \dots u_m$. Then $|v| = 1 + \min_{i \in [m]} |u_i|$. To obtain a dependence on $M$ in our error elimination result, we use that the minimum is less than the average: $ 1 + \min_{i \in [m]} |u_i| \leq  1 + \frac{1}{m}\sum_{i \in [m]} |u_i|$.
\end{itemize}
    \item Otherwise, an adversary determines the attachment of the new node (number of parent nodes and which parent nodes) as well as its truth value ($\CF$ or $\CT$). The worst-case contribution to the potential occurs when the adversary places the new node as far away as possible from a minimal-false node and connects the node to many parent nodes, though the increase in the potential is still bounded in this case.
\end{itemize}
We show that the expected change in the potential is negative at each time step in the CKP when the potential is not zero. The potential forming a super-martingale then allows us to show error elimination.

\paragraph{Proof of Theorem \ref{thm:intro-error-survival} (Error survival)} The full proofs for error survival can be found in Section \ref{section:error-survival}. 

We study error survival for generalized CKPs. These results indicate that the occurrence of error elimination in families of processes of knowledge accumulation is non-trivial. Our proofs also shed light on the components of knowledge accumulation processes that cause errors to propagate. First, for any CKP, we can prove that error survives whenever $p < \epsilon$; however, we also obtain more refined bounds depending on the other features of CKPs. In this section, we highlight the ideas of the proof of the more refined bounds. Our error survival results depend on a lower bound $\min (M)$ on the number of parent nodes that new nodes connect to. 

To prove error survival, we identify several key structural components of the CKP that propagate the error effect. The components identified are: $\CF$ nodes, root nodes ($\CT$ nodes with a $\PF$ parent), and leaf nodes. First, $\CF$ nodes introduce errors to a CKP. Root nodes (and the respective sub-DAGs rooted at these nodes) are $\False$ because they have a $\False$ ancestor; until marked as $\PF$, these nodes can yield future undetected $\False$ nodes. Lastly, leaf nodes propagate errors because there is a relatively low probability of connecting to a leaf node and checking if it has a $\False$ ancestor node. We use that the expected number of non-root leaf nodes that a new node with $m \sim M$ parents connects to is less than $1$. On the other hand, every new node is a leaf node, and thus the expected number of leaves grows over time.

Using these three structural components, we define a potential on DAGs which incorporates the number of minimal false ($\CF$ and root) nodes and leaf nodes. We show the expected change in the potential is positive at each time step in the CKP, which we then prove implies error survival.

\paragraph{Proof of Theorem \ref{thm:intro-monotonicity} (Monotonicity for simple tree-CKPs)} The full proofs for monotonicity can be found in Section \ref{section:monotonicity}. 

We prove monotonicity of error elimination with respect to the two checking parameters: $p$ and $k$. We separate the proof of the monotonicity of error elimination into two parts: monotonicity with respect to $p$, and monotonicity with respect to $k$. Since the two parts are similar, we now focus on the proof of monotonicity with respect to $p$. Below, we consider CKPs that evolve according to the same attachment function $a$ and perform checks of the same height $k$, and therefore only denote the simple tree-CKPs by their checking probability $p$, as CKP($p$).

For the proof, we construct a coupling of CKP($p_1$) and CKP($p_2$), where $p_2 \geq p_1$, such that CKP($p_2$) (which runs checks with higher probability and thus has fewer nodes) is embedded inside the coupled CKP($p_1$). At a high level, we ensure that all nodes in CKP($p_1$) are also nodes in CKP($p_2$) or have already been removed by the CKP($p_2$). When a new node is added to CKP($p_1$), if it connects to a node that is still alive (not caught to be \False) in the  CKP($p_2$), we also add it to the CKP($p_2$). If it connects to a node that was found to be \False\ in the CKP($p_2$), it is not added to the CKP($p_2$) and the CKP($p_2$) does not update. 

At each time step of the coupling, with probability $p_1$, both CKPs perform checks. With probability $p_2 - p_1$, only the CKP($p_2$) performs a check. We keep track of the nodes that are alive in the CKP($p_1$) but dead (caught to be \False) in the CKP($p_2$), which we call ``zombie'' nodes.

Because the CKP($p_2$) may not update when the CKP($p_1$) does, the two CKPs may evolve at different rates, but we can still always pair time steps in one process with the other to ensure that error elimination in the CKP($p_1$) implies error elimination in the corresponding CKP($p_2$). It is possible to generate every CKP($p_2$) through the mechanism considered in this coupling. Therefore, assuming the coupled CKP($p_1$) process eliminates all errors, we can conclude that error elimination is monotonic with respect to the checking probability. 

\section{The family of generalized Cumulative Knowledge Processes}\label{section:model-definitions}
We consider a family of processes that model knowledge accumulation. This family of processes includes and generalizes the Cumulative Knowledge Process of \cite{BMMS:23}, in which each new node depended on only one parent and attached according to a preferential attachment rule. Our family of processes extends the model simultaneously in several directions: it allows nodes to connect to a random number of parent nodes, allows more flexibility in the attachment procedure, and incorporates adversarial errors along with random (accidental) errors. We also consider different natural heuristics for checking for errors that are more suitable for this general family of processes.

\paragraph{Labels} Each node in a given CKP state is labeled according to $L_t$, where the $i$th element of $L_t$ is the label of the node added at the $i$th time step and indicates the private and public truth values of this node. Each label is from the set $\{\PF, \CF, \CT\}$, which respectively stand for ``Proclaimed False,'' ``Conditionally False,'' and ``Conditionally True.'' A node is $\False$ if it or any of its ancestors is $\CF$, and is considered $\True$ otherwise. Nodes that are labeled as $\CT$ and $\CF$ are both perceived as being $\True$ until they are checked. $\CF$ and $\CT$ nodes are also referred to as ``Proclaimed True'' ($\PT$) because of this. On the other hand, $\PF$ nodes have publicly been found to be $\False$. 

Each process in the family of processes considered is called a \textit{generalized Cumulative Knowledge Process} -- defined below -- and is specified by its features.

\begin{defn}
    A generalized Cumulative Knowledge Process (CKP) with features
    \begin{equation}\label{ckp-parameters}
        \left(\textsc{Attach} = (\mathbf{a}, M), \textsc{Error} = (\e), \textsc{Adversary} = (q, r), \textsc{Check} = (p, k) \right)
    \end{equation}
    consists of a sequence $X_0, X_1, X_2, \dots$, where at each time $t$, $X_t = (\G_t, L_t)$. Here, $\G_t$ is a directed acyclic graph of $t + 1$ and $L_t$ is an accompanying set of labels. Each node in $\G_t$ is labeled from the set $\L = \{\PF, \CF, \CT\}$; in other words, $L_t \in \L^t$. Also, $\mathbf{a}: \mathbb{Z}_{\geq 0} \to \mathbb{R}_{\geq 0}$ is a function, $M$ is a bounded random variable $(1 \leq M < \infty)$, and the parameters in the remaining features are all constants.

    The starting state $X_0$ consists of a $\CT$ node (though our analysis could allow for multiple $\CT$ nodes in $X_0$). Thus, $\False$ nodes are those with a $\CF$ ancestor added at some error introduction phase. See Figure~\ref{fig:evolution-general} for an example of an evolution step.

    The evolution of the generalized CKP can be broken up into the five following steps. Given a generalized CKP with features (\ref{ckp-parameters}) at state $X_t = (\G_t, L_t)$, the next state $X_{t+1}$ is determined as follows.
    \begin{enumerate}    
        \item \textbf{Non-adversarial step.} With probability $1 - q$, perform the following steps (non-adversarial step):
        \begin{enumerate}
            \item \textbf{Choosing the number of parents.} Draw $m \sim M$. This is the number of parent nodes that the new node connects to.
    
            \item \textbf{Selecting parents.} If the DAG is entirely $\PF$, i.e., $L_t = \{\PF, \dots, \PF\}$, then the process stops and $X_{t+1} = X_t$. Otherwise, $m$ random $\PT$ nodes $u_i \in \G_t^\PT$ are chosen with probability proportional to $\mathbf{a}(\deg_{\PT}(u_i))$; that is, according to the attachment function $\mathbf{a}$. (Note that the new node can choose the same parent multiple times.)
            \item \textbf{Creating a new node.} A new node $v$ is created, with parents $u_1, \dots u_m$. 
            \item \textbf{Error introduction.} Label $v$ as $\CF$ with probability $\e$ and $\CT$ with probability $1-\e$. 
            \item \textbf{Error checking.} With probability depending on $p$, a given checking mechanism is applied to $v$. See the definition of checking procedures below (Definition \ref{checking-procedures}) for the options of checking mechanisms we consider.
        \end{enumerate}
        \item \textbf{Adversarial step.} Otherwise (with probability $q$), an adversarial step is performed as follows.
        \begin{enumerate}
            \item An adversary chooses any number of parent nodes -- up to a maximum of $r$ --  and chooses which parent nodes the new node connects to (however they wish to). A new node $v$ is created with these parents, and the adversary chooses the truth value ($\CT$ or $\CF$) of this node.
        \end{enumerate}
            
    \end{enumerate}
\end{defn}

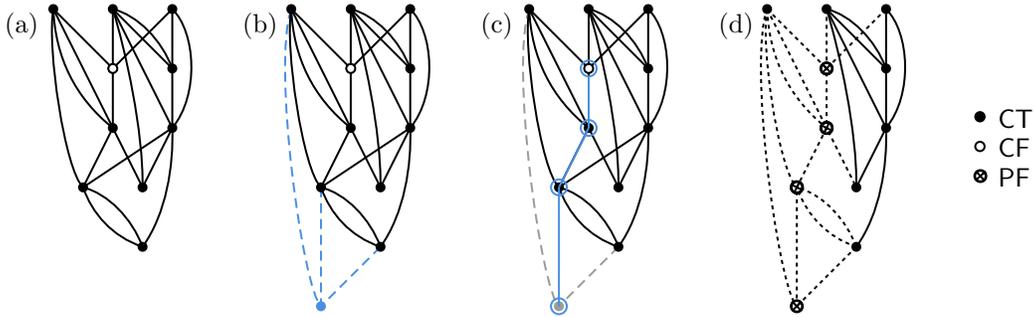
\begin{figure}[hbtp]
    \centering
    \input{Visuals/evolution}
    \caption{One non-adversarial evolution step of a CKP with $\textsc{Check} = (p, k=3)$ and $\textsc{Attach} = (\mathbf{a}, M=3)$ with checking mechanism \textsc{Stringy}. Node labels $\PF, \CF$, and $\CT$ are respectively represented by crossed, empty, and filled circles. (a) the initial CKP state $X_t$; (b) the result of steps (i)-(iv) in which a new \CT\ node is added; (c) step (v), a random path of length $k=3$ is checked and stops at a \CF\ node; (d) all visited descendants of the \CF\ node are marked $\PF$; this is $X_{t+1}$.}
    \label{fig:evolution-general}
\end{figure}

\begin{defn}[Checking procedure]\label{checking-procedures} Once the new node attaches to a state in a CKP, it performs a checking mechanism with probability $p$ (specified in the parameters of the model). This checking mechanism is \textit{local} in the sense that nodes access the hidden truth value of ancestors within a distance of $k$ (a small constant, also specified in the parameters of the model) when performing a check.

Consider the $t$-th step $X_t = (\G_t, L_t)$ of a CKP. Several of the checking mechanisms mentioned below check nodes until a \textit{minimal false} node is reached. Let the sub-DAG of roots ($\G_t^R$) consist of \CT\ nodes with one or more \PF\ parents. Let $\G_t^\CF$ denote the sub-DAG of nodes labeled $\CF$. Then the minimal false nodes $\F_t = \G_t^\CF \cup \G_t^R$ are the $\CF $ nodes and roots in $\G_t$.

We analyze the following checking mechanisms, listed roughly in order of the amount of nodes that can be deleted in one checking step. For nodes $u$ and $v$, let $\dist(u,v)$ denote the number of edges on the shortest path between $u$ and $v$. 

\begin{enumerate}
    \item \textbf{\textsc{Stringy}}: With probability $p$, check one random path of height $k$ up from $u$, stopping at the first \CF\ or \PF\ node reached. If one is found, flag it and its descendants along the path checked as \PF. 
    \item \textbf{\textsc{BFS}}: With probability $p$, perform a breadth-first search (BFS) starting from node $v$ up to depth $k$ (i.e., all nodes $u$ checked have $\dist(u,v) \leq k$), stopping at the first minimal false node reached.
    Flag as \PF\ this node and its descendants that were seen along the BFS from $v$.

    \item \textbf{\textsc{Exhaustive BFS}}: For each parent $u_i$ of $v$, one at a time, with probability $p$, check $v$ and then perform a BFS from $u_i$ up to depth $k-1$ (i.e., all checked nodes $u$ have $\dist(u,u_i) \leq k-1$). Stop the checking procedure at the first minimal false node reached. Flag as \PF\ this node and its descendants seen along the BFS from $u_i$. 
    
    \item \textbf{\textsc{Parent-wise BFS}}: Do the same thing as in the \textsc{Exhaustive BFS}, but instead of stopping the check completely as soon as a minimal false node is found, restart the check for the next parent $u_{i+1}$. A maximum of $m$ minimal false nodes can now be found (1 per parent $u_i$), rather than 1.  
    
    \item \textbf{\textsc{Complete}}: For each parent $u_i$ of $v$, with probability $p$, check $v$ and all nodes that are within distance $k-1$ of $u_i$. This is equivalent to doing the \textsc{Exhaustive BFS} but continuing after finding any minimal false node. Now, if each parent performs a check, all \CF\ and root nodes within distance $k$ can be found. 
\end{enumerate}
\end{defn}

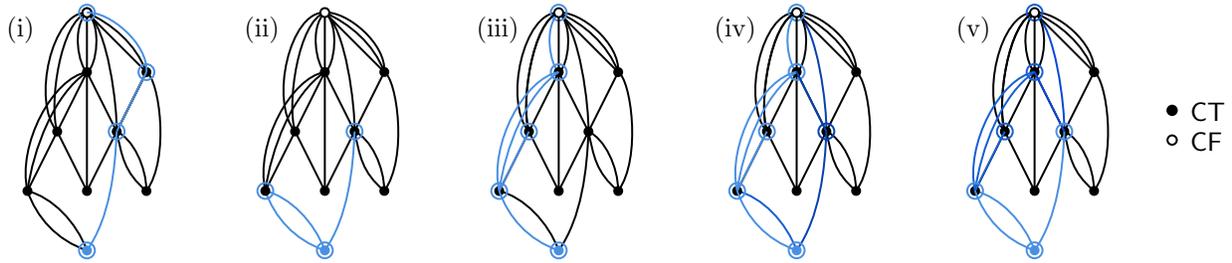
\begin{figure}[hbtp]
    \centering
    \input{Visuals/mechanisms}
    \caption{Visualizations of various checking mechanisms. (i) \textsc{Stringy} with $k=3$. (ii) \textsc{BFS} with $k=1$. (iii) \textsc{Exhaustive BFS} with $k=3$, which we note stopped as soon as it found the \CF\ node. (iv) \textsc{Parent-wise BFS} with $k=2$, which found one more node than the \textsc{Exhaustive} check. (v) \textsc{Complete} with $k=2$.
    }
    \label{fig:checking-mechanisms}
\end{figure}

While the details of these mechanisms differ, at a broader level we aim to theoretically justify that natural local checking heuristics with certain reasonable parameters can indeed eliminate all of the errors in a broad family of processes of knowledge accumulation, and also prove that errors survive forever under certain parameters even when local checking is performed.

Note that in our model, the discovery of a False node via a local check does not propagate to the (False) sub-DAG of this node. We believe this is a reasonable and important assumption: in a large network, if somebody finds an error in a paper, they may report it to the author of the erroneous paper and the few other papers that were checked in this procedure, but they will probably not have the time nor means to contact all papers that rely on this result (which potentially include very distant descendants); there also does not exist any centralized agency that would perform such a clean-up.

\paragraph{Simple CKP} Several results in this paper involve the following simplification of the model, in which there is no error introduction ($\e = 0$), no adversarial behavior $(q = 0$), and we begin with a $\CF$ node to avoid trivialities. In this simplification, the entire DAG is \False.
\begin{defn}
    A simple CKP with features $(\textsc{Attach} = (\mathbf{a}, M), \textsc{Check} = (p, k))$ sets $\e = 0$ and $q = 0$, and $X_0$ consists of one $\CF$ node. See Figure~\ref{fig:evolution-simple}.
\end{defn}
\begin{figure}[hbtp]
    \centering
    \input{Visuals/evolution-simple}
    \caption{One non-adversarial evolution step of a simple CKP with $\textsc{Check} = (p, k=4)$ and $\textsc{Attach} = (\mathbf{a}, M=3)$ with \textsc{Exhaustive BFS} checking mechanism. Steps (a) and (b) are as in Figure~\ref{fig:evolution-general}; (c) the first parent hits the probability $p$ and a BFS is performed, which finds the original \CF\ node and stops; (d) all visited descendants of the \CF\ node are marked \PF. Note in this case, the \textsc{Parent-wise BFS} and \textsc{Complete} checks would have marked the same set of nodes as \PF.}
    \label{fig:evolution-simple}
\end{figure}
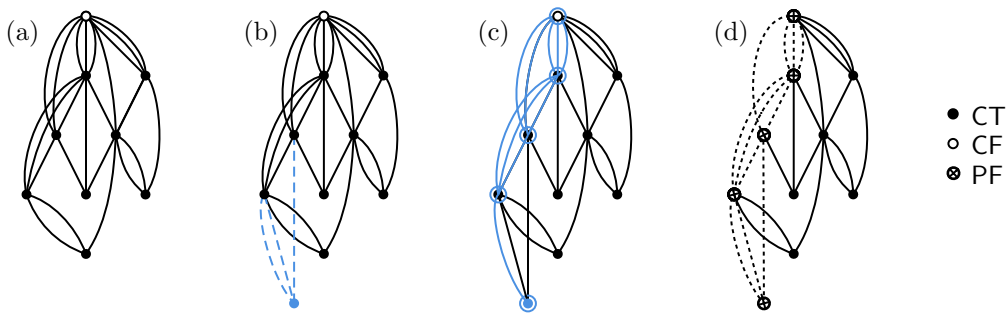

At a high level, the reason for considering this simplified model is that it can be connected to the smallest erroneous components within the generalized CKP in the family of processes considered. The smallest erroneous components consist of a $\CF$ or $\PF$ node and all of its descendants.

\section{Proof of Error Elimination} \label{section:error-elimination}

In this section, we study error elimination for generalized CKPs. The results here imply Theorem~\ref{thm:intro-error-elimination}.
We first describe the potential that is used to prove both the simple and general CKP extinction results. 

\paragraph{Minimum-distance potential} To specify the potential, we first define the sub-DAG of \textbf{\PT\ \False\ nodes} $\G_t^{\PT, F}$, \CF\ nodes and \CT\ nodes with some \CF\ or \PF\ ancestor (i.e., also some root ancestor). 
Now consider the following potential $\Phi$ on a CKP state $X_t = (\G_t, L_t)$:
\begin{equation}\label{minimum-distance-potential}
  \Phi(\G_t) = \sum_{v \in \G_t^{\PT,F}} \mathbf{a}(\deg(v)) \cdot c^{|v|},
\end{equation}
where $|v|$ is the number of edges in the shortest path from $v$ to a minimal false node (if $v$ is itself a minimal false node, then $|v| = 0$)
and $c > 1$, which we fix in the proofs of Lemmas~\ref{lemma:error-elimination-supermartingale-simple} and \ref{lemma:error-elimination-supermartingale-general}. This potential satisfies $\Phi(\G_t) \geq |\G_t^{\PT,F}| \geq |\G_t^{\PT, (u)}|$ for any $\CF$ node $u$, i.e., it upper bounds the number of $\PT$ descendants $|\G_t^{\PT, (u)}|$ of any $\CF$ node.
Note that in the simple CKP case, $\G_t^{\PT,F} = \G_t^\PT$ since all nodes are \False.

\paragraph{BFS-components}
The reason we choose this potential is that we can partition the \PT\ \False\ sub-DAG into disjoint components according to their closest minimal false node, and split up $\Phi$ accordingly. Recall from Definition \ref{checking-procedures} that minimal false nodes are $\CF$ nodes and roots (which are defined as $\CT$ nodes with one or more $\PF$ parents). The partitioning is thus indexed over $\F_t$, and we denote the components $\{\G_t[u]\}_{u \in \F_t}$. 

We now describe the assignment of \PT\ \False\ nodes into {BFS-components} $\{\G_t[u]\}_{u \in \F_t}$. 
For each $v \in \G_t^{\PT,F}$, let $u_F \in \F_t$ be the closest \CF\ or root node 
\[ \argmin_{u \in \F_t} \, \dist(u, v), \]
chosen uniquely by breaking ties according to the BFS ordering, i.e., if this minimum is attained for several $u \in \F_t$, we set $u_F$ to be the first one reached in a BFS search starting from $v$. 
See Figure~\ref{fig:BFS-partition} for visualizations of this partition. 

We note that each $\G_t[u]$ is connected, because if some $w$ is in $\G_t[u]$, all of its ancestor nodes along its shortest path to $u$ must also be in $\G_t[u]$.
Also, $|\G_t[u]| \geq 1$ for each $u \in \F_t$, as it at least contains $u$ itself. 

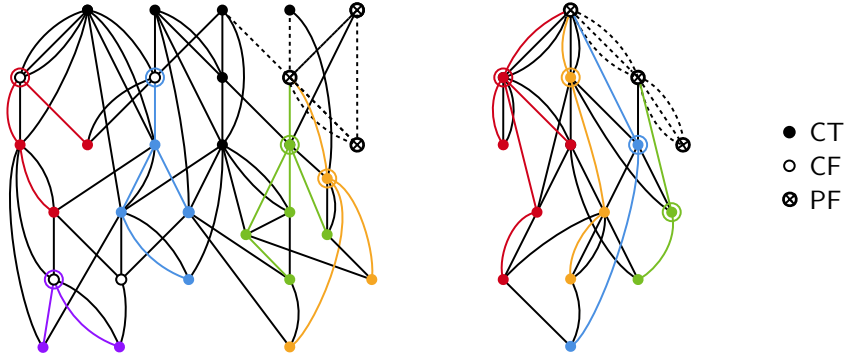
\begin{figure}[hbtp]
    \centering
    \input{Visuals/BFS-partition}
    \caption{A partition of a general (left) and simple (right) CKP state into BFS-components, denoted in the figure by different colors. The circled nodes are the minimal false nodes in $\F_t$. Notice the tie-breaking based on the BFS left-to-right ordering.}
    \label{fig:BFS-partition}
\end{figure}

We can now write the minimum-distance potential equivalently as
\begin{equation}\label{eq:potential-partition}
  \Phi(\G_t) = \sum_{u \in \F_t} \Phi(\G_t[u]) \ \text{ where } \ 
  \Phi(\G_t[u]) = \sum_{v \in \G_t[u]} \mathbf{a}(\deg(v)) \cdot c^{\dist(u, v)}.
\end{equation}

\subsection{Error elimination in the simple CKP}

We begin by proving the error elimination (see Definition \ref{def:error-elimination}) in the simple CKP setting, which highlights several key ideas that we then use to analyze the generalized CKP model. 

\begin{thm}[Simple CKP Error Elimination] \label{thm:error-elimination-simple}
    For the \textsc{Exhaustive BFS}, \textsc{Parent-wise BFS} and \textsc{Complete} checking mechanisms, for every $b \in \mathbb{R}_{\geq 0}$, $(0, b)$-regular attachment function $\mathbf{a}$, bounded random variable $M \geq 1$, $p \in (0, 1]$ and $k \geq 2$ satisfying
    \begin{equation}
        \max \left(\frac{b + 3 \cdot \mathbf{a}(0) \cdot \E\{1/M\}}{b + 3 \cdot \mathbf{a}(0) \cdot \E\{1/M\} + 2/3}, \frac{b + 3 \cdot \mathbf{a}(0) \cdot \E\{1/M\}}{((k - 1) \mathbf{a}(1) + \mathbf{a}(0)) \cdot 2/3} \right) \leq p \leq 1,
    \end{equation}
    the error effect is eliminated in the $(\textsc{Attach} = (\mathbf{a}, M), \textsc{Check} = (p, k))$-simple CKP (see Definition~\ref{def:error-elimination}).
\end{thm}

For concreteness, we give some explicit values in the preferential attachment setting. For preferential attachment CKPs, $p \geq 6/7$ upper bounds the first term in the maximum for all $M \geq 1$. As for the second term, we have a non-trivial ($\leq 1$) upper bound if $k \geq 2$ and $\E \{1/M\} \leq 1/3$; and if $k \geq 3$ and $\E \{1/M\} \leq 7/9 \approx 1/1.3$. Note that these are satisfied respectively when we have the deterministic lower bounds $M \geq 3$ and $M \geq 2$. 
\vspace{.25\bs}

In the following lemma, we prove that the sequence of minimum-distance potential values over time forms a super-martingale, for certain settings of the features.

\begin{lem}\label{lemma:error-elimination-supermartingale-simple}
Let $X_t = (\G_t, L_t)$ be the state of the simple CKP with features $\textsc{Attach} = (\mathbf{a}, M)$ and $\textsc{Check} = (p, k)$ at time $t$, for features that satisfy all conditions from Theorem~\ref{thm:error-elimination-simple}.
If $\Phi(\G_t) > 0$, then $\Delta_t = \Phi(\G_{t+1}) - \Phi(\G_t)$ satisfies
$\E\left\{\Delta_t \mid \G_t \right\} < 0$. 
\end{lem}

\begin{proof}
  Fixing a sequence $S = \{ \G_t^\PT[u_1], \G_t^\PT[u_2], \dots \}$ of BFS-components of $\G_t^\PT$,
  we define $E_S$ to be the event that a new node added to the CKP $\G_t^\PT$ connects to components in the order given by $S$. 
  To complete the proof, it suffices to show that for any arbitrary $S$, 
  $
    \E\left\{ \Delta_t \mid E_S, \G_t \right\} < 0.
  $

  We denote by $\alpha_i$ the probability that an edge from a new node $v$ to the $i$-th BFS-component $\G_t^\PT[u_i]$ connects within distance $k$ of $u_i$, 
  i.e., $\dist(v, u_i) \leq k$. 
  If $|\G_t^\PT[u_i]| < k$, then this probability must be 1. Otherwise, in the worst case $\G_t^\PT[u_i]$ consists of a chain of length $k$ with
  a hanging sub-component. Due to the monotonicity of the attachment function, the probability of connecting within distance $k$ of $u_i$ in this case is at least $[(k-1) \cdot \mathbf{a}(1) + \mathbf{a}(0)]/Z_i$, where $Z_i = \sum_{v \in \G_t^\PT[u_i]} \mathbf{a}(\deg_\PT(v))$. 
  
  Combining these cases, we have
  \[
    \alpha_i \geq \min(1, ((k-1) \cdot \mathbf{a}(1) + \mathbf{a}(0))/Z_i). 
  \]

  In the simple CKP process, a new \CT\ node $v$ is added. 
  If a check is successful, the potential of one of its parents' BFS-component decreases. 
  On the other hand, if no check is successful, the potential increases, from two contributions:
  \begin{itemize}
    \item the degree of each of the parents $p_i$ increases by 1, and
    \item the new node $v$ adds a $\mathbf{a}(0) \cdot c^{|v|}$ term, where $|v| = \min_{i} |p_i| + 1$. 
  \end{itemize}   
  We would like to break up this change in potential into BFS-components, i.e., the change in component $i$'s potential $\Phi_i \ceq \Phi(\G_t^\PT[u_i])$. 
  To do this for the contribution of $v$, we observe that, if $v$ has $M = m$ parents 
  (we later take an expected value over this combination factor),
  \begin{equation}\label{eq:min-less-than-avg}
    c^{|v|} = \min_{i = 1}^m c^{|p_i| + 1} \leq \sum_{i = 1}^m c^{|p_i| + 1} / m.
  \end{equation}
  
  We can now analyze component-by-component the change in $\Phi_i$, which we denote by $\Delta_t^{[i,m]}$. Set the checking mechanism to be \textsc{Exhaustive BFS}.
  For each parent edge $i \geq 1$, the checking procedure is still ongoing (has not yet been successful) with probability $(1-p\alpha_1) \cdots (1-p\alpha_{i-1})$.  
  Now, there are two possible events for the corresponding BFS-component $\Phi_i$:
  \begin{enumerate}
    \item {\sl A check is performed and successful}, which happens with probability $p \alpha_i$.
      In this case, $u_i$ is deleted from $\G_t^\PT[u_i]$. At time $t+1$, 
      $\G_t^\PT[u_i]$ loses nodes and is partitioned into new components $\{\G_{t+1}^\PT[u_{ij}]\}$, where $u_{ij}$ are children of $u_i$. 
      In each of these new components, every node $v' \in \G_{t+1}^\PT[u_{ij}]$ is closer to $u_{ij}$ than it was to $u_i$, 
      specifically, $\dist(u_{ij}, v') \leq \dist(u_i, v') - 1$. 
      Therefore,
      \begin{equation*}
          \sum_j \Phi(\G_{t+1}^\PT[u_{ij}]) < c^{-1} \Phi_i
      \end{equation*}\vspace{-\bs}
      
      and so
      \begin{equation}
        \E\bcurly{\Delta_t^{[i,m]} \ind{i \text{th check success}} \bigmid E_S, \G_t } < p \alpha_i (c^{-1} -1) \Phi_i.
      \end{equation} 
      The potentials of the other BFS-components either remain unchanged or decrease.
      
    \item {\sl Otherwise, the check fails and we move on to the next edge}, with probability $1-p\alpha_i$. 
      We denote the new component resulting from the connection by $\G_{t+1}^\PT[u_i]$.
      In the worst case, no check for any future edge is successful and we have two contributions to the change in potential: 
      the parent $w$ chosen by this edge gains 1 to its degree, contributing $\mathbf{a}(\deg(w) + 1) - \mathbf{a}(\deg(w)) \cdot c^{|w|} \leq b \cdot c^{|w|}$ due to the $(0, b)$-regularity of the attachment function. The new node also contributes $\mathbf{a}(0) \cdot \frac{c^{|w|+1}}{m}$, by \eqref{eq:min-less-than-avg}.  
      Therefore, 
      \begin{align*} \nonumber
        \E\bcurly{ \Delta_t^{[i,m]} \ind{i \text{th check fail}} \bigmid E_S, \G_t } 
        &\leq \E\bcurly{ \Delta_t^{[i,m]} \ind{i \text{th check and all future checks fail}} \bigmid E_S, \G_t } \\ \nonumber
        &\leq \sum_{w \in \G_t^\PT[u_i]} \frac{\mathbf{a}(\deg(w))}{Z_i} \cdot \bpar{b \cdot c^{|w|} + \mathbf{a}(0) \cdot \frac{c^{|w|+1}}{m}  }  \\ \nonumber
        & = \frac{b + \mathbf{a}(0) \cdot c/m}{Z_i} \sum_{w \in \G_t^\PT[u_i]} \mathbf{a}(\deg(w)) \cdot c^{|w|} \\
        &= \frac{b + \mathbf{a}(0) \cdot c/m}{Z_i} \Phi_i.
      \end{align*}
  \end{enumerate}
  Combining cases (i) and (ii), for each BFS-component, in the event that the check is still ongoing,
  \begin{equation}\label{eq:extinction-simple-one-edge}
    \E \bcurly{\Delta_t^{[i,m]} \bigmid E_S, \G_t} < p \alpha_i (c^{-1} - 1)\Phi_i + (1-p\alpha_i)\frac{b + \mathbf{a}(0) \cdot c/m}{Z_i} \Phi_i.
  \end{equation}
  We can now compute $\E\{\Delta_t \mid E_S, \G_t\}$, taking an expectation over the combination factor $M$. Let $M$ be equal to $m$ with probability $q_m$ for each $m \geq 1$. 
  We have
  \begin{align}
    \E\{\Delta_t \mid E_s, \G_t\} 
    &\leq q_1 \E\bcurly{\Delta_t^{[1,1]} \bigmid E_S, \G_t } + q_2 \E\bcurly{\Delta_t^{[1,2]} + \ind{e = 1 \text{ check fail}}\Delta_t^{[2,2]} \bigmid E_S, \G_t } \nonumber\\
      &\quad + q_3 \E\bcurly{\Delta_t^{[1,3]} + \ind{e = 1 \text{ check fail}}(\Delta_t^{[2,3]} + \ind{e = 2 \text{ check fail}}\Delta_t^{[3,3]}) \bigmid E_S, \G_t } + \dots \nonumber \\ 
      &= \sum_{m \geq 1} q_m \sum_{i = 1}^m \Big( \prod_{j=1}^{i-1}(1- p \alpha_j)\Big) \E\bcurly{ \Delta_t^{[i,m]}\bigmid E_S, \G_t}, \label{eq:extinction-simple-delta-t}
  \end{align}
  where we used $\E\{\ind{e = i \text{ check fail}}\} = (1-p \alpha_i)$. Inverting the sums over $m$ and $i$, we rewrite this as
  \begin{align*}
    \E\{\Delta_t \mid E_s, \G_t\} 
    &\leq \sum_{i \geq 1} \Big( \prod_{j=1}^{i-1}(1-p \alpha_j)\Big) \sum_{m \geq i} q_m \E\bcurly{\Delta_t^{[i,m]} \bigmid E_S, \G_t } \\
    &< \sum_{i \geq 1} \Big( \prod_{j=1}^{i-1}(1-p \alpha_j)\Big) \sum_{m\geq i} q_m \bpar{ p \alpha_i (c^{-1}-1)\Phi_i + (1-p\alpha_i)\frac{b+ \mathbf{a}(0) \cdot c/m}{Z_i} \Phi_i } \\ 
    &\leq 
    \sum_{i \geq 1} \Big( \prod_{j=1}^{i-1}(1-p \alpha_j)\Big) \Big(\sum_{m\geq i} q_m\Big) \bigg( p\alpha_i (c^{-1}-1) + \frac{1-p\alpha_i}{Z_i} \Big(b + \mathbf{a}(0) \cdot c \cdot \E\{1/M\}\Big) \bigg) \Phi_i.
  \end{align*}
  This expected difference term is negative if for each $i \geq 1$, 
  \begin{equation}\label{eq:simple-extinction-condition}
      p\alpha_i (c^{-1}-1) + \frac{1-p\alpha_i}{Z_i} \Big(b + \mathbf{a}(0) \cdot c \cdot \E\{1/M\} \Big) \leq 0.
  \end{equation}
  Set $c = 3$.
  If $\alpha_i = 1$, as $Z_i \geq 1$, this is satisfied for
  \begin{equation*}
      p \geq \frac{b + 3 \cdot \mathbf{a}(0) \cdot \E\{1/M\}}{b + 3 \cdot \mathbf{a}(0) \cdot \E\{1/M\} + 2/3}.
  \end{equation*}
  Otherwise, we know $1 > \alpha_i \geq [(k-1) \cdot \mathbf{a}(1) + \mathbf{a}(0)]/Z_i$, and then \eqref{eq:simple-extinction-condition} holds if 
  \[
    -p \cdot \left((k-1) \mathbf{a}(1) + \mathbf{a}(0)\right) \cdot (2/3) + (b+3\mathbf{a}(0) \E\{1/M\}) \leq 0, 
  \]
  i.e., for all $p$ satisfying $p \cdot \left((k-1) \mathbf{a}(1) + \mathbf{a}(0)\right) \geq 3/2 (b+3\mathbf{a}(0) \E\{1/M\})$. 
\end{proof}

\begin{rem}
  Note that for the \textsc{Parent-wise BFS} and \textsc{Complete BFS} checking mechanisms, we can only delete more at each step if a check occurs, and add the same amount if no check occurs. Therefore both \eqref{eq:extinction-simple-one-edge} and \eqref{eq:extinction-simple-delta-t} (in fact without the $\ind{e=i \text{ check fail}}$ indicators) still hold.
\end{rem}

\begin{rem}\label{remark:noisy-check}
    Our results hold for checks that are \textit{noisy} as well. A noisy check would only detect a $\CF$ node to be in error with probability $p_e$. If $p$ is the probability of performing a check, then the results in this section hold for noisy checks with $p$ replaced by $p \cdot p_e$. Throughout the proof, $p \cdot \alpha_i$ can be replaced by $p \cdot p_e \cdot \alpha_i$, which is now a lower bound that the check of the $i$-th BFS-component is successful.
\end{rem}

We now obtain Theorem \ref{thm:error-elimination-simple} from Lemma \ref{lemma:error-elimination-supermartingale-simple} as follows. 

\begin{proof}[Proof of Theorem \ref{thm:error-elimination-general} given Lemma~\ref{lemma:error-elimination-supermartingale-simple}]
    This is the same as the proof of \cite[Theorem 2.1]{BEMMS} and we provide the proof for completeness. 
    We show that for a CKP given by $\{\G_t\}_{t = 1}^{\infty}$ with the conditions from Theorem \ref{thm:error-elimination-general}, the sequence of potential values $\{\Phi(\G_t)\}_{t = 1}^{\infty}$ converges to $0$ almost surely. Through Lemma \ref{lemma:error-elimination-supermartingale-simple}, we have shown that $\{\Phi(\G_t)\}_{t = 1}^{\infty}$ is a positive super-martingale. The martingale convergence theorem tells us that as $t \to \infty$, $\Phi(\G_t)$ converges almost surely to a limit $Y$. Combining this with the fact that $|\Phi(\G_t) - \Phi(\G_{t+1})| \geq 1$ when $\Phi(\G_t) > 0$, we can conclude that this limit $Y$ must be $0$, i.e., $\P\{Y = 0\} = 1$. 
    We previously noted that for any $\CF$ node $u$, $|\G_t^{\PT, (u)}| \leq \Phi(\G_t)$, which implies that $|\G_t^{\PT, (u)}|$ also converges to $0$ almost surely, yielding error elimination.
\end{proof}

\begin{rem}
   The monotonicity of the attachment function $\mathbf{a}: \mathbb{Z}_{\geq 0} \to \mathbb{R}_{\geq 0}$ (implicit in the $b$-regularity assumption) is used to claim that, conditioned on connecting to a BFS-component $\G_t^{\PT}[u_i]$ of size larger than $k$, the probability of connecting within distance $k$ of $u_i$ is at least $[(k-1) \cdot \mathbf{a}(1) + \mathbf{a}(0)]/Z_i$, where $Z_i = \sum_{v \in \G_t^\PT[u_i]} \mathbf{a}(\deg_\PT(v))$. With the monotonicity assumption removed, a weaker error elimination guarantee can still be obtained. The probability of connecting within distance $k$ of $u_i$ is still at least $k \cdot \min_d \mathbf{a}(d)/Z_i$, and the rest of the proof follows in the same way as before. Note that, in this case, we still would need $\min_d \mathbf{a}(d) \geq 1$ to prove error elimination, because the proof requires that $|\G_t^{\PT,F}| \leq \Phi(\G_t)$.
\end{rem}

\subsection{Error elimination in the generalized CKP}

We now incorporate the introduction of new errors to the CKP, as well as the adversarial behavior. We prove the following result about error elimination (see Definition \ref{def:error-elimination}) for general CKPs. 

\begin{thm}[General CKP error elimination]
\label{thm:error-elimination-general}
    For CKPs with the \textsc{Exhaustive BFS}, \textsc{Parent-wise BFS} and \textsc{Complete} checking mechanisms, 
    for every $(0, b)$-regular attachment function and every combination factor $M \geq 1$, and any adversary bound $r \in \Z_{\geq 0}$, there exists $ q_0 > 0$ such that for every low enough adversarial probability $q \leq q_0$, and every error probability $\epsilon < 1$, there exist $p_0 < 1$ and $k_0 \in \N$ such that for all large enough checking parameters $p \geq p_0$ and $k \geq k_0$, error elimination occurs (Definition~\ref{def:error-elimination}). 
    Specifically, this holds for parameters satisfying \eqref{equation:general-ckp-survival-expression}.
\end{thm}

Note that $M=1$ is a worst-case scenario, and we recover the corresponding bounds. For instance, we recover that for a preferential attachment model with no adversary, a checking depth of $k \geq 4$ is required for error elimination. Also, we now find that for a uniform attachment model without adversarial steps, a checking depth of $k \geq 5$ is required. 

Despite our analysis reducing to $M=1$, extending this result to the multi-parent CKP setting still requires work. Based on the simulations in appendix \ref{section:simulations} for preferential attachment, we conjecture that the true error elimination threshold does vary with $M$ for general CKPs and leave as open what this dependence on $M$ should be. 

As in the simple case, we prove Theorem~\ref{thm:error-elimination-general} by showing that the sequence of minimum-distance potential values forms a super-martingale.

\begin{lem}\label{lemma:error-elimination-supermartingale-general} 
Let $X_t = (\G_t, L_t)$ be the state of a CKP with features $(\textsc{Attach} = (\mathbf{a}, M), \textsc{Error} = (\e), \textsc{Adversary} = (q, r), \textsc{Check} = (p, k))$  at time $t$. Suppose  $\mathbf{a}$ is a $(0, b)$-regular attachment function and the features satisfy the following:
\begin{align} \label{equation:general-ckp-survival-expression}
    (1 - q) \cdot \Big[ &(1-\e)\max\bpar{-\frac{1}{2}p((k-1) \mathbf{a}(1) + \mathbf{a}(0)) + (b + 2 \mathbf{a}(0)), -\frac{1}{2}p + (b + 2 \mathbf{a}(0)) (1-p)} \nonumber \\
    & + \e (b + \mathbf{a}(0))  (1-p) \Big] 
    + q \cdot \frac{(b + 2) (r \cdot b + 2 \mathbf{a}(0))}{\mathbf{a}(0) + \mathbf{a}(1)} \leq 0.
\end{align}
Then, when $\Phi(\G_t) > 0$, $\Delta_t = \Phi(\G_{t+1}) - \Phi(\G_t)$ satisfies $\E\{\Delta_t \mid \G_t\} < 0$. 
\end{lem}

\begin{proof}
If the new node joins non-adversarially to the CKP (with probability $1 - q$), we analyze the change in potential as follows. In the minimum-distance potential, let $c = 2$ for the analysis. First, recall that previously, in the simple CKP, all \PT\ nodes are \False, and so the minimum-distance potential sums over all \PT\ nodes. In particular, a new node $v$ must connect to $M$ BFS-components (potentially with replacement). In the general CKP however, a new node $v$ connects to some number $0 \leq M' \leq M$ of BFS-components, and $M - M'$ \True\ nodes. 

If $M'=0$, $v$ is \True\ and there is no change in the potential $\Phi(\G_t)$, so we assume that $M' \geq 1$. 

We can begin with the same analysis as in the proof of Lemma~\ref{lemma:error-elimination-supermartingale-simple}. As in the previous proof, let $E_S$ denote the event that a new node added to the CKP connects to components in the order given by $S$. If $v$ is \CT, we have the same two events (i) and (ii), so we have for each $i = 1, \dots, M'$, on the event that the check is still ongoing,
\begin{equation}\label{eq:extinction-case-CT}
    \E\bcurly{\Delta_t^{[i,M']} \ind{\text{non-adv. step, } v \text{ is } \CT} \mid E_S, \G_t } < (1 - q)  (1-\e)  \bpar{ p \alpha_i (c^{-1} - 1) + (1-p\alpha_i) \frac{b+\mathbf{a}(0) \cdot c/M'}{Z_i} }\Phi_i,
\end{equation}
where $Z_i = \sum_{v \in \G_t^\PT[u_i]} \mathbf{a}(\deg_\PT(v))$ for $u_i \in \F_t$ which is the minimal false ancestor of $v$ accessible by the shortest path.

On the other hand, if $v$ is \CF, (i) if a check is performed for any edge, then there is no change in potential as $v$ is simply removed, and (ii) if there is no check, for each $i = 1, \dots, M'$, 
\begin{align}
\label{eq:extinction-case-CF}
    \E\bcurly{\Delta_t^{[i,M']} &\ind{\text{non-adversarial step, } v \text{ is } \CF, \text{ no check}}  \mid E_S, \G_t } \nonumber \\
    &\leq (1 - q) \cdot  \e (1-p)^{M'} \bigg( \frac{\mathbf{a}(0)}{M'} + \sum_{w \in \G_t^\PT[u_i]} \frac{\mathbf{a}(\deg(w))}{Z_i} \left( \mathbf{a}(\deg(w) + 1) - \mathbf{a}(\deg(w)) \right) c^{|w|} \bigg) \nonumber \\ 
    &\leq (1 - q) \cdot \e (1-p)^{M'}\bpar{\frac{\mathbf{a}(0)}{M'} + \frac{b}{Z_i} \Phi_i } \nonumber \\ 
    &\leq (1 - q) \cdot \e (1-p)^{M'} \frac{b+\mathbf{a}(0)/M'}{Z_i} \Phi_i.
\end{align}
The worst case for both \eqref{eq:extinction-case-CT} and \eqref{eq:extinction-case-CF} is when $M' = 1$. We find that, picking $c = 2$,
\begin{align}
    \E\{& \Delta_t \ind{\text{non-adversarial step}} \mid \G_t\} \nonumber \\
    &< (1 - q) \cdot \Big((1-\e)\max\bpar{-\frac{1}{2}((k-1) \cdot \mathbf{a}(1) + \mathbf{a}(0)) \cdot p + (b + 2 \mathbf{a}(0)), -\frac{1}{2}p + (b + 2 \mathbf{a}(0)) \cdot (1-p)} \nonumber \\
    & \ + \e \cdot (b + \mathbf{a}(0)) \cdot (1-p) \Big) \cdot \sum_{u \in \F_t} Z_u \frac{\Phi(G_t[u])}{Z_u},
\end{align}
for $Z_u = \sum_{v \in \G_t^\PT[u]} \mathbf{a}(\deg_\PT(v))$. Note that $\sum_{u \in \F_t} Z_u \frac{\Phi(G_t[u])}{Z_u} = \Phi(\G_t)$.

If the new node joins adversarially to the CKP (with probability $q$), then the new node's placement (attachment to existing nodes) and truth value ($\CT$ or $\CF$) are chosen by an adversary. Recall that the adversary can choose up to $r$ parent nodes for the new node. For each $\False$ parent node $v$ chosen, the potential increases by at most $b \cdot c^h$, where $h$ is the maximum length of a shortest path from a node to a minimal false ancestor node in $\G_t$. Additionally, if the new node is chosen to be $\CT$, the new node adds at most $\mathbf{a}(0) \cdot c^{h + 1}$ to the potential. If the new node is chosen to be $\CF$, then the new node instead adds at most $\mathbf{a}(0)$ to the potential, which is lower. Therefore, overall, the contribution to the potential in an adversarial step is $\leq (r \cdot b + \mathbf{a}(0) \cdot c) \cdot c^h.$ Noting that the contribution to the potential along this longest shortest path from a node to a minimal false ancestor node in $\G_t$ is at least $\mathbf{a}(0) \cdot c^h + \sum_{i = 0}^{h-1} \mathbf{a}(1) \cdot c^i = \mathbf{a}(1) \cdot c^h - \mathbf{a}(1) + \mathbf{a}(0)c^h = (\mathbf{a}(1) + \mathbf{a}(0)) \cdot c^h - \mathbf{a}(1)$, we see that $\Phi(G_t) \geq (\mathbf{a}(1) + \mathbf{a}(0)) \cdot c^h - \mathbf{a}(1)$. Therefore, 
\begin{align}
    \E\bcurly{\Delta_t  \ind{\text{adversarial step}} \mid \G_t } &\leq q \cdot \frac{(b + 2) (r \cdot b + \mathbf{a}(0) \cdot c)}{\mathbf{a}(0) + \mathbf{a}(1)} \cdot \Phi(\G_t).
\end{align}
 Combining all of these cases, we find that $ \E\bcurly{\Delta_t   \mid \G_t } < 0$ when
\begin{align*}
    (1 - q) \Big[&(1-\e)\max\bpar{-\frac{1}{2}p((k-1) \mathbf{a}(1) + \mathbf{a}(0)) + (b + 2 \mathbf{a}(0)), -\frac{1}{2}p + (b + 2 \mathbf{a}(0)) (1-p)}
    \\
    & + \e \cdot (b + \mathbf{a}(0)) (1-p) \Big] 
    + q \cdot \frac{(b + 2) (r \cdot b + 2 \mathbf{a}(0))}{\mathbf{a}(0) + \mathbf{a}(1)} \leq 0,
\end{align*}
as required.
\end{proof}

Theorem \ref{thm:error-elimination-general} follows from Lemma \ref{lemma:error-elimination-supermartingale-general} by the same proof as in the simple CKP setting.

\subsection{Proportion of undetected False nodes}\label{subsection:fraction-false-nodes}

We can apply the techniques of \cite{BEMMS} to translate our error elimination result into the following guarantee about the proportion of $\False$ nodes among active nodes (nodes that can obtain future children) in CKPs at any time step.

\begin{prop}\label{prop:proportion-of-false-nodes}
    Let $X_t = (\G_t, L_t)$ be the state at time $t$ of a CKP with features satisfying (\ref{equation:general-ckp-survival-expression}) and $q = 0$ (no adversarial steps). Consider $\eps \in (0, 1)$ and $p \in (0, 1]$. 
    
    Let $\G_t^T$ be the sub-DAG of $\True$ nodes in $X_t$ and let $\G_t^{\PT, F}$ be the set of $\PT$ $\False$ nodes in $X_t$, which are $\False$ nodes that are still Proclaimed True (i.e. ``undetected''). Then, for all $t \geq 0$, 
    $$\E\{ |\G_t^{\PT, F}| \} \leq \frac{\eps (1 - p) \mathbf{a}(0)}{1 - \eps}\E\{ |\G_t^T| \} .$$ 
\end{prop}

\begin{proof}
    This is similar to the proof of \cite[Theorem 4.3]{BEMMS} and we provide the proof for completeness. Let $\Phi(\G_t)$ denote the minimum distance potential (\ref{minimum-distance-potential}) introduced in Section \ref{section:error-elimination}. Recall from Lemma \ref{lemma:error-elimination-supermartingale-general} that, for CKPs whose features satisfy (\ref{equation:general-ckp-survival-expression}), when $\Phi(\G_t) > 0$ then $\E\{\Phi(\G_{t+1}) - \Phi(\G_t) | \G_t\} < 0$. Let us define a new potential that incorporates the minimum distance potential. Define $\Psi(\G_t)$ as follows:
    $$\Psi(\G_t) = \frac{\eps ( 1- p) \mathbf{a}(0)}{1 - \eps} |\G_t^T| - \Phi(\G_t).$$ 

    We want to prove that $\Psi(\G_t)$ is a sub-martingale. Suppose that the new node connects to $m \sim M$ parents, all of which are $\True$. If the new node is $\CT$ (with probability $1 - \epsilon$), then it is also $\True$ and the potential increases by: $\Psi(\G_{t + 1} ) - \Psi(\G_t) = \frac{\eps ( 1- p)}{1 - \eps}$. If the new node is $\CF$ (with probability $\epsilon$) and no check is performed then it becomes the root of a new $\PT$ $\False$ sub-DAG and $\Psi(\G_{t + 1} ) - \Psi(\G_t) = - \Phi(\G_{t+1}) + \Phi(\G_{t} ) = -\mathbf{a}(0)$. Therefore, the expected change in the potential when all parents are $\True$ is:
    $$\E\{\Psi(\G_{t + 1} ) - \Psi(\G_t)| \G_t\} = (1 - \eps) \cdot \frac{\eps ( 1- p) \cdot \mathbf{a}(0)}{1 - \eps} - \epsilon \cdot (1 - p) \cdot \mathbf{a}(0) = 0.$$

    Suppose that out of the $m \sim M$ parents, at least one is $\False$. The new node belongs to $\G_t^{\PT, F}$, and because the CKP features satisfy (\ref{equation:general-ckp-survival-expression}), we can apply the proof of Lemma \ref{lemma:error-elimination-supermartingale-general} and can state that: $\E\{\Phi(\G_{t+1}) - \Phi(\G_t) | \G_t\} < 0$, i.e.:
    $$\E\{\Psi(\G_{t+1}) - \Psi(\G_t) | \G_t\} > 0.$$

    Combining these two cases, we find that $\Psi(\G_t)$ is a sub-martingale, yielding $\E\{\Psi(\G_t)\} \geq \E\{\Psi(\G_0)\} > 0$. Because $|\G_t^{\PT, F}| \leq \Phi(\G_t)$, we have:
    $$ \frac{\eps ( 1- p) \mathbf{a}(0)}{1 - \eps} \cdot \E\{|\G_t^T|\} - \E\{|\G_t^{\PT, F}|\} \geq \E\left\{ \frac{\eps ( 1- p) \cdot \mathbf{a}(0)}{1 - \eps} \cdot |\G_t^T| - \Phi(\G_t)\right\} = \E\{\Psi(\G_t)\} > 0,$$
    as required.
\end{proof}

\section{Proof of Error Survival} \label{section:error-survival}

We now shift our focus to the study of \textit{error survival} for the family of generalized CKPs. We ask the question: for what settings of features do errors propagate forever in generalized CKPs? Our results primarily depend on the checking probability $p$ and the attachment features $\textsc{Attach} = (\mathbf{a}, M)$. 

\subsection{Error survival in the simple CKP}

In this section, we prove the following results about error survival (see Definition \ref{def:error-survival}) in simple CKPs. Let $\min(M)$ denote the minimum value that the random variable $M$ can take. Recall Definition~\ref{def:regular-attachment} of $(b_1, b_2)$-regular attachment functions. When $b_1 \leq \mathbf{a}(d + 1) - \mathbf{a}(d)$ but the growth is not necessarily bounded above, we call the function $(b_1, \infty)$-regular. 

\begin{thm}[Simple CKP error survival]
\label{thm:simple-ckp-error-survival-exhaustive}
    For all simple CKPs with $(\textsc{Attach} = (\mathbf{a}, M), \textsc{Check} = (p, k))$ such that $\mathbf{a}$ is $(\mathbf{a}(0), \infty)$-regular and $M \geq 1$ and $p \in [0, 1)$ satisfy 
    \begin{equation}\label{eq:simple-survival-condition1}
    \frac{\E\{M\}}{\min (M) + 1}
    < 1 ~\text{ and }~ p \leq \frac{1}{2}\bpar{ 1 - \frac{\E\{M\}}{\min (M) + 1} },
    \end{equation}
    the error effect survives (see Definition~\ref{def:error-survival}) with positive probability under the \textsc{Stringy} and \textsc{BFS} checking mechanisms. Similarly, for the
    \textsc{Exhaustive BFS} and \textsc{Parent-wise BFS} checking mechanisms, survival holds if the features satisfy \eqref{eq:simple-survival-condition1} with $p$ replaced by $p\E\{M\}$.
\end{thm}

This proof in fact works for any checking mechanism that has a bounded total checking probability (e.g., $p$ in the \textsc{Stringy} and \textsc{BFS} cases) and a bounded number of possible deleted minimal \False\ nodes, and we could obtain a bound on $p$ in terms of these parameters.
\vspace{.5\bs}

We again prove this result using the potential method, by constructing a potential and showing that its sequence of values forms a super-martingale. 

\paragraph{Minimal-False and Leaves Potential}
We first recall that $\F_t$ denotes the minimal false nodes, i.e., roots and \CF\ nodes. 
For the checks stopping at the first minimal false node seen, note that we verify if a node is a root by checking if any parent node is $\PF$, stopping at the first such parent node found.

For the potential, we now additionally define \textbf{non-root leaves} $\G_t^L$, which are nodes with no \PF\ parents and out-degree zero. Specifically, we let $\L_t \subset \G_t^L$ be the \textbf{\CT\ non-root leaves}. 
We define
\begin{equation}
    \Phi_{F,L}(\G_t) = |\F_t| + |\L_t|,
\end{equation}
and note that $\Phi_{F,L}(\G_t) \leq |\G_t^{\PT,F}|$, the number of \PT\ \False\ nodes. 
We prove the following property about the expected change in the potential in a time step.

\begin{lem}\label{lemma:survival-rl-potential}
    Let $X_t = (\G_t, L_t)$ be the state of the simple CKP with features $(\textsc{Attach} = (\mathbf{a}, M), \textsc{Check} = (p, k))$ at time $t$ such that $\mathbf{a}$ is $(\mathbf{a}(0), \infty)$-regular and $M \geq 1$ and $p \in [0, 1)$ satisfy 
    \begin{equation}\label{eq:survival-rl-potential-stringy}
    \frac{\E\{M\}}{\min (M) + 1}
    \leq 1 ~\text{ and }~ p \leq \frac{1}{2}\bpar{ 1 - \frac{\E\{M\}}{\min (M) + 1} }.
    \end{equation}
    Let $\Delta_t =\Phi_{F,L}(\G_{t+1}) - \Phi_{F,L}(\G_t) $. If $\Phi_{F,L}(\G_t) > 0$, then $\E\{\Delta_t \mid \G_t\} > 0$
    when the CKP uses the \textsc{Stringy} or \textsc{BFS} checking mechanism. When the CKP uses the \textsc{Exhaustive} or \textsc{Parent-wise BFS} checking mechanism, we have $\E\{\Delta_t \mid \G_t\} > 0$ if \eqref{eq:survival-rl-potential-stringy} holds with $p$ replaced by $p\E\{M\}$. 
\end{lem}
When $\Phi_{F,L}(\G_t) = 0$, then for all $t' > t$, we also have $\Phi_{F,L}(\G_{t'}) = 0$ because the process has stopped (i.e., all errors have been eliminated).

\begin{proof}[Proof of Lemma \ref{lemma:survival-rl-potential} for \textsc{Stringy} and \textsc{BFS}]
    First, let $\beta$ be the probability that when an edge from a new node to $\G_t^\PT$ is created, it is connected to a \CT\ non-root leaf. Letting $Z \ceq \sum_{v \in \G_t^\PT} \mathbf{a}(\deg_{\PT}(v))$, we have that
    \begin{equation}\label{eq:beta-bound}
        \beta = \frac{|\L_t| \cdot \mathbf{a}(0)}{Z} < \frac{1}{\min(M)+1}.
    \end{equation}
    Indeed, since a non-root leaf has $\PT$-indegree at least $\min(M)$, it contributes at least $\left(\min(M)+1\right) \cdot \mathbf{a}(0)$ to the potential: $\mathbf{a}(0)$ for the node itself and $\min(M) \cdot \mathbf{a}(0)$ for increasing the outdegrees of its $\PT$ parent nodes, because $\mathbf{a}(d + 1) - \mathbf{a}(d) \geq \mathbf{a}(0)$ for all $d \in \mathbb{Z}_{\geq 0}$. Thus we have $Z > (\min(M)+1)\cdot |\L_t| \cdot \mathbf{a}(0)$, yielding the stated bound. Notice that the worst-case structure is a rooted star DAG. 
    
    Let $v$ be the new node added at time $t+1$, with $m \sim M$ edges created between $v$ and $\G_t^\PT$. 
    We break the analysis into the change in the minimal false nodes $|\F_t|$ and the non-root leaves $|\L_t|$, denoted by $\Delta_{|\F_t|}$ and $\Delta_{|\L_t|}$, respectively.
    The following holds for both the \textsc{Stringy} and \textsc{BFS} checking mechanisms, which perform an overall check with probability $p$.
    \begin{enumerate}
        \item $|\F_t|$ can only change if a check is performed and successful. Since we can remove at most one minimal false node, we have $\Delta_{|\F_t|} = -1$ with probability at most $p$. 
        \item $|\L_t|$ experiences an increase from the new node $v$ which is a non-root leaf, and a decrease from existing leaves gaining a child. Note that if an existing leaf has one of its parents checked and marked as \PF, this leaf then becomes a minimal false node, resulting in no net change in $\Phi_{F,L}$. 
        \begin{itemize}
            \item With probability at least $1-p$, there is no successful check and $v$ contributes $+1$ to $\Delta_{|\L_t|}$. 
            \item Each edge connects to a non-root leaf parent with probability $\beta$, contributing $-1$ to $\Delta_{|\L_t|}$. 
        \end{itemize}
    \end{enumerate}
    Combining, 
    \[
    \E\{\Delta_t \mid {M = m}\} \geq -p + (1-p) + m \beta (-1) > 1-2p - \frac{m}{\min(M) + 1},
    \]
    and taking an expectation over $M$ yields the bound 
    \begin{equation}\label{eq:delta_t-stringy-BFS}
        \E\{\Delta_t \mid \G_t \} > 1 - 2p - \frac{\E\{M\}}{\min(M)+1}.
    \end{equation}
For $M = m$ a fixed constant, this is positive if $p \leq 1/(2m+2)$.
\end{proof}

\begin{proof}[Proof of Lemma \ref{lemma:survival-rl-potential} for \textsc{Exhaustive} and \textsc{Parent-wise BFS}]
The analysis is very similar to the previous case. The differences in (i) and (ii), which hold for both checking mechanisms, are as follows: 
\begin{enumerate}
    \item Now, each of the $m$ edges performs a separate check with probability $p$, potentially removing a minimal false node and contributing $-1$ to $\Delta_{|\F_t|}$. 
    \item As for the change in $|\L_t|$,
    \begin{itemize}
        \item With probability at least $1-mp$, there is no successful check and $v$ contributes $+1$ to $\Delta_{|\L_t|}$.
        \item Each edge once again connects to a non-root leaf parent with probability $\beta$.
    \end{itemize}
\end{enumerate}
Combining, 
\[
    \E\{\Delta_t \mid M = m\} > -mp + (1-mp) - \frac{m}{\min(M) + 1},
\]
which gives 
\begin{equation}\label{eq:delta_t-exhaustive}
    \E\{\Delta_t \mid \G_t\} \geq 1 - 2p \E\{M\} - \frac{\E\{M\}}{\min(M)+1}
\end{equation}
after the expectation over $M$.
For a non-trivial bound, we need ${\E\{M\}}/{(\min(M) + 1)} < 1$.
\end{proof}

We now use the properties of the minimal-false and leaves potential to prove Theorem~\ref{thm:simple-ckp-error-survival-exhaustive}. First, we recall a Lemma about sub-martingales from \cite{BEMMS}. We then consider a truncated version of our potential to ultimately prove error survival.

\begin{lem}[\cite{BEMMS}, Lemma 3.6] \label{lem:submartingale-BEMMS}
    Let $\{X_t\}_{t \geq 0}$ be a non-negative sub-martingale with $X_0 > 0$. Assume there exist constants $c_1, c_2 > 0$ such that, for every $t \geq 0$, when $X_t \neq 0$:
    \begin{enumerate}
        \item $|X_{t+1} - X_t| \leq c_1$ almost surely, and 
        \item $\E\{X_{t+1} - X_t | X_t\} > c_2$.
    \end{enumerate}
    With positive probability, $X_t > 0$ for all $t \in \mathbb{N}$.
\end{lem}

\begin{proof}[Proof of Theorem \ref{thm:simple-ckp-error-survival-exhaustive} given Lemma~\ref{lemma:survival-rl-potential}]
For any large constant $C$, there exists some time $t_0$ such that $\Phi_{F,L}(\G_{t_0}) \geq C$ with probability bounded away from 0. We condition on this event and consider an upper-bounded version $Y_t$ of the process $X_t = \Phi_{F,L}(\G_{t_0+t})$. To do so, we define $Y_0 = \Phi(\G_{t_0})$ and for each $t \geq 0$, let 
$Y_{t+1} = \max(Y_t + \min(X_{t+1} - X_t, 1), 0)$. Note that since $X_t \geq 0$, we have $Y_t \leq X_t$. 

Then, $Y_t$ is a non-negative sub-martingale satisfying the conditions of Lemma~\ref{lem:submartingale-BEMMS} for $c_1 = 2$, (in the \textsc{Stringy} case), $1 + \max(M)$ (in the \textsc{BFS} and \textsc{Exhaustive BFS} cases) or $2\max (M)$ (for the \textsc{Parent-wise BFS}), and $c_2$ being the right hand side of either \eqref{eq:delta_t-stringy-BFS} or \eqref{eq:delta_t-exhaustive}. So we have 
\[\P\{\min_{t \geq 0} X_t \geq 0 \} \geq \P\{\min_{t \geq 0} Y_t \geq 0\} > 0\] 
by Lemma~\ref{lem:submartingale-BEMMS}. Since $X_t = \Phi_{F,L}(\G_t) \leq |\G_t^{\PT,F}| = |\G_t^\PT|$, this proves error survival.
\end{proof}

\subsection{Error survival in the generalized CKP}

We now prove error survival in the generalized CKP setting. We first make note of some changes that arise when arguing about generalized CKPs instead of simple CKPs. Recall that, as in Definition \ref{def:error-survival}, for error survival we must argue that with some positive probability, there exists a $\CF$ node $u$ such that for all time steps $t$, there exists at least one $\PT$ node in the sub-DAG rooted at $u$. In the simple CKP setting, there was only one $\CF$ node (the one added in the first step), and the potential argument could therefore be made globally. In the general setting, we focus on sub-DAGs rooted at individual $\CF$ nodes, which now may not be the whole CKP. While a similar argument for simple CKPs carries over to the general CKP setting -- namely the idea of tracking the number of minimal false nodes and non-root leaves -- we must condition on connecting to the sub-DAG we are tracking in the analysis. 

Additionally, in the generalized CKP setting, we obtain bounds that also depend on the error probability $\eps$ and adversarial features $\textsc{Adversary} = (q, r)$. With respect to the dependency on $\eps$, we prove that if the checking probability $p$ is less than $\eps$ then errors survive forever. Beyond this, for the bounds that mirror the simple CKP case, the results hold when $M$ is small (e.g. when $M$ is constant, it's needed that $M \leq 3$). However, the strengthened results when $M$ is small provide evidence that in general there are parameters $p$ beyond $p < \eps$ such that errors survive.

We prove the following theorem, which is a more formal version of Theorem~\ref{thm:intro-error-survival}.

\begin{thm}[General CKP error survival]\label{thm:general-ckp-error-survival}
    Let $\mathbf{a}$ be a $(\mathbf{a}(0), b)$-regular function and let $M \geq 1$ be a bounded random variable. For CKPs with features that satisfy either
    \begin{equation}\label{eq:general-survival-condition-eps}
    p < \e 
    \end{equation}
    or both $\eta_M \ceq \frac{2\E\{M\} + \min(M) - 1}{2 (\min(M) + 1)}\leq 1$ and
    \begin{equation}\label{eq:general-survival-condition1}
    (1 - q) \bigg[ p \left(-2 + \eps \left( 1 + \frac{b}{b + \mathbf{a}(0)} \eta_M \right)  \right) + 1 - \eta_M \left( 1 - \frac{\eps \cdot \mathbf{a}(0)}{b + \mathbf{a}(0)}\right) \bigg] - q r \geq 0,
    \end{equation}
    the error effect survives (see Definition~\ref{def:error-survival}) with positive probability under the \textsc{Stringy} and \textsc{BFS} checking mechanisms. Similarly,  survival holds with positive probability for the CKPs with
    \textsc{Exhaustive BFS} and \textsc{Parent-wise BFS} checking mechanisms if the features satisfy \eqref{eq:general-survival-condition-eps} or  \eqref{eq:general-survival-condition1} with $p$ replaced by $p\E\{M\}$.
\end{thm}

Note that this result holds for any $k \in \N$. We define a potential and show that for the features satisfying (\ref{eq:general-survival-condition1}) in the theorem above, the expected change in the potential is positive. 
\begin{proof}[Proof of Theorem~\ref{thm:general-ckp-error-survival}]
    This follows by applying the same steps as in the proof of Theorem~\ref{thm:simple-ckp-error-survival-exhaustive} to obtain error survival from a positive expected change in potential, see Lemma~\ref{lemma:minimal-false-potential-survival} for the $p < \e$ bound, and Lemmas~\ref{lemma:minimal-false-potential-survival} and \ref{lemma:survival-generalcase-secondpotential-exhaustiveBFS} for the other, for the two types of checking mechanisms. To apply the same proof, we need that the change in potential is bounded below, which still holds in this context (with the lower bound of $\min\left(-\max (M), -r\right)$).
\end{proof}

We define and analyze two different potentials and find the range of settings of features for which the expected change in potential is positive. While the first potential gives the intuitive $p < \e$ bound, the second bound often yields a more extensive range when $\E\{M\}$ and $\min(M)$ are small, with more dependence on the various features.

\paragraph{Minimal-False potential}
We define 
\[\Phi_{F}(\G_t) = |\F_t|.
\]

\begin{lem} \label{lemma:minimal-false-potential-survival}
Consider any CKP. Consider a node $u$ that was originally $\CF$ and the sub-DAG $\G_t^{(u)}$ rooted at this node. Let $t \in \mathbb{N}$ and let $\Delta_t =\Phi_{R,F}(\G_{t+1}^{(u)}) - \Phi_{R,F}(\G_t^{(u)}) $. When $\Phi_{R,F}(\G_t^{(u)}) > 0$ then
    $\E\{\Delta_t \mid \G_t\} \geq (\eps - p) ( 1 - q)$
    if the CKP uses \textsc{Stringy} or \textsc{BFS} checking mechanisms, 
    and 
    $\E\{\Delta_t \mid \G_t\} \geq (\eps - \E\{M\}p)(1 - q)$
    if the CKP uses \textsc{Exhaustive} or \textsc{Parent-wise BFS}. 
\end{lem}

\begin{proof}
    If the new node does not connect to a parent in $\G_t^{(u)}$ then there is no change to the potential $\Phi_{R,F}(\G_t^{(u)})$ and $\Delta_t = 0$. Therefore, consider the case where the new node connects to a parent in $\G_t^{(u)}$.

    Let us start with the analysis for the \textsc{Stringy} and \textsc{BFS} checking mechanisms. Suppose that the next step is non-adversarial. If the new node is $\CT$, the potential only changes if there is a check, and $\Delta_t \geq -1$, because the check stops at the first minimal false node found. 
    If it is $\CF$, it only changes if there is no check, and then $\Delta_t = 1$. Otherwise, if the next step is adversarial, the potential cannot decrease and $\Delta_t \geq 0$.

    Combining all cases yields the stated bound. 
    For the \textsc{Exhaustive} and \textsc{Parent-wise BFS} checking mechanisms, the same analysis can be performed edge-by-edge.
\end{proof}

We now show the case~\eqref{eq:general-survival-condition1}, which yields stronger results for some settings of the features, when $\E\{M\}$ and $\min(M)$ are small. We adapt the minimal-false and leaf potential from the simple CKP case. For this potential, we revise the definition of leaf nodes to be nodes with $\deg_\CT = 0$. A node $v$ is thus now a \textbf{$\CT$ non-root leaf} ($v \in \L_t$) if it is a \CT\ node with no \PF\ parents and $\deg_\CT(v) = 0$. 

\paragraph{General Minimal-False and Leaves potential} Consider a node $u$ that was originally $\CF$ and the sub-DAG $\G_t^{(u)}$ rooted at this node. Define the potential as
\begin{equation}
    \Phi_{F,L}(\G_t^{(u)}) = |\F_t| + \sum_{v \in \L_t^{(u)}} \frac{\mathbf{a}(0)}{\mathbf{a}(\deg_{\CF}(v))},
\end{equation}
where $\L_t^{(u)} = \L_t \cap \G_t^{(u)}$.

The reason we modify the potential so that each $v \in \L_t^{(u)}$ contributes $\mathbf{a}(0)/\mathbf{a}(\deg_{\CF}(v))$ is that, with our new definition of $\CT$ non-root leaves, we can no longer upper-bound the probability of connecting to a non-root leaf as in the simple CKP case. However, we can still upper-bound $\mathbf{a}(0)\cdot|\L_t|/Z$ by $1 / (\min(M) + 1)$, and making this change to the potential yields an $\mathbf{a}(0)|\L_t|/Z$ term in the analysis.

We prove the following lemma regarding this potential.
\begin{lem} \label{lemma:survival-generalcase-secondpotential-stringy}
    Let $\mathbf{a}$ be a $(\mathbf{a}(0), b)$-regular function and let $M \geq 1$ be a bounded random variable such that ${2 \E\{M\} + \min(M) - 1}\leq {2 (\min(M) + 1)}$. Consider any CKP with the \textsc{Stringy} or \textsc{BFS} checking mechanism, a node $u$ that was originally $\CF$, and the sub-DAG 
    $\G_t^{(u)}$ rooted at this node. Let $t \in \mathbb{N}$ and $\Delta_t =\Phi_{F,L}(\G_{t+1}^{(u)}) - \Phi_{F,L}(\G_t^{(u)})$. When $\Phi_{F,L}(\G_t^{(u)}) > 0$ then $\E\{\Delta_t \mid \G_t\}$ is strictly greater than
    \begin{equation}\label{eq:survival-general-condition}
        \begin{aligned}
        (1 - q) \bigg[ p &\left(-2 + \eps \left( 1 + \frac{b}{b + \mathbf{a}(0)} \frac{2 \E\{M\} + \min(M) - 1}{2(\min(M)+1)} \right)  \right) + \\
        &1 - \frac{2\E\{M\} + \min(M) - 1}{2(\min(M)+1)} \left( 1 - \frac{\eps \cdot b}{b + \mathbf{a}(0)}\right) \bigg] - q r,
    \end{aligned}
    \end{equation}
\end{lem}

\begin{proof}
We first prove this lemma for the checking mechanisms that perform an overall check with probability $p$.

Consider the case where the next step is non-adversarial (with probability $1 - q$). Here, we break the analysis into four cases, depending on whether the new node is $\CT$ or $\CF$ and whether it runs a check. Recall that the number of edges $m$ that the new node creates is chosen according to the random variable $M$. Let $Z = \sum_{w \in \G_t^\PT} \mathbf{a}(\deg(w))$ and $Z_u = \sum_{w \in \G_t^\PT(u)} \mathbf{a}(\deg(w))$.

Note that, if the new node does not connect to $\G_t^{(u)}$, then $\Delta_t = 0$. Therefore, we condition on the case that it does connect to $\G_t^{(u)}$.
\begin{enumerate}
    \item {\sl New node is $\CT$ and no check:} Since the new node is a $\CT$ non-root leaf, it contributes $+1$ to the potential. 
    For each edge, the edge may connect to a non-root leaf $v$ with probability proportional to $\mathbf{a}(\deg(v)) = \mathbf{a}(\deg_{CF}(v))$. In this case, there is a change of $-\mathbf{a}(0)/\mathbf{a}(\deg_{\CF}(v))$ to the potential because adding the $\CT$ node to the leaf node makes it no longer a leaf. 
    Since we condition on at least one edge connecting to a parent in $\G_t^{(u)}$, for one of the edges, the probability of connecting to a non-root leaf is $\sum_{v \in \L_t^{(u)}} \frac{\mathbf{a}(\deg_{\CF}(v))}{Z_u}$, while for the rest of the edges, this probability is $\sum_{v \in \L_t^{(u)}} \frac{\mathbf{a}(\deg_{\CF}(v))}{Z}$ as the parent may be chosen among any node in $\G_t^{\PT}$.
    Therefore,
    \begin{align*}
        \E\{&\Delta_t \ind{\text{non-adversarial step, new node is } \CT \text{, no check}} \mid \G_t \} \\
        &= (1 - \eps) (1-p) \bigg( 1 + \left(\E\{M\} - 1 \right) \sum_{v \in \L_t^{(u)}} \frac{\mathbf{a}(\deg_{\CF}(v))}{Z} \frac{-\mathbf{a}(0)}{\mathbf{a}(\deg_{\CF}(v))} \\ 
        & \hspace{3cm} + \sum_{v \in \L_t^{(u)}} \frac{\mathbf{a}(\deg_{\CF}(v))}{Z_u} \frac{-\mathbf{a}(0)}{\mathbf{a}(\deg_{\CF}(v))}\bigg) \\
        &= (1-\eps) (1-p) \left(1 - \left(\E\{M\} - 1 \right) \mathbf{a}(0) \frac{|\L_t^{(u)}|}{Z} - \textbf{a}(0) \frac{|\L_t^{(u)}|}{Z_u}\right).
    \end{align*}
    Now, ${\mathbf{a}(0) |\L_t^{(u)}|}/{Z} < 1/(\min(M)+1)$ because every non-root leaf has at least $\min(M)$ parent nodes for which it contributes at least $\mathbf{a}(0)$ to the attachment weight. Additionally, ${\mathbf{a}(0) |\L_t^{(u)}|}/{Z_u} < 1/2$ because every non-root leaf in $\G_t^{(u)}$ has at least one parent in $\G_t^{(u)}$ for which it contributes at least $\mathbf{a}(0)$ to the attachment weight.
    Therefore, we obtain
    $$\E\{\Delta_t \ind{\text{non-adversarial step, new node is } \CT \text{, no check}} \mid \G_t \} >  (1-\eps) (1-p) \left( 1 - \frac{2\E\{M\} + \min(M) - 1}{2 (\min(M)+1)}\right).
    $$
    
    \item {\sl New node is $\CT$ and check:} We assume the check is successful, since 
    an unsuccessful check will result in an expected increase in the potential, while a successful check will result in a decrease. We can remove at most one minimal false node, and additionally, if any of the parent nodes is a $\CT$ non-root leaf, it could be removed. 
    Therefore,
    \begin{align*}
        \E\{&\Delta_t \ind{\text{non-adversarial step, new node is } \CT \text{, check}} \mid \G_t \} \\
        &= (1- \eps) p \Big(- 1 + \left(\E\{M\} - 1 \right) \sum_{v \in \L_t^{(u)}} \frac{\mathbf{a}(\deg_{\CF}(v))}{Z} \frac{-\mathbf{a}(0)}{\mathbf{a}(\deg_{\CF}(v))} + \sum_{v \in \L_t^{(u)}} \frac{\mathbf{a}(\deg_{\CF}(v))}{Z_u} \frac{-\mathbf{a}(0)}{\mathbf{a}(\deg_{\CF}(v))}\Big) \\
        &> - (1-\eps) p \left(1 + \frac{2\E\{M\} + \min(M) - 1}{2 (\min(M)+1)} \right)
    \end{align*}

    \item {\sl New node is $\CF$ and no check:} The new node contributes $+1$ to the count of $\CF$ nodes in the potential. 
    Each edge connects to a non-root leaf $v$ with probability proportional to $\mathbf{a}(\deg_{CF}(v))$. In this case, since a leaf node with an additional $\CF$ child remains a leaf, the potential changes by 
    \[\frac{\mathbf{a}(0)}{\mathbf{a}(\deg_{\CF}(v) + 1)} - \frac{\mathbf{a}(0)}{\mathbf{a}(\deg_{\CF}(v))}  
    \geq - \frac{b}{b+\mathbf{a}(0)} \cdot \frac{\mathbf{a}(0)}{\mathbf{a}(\deg_{\CF}(v))}.\] 
    
    Therefore,
    \begin{align*}
    \E\{ &\Delta_t  \ind{\text{non-adversarial step, new node is } \CF \text{, no check}} \mid \G_t \} \\
    &= \eps (1-p)  \bigg( 1 + \left(\E\{M\} - 1 \right) \sum_{v \in \L_t^{(u)}} \frac{\mathbf{a}(\deg_{\CF}(v))}{Z}  \frac{-b}{b + \mathbf{a}(0)}\frac{\mathbf{a}(0)}{\mathbf{a}(\deg_{\CF}(v))} \\ 
    &\hspace{2cm} + \sum_{v \in \L_t^{(u)}} \frac{\mathbf{a}(\deg_{\CF}(v))}{Z_u}  \frac{-b}{b + \mathbf{a}(0)}\frac{\mathbf{a}(0)}{\mathbf{a}(\deg_{\CF}(v))}\bigg) \\
    &> \eps  (1-p) \left(1 - \frac{b}{b + \mathbf{a}(0)} \cdot \left( \frac{2\E\{M\} + \min(M) - 1}{2(\min(M)+1)}\right)\right).
    \end{align*}
    
    \item {\sl New node is $\CF$ and check:} The check removes only the new node, and $$\E\{\Delta_t \ind{\text{non-adversarial step, new node is } \CF \text{, check}} \mid \G_t \} = 0.$$
\end{enumerate}

Next, consider the case where the next step is adversarial (with probability $q$). The worst-case decrease for the potential occurs when the new node is $\CT$ and attaches to $r$ parent nodes that are $\CT$ non-root leaves, which causes these parent nodes to no longer be $\CT$ non-root leaves. Therefore, 
$$\E\{\Delta_t \ind{\text{adversarial step}}\} \geq -r.$$

Combining all four cases yields \eqref{eq:survival-general-condition} and the associated bound on $p$. 
\end{proof}

We now state the comparable Lemma for the \textsc{Exhaustive BFS} and \textsc{Parent-wise BFS} checking mechanisms, which are the ones that run separate checks with probability $p$ for each parent node.

\begin{lem} \label{lemma:survival-generalcase-secondpotential-exhaustiveBFS}
Let $\mathbf{a}$ be a $(\mathbf{a}(0), b)$-regular function and let $M \geq 1$ be a bounded random variable such that ${2 \E\{M\} + \min(M) - 1}\leq {2 (\min(M) + 1)}$. Consider any CKP with the \textsc{Exhaustive BFS} or \textsc{Parent-wise BFS} checking mechanism, a node $u$ that was originally $\CF$, and the sub-DAG 
    $\G_t^{(u)}$ rooted at this node. Let $t \in \mathbb{N}$ and $\Delta_t =\Phi_{F,L}(\G_{t+1}^{(u)}) - \Phi_{F,L}(\G_t^{(u)})$. When $\Phi_{F,L}(\G_t) > 0$ then 
    $\E\{\Delta_t \mid \G_t\}$ is greater than
    \begin{equation}\label{eq:survival-general-condition-exhaustiveBFS}
    \begin{aligned}
        (1 - q) \cdot \bigg[ & p\E\{M\} \left(-2 + \eps \left( 1 + \frac{b}{b + \mathbf{a}(0)} \frac{2 \E\{M\} + \min(M) - 1}{2(\min(M)+1)} \right)  \right) \\ 
        &+ 1 - \frac{2 \E\{M\} + \min(M) - 1}{2(\min(M)+1)} \left( 1 - \frac{\eps \cdot b}{b + \mathbf{a}(0)}\right) \bigg] - qr.
    \end{aligned}
    \end{equation}
\end{lem}

\begin{proof}
As for Lemma \ref{lemma:survival-rl-potential} for the simple CKP, we can modify the previous proof by arguing about the change in potential one edge at a time. We obtain similar expressions for the various cases in the proof of Lemma~\ref{lemma:survival-generalcase-secondpotential-stringy} for the new node being \CT\ or \CF; writing the expressions out gives 
\eqref{eq:survival-general-condition-exhaustiveBFS}.
\end{proof}

\subsection{Error survival for some irregular attachment functions}

\begin{lem}\label{lem:growing-a}
    Suppose $\mathbf{a}(d) \gtrsim d^3$. Then, for any adversarial parameters, any error parameters, and any checking parameters $p < 1$, $k \in \N$, the error survives forever with positive probability.
\end{lem}
\begin{proof}
    With positive probability, the CKP evolves to consist of one \CF\ node, followed by a path of $k$ \CT\ nodes (if $\min(M) > 1$, this is a path of multi-edges), followed by a \CT\ node $v$ with 1 \CT\ child. We denote this as the starting state. Consider the event $E$ that for each future step of the process, the new node connects to $v$. This event implies error survival since the \CF\ node will then never be reached; we prove it occurs with positive probability. 
    
    Let $E_n$ denote the event that the new node connects to $v$ at step $n$ after initialization, where without loss of generality, we proceed edge-by-edge. Conditioned on events up to $n-1$, $v$ has out-degree $n$, so $\P\{E_n \mid E_1, \dots, E_{n-1}\} = \mathbf{a}(n) / \bbr{\mathbf{a}(n) + (k+1)\mathbf{a}(\min (M)) + n\mathbf{a}(0)}$. Then, 
    \begin{equation}
        \P\{E\} = \prod_{n=1}^\infty \P\{E_n \mid E_1, \dots, E_{n-1}\} = \prod_{n=1}^\infty \bpar{1 - \frac{(k+1)\mathbf{a}(\min (M)) + n}{\mathbf{a}(n) + (k+1)\mathbf{a}(\min (M)) + n}}
    \end{equation}
    converges to a non-zero value if $\mathbf{a}(n) \gtrsim n^3$, completing the proof.
\end{proof}
\begin{rem}\label{rem:holes}
    An even simpler argument can be made for distributions with holes. Suppose there exists some $\Tilde{d} \in \N$ such that $\mathbf{a}(\Tilde{d}) = 0$. Then, with positive probability, the CKP evolves to consist of $k+1$ levels of nodes each with $\Tilde{d}$ children, where the root is $\CF$ and the rest of the nodes are $\CT$. (If $\min(M) = 1$, we can consider a complete $\Tilde{d}$-ary tree; if $\min(M) > 1$, an analogous structure can be constructed.) Then, new nodes added from this point on can only connect below the leaves of this constructed structure, and the $\CF$ node cannot be reached.
\end{rem}

\section{Monotonicity of error elimination for the tree-like simple CKP} \label{section:monotonicity}

We now present couplings of tree-like simple CKP processes that have different checking parameters $\textsc{Check} = (p, k)$ and follow the same attachment procedure $\mathbf{a}: \mathbb{Z}_{\geq 0} \to \mathbb{R}_{\geq 0}$, through which we obtain the following monotonicity result with respect to the checking probability $p$ or the checking depth $k$. In the sections below, $M = 1$ always. We prove the following theorem.

\begin{thm}[Monotonicity of error elimination for checking parameters] \label{thm:monotonicity-p-k}
    For all $p_2 \geq p_1$ and $k_2 \geq k_1$, if the error effect is completely eliminated in the simple CKP with features $(\textsc{Attach} = (\mathbf{a}, M=1),  \textsc{Check} = (p_1, k_1))$, then the error effect is completely eliminated in the simple CKP with features $(\textsc{Attach} = (\mathbf{a}, M=1),  \textsc{Check} = (p_2, k_2))$.
\end{thm}

Therefore, if the error effect in the simple CKP with $\textsc{Check} = (p_2, k_2)$ survives with positive probability, then the error effect in the simple CKP with $\textsc{Check} = (p_1, k_1)$ and the same attachment procedure survives as well with positive probability.
\vspace{.25\bs}

We prove this theorem by separately fixing $p$ and $k$ in Lemmas~\ref{lemma:monotonicity-proof-for-p} and \ref{lemma:monotonicity-proof-for-k}, noting that we can then chain error elimination from the simple CKP with $\textsc{Check} = (p_1, k_1)$ to the simple CKP with $\textsc{Check} = (p_2, k_1)$ and finally to the simple CKP with $\textsc{Check} = (p_2, k_2)$ (where all three CKPs follow the same attachment procedure $\mathbf{a}$).

Remark that, while we can obtain monotonicity results for each of $p$ and $k$, there is no monotonicity with respect to the product $p \cdot k$. That is, there exist parameters satisfying $p_2 \cdot k_2 > p_1 \cdot k_1$ such that a simple CKP with $\textsc{Check} = (p_1, k_1)$ eliminates all errors but errors survive with positive probability in the simple CKP with $\textsc{Check} = (p_2, k_2)$ that evolves according to the same attachment procedure. For example, \cite[Theorems~2.1, 2.2]{BMMS:23} imply that the error effect in the simple CKP that evolves according to preferential attachment with $\textsc{Check} = (6/7, 4)$ is completely eliminated, while the error effect in the 
simple CKP that evolves according to preferential attachment with $\textsc{Check} = (1/4, 20)$ survives with positive probability.

\subsection{Overview of the couplings} \label{section:coupling-monotonicity}

We first describe the motivation and high-level ideas of the coupling. This coupling is inspired by the coupling given in Section 5 of \cite{flaxman_frieze_vera_2007}, which couples two preferential attachment processes: one with adversarial deletions at each time step $G_s$
and one with adversarial deletions at a fixed ending time step $G_s^*$. 
The process keeps track of which nodes and edges are ``alive'' in both processes, and which ones are alive only in $G_{s}^{*}$ and not $G_{s}$, to form their coupling and argue about the impact of adversarially deleted vertices and edges. 
Here, we also keep track of nodes that are ``alive'' in either both processes or only one, and
similarly 
ensure that the two processes are each individually generated according to preferential attachment. 

First, recall that a tree-like CKP with checking probability $p$ and checking depth $k$ and attachment procedure $\mathbf{a}$ is defined as a series of states $X_t = (\T_t, L_t)$ where each $\T_t$ is a finite rooted tree and labels are chosen from 
$\L = \{\CT, \CF, \PF\}$.

Consider the coupling with $k$ and $\mathbf{a}$ fixed. Then, $p$ is our only variable so we can simplify the notation and write $p_1$-CKP and $p_2$-CKP. 
For the purpose of the coupling, we define a new process that is given as a series of states $Z_0, Z_1, \dots$, where each $Z_t$ is given by a finite rooted tree $\mathcal{T}_t$ with each of the vertices labeled from $\L' = \{\CT, \CF, \PF, \ZCF, \ZCT, \ZNCT, \ZNPF\}$:
\begin{itemize}
    \item $\ZCF$ and $\ZCT$ stand for zombie-$\CF$ and zombie-$\CT$: these nodes are alive (not yet found to be \False) in the $p_1$-CKP but are dead (found to be \False) in the $p_2$-CKP.
    \item $\ZNCT$ nodes are nodes that should never have been added to the $p_2$-CKP (as their parent was already found to be \False). They are also zombie nodes: alive in the $p_1$-CKP but dead in the $p_2$-CKP. These nodes will ultimately be removed in the construction of the $p_2$-CKP. Note that there are no \ZNCF\ nodes since we are in the simple process and all added nodes are \CT.
    \item \ZNPF\ nodes are those that were once 
    \ZNCT, then are found to be \False\ in the $p_1$-CKP. These are also removed in the $p_2$-CKP as they never existed in the first place.
\end{itemize}
See Figure~\ref{fig:evolution-monotonicity} for a visualization.
At a high level, 
we design this coupling such that 
\begin{itemize}
    \item the $p_1$-CKP is generated by relabeling $\ZCF
    \Rightarrow \CF$ and $\{\ZCT, \ZNCT\} \Rightarrow \CT$.
    \item the $p_2$-CKP is generated as a ``slowed down'' subset of $\{Z_i\}$ by relabeling $\{\ZCF, \ZCT\} \Rightarrow \PF$, and removing all 
    $\ZNCT$ and $\ZNPF$ vertices.
\end{itemize}
This coupling allows us to track the error elimination in both processes together.
The coupling with $p$ fixed and $k$ varying is similar, and in fact, uses all the same labels in $\L'$. 

We note that this proof only works for the simple tree-CKP because of issues that arise when $\CT$ or $\CF$ nodes have $\CF$ children (for extending from simple to general) or when $\CF$ nodes have multiple paths to various $\False$ ancestors (for extending to the DAG CKP). We leave as an open question whether a variation of this technique can be used to prove monotonicity more generally.

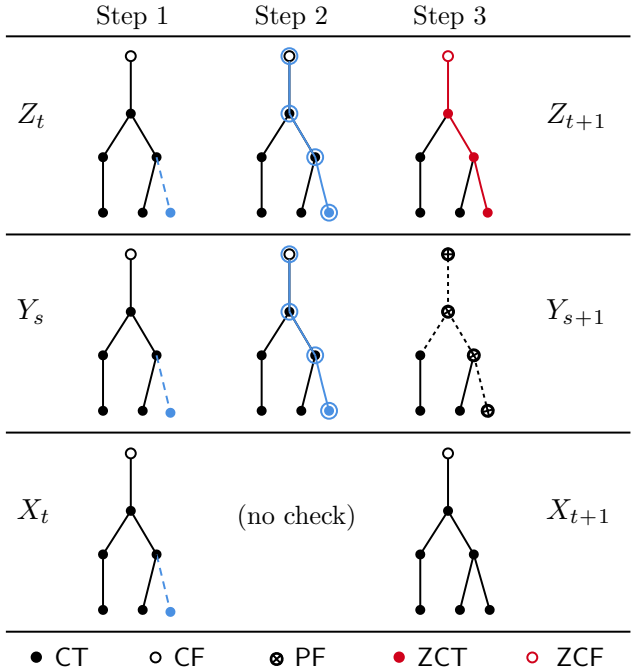
\begin{SCfigure}[.8][hbtp]
    \centering
    \input{Visuals/monotonicity-p}
    \caption{Evolution of $Z_t$ and the coupled processes $X_t$ and $Y_s$ when performing a check in $Y_s \sim p_2$-CKP but not in $X_t \sim p_1$-CKP, i.e., $U \in (p_1, p_2]$. 
    Recall that 
    the $Y$ process does not update at each time step of the $Z$ process; let 
    $s = \textsc{Y-time}(t)$ be the time step of the $Y$ process corresponding to time $t$ in the $Z$ process. 
    If at a future time, a new node picks one of the red nodes in $Z_{t+1}$ as its parent, it will then be labeled \ZNCT.
    }
    \label{fig:evolution-monotonicity}
\end{SCfigure}

\subsection{Monotonicity with respect to $p$} \label{subsection:coupling-p-values}
We describe the coupling of the $p_1$-simple CKP and $p_2$-simple CKP, where $p_2 \geq p_1$ and both CKPs have checking depth $k$ and attachment function $\mathbf{a}$. Let the $p_1$-CKP be given by the series of states $\{X_i\}$, and the $p_2$-CKP be given by the series of states $\{Y_i\}$. We define the coupling $(\{Z_i\},\{X_i\},\{Y_i\}) $ as follows. 
Each $Z_t$ is given by a finite rooted tree, labels to each of the vertices, and a mapping $\textsc{Y-time}(t)$ that indicates the time step in the $Y$ process corresponding to time step $t$ in the $Z$ process.

\paragraph{Initial state} $Z_0$ is given by a single \CF\
node. Let $X_0 = Y_0 = Z_0$, and $\textsc{Y-time}(0) = 0$.

\paragraph{State evolution} Given $\{Z_0, Z_1, \dots, Z_t\}$, we define $Z_{t+1}$ through the following series of steps.
\begin{itemize}
    \item \textbf{Choose parent node:} If every node in $\mathcal{T}_t$ is $\PF$, the process stops. Set $Z_{t+1} = Z_t$. Otherwise, choose a node $u \in \mathcal{T}_t$ according to the attachment function $\mathbf{a}$ among nodes with 
    label in $T = \{\CT, \CF, \ZCT, \ZCF, \ZNCT\}$. In this setting, the function $\mathbf{a}$ takes in as input $\deg_{T}(u)$, which is the number of children of $u$ with label in $T$. 
    \item \textbf{Add new node:}
    Let $\mathcal{T}_{t+1}$ be $\mathcal{T}_t$ with the new leaf node $v$ connected to parent $u$.
    
    \item \textbf{Specify the} \textsc{Update} \textbf{value:} If the parent $u$ has label in $\{\CT,\CF\}$, set \textsc{Update} $\gets 1$. This means that the $Y$ process updates. Otherwise, if $u$ has label in 
    $\{\ZCT, \ZCF, \ZNCT\}$, 
    set \textsc{Update} $\gets 0$. 
    \item \textbf{Label the new node:} If the parent $u$ has label in $\{\CT,\CF\}$, label the $v$ as $\CT$. 
    Otherwise, if $u$ has label in $\{\ZCT, \ZCF, \ZNCT\}$, label $v$ as $\ZNCT$. 
    \item \textbf{Error detection phase:} Choose $U \sim \mathrm{Unif}([0, 1])$. We denote $v_0 = v$, and let $v_{i+1}$ be the parent of $v_i$ for each $i = 0, \dots, k$.
    \begin{enumerate}
        \item If $U \in [0, p_1]$: a check is performed in both CKPs. 
        
        If $v$ has label $\CT$,
        then perform the following check of a path with up to $k$ edges. Define and set a variable \textsc{Zombie} $\gets 0$. For $i = 0, \dots, k$,
        \begin{itemize}
            \item If $v_i$ is $\PF$, $\CF$ 
            or $\ZCF$ ($\CF$ in the $X$ process and $\PF$ in the $Y$ process), relabel all of the nodes $\{v_0, ..., v_i\}$ as $\PF$ and end the error detection phase.
            \item If $\textsc{Zombie}=0$, and if $v_i$ is $\ZCT$ ($\CT$ in the $X$ process and $\PF$ in the $Y$ process), relabel the nodes $\{v_0, ..., v_i\}$ by replacing $\CT$ labels with $\ZCT$ labels. (Note that by construction, all nodes encountered in this case up until the $\ZCT$ node must have been $\CT$.) Set $\textsc{Zombie} \gets 1$ and continue the for loop.
        \end{itemize}
        
        Otherwise, if $v$ has label $\ZNCT$, 
        for $i = 0, \dots, k$,
        \begin{itemize}
            \item 
            If $v_i$ is $\CF$, $\ZCF$, or $\PF$, relabel $\{v_0, ..., v_i\}$ as $\PF$ and end the error detection phase. 
            \item If $v_i$ is 
            $\ZNPF$, then all nodes on the path checked so far were all $\ZNCT$. Relabel all of these nodes as $\ZNPF$ and end the error detection phase.
        \end{itemize}

        \item Otherwise, if $U \in (p_1, p_2]$: a check is performed in the $p_2$-CKP but not in the $p_1$-CKP.

        If $v$ has label $\CT$, 
        then for $i= 0, \dots, k$:
        \begin{itemize}
            \item If $v_i$ is in $\{\CF,\ZCT,
            \PF\}$, relabel the nodes $\{v_0, ..., v_i\}$ by replacing $\CF$ labels with $\ZCF$ labels and replacing $\CT$ labels with $\ZCT$ labels. End the error detection phase.
        \end{itemize}
        
        Otherwise, if $v$ has label $\ZNCT$, 
        end the error detection phase.

        \item Finally, if $U \in (p_2, 1]$: no check in either CKP, and we do not relabel anything.
    \end{enumerate}
    \end{itemize}
\paragraph{Updating $\{X_i\}$, $\{Y_i\}$ from $Z_t$} We also update the mapping $\textsc{Y-time}$.
    \begin{enumerate}
        \item $X_{t+1}$ is obtained from $Z_{t+1}$ by relabeling $\ZCF
        \Rightarrow \CF$, $\{\ZCT, \ZNCT\}\Rightarrow \CT$ and $\ZNPF \Rightarrow \PF$.
        \item If $\textsc{Update} = 0$, we do not add a state to the list of $\{Y_i\}$. Let $\textsc{Y-time}(t+1) = \textsc{Y-time}(t)$.
        
        If $\textsc{Update} = 1$, 
        let $\textsc{Y-time}(t+1) = \textsc{Y-time}(t)+1$. 
        Add to the $Y$-sequence a new state $Y_{\textsc{Y-time}(t+1)}$, which is defined from $Z_{t+1}$ by relabeling $\{\ZCT, \ZCF\}\Rightarrow \PF$ and removing all $\ZNCT$
        and $\ZNPF$ nodes.
        Note that if a node is labeled \ZNCT\ or \ZNPF, 
        then so will all of the vertices in the subtree rooted at this node.
\end{enumerate}

We now prove that this coupling has the correct marginals, i.e., the two processes $\{X_i\}$ and $\{Y_j\}$ are indeed CKPs with the correct parameters.

\begin{lem} \label{lemma:monotonicity-coupling-p-X-correctness}
    Let $\{X_i\}$ be generated according to the process above. Then $\{X_i\} \sim (\textsc{Attach} = (\mathbf{a}, M), \textsc{Check} = (p_1, k))$-simple CKP. 
\end{lem}

\begin{proof}
The initial state is consistent with the simple CKP process with $\textsc{Attach} = (\mathbf{a}, M = 1)$ and $\textsc{Check} = (p_1, k)$. We now verify this for the evolution, where we note that unlike $\{Y_i\}$, for every time step that the $Z$-process is updated, the $X$-process is also updated. 
   \begin{itemize}
       \item \textbf{Choose parent node:} We note that $Z_t^T$, the set of $T = \{\CT, \CF, \ZCT, \ZCF, \ZNCT\}$ nodes in $Z_t$, is the same as $X_t^\PT$, the set of \PT\ nodes in $X_t$. This also implies that the $T$-degree of a node $u$ in $Z_t$ is equal to the \PT-degree of this node in $X_t$: $\deg_T^{Z_t}(u) = \deg_\PT^{X_t}(u)$. Therefore, a parent $u$ is chosen among $Z_t^T = X_t^\PT$ with probability 
       $$\frac{\mathbf{a}(\deg_T^{Z_t}(u))}{\sum_{v \in Z_t^{T}} \mathbf{a}(\deg_T^{Z_t}(v))} = \frac{\mathbf{a}(\deg_\PT^{X_t}(u))}{\sum_{v \in X_t^{\PT}} \mathbf{a}(\deg_\PT^{X_t}(v))},$$ as required.
       \item \textbf{Add and label the new node:} Same process as the $(\textsc{Attach} = (\mathbf{a}, M), \textsc{Check} = (p_1, k))$-simple CKP. 
       \item \textbf{Error detection phase:} In the $\{Z_j\}$ process, 
       with probability $p_1$, a check is performed which checks a path of length up to $k$ and stops at the first node found that is labeled $\PF$, $\CF$, $\ZCF$, 
       or $\ZNPF$ (i.e., $\PF$ or $\CF$ in the $\{X_i\}$ process); once such a node is found the entire path is labeled $\PF$. Along the way, nodes may be relabeled from $\CT$ to $\ZCT$, but this does not change the labeling of nodes in the $\{X_i\}$ process. Therefore, with probability $p_1$, a check is performed in the $\{X_i\}$ process that is consistent with how a check is performed in a $(\textsc{Attach} = (\mathbf{a}, M), \textsc{Check} = (p_2, k))$-simple CKP, as required.
       
       With probability $p_2 - p_1$, the $\{Z_j\}$ process also performs a check. However, this check can only re-label $\CT$ nodes as $\ZCT$ nodes and $\CF$ nodes as $\ZCF$ nodes, which does not change the labeling of the coupled $\{X_i\}$ process.
   \end{itemize} 

We have verified that the initialization and state evolution with the checking procedure all align with the $(\textsc{Attach} = (\mathbf{a}, M), \textsc{Check} = (p_1, k))$-simple CKP process.
\end{proof}

We now show that $\{Y_j\} \sim (\textsc{Attach} = (\mathbf{a}, M), \textsc{Check} = (p_2, k))$-simple CKP. 
Recall that for any time step $t$ of the $\{Z_i\}$ process, $\textsc{Y-time}(t)$ is equal to the number of steps $i \leq t$ for which $\textsc{Update} = 1$. 
That is, the states $Z_t$ and $X_t$ correspond to the state $Y_s$ in the coupling, for $s := \textsc{Y-time}(t)$.
To justify the correctness of the checking procedure for the $\{Y_j\}$ process, we first state the following lemma that follows simply by construction.

\begin{lem}\label{zombie-ancestors-are-zombies}
    If a node $v$ is labeled $\ZCT$ or $\ZCF$ in $Z_t$, all ancestor nodes of $v$ in the path to the nearest $\PF$ node must be labeled $\ZCT$ or $\ZCF$. 
    Furthermore, if a node $v$ is labeled $\ZNCT$,
    there exists an $\ell \in \N$ such that the closest $\ell$ ancestor nodes from $v$ are labeled $\ZNCT$
    and all subsequent ancestor nodes of $v$ on the path to the nearest $\PF$ node are labeled $\ZCT$ or $\ZCF$.
\end{lem}

\begin{proof}
    If a node $v$ is $\ZCT$ or $\ZCT$ in $Z_t$, then it was originally created $\CT$ or $\CF$ in $\{Z_i\}$. By construction, $v$ became a zombie when a check through it found a $\ZCT$ node (if $U \in [0, p_1]$) or a \CF, \ZCF, or \PF\ node (if $U \in (p_1, p_2]$). When this occurs, all nodes along the path to the node found are marked $\ZCT$ or $\ZCF$, so all ancestor nodes of $v$ on the path to the nearest $\PF$ node must be labeled $\ZCT$ or $\ZCF$.
    A similar argument can be given if $v$ is $\ZNCT$:
    it then connected to a $\ZNCT$ or $\ZCF$ 
    parent when created.
\end{proof}

\begin{lem} \label{lemma:monotonicity-coupling-p-Y-correctness}
    Let $\{Y_i\}$ be generated according to the process above. Then $\{Y_i\} \sim (\textsc{Attach} = (\mathbf{a}, M), \textsc{Check} = (p_2, k))$-simple CKP.
\end{lem}

\begin{proof}
The initial state is consistent with the simple CKP with $\textsc{Attach} = (\mathbf{a}, M)$ and $\textsc{Check} = (p_2, k)$. We verify that this is also the case for state evolution.

For a time step $i \in [0, t]$ in which $Z_i$ is updated, the $Y$-process is only updated when $\textsc{Update} = 1$. Indeed, $\textsc{Update} = 0$ when the parent node is in $\{\ZCT, \ZCF, \ZNCT
\}$; the new node added is therefore 
$\ZNCT$, and all these nodes are removed in the coupled $Y$-process.
Therefore, we only need to focus on the evolution of the $Y$-process between two steps in the $Z$-process when $\textsc{Update} = 1$; all steps where $\textsc{Update} = 0$ do not impact the resulting $Y$-process. 
Call these two steps $Y_{s}$ and $Y_{s + 1}$, where $s = \textsc{Y-time}(t)$. Recall that to go from the $Z$ to $Y$ process, we relabel $\{\ZCT, \ZCF\} \Rightarrow \PF$ and remove all $\ZNCT$
and $\ZNPF$ nodes.
\begin{itemize}
    \item \textbf{Choose parent node:} When $\textsc{Update} = 1$, the new node connects to a \CT\ or \CF\ parent node. We want to verify that this connection is distributed according to the attachment function $\mathbf{a}$ \textit{in the graph corresponding to $Y_s$}. Note that since all children of a \CT\ or \CF\ node in $\{Z_i\}$ must also be \CT\ or \CF,
    the degree of a $\{\CT, \CF\}$ node in $Z_t$ equals the degree of the corresponding node in $Y_s$. 
    
    Therefore, conditioning on 
    the parent node being $\CF$ or $\CT$ (when 
    $\textsc{Update} = 1$), we find that $u$ is chosen as the parent node of the new node in $Y_s$ with probability
    $$\frac{\mathbf{a}(\deg_{T}^{Z_t}(u))}{\sum_{v \in Z_t^{\CF} \cup Z_t^{\CT}} \mathbf{a}(\deg_{T}^{Z_t}(v))} = \frac{\mathbf{a}(\deg_\PT^{Y_s}(u))}{\sum_{v \in Y_s^{\PT}} \deg_\PT^{Y_s}(v)},$$
    as required.
    \item \textbf{Add and label the new node:} Same as in the $(\textsc{Attach} = (\mathbf{a}, M), \textsc{Check} = (p_2, k))$-simple CKP.
       \item \textbf{Error detection phase:} Recall that
        if a node is labeled $\PF$ in $Y_s$, it must have a label among $\{\PF, \ZCT, \ZCF\}$ in $Z_t$.

        With probability $p_1$, 
        if the new node $v$ is labeled
        \CT, 
        a check is performed which checks a path of length up to $k$ and stops at the first node that is labeled $\PF$, $\CF$, or $\ZCF$ in $Z_t$ (i.e., $\PF$ or $\CF$ in $Y_s$); if such a node is found, the entire path is labeled $\PF$. Along the way, if a \ZCT\ ancestor is found, nodes on the path may be relabeled from $\CT$ to $\ZCT$; in $Y_s$, this corresponds to finding a \PF\ ancestor and relabeling the corresponding nodes from $\CT$ to $\PF$. 
        Even though the check continues in $Z_t$ after the \ZCT\ ancestor was found, Lemma~\ref{zombie-ancestors-are-zombies} guarantees that all ancestor nodes of this $\ZCT$ ancestor are themselves $\ZCT$ or $\ZCF$, so they are already marked $\PF$ in $Y_s$.
        We can disregard the case where the new node in $Z_t$ has label \ZNCT, as $\textsc{Update} = 0$ and these nodes are removed in $Y_s$. 
        Therefore this check has the desired behaviour.

        With probability $p_2-p_1$, if $v$ is \CT, we also have a check which stops at the first $\CF$, $\PF$ or $\ZCT$
        node 
        (i.e., $\CF$ or $\PF$ in $Y_j$), and makes nodes along the checked path either $\ZCF$ or $\ZCT$, i.e., $\PF$ in $Y_j$. This also is consistent with the $(\textsc{Attach} = (\mathbf{a}, M), \textsc{Check} = (p_2, k))$-simple CKP. 
\end{itemize} 

We have verified that the initialization and state evolution with the checking procedure all align with the $(\textsc{Attach} = (\mathbf{a}, M), \textsc{Check} = (p_2, k))$-simple CKP. 
\end{proof}

\paragraph{Proof of monotonicity with respect to $p$}
Using this coupling, we can now prove the following required lemma for error elimination with $k$ fixed.

\begin{lem} \label{lemma:monotonicity-proof-for-p}
    For all $p_2 \geq p_1$ and all $k$, if the error effect is completely eliminated in the simple CKP with features $(\textsc{Attach} = (\mathbf{a}, M=1),  \textsc{Check} = (p_1, k))$, then the error effect is completely eliminated in the simple CKP with features $(\textsc{Attach} = (\mathbf{a}, M=1),  \textsc{Check} = (p_2, k))$.
\end{lem}

\begin{proof}
    Assume that for every $(\textsc{Attach} = (\mathbf{a}, M), \textsc{Check} = (p_1, k))$-simple CKP, the error effect is eliminated completely: there exists some time $t$ such that the sub-tree of the first node, i.e., the entire tree $\T_{t}$, is entirely marked \PF.
    Given a $(\textsc{Attach} = (\mathbf{a}, M), \textsc{Check} = (p_2, k))$-simple CKP, we consider its generation according to the $(\{Z_i\}, \{X_i\}, \{Y_i\})$ coupling defined in 
    Section~\ref{section:coupling-monotonicity}.
    Let $t$ be the time at which the $X_t = (\T_t, L_t)$ state is entirely marked \PF. Then, let $s = \textsc{Y-time}(t)$ be the corresponding time step of the $Y$-process. Since every \PF\ node in $X_t$ is either \PF\ or \ZNPF\ in $Z_s$, it consequently is either \PF\ or does not exist in $Y_s$. Therefore, by construction $Y_s$ is also entirely marked \PF, and the error is eliminated. 
\end{proof}

\subsection{Monotonicity with respect to $k$}

We now present the coupling of the $(\textsc{Attach} = (\mathbf{a}, M), \textsc{Check} = (p, k_1))$-simple CKP and $(\textsc{Attach} = (\mathbf{a}, M), \textsc{Check} = (p, k_2))$-simple CKP processes for $k_1 \leq k_2$. This coupling follows the same overarching structure as the coupling used to show monotonicity with respect to $p$, and so we highlight the components of the coupling that differ in this setting. The key difference arises in the \textit{checking procedure} utilized in the constructed $Z$-process, which must handle the checking for different values of $k$ in the coupled processes differently than how the checking for different $p$ values was handled.

Consider the coupling with $p$ and $\mathbf{a}$ fixed. Then, $k$ is our only variable and so we can simplify the notation and write $k_1$-CKP and $k_2$-CKP. Let the $k_1$-CKP be given by the series of states $\{X_i\}$, and let the $k_2$-CKP be given by the series of states $\{Y_i\}$. We once again define the coupling  $(\{Z_i\},\{X_i\},\{Y_i\})$ and the mapping $\textsc{Y-time}(t)$, with key difference as follows in the state evolution.

\paragraph{State evolution} Given $\{Z_0, Z_1, \dots, Z_t\}$, we define $Z_{t+1}$ through the following series of steps. The first four steps are the same as those in section~\ref{subsection:coupling-p-values}: choose a parent node, add the new node, specify the \textsc{Update} value, and label the new node. The error detection phase is different in this setting. Let $v_0 = v$ be the new node added, and for every $v_i$ in the path checked, let $v_{i + 1}$ denote its parent. 
\begin{itemize}
    \item \textbf{Error detection phase:} If the new node is labeled 
    $\CT$, perform the following procedure. Set indicator variable \textsc{Zombie} $\gets 0$. For $i = 0, \dots, k_1$:
    \begin{itemize}
        \item If $v_i$ is $\CF$, $\ZCF$, or $\PF$, label all the nodes $\{v_0, \dots, v_i\}$ as $\PF$ and stop the check. Else, continue the checking procedure.
        \item If $\textsc{Zombie} = 0$, and if $v_i$ is 
        $\ZCT$, relabel all the nodes $\{v_0, \dots, v_i\}$ from $\CT$ and $\CF$ to $\ZCT$ and $\ZCF$ correspondingly. Set $\textsc{Zombie} \gets 1$. Continue the checking procedure.
    \end{itemize}
    Next, for $i$ from $k_1 + 1$ to $k_2$:
    \begin{itemize}
        \item If $v_i$ is $\ZCF$, $\ZCT$, $\CT$, or $\PF$, relabel all the nodes $\{v_0, \dots, v_i\}$ from $\CT$ and $\CF$ to $\ZCT$ and $\ZCF$ correspondingly, and stop the check. Else, continue the checking procedure.
    \end{itemize}
    Otherwise, if the new node is labeled
    $\ZNCT$, then for $i = 0, \dots, k_1$: 
    \begin{itemize}
        \item   If $v_i$ is $\CF$, $\ZCF$, or $\PF$, label all the nodes $\{v_0, \dots, v_i\}$ as $\PF$ and stop the check. 
        \item If $v_i$ is 
        $\ZNPF$, then all nodes on the path checked so far were $\ZNCT$; relabel all of these nodes as $\ZNPF$. Else, continue the checking procedure.

    \end{itemize}
\end{itemize}
Finally, the processes of creating the state $Z_{t+1}$ and $\textsc{Y-time}$ and of forming $\{X_i\}$ and $\{Y_j\}$ from $\{Z_i\}$ are the same as for \ref{subsection:coupling-p-values}.

We now state the lemmas regarding the correctness of the $\{X_i\}$ and $\{Y_j\}$ processes.

\begin{lem}
    Let $\{X_i\}$ be generated according to the process above. Then $\{X_i\} \sim (\textsc{Attach} = (\mathbf{a}, M), \textsc{Check} = (p, k_1))$-simple CKP.
\end{lem}

\begin{proof}
The only component of the $\{X_i\}$ process that needs to be analyzed is the error detection phase of the state evolution. All other components of the correctness proof follow from the proof of Lemma \ref{lemma:monotonicity-coupling-p-X-correctness}, due to the similarities between the couplings for $p$ values and for $k$ values.
Recall that if a node is labeled $\CF$ in $X_t$, it must have had a label from the set $\{\CF, \ZCF, \ZNCF\}$ in $Z_t$. If a node is labeled $\PF$ in $X_t$, it must have been labeled $\PF$ or $\ZNPF$ in $Z_t$.

With probability $p$, a check is performed which checks a path of length up to $k_1$ and stops at the first node found that is labeled $\PF$, $\CF$, $\ZCF$
or $\ZNPF$ (so, $\PF$ or $\CF$ in the $\{X_i\}$ process); once such a node is found the entire path is labeled $\PF$. Along the way, nodes may be relabeled from $\CT$ to $\ZCT$, but this does not change the labeling of nodes in the $\{X_i\}$ process. The check also examines the following $k_2 - k_1$ nodes, but this process only relabels nodes from $\CT$ to $\ZCT$ and from $\CF$ to $\ZCF$, which does not change the labeling of nodes in the $\{X_i\}$ process.

Therefore, with probability $p$, a check is performed in the $\{X_i\}$ process that is consistent with how a check is performed in a simple CKP with $\textsc{Check} = (p, k_1))$ (and $\textsc{Attach} = (\mathbf{a}, M)$). We can conclude that $\{X_i\} \sim (\textsc{Attach} = (\mathbf{a}, M), \textsc{Check} = (p, k_1))$-simple CKP.
\end{proof}

\begin{lem}
    Let $\{Y_j\}$ be generated according to the process above. Then $\{Y_j\} \sim (\textsc{Attach} = (\mathbf{a}, M), \textsc{Check} = (p, k_2))$-simple CKP.
\end{lem}

\begin{proof}
   We again only need to analyze the error detection phase of the state evolution of the $\{Y_j\}$ process. 
   We need to show that the error detection phase performs a check with probability $p$, and that such a check examines the path of ancestor nodes of up to length $k_2$, stopping at the first node labeled $\CF$ or $\PF$ in $Y_j$. Suppose that $j = \textsc{Y-time}(t)$.
   Recall that if a node is labeled $\CF$ in $Y_j$, it must have had a label of $\CF$ in $Z_t$. If a node is labeled $\PF$ in $Y_j$, it must have had a label in the set $\{ \PF, \ZCT, \ZCF\}$ in $Z_t$.

   With probability $p$, if the new node is labeled $\CF$ or $\PF$, a check is performed which checks a path of length up to $k_1$ and stops at the first node that is labeled $\PF$, $\CF$, or $\ZCF$ (so, $\PF$ or $\CF$ in the $\{Y_j\}$ process); once such a node is found the entire path is labeled $\PF$. Along the way, nodes may be relabeled from $\CT$ to $\ZCT$; in $Y_j$, this relabels the corresponding nodes from $\CT$ to $\PF$. This occurs when a $\ZCT$ ancestor of a node has been found (i.e. a $\PF$ ancestor node in $Y_j$), aligning with the desired checking mechanism for $Y_j$. Even though the check continues in $Z_t$ after the $\ZCT$ ancestor was found, Lemma \ref{zombie-ancestors-are-zombies} guarantees that all ancestor nodes of this $\ZCT$ ancestor are themselves $\ZCT$ or $\ZCF$, so they are already marked $\PF$ in $Y_j$. We can disregard the case where the new node and its parent node have labels in the set $\{\ZCT, \ZCF, \ZNCT, \ZNCF\}$, as these nodes are removed when moving from $Z_t$ to $Y_j$.

   For the remaining $k_2 - k_1$ steps of the check, nodes along a path are checked, stopping at the first node that is labeled $\CF$, $\ZCT$, $\ZCF$, or $\PF$ (i.e. $\CF$ or $\PF$ in $Y_j$). It makes nodes along the checked path either $\ZCF$ or $\ZCT$, i.e. $\PF$ in $Y_j$.

   Altogether, the error detection phase is consistent with that of the simple CKP with $\textsc{Attach} = (\mathbf{a}, M)$ and $ \textsc{Check} = (p, k_1)$ defined by $Y_j$. We can conclude that $\{Y_j\} \sim (\mathbf{a}, M), \textsc{Check} = (p, k_2))$-simple CKP.
\end{proof}

\paragraph{Proof of monotonicity with respect to \textit{k}}
We prove the following theorem regarding the monotonicity of error elimination with respect to the checking parameter $k$. 

\begin{lem}\label{lemma:monotonicity-proof-for-k}
    For all $k_2 \geq k_1$ and all $p$, if the error effect is completely eliminated in the simple CKP with features $(\textsc{Attach} = (\mathbf{a}, M=1),  \textsc{Check} = (p, k_1))$, then the error effect is completely eliminated in the simple CKP with features $(\textsc{Attach} = (\mathbf{a}, M=1),  \textsc{Check} = (p, k_2))$.
\end{lem}

\begin{proof}
    The proof is exactly the same as the proof of Lemma~\ref{lemma:monotonicity-proof-for-p} since the argument entirely relied on comparing the labeling of nodes in the coupled $X$ and $Y$ processes, which carries over to this setting.
\end{proof}

\bibliographystyle{alpha}
\bibliography{bibliography}

\appendix

\section{Simulations}\label{section:simulations}

\paragraph{Set-up of the simulations} We now present our experiments on simulations of CKPs where new units attach to $m$ parents according to preferential attachment, for various values of $m$. Preferential attachment is a common modeling choice for processes of knowledge accumulation (such as citation networks and information on the web) \cite{amaral2000classes,jeong2003measuring}. In the rest of this section, when we refer to CKPs, we specifically mean CKPs with the preferential attachment function.

We provide simulations for the simple and general CKP models equipped with the \textsc{Exhaustive BFS} checking mechanism. In our experiments, we run each CKP simulation for 2000 timesteps, where at each timestep one new node is introduced into the CKP, connects to parent nodes, and possibly performs a check. We perform 20 trials per combination of $m$, $p$, and $k$ values, and the heatmaps reflect the percentage of trials that have survived up to the 2000th timestep. We run simulations for $p \in \{0, 0.1, \dots, 0.9, 1\}$ and $k \in \{1, 2, \dots, 10\}$. For all simulations involving general CKPs, $\eps$ was set to $0.25$.

For each simulation, we initialize the CKP with a chain of 25 nodes, where the $0$-th node is labeled $\CF$, the other nodes in the chain are labeled $\CT$, and there are $m$ edges between each node $i$ in the chain and its parent node $i - 1$. We note that the results proven in this paper hold under any valid CKP initialization, and we choose this initialization to make it initially more difficult to eliminate errors in the CKP. In the simple CKP model, the first node should be labeled $\CF$ (as in this initialization). While the general CKP model is defined so that the first node is labeled $\CT$, we still utilize the chain initialization so that all nodes are $\False$, which simplifies the experiments we have performed (because otherwise, error elimination is tougher to measure in a simulation). One may view the general CKP model with initial node $\CF$ as a $\False$ sub-component that we are trying to eliminate in a larger CKP.

We say that a CKP has \textit{survived} if there exist $\PF$ nodes in the DAG at the 2000th timestep. We note that this does not ensure survival for all times $t \in \mathbb{N}$, but instead should be viewed as an indication that the CKP has survived (and may continue to survive) for a long period of time. A CKP has \textit{eliminated error} if it has not survived according to this definition.

\begin{figure}[hbtp]
    \centering
    \textcolor{white}{
    \subfigure{\includegraphics[height=0.32\textwidth]{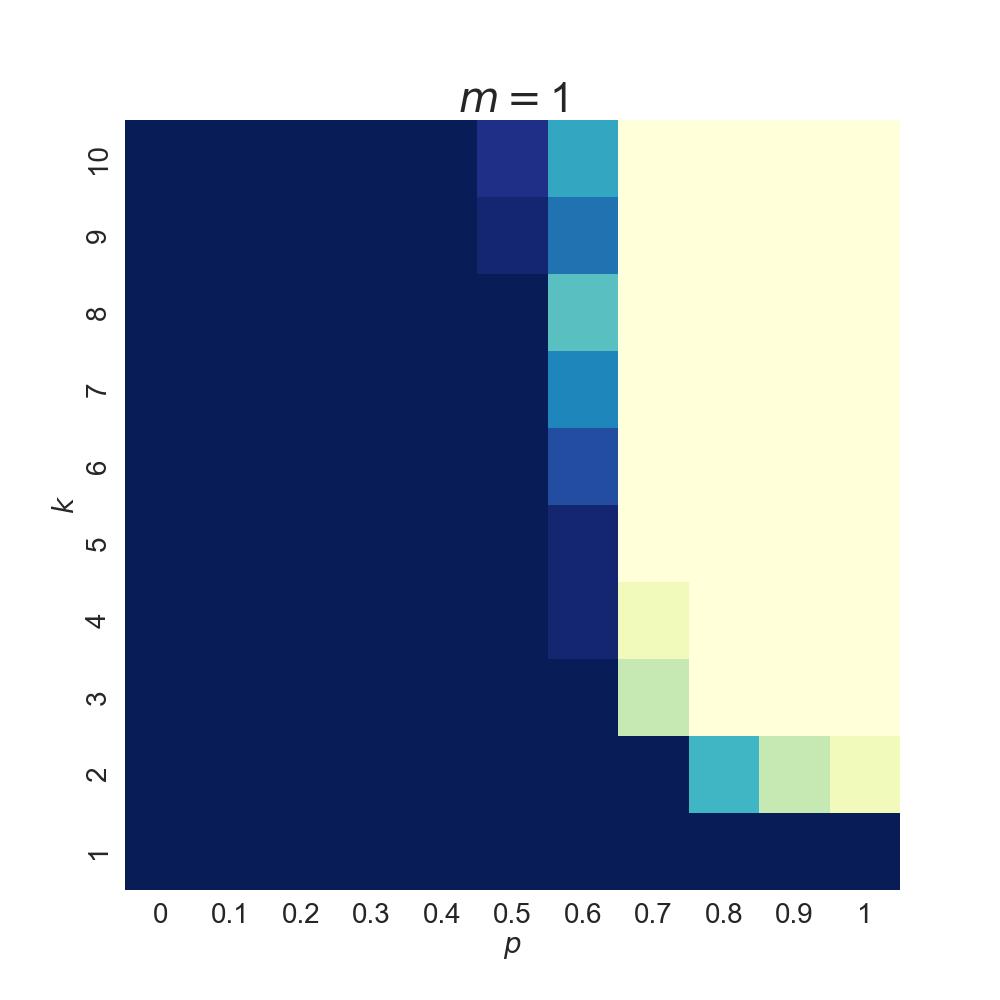}} 
    \subfigure{\includegraphics[height=0.32\textwidth]{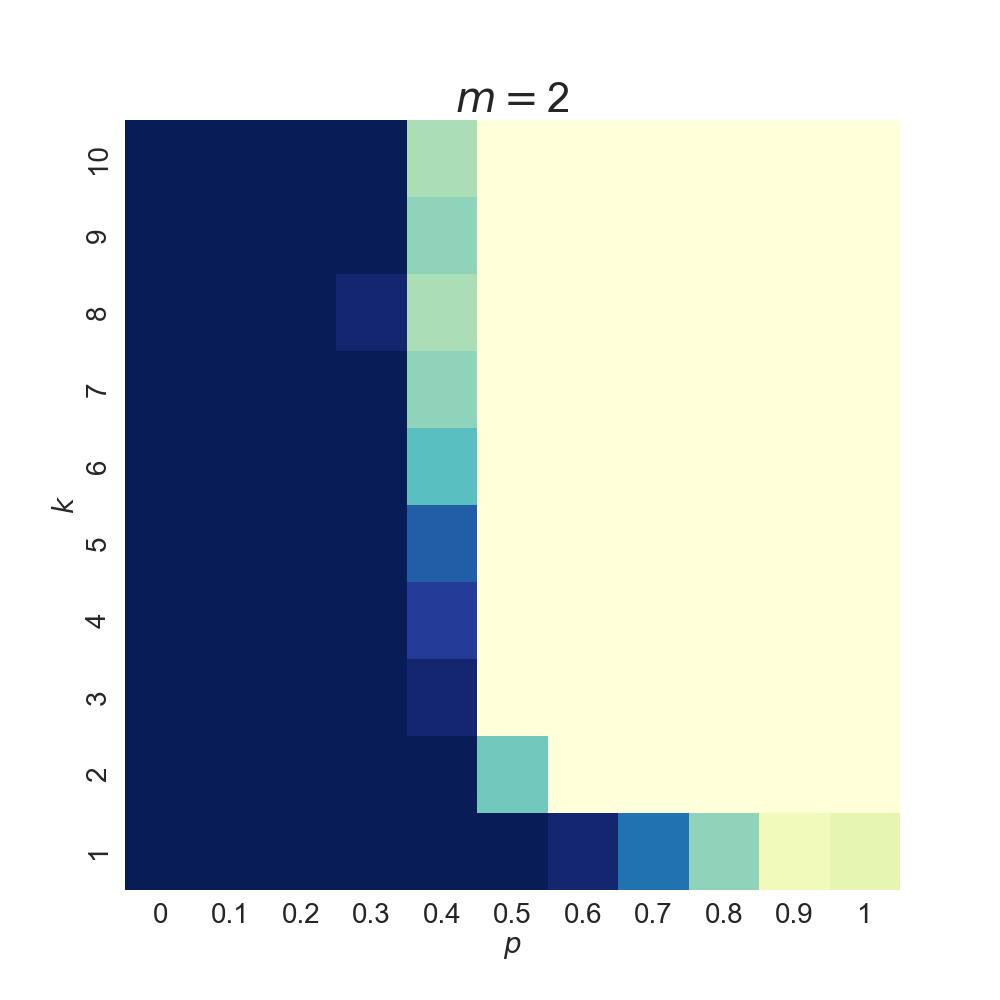}} 
    \subfigure{\includegraphics[height=0.32\textwidth]{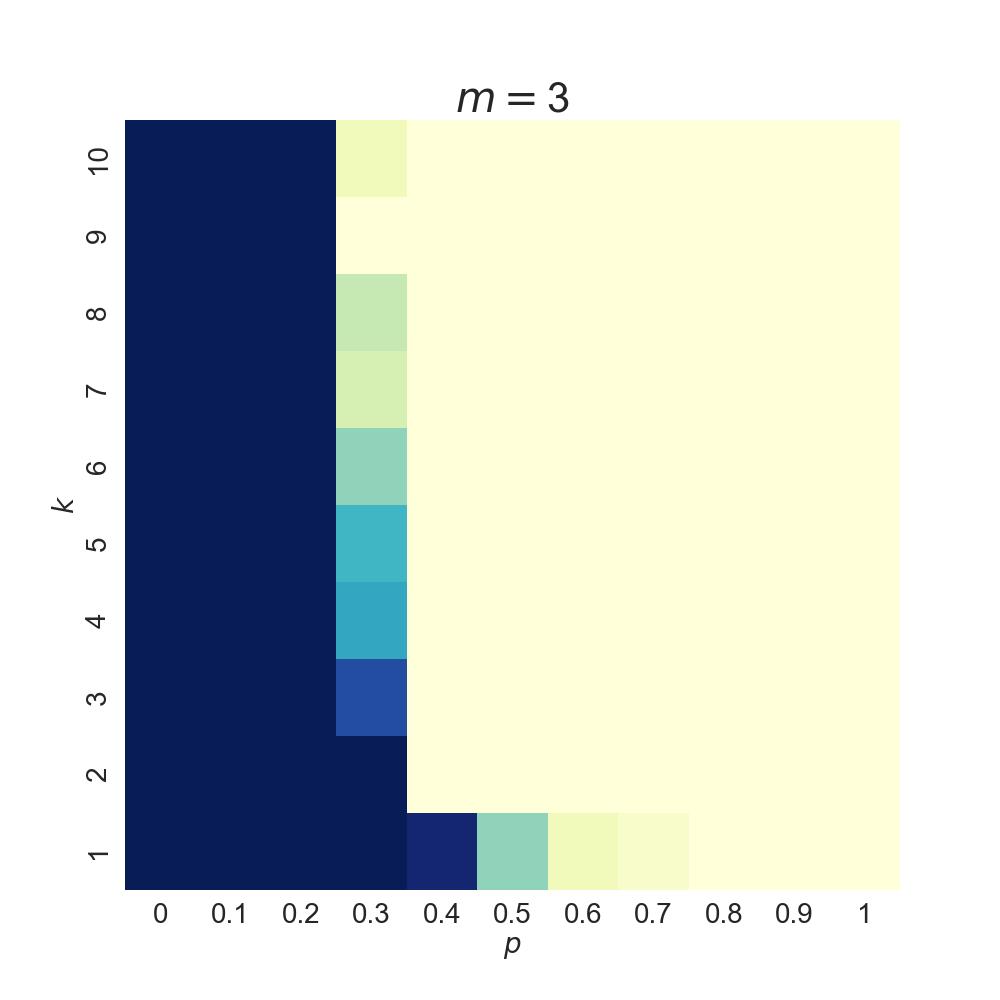}}\\ \vspace{-.6cm}
    \subfigure{\includegraphics[height=0.32\textwidth]{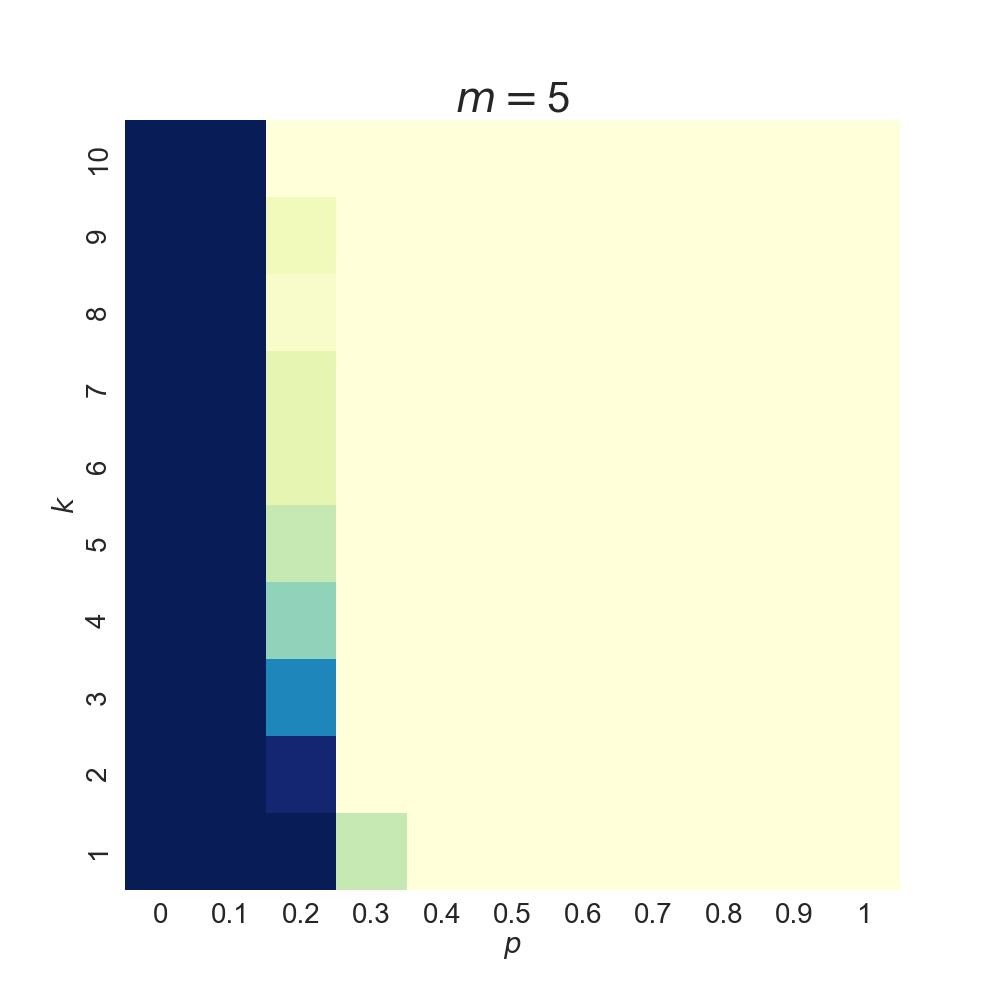}}
    \subfigure{\includegraphics[height=0.32\textwidth]{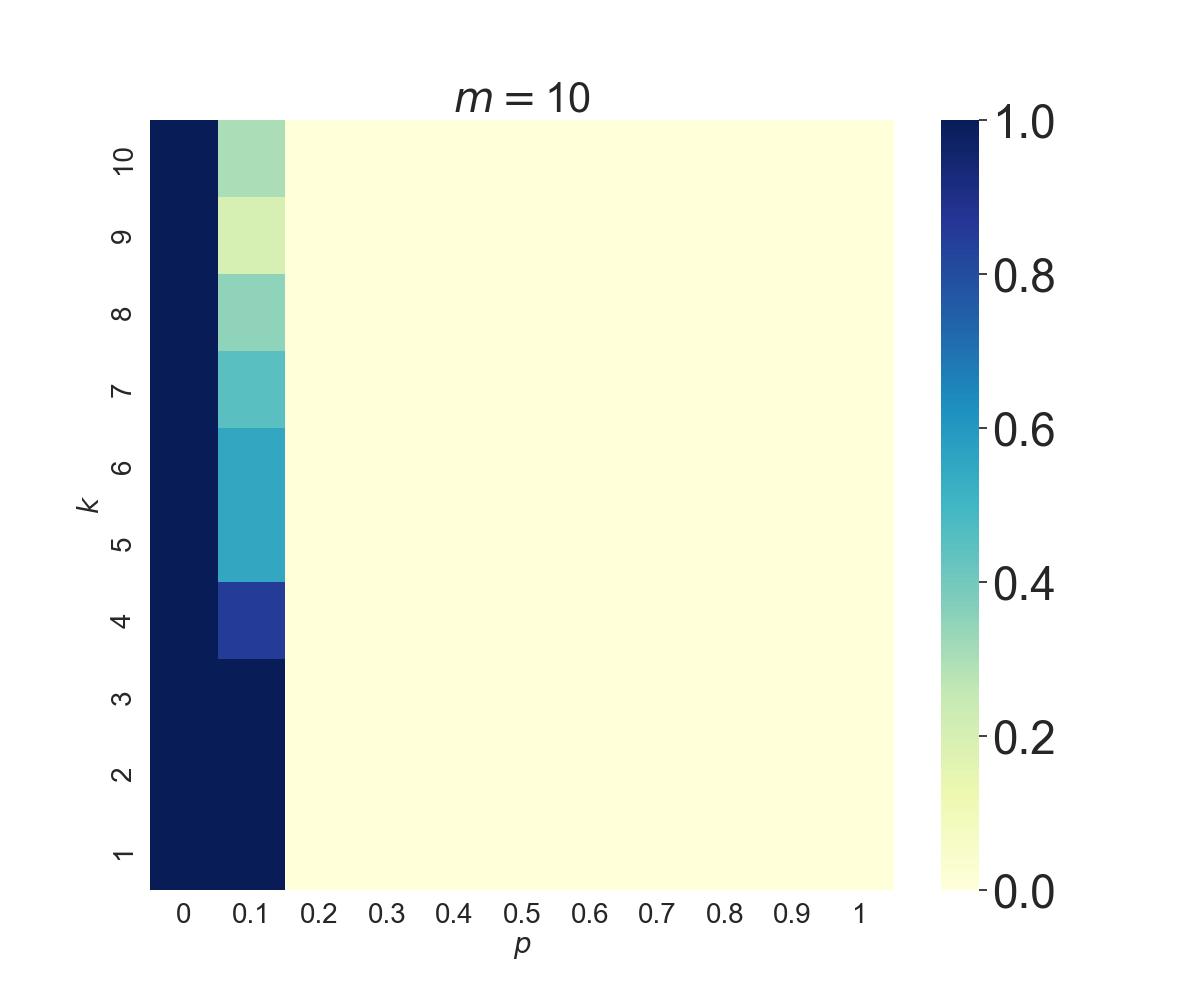}}}
    \vspace{-.6cm}
    \caption{\textbf{General CKP experiments.} For each choice of $m, p, k$, we generate 20 general CKP simulations with $M = m$, $\eps = 0.25$, $\textsc{Check} = (p, k)$, and checking mechanism \textsc{Exhaustive BFS}. The initial CKP state is a chain of $25$ nodes (one $\CF$ node followed by $24$ $\CT$ nodes) with $m$ edges between each node and its parent. The CKP evolves according to preferential attachment. 
    The percentage of simulations that survive until the 2000th timestep is displayed. \textbf{Error elimination phase transitions and monotonicity can be observed in the simulations.}}
    \label{fig:general-ckp-simulations}
\end{figure}

\paragraph{Discussion and conjectures} We now discuss potential phenomena suggested by the experiments.

\vspace{.5\bs} \noindent \textit{Phase transition:} The heatmaps indicate that there exists a critical $p_0$ value for which error survives in a $(p, k)$-CKP for $p < p_0$ and error is eliminated when $p > p_0$. The heatmaps also provide us with reasonable guesses for what the critical values may be for different $m$.

\vspace{.5\bs} \noindent \textit{Monotonicity of error elimination:} Our heatmaps suggest that for $m$-parent CKPs, there is monotonicity of error elimination with respect to checking parameters $p$ and $k$. The heatmaps also suggest that there is monotonicity of error elimination with respect to $m$, as the region of error elimination encompasses a broader range of combinations of $p$ and $k$ values as $m$ increases.

\vspace{.5\bs} \noindent \textit{Different checking mechanisms:} We note that, while not shown in this paper, the simulations run on different checking mechanisms yield heatmaps with different $p$ and $k$ thresholds for error elimination. However, for all checking mechanisms we have run simulations with, we observe the same overall phenomena (including the existence of a phase transition and various types of monotonicity).

\begin{figure}[hbtp]
    \centering 
    \textcolor{white}{
    \subfigure{\includegraphics[height=0.32\textwidth]{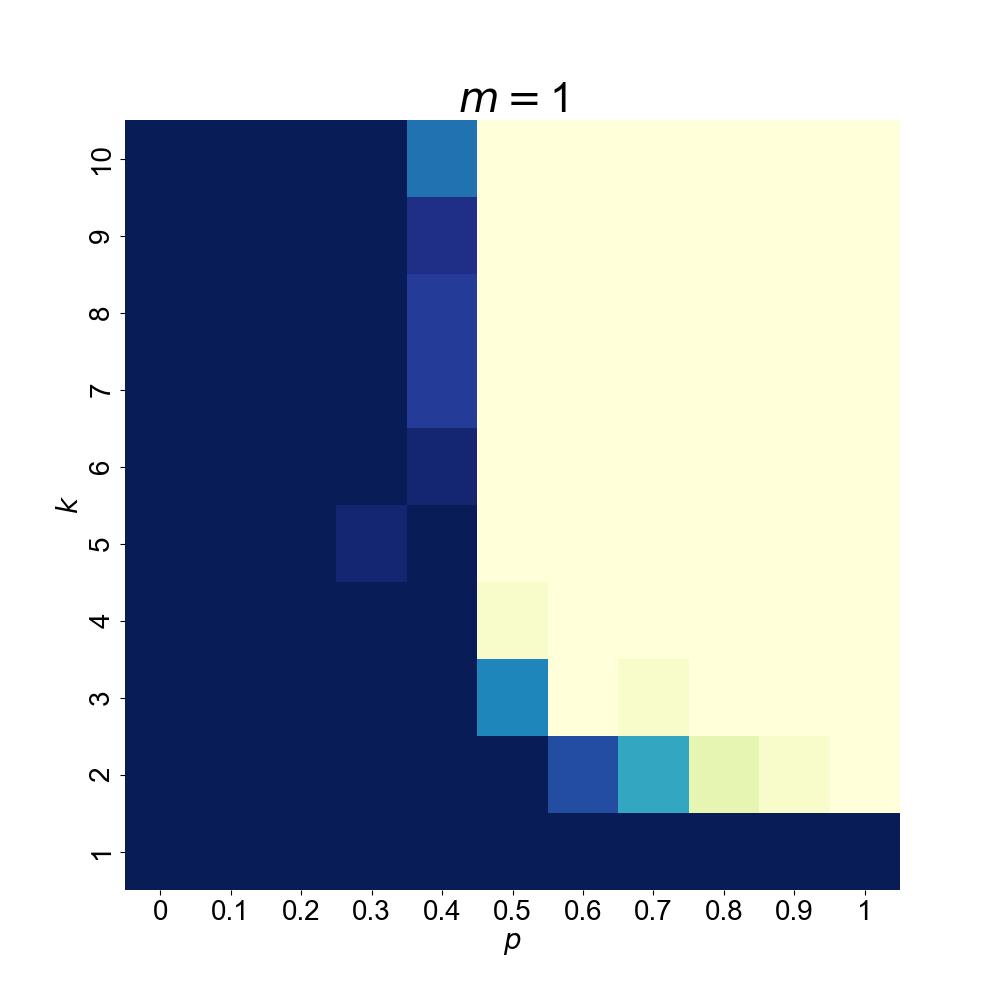}} 
    \subfigure{\includegraphics[height=0.32\textwidth]{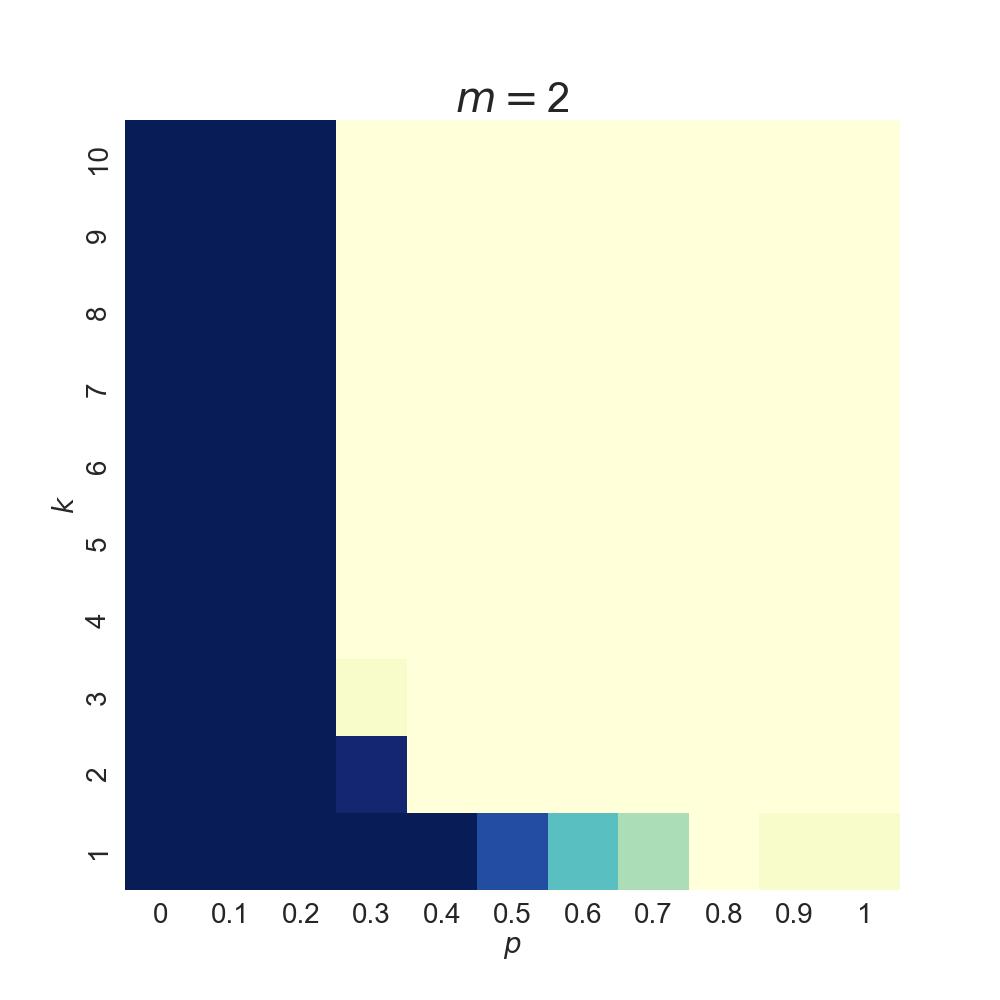}} 
    \subfigure{\includegraphics[height=0.32\textwidth]{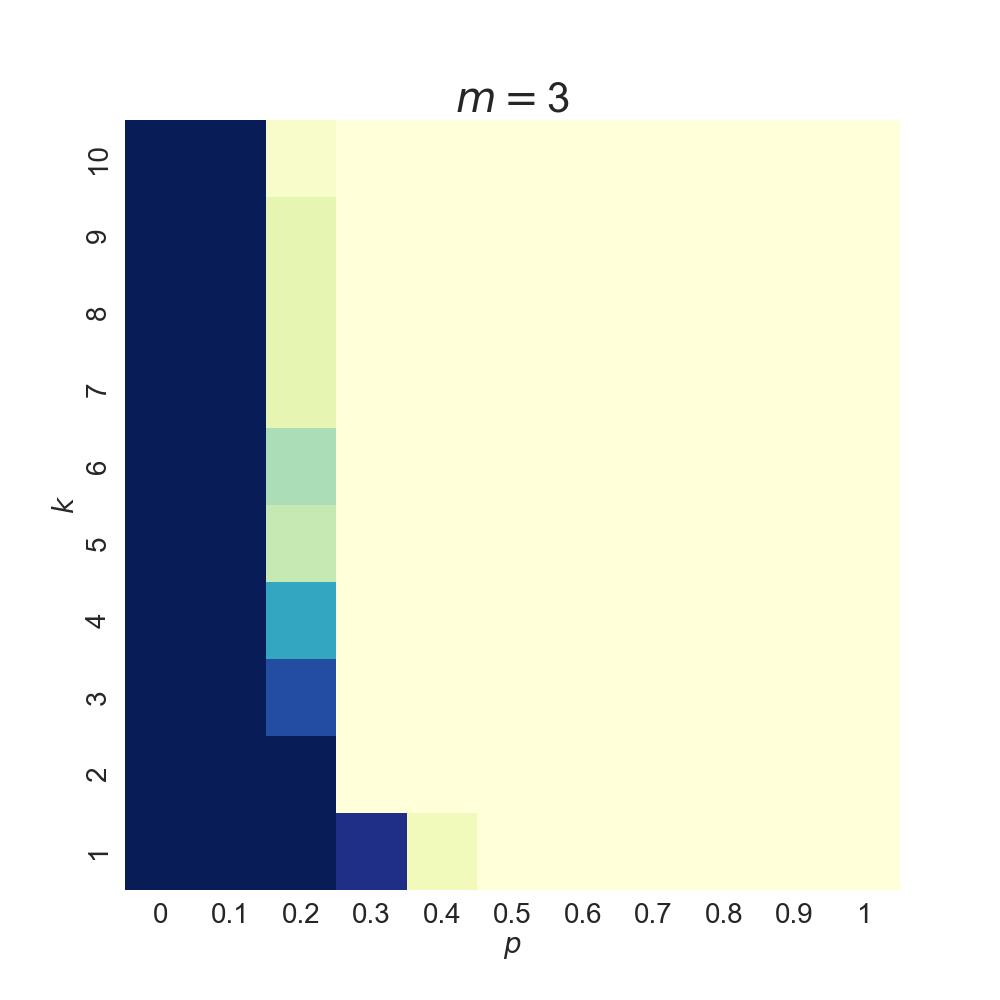}} \\ \vspace{-.6cm}
    \subfigure{\includegraphics[height=0.32\textwidth]{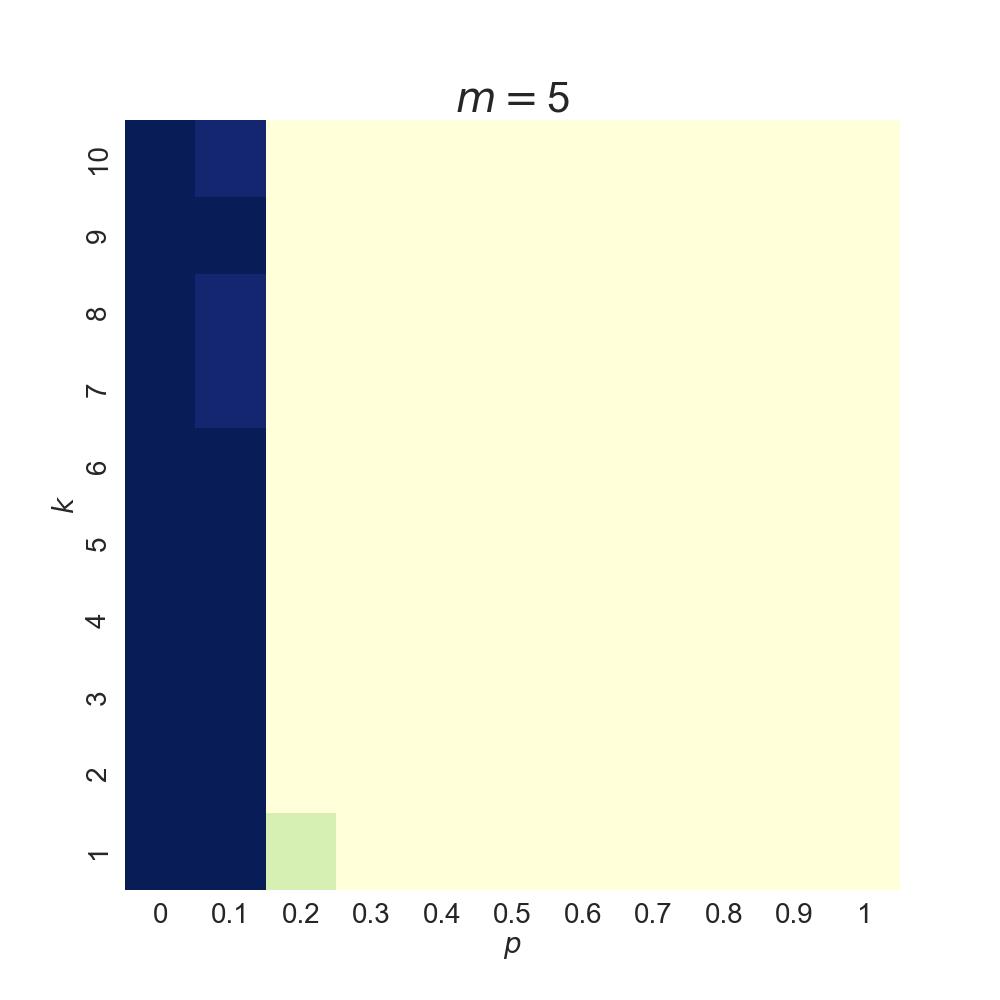}}
    \subfigure{\includegraphics[height=0.32\textwidth]{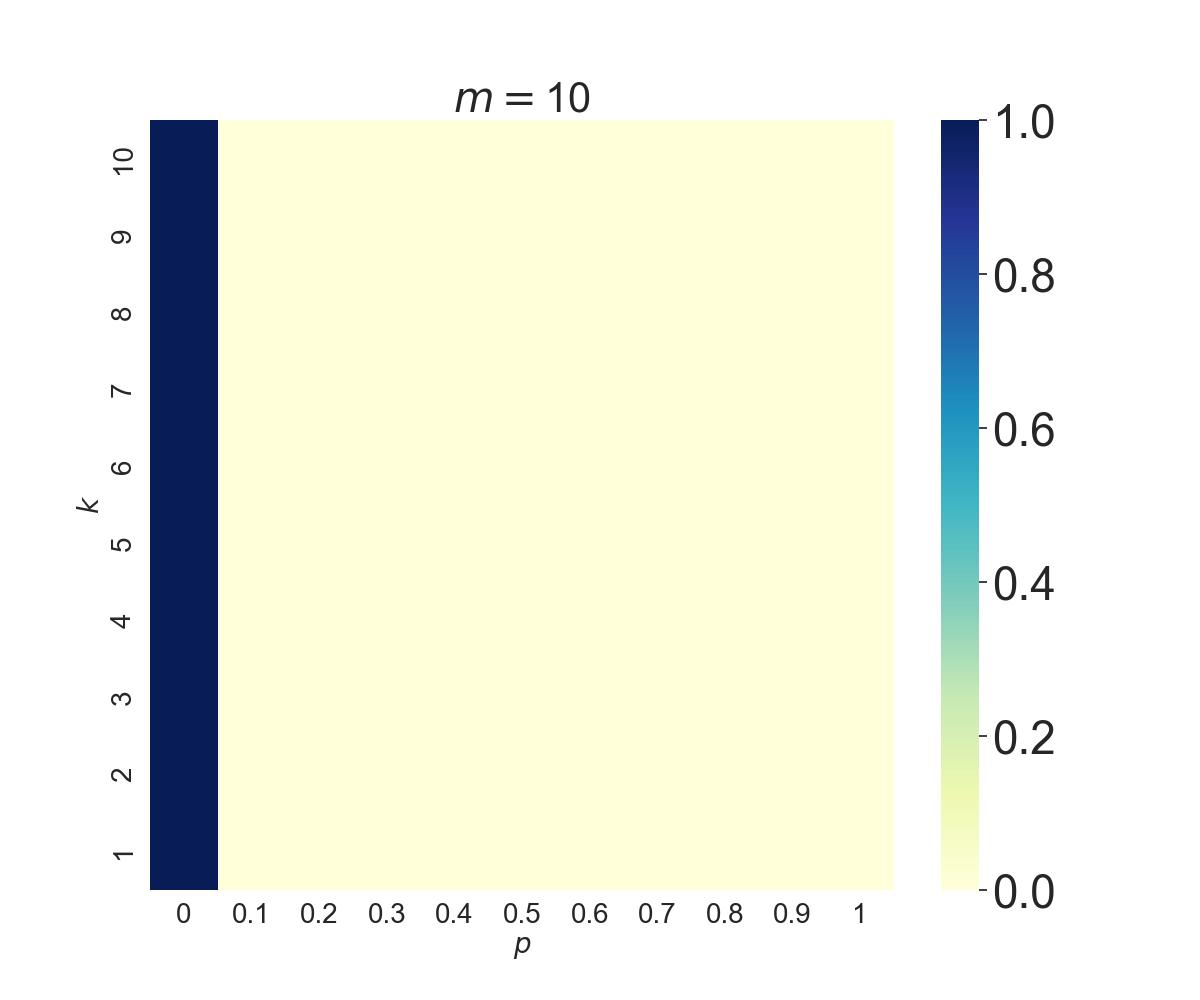}} 
    }
    \vspace{-.6cm}
    \caption{\textbf{Simple CKP experiments:} For each choice of $m, p, k$, we generate 20 simple CKP simulations with $M = m$, $\textsc{Check} = (p, k)$, and checking mechanism \textsc{Exhaustive BFS}. The initial CKP state is a chain of $25$ nodes (one $\CF$ node followed by $24$ $\CT$ nodes) with $m$ edges between each node and its parent. The CKP evolves according to preferential attachment. 
    The percentage of simulations that survive until the 2000th timestep is displayed. \textbf{Error elimination phase transitions and monotonicity can be observed in the simulations.} Note that the phase transitions occur at different $p$ values than in the general CKP case.}
    \label{fig:simple-ckp-simulations}
\end{figure}

\end{document}

%% file: cmds.tex
\usepackage{dsfont}
\newcommand{\1}{\mathds{1}}

\newcommand{\E}{\mathbb{E}}

\newcommand{\N}{\mathbb{N}}
\renewcommand{\P}{\mathbb{P}}

\newcommand{\R}{\mathbb{R}}
\newcommand{\Z}{\mathbb{Z}}

\newcommand{\bpar}[1]{\left(#1\right)}
\newcommand{\bbr}[1]{\left[#1\right]}
\newcommand{\bcurly}[1]{\big\{#1\big\}}

\newcommand{\bs}{\baselineskip}

\DeclareMathOperator*{\argmin}{arg\,min}

\newcommand{\dist}{\mathrm{dist}}

\newcommand{\ind}[1]{\1_{[#1]}} % indicator 
\newcommand{\bigmid}{\ \big\vert \ }

\newcommand{\e}{\epsilon}
\newcommand{\eps}{\varepsilon}

\newcommand{\ceq}{\coloneqq}

\renewcommand{\P}{\mathbb{P}}

\newcommand{\F}{\mathcal{F}}
\newcommand{\G}{\mathcal{G}}
\newcommand{\T}{\mathcal{T}}

\renewcommand{\L}{\mathcal{L}}

\newcommand{\PF}{\textsf{PF}}
\newcommand{\CF}{\textsf{CF}}
\newcommand{\PT}{\textsf{PT}}
\newcommand{\CT}{\textsf{CT}}
\newcommand{\True}{{\textsf{True}}}
\newcommand{\False}{{\textsf{False}}}

\newcommand{\ZCF}{\textsf{ZCF}}
\newcommand{\ZCT}{\textsf{ZCT}}
\newcommand{\ZNCF}{\textsf{ZNCF}}
\newcommand{\ZNCT}{\textsf{ZNCT}}

\newcommand{\ZNPF}{\textsf{ZNPF}}

%% file: Visuals/4-step-evolution-CKP.tex
\tikzset{every picture/.style={line width=0.75pt}} %set default line width to 0.75pt        

\begin{tikzpicture}[x=0.75pt,y=0.75pt,yscale=-1,xscale=1]
%uncomment if require: \path (0,481); %set diagram left start at 0, and has height of 481

%Straight Lines [id:da05800507655522569] 
\draw [color={rgb, 255:red, 155; green, 155; blue, 155 }  ,draw opacity=1 ]   (405.01,255) -- (389.99,280) ;
%Curve Lines [id:da2808822994807607] 
\draw  [dash pattern={on 2.25pt off 1.5pt}]  (345.01,255) .. controls (352.37,269.6) and (351.7,287.6) .. (345.04,305) ;
%Straight Lines [id:da11041169173020893] 
\draw  [dash pattern={on 2.25pt off 1.5pt}]  (169.95,280) -- (185,305) ;
\draw [shift={(185,305)}, rotate = 58.96] [color={rgb, 255:red, 0; green, 0; blue, 0 }  ][fill={rgb, 255:red, 0; green, 0; blue, 0 }  ][line width=0.75]      (0, 0) circle [x radius= 2.01, y radius= 2.01]   ;
%Straight Lines [id:da825039999170277] 
\draw    (169.95,280) -- (160.01,305) ;
\draw [shift={(160.01,305)}, rotate = 111.69] [color={rgb, 255:red, 0; green, 0; blue, 0 }  ][fill={rgb, 255:red, 0; green, 0; blue, 0 }  ][line width=0.75]      (0, 0) circle [x radius= 2.01, y radius= 2.01]   ;
%Straight Lines [id:da9866439210357763] 
\draw    (95.02,30) ;
\draw [shift={(95.02,30)}, rotate = 0] [color={rgb, 255:red, 0; green, 0; blue, 0 }  ][fill={rgb, 255:red, 0; green, 0; blue, 0 }  ][line width=0.75]      (0, 0) circle [x radius= 2.01, y radius= 2.01]   ;
%Straight Lines [id:da2662007322639628] 
\draw    (125.02,30) ;
\draw [shift={(125.02,30)}, rotate = 0] [color={rgb, 255:red, 0; green, 0; blue, 0 }  ][fill={rgb, 255:red, 0; green, 0; blue, 0 }  ][line width=0.75]      (0, 0) circle [x radius= 2.01, y radius= 2.01]   ;
%Straight Lines [id:da16423454672288795] 
\draw    (155,30) -- (169.98,55) ;
\draw [shift={(169.98,55)}, rotate = 59.07] [color={rgb, 255:red, 0; green, 0; blue, 0 }  ][fill={rgb, 255:red, 0; green, 0; blue, 0 }  ][line width=0.75]      (0, 0) circle [x radius= 2.01, y radius= 2.01]   ;
\draw [shift={(155,30)}, rotate = 59.07] [color={rgb, 255:red, 0; green, 0; blue, 0 }  ][fill={rgb, 255:red, 0; green, 0; blue, 0 }  ][line width=0.75]      (0, 0) circle [x radius= 2.01, y radius= 2.01]   ;
%Straight Lines [id:da4180438981283604] 
\draw    (185,30) -- (170.5,54.13) ;
\draw [shift={(169.98,55)}, rotate = 121] [color={rgb, 255:red, 0; green, 0; blue, 0 }  ][line width=0.75]      (0, 0) circle [x radius= 2.01, y radius= 2.01]   ;
\draw [shift={(185,30)}, rotate = 121] [color={rgb, 255:red, 0; green, 0; blue, 0 }  ][fill={rgb, 255:red, 0; green, 0; blue, 0 }  ][line width=0.75]      (0, 0) circle [x radius= 2.01, y radius= 2.01]   ;
%Shape: Circle [id:dp03062442943571786] 
\draw  [color={rgb, 255:red, 0; green, 0; blue, 0 }  ,draw opacity=1 ][fill={rgb, 255:red, 255; green, 255; blue, 255 }  ,fill opacity=1 ] (167.73,55) .. controls (167.73,53.76) and (168.73,52.75) .. (169.98,52.75) .. controls (171.22,52.75) and (172.23,53.76) .. (172.23,55) .. controls (172.23,56.24) and (171.22,57.25) .. (169.98,57.25) .. controls (168.73,57.25) and (167.73,56.24) .. (167.73,55) -- cycle ;
%Straight Lines [id:da5787089795776864] 
\draw  [dash pattern={on 2.25pt off 1.5pt}]  (169.95,105) -- (169.98,130) ;
\draw [shift={(169.98,130)}, rotate = 89.95] [color={rgb, 255:red, 0; green, 0; blue, 0 }  ][fill={rgb, 255:red, 0; green, 0; blue, 0 }  ][line width=0.75]      (0, 0) circle [x radius= 2.01, y radius= 2.01]   ;
%Straight Lines [id:da9406692832742876] 
\draw    (154.98,80) -- (169.95,105) ;
\draw [shift={(169.95,105)}, rotate = 59.07] [color={rgb, 255:red, 0; green, 0; blue, 0 }  ][fill={rgb, 255:red, 0; green, 0; blue, 0 }  ][line width=0.75]      (0, 0) circle [x radius= 2.01, y radius= 2.01]   ;
\draw [shift={(154.98,80)}, rotate = 59.07] [color={rgb, 255:red, 0; green, 0; blue, 0 }  ][fill={rgb, 255:red, 0; green, 0; blue, 0 }  ][line width=0.75]      (0, 0) circle [x radius= 2.01, y radius= 2.01]   ;
%Straight Lines [id:da3371144437307507] 
\draw    (184.98,80) -- (170.47,104.13) ;
\draw [shift={(169.95,105)}, rotate = 121] [color={rgb, 255:red, 0; green, 0; blue, 0 }  ][line width=0.75]      (0, 0) circle [x radius= 2.01, y radius= 2.01]   ;
\draw [shift={(184.98,80)}, rotate = 121] [color={rgb, 255:red, 0; green, 0; blue, 0 }  ][fill={rgb, 255:red, 0; green, 0; blue, 0 }  ][line width=0.75]      (0, 0) circle [x radius= 2.01, y radius= 2.01]   ;
%Shape: Circle [id:dp768166752460446] 
\draw  [color={rgb, 255:red, 0; green, 0; blue, 0 }  ,draw opacity=1 ][fill={rgb, 255:red, 255; green, 255; blue, 255 }  ,fill opacity=1 ] (167.7,105) .. controls (167.7,103.76) and (168.71,102.75) .. (169.95,102.75) .. controls (171.2,102.75) and (172.2,103.76) .. (172.2,105) .. controls (172.2,106.24) and (171.2,107.25) .. (169.95,107.25) .. controls (168.71,107.25) and (167.7,106.24) .. (167.7,105) -- cycle ;
%Curve Lines [id:da35940437379068757] 
\draw  [dash pattern={on 2.25pt off 1.5pt}]  (154.98,80) .. controls (148.98,95.75) and (157.5,122.25) .. (169.98,130) ;
%Straight Lines [id:da46831222642485726] 
\draw    (375,30) -- (389.98,55) ;
\draw [shift={(389.98,55)}, rotate = 59.07] [color={rgb, 255:red, 0; green, 0; blue, 0 }  ][fill={rgb, 255:red, 0; green, 0; blue, 0 }  ][line width=0.75]      (0, 0) circle [x radius= 2.01, y radius= 2.01]   ;
\draw [shift={(375,30)}, rotate = 59.07] [color={rgb, 255:red, 0; green, 0; blue, 0 }  ][fill={rgb, 255:red, 0; green, 0; blue, 0 }  ][line width=0.75]      (0, 0) circle [x radius= 2.01, y radius= 2.01]   ;
%Straight Lines [id:da2648759148372747] 
\draw    (405,30) -- (390.5,54.13) ;
\draw [shift={(389.98,55)}, rotate = 121] [color={rgb, 255:red, 0; green, 0; blue, 0 }  ][line width=0.75]      (0, 0) circle [x radius= 2.01, y radius= 2.01]   ;
\draw [shift={(405,30)}, rotate = 121] [color={rgb, 255:red, 0; green, 0; blue, 0 }  ][fill={rgb, 255:red, 0; green, 0; blue, 0 }  ][line width=0.75]      (0, 0) circle [x radius= 2.01, y radius= 2.01]   ;
%Shape: Circle [id:dp614423071425075] 
\draw  [color={rgb, 255:red, 0; green, 0; blue, 0 }  ,draw opacity=1 ][fill={rgb, 255:red, 255; green, 255; blue, 255 }  ,fill opacity=1 ] (387.73,55) .. controls (387.73,53.76) and (388.73,52.75) .. (389.98,52.75) .. controls (391.22,52.75) and (392.23,53.76) .. (392.23,55) .. controls (392.23,56.24) and (391.22,57.25) .. (389.98,57.25) .. controls (388.73,57.25) and (387.73,56.24) .. (387.73,55) -- cycle ;
%Straight Lines [id:da08573282990533027] 
\draw    (95,80) -- (109.98,105) ;
\draw [shift={(109.98,105)}, rotate = 59.07] [color={rgb, 255:red, 0; green, 0; blue, 0 }  ][fill={rgb, 255:red, 0; green, 0; blue, 0 }  ][line width=0.75]      (0, 0) circle [x radius= 2.01, y radius= 2.01]   ;
\draw [shift={(95,80)}, rotate = 59.07] [color={rgb, 255:red, 0; green, 0; blue, 0 }  ][fill={rgb, 255:red, 0; green, 0; blue, 0 }  ][line width=0.75]      (0, 0) circle [x radius= 2.01, y radius= 2.01]   ;
%Straight Lines [id:da05116810754654477] 
\draw    (125,80) -- (110.5,104.13) ;
\draw [shift={(109.98,105)}, rotate = 121] [color={rgb, 255:red, 0; green, 0; blue, 0 }  ][line width=0.75]      (0, 0) circle [x radius= 2.01, y radius= 2.01]   ;
\draw [shift={(125,80)}, rotate = 121] [color={rgb, 255:red, 0; green, 0; blue, 0 }  ][fill={rgb, 255:red, 0; green, 0; blue, 0 }  ][line width=0.75]      (0, 0) circle [x radius= 2.01, y radius= 2.01]   ;
%Shape: Circle [id:dp739049102873557] 
\draw  [color={rgb, 255:red, 0; green, 0; blue, 0 }  ,draw opacity=1 ][fill={rgb, 255:red, 255; green, 255; blue, 255 }  ,fill opacity=1 ] (107.73,105) .. controls (107.73,103.76) and (108.73,102.75) .. (109.98,102.75) .. controls (111.22,102.75) and (112.23,103.76) .. (112.23,105) .. controls (112.23,106.24) and (111.22,107.25) .. (109.98,107.25) .. controls (108.73,107.25) and (107.73,106.24) .. (107.73,105) -- cycle ;
%Straight Lines [id:da506175933573507] 
\draw    (389.95,105) -- (389.98,130) ;
\draw [shift={(389.98,130)}, rotate = 89.95] [color={rgb, 255:red, 0; green, 0; blue, 0 }  ][fill={rgb, 255:red, 0; green, 0; blue, 0 }  ][line width=0.75]      (0, 0) circle [x radius= 2.01, y radius= 2.01]   ;
%Straight Lines [id:da28505924296155916] 
\draw    (375.03,80) -- (390,105) ;
\draw [shift={(390,105)}, rotate = 59.07] [color={rgb, 255:red, 0; green, 0; blue, 0 }  ][fill={rgb, 255:red, 0; green, 0; blue, 0 }  ][line width=0.75]      (0, 0) circle [x radius= 2.01, y radius= 2.01]   ;
\draw [shift={(375.03,80)}, rotate = 59.07] [color={rgb, 255:red, 0; green, 0; blue, 0 }  ][fill={rgb, 255:red, 0; green, 0; blue, 0 }  ][line width=0.75]      (0, 0) circle [x radius= 2.01, y radius= 2.01]   ;
%Straight Lines [id:da2820589065618577] 
\draw    (405.03,80) -- (390.52,104.13) ;
\draw [shift={(390,105)}, rotate = 121] [color={rgb, 255:red, 0; green, 0; blue, 0 }  ][line width=0.75]      (0, 0) circle [x radius= 2.01, y radius= 2.01]   ;
\draw [shift={(405.03,80)}, rotate = 121] [color={rgb, 255:red, 0; green, 0; blue, 0 }  ][fill={rgb, 255:red, 0; green, 0; blue, 0 }  ][line width=0.75]      (0, 0) circle [x radius= 2.01, y radius= 2.01]   ;
%Curve Lines [id:da3734722406401495] 
\draw    (374.98,80) .. controls (368.98,95.75) and (377.5,122.25) .. (389.98,130) ;
%Shape: Circle [id:dp712012365987609] 
\draw  [color={rgb, 255:red, 0; green, 0; blue, 0 }  ,draw opacity=1 ][fill={rgb, 255:red, 255; green, 255; blue, 255 }  ,fill opacity=1 ] (387.75,105) .. controls (387.75,103.76) and (388.76,102.75) .. (390,102.75) .. controls (391.25,102.75) and (392.25,103.76) .. (392.25,105) .. controls (392.25,106.24) and (391.25,107.25) .. (390,107.25) .. controls (388.76,107.25) and (387.75,106.24) .. (387.75,105) -- cycle ;
%Straight Lines [id:da7467198814665629] 
\draw    (109.98,180) -- (110,205) ;
\draw [shift={(110,205)}, rotate = 89.95] [color={rgb, 255:red, 0; green, 0; blue, 0 }  ][fill={rgb, 255:red, 0; green, 0; blue, 0 }  ][line width=0.75]      (0, 0) circle [x radius= 2.01, y radius= 2.01]   ;
%Straight Lines [id:da9055177423659957] 
\draw    (95,155) -- (109.98,180) ;
\draw [shift={(109.98,180)}, rotate = 59.07] [color={rgb, 255:red, 0; green, 0; blue, 0 }  ][fill={rgb, 255:red, 0; green, 0; blue, 0 }  ][line width=0.75]      (0, 0) circle [x radius= 2.01, y radius= 2.01]   ;
\draw [shift={(95,155)}, rotate = 59.07] [color={rgb, 255:red, 0; green, 0; blue, 0 }  ][fill={rgb, 255:red, 0; green, 0; blue, 0 }  ][line width=0.75]      (0, 0) circle [x radius= 2.01, y radius= 2.01]   ;
%Straight Lines [id:da5370816874159682] 
\draw    (125,155) -- (121.4,160.99) -- (110.5,179.13) ;
\draw [shift={(109.98,180)}, rotate = 121] [color={rgb, 255:red, 0; green, 0; blue, 0 }  ][line width=0.75]      (0, 0) circle [x radius= 2.01, y radius= 2.01]   ;
\draw [shift={(125,155)}, rotate = 121] [color={rgb, 255:red, 0; green, 0; blue, 0 }  ][fill={rgb, 255:red, 0; green, 0; blue, 0 }  ][line width=0.75]      (0, 0) circle [x radius= 2.01, y radius= 2.01]   ;
%Curve Lines [id:da278658599601029] 
\draw    (95,155) .. controls (89,170.75) and (97.52,197.25) .. (110,205) ;
%Shape: Circle [id:dp4784087094101458] 
\draw  [color={rgb, 255:red, 0; green, 0; blue, 0 }  ,draw opacity=1 ][fill={rgb, 255:red, 255; green, 255; blue, 255 }  ,fill opacity=1 ] (107.73,180) .. controls (107.73,178.76) and (108.73,177.75) .. (109.98,177.75) .. controls (111.22,177.75) and (112.23,178.76) .. (112.23,180) .. controls (112.23,181.24) and (111.22,182.25) .. (109.98,182.25) .. controls (108.73,182.25) and (107.73,181.24) .. (107.73,180) -- cycle ;
%Straight Lines [id:da783145441643115] 
\draw    (169.95,180) -- (169.98,205) ;
\draw [shift={(169.98,205)}, rotate = 89.95] [color={rgb, 255:red, 0; green, 0; blue, 0 }  ][fill={rgb, 255:red, 0; green, 0; blue, 0 }  ][line width=0.75]      (0, 0) circle [x radius= 2.01, y radius= 2.01]   ;
%Straight Lines [id:da41514699011431455] 
\draw    (154.98,155) -- (169.95,180) ;
\draw [shift={(169.95,180)}, rotate = 59.07] [color={rgb, 255:red, 0; green, 0; blue, 0 }  ][fill={rgb, 255:red, 0; green, 0; blue, 0 }  ][line width=0.75]      (0, 0) circle [x radius= 2.01, y radius= 2.01]   ;
\draw [shift={(154.98,155)}, rotate = 59.07] [color={rgb, 255:red, 0; green, 0; blue, 0 }  ][fill={rgb, 255:red, 0; green, 0; blue, 0 }  ][line width=0.75]      (0, 0) circle [x radius= 2.01, y radius= 2.01]   ;
%Straight Lines [id:da7771799308001417] 
\draw    (184.98,155) -- (170.47,179.13) ;
\draw [shift={(169.95,180)}, rotate = 121] [color={rgb, 255:red, 0; green, 0; blue, 0 }  ][line width=0.75]      (0, 0) circle [x radius= 2.01, y radius= 2.01]   ;
\draw [shift={(184.98,155)}, rotate = 121] [color={rgb, 255:red, 0; green, 0; blue, 0 }  ][fill={rgb, 255:red, 0; green, 0; blue, 0 }  ][line width=0.75]      (0, 0) circle [x radius= 2.01, y radius= 2.01]   ;
%Curve Lines [id:da9763989376937783] 
\draw    (154.98,155) .. controls (148.98,170.75) and (157.5,197.25) .. (169.98,205) ;
%Shape: Circle [id:dp4805057716250928] 
\draw  [color={rgb, 255:red, 0; green, 0; blue, 0 }  ,draw opacity=1 ][fill={rgb, 255:red, 255; green, 255; blue, 255 }  ,fill opacity=1 ] (167.7,180) .. controls (167.7,178.76) and (168.71,177.75) .. (169.95,177.75) .. controls (171.2,177.75) and (172.2,178.76) .. (172.2,180) .. controls (172.2,181.24) and (171.2,182.25) .. (169.95,182.25) .. controls (168.71,182.25) and (167.7,181.24) .. (167.7,180) -- cycle ;
%Straight Lines [id:da554334619823499] 
\draw  [dash pattern={on 2.25pt off 1.5pt}]  (169.98,205) -- (170,230) ;
\draw [shift={(170,230)}, rotate = 89.95] [color={rgb, 255:red, 0; green, 0; blue, 0 }  ][fill={rgb, 255:red, 0; green, 0; blue, 0 }  ][line width=0.75]      (0, 0) circle [x radius= 2.01, y radius= 2.01]   ;
%Curve Lines [id:da262198880704579] 
\draw  [dash pattern={on 2.25pt off 1.5pt}]  (184.98,155) .. controls (189.5,180.75) and (182,221.25) .. (170,230) ;
%Straight Lines [id:da8523343970353211] 
\draw    (269.94,180) -- (269.96,205) ;
\draw [shift={(269.96,205)}, rotate = 89.95] [color={rgb, 255:red, 0; green, 0; blue, 0 }  ][fill={rgb, 255:red, 0; green, 0; blue, 0 }  ][line width=0.75]      (0, 0) circle [x radius= 2.01, y radius= 2.01]   ;
%Straight Lines [id:da39365332872556247] 
\draw    (254.96,155) -- (269.94,180) ;
\draw [shift={(269.94,180)}, rotate = 59.07] [color={rgb, 255:red, 0; green, 0; blue, 0 }  ][fill={rgb, 255:red, 0; green, 0; blue, 0 }  ][line width=0.75]      (0, 0) circle [x radius= 2.01, y radius= 2.01]   ;
\draw [shift={(254.96,155)}, rotate = 59.07] [color={rgb, 255:red, 0; green, 0; blue, 0 }  ][fill={rgb, 255:red, 0; green, 0; blue, 0 }  ][line width=0.75]      (0, 0) circle [x radius= 2.01, y radius= 2.01]   ;
%Straight Lines [id:da7109020620286572] 
\draw    (284.96,155) -- (270.46,179.13) ;
\draw [shift={(269.94,180)}, rotate = 121] [color={rgb, 255:red, 0; green, 0; blue, 0 }  ][line width=0.75]      (0, 0) circle [x radius= 2.01, y radius= 2.01]   ;
\draw [shift={(284.96,155)}, rotate = 121] [color={rgb, 255:red, 0; green, 0; blue, 0 }  ][fill={rgb, 255:red, 0; green, 0; blue, 0 }  ][line width=0.75]      (0, 0) circle [x radius= 2.01, y radius= 2.01]   ;
%Curve Lines [id:da7842750593010146] 
\draw    (254.96,155) .. controls (248.96,170.75) and (257.49,197.25) .. (269.96,205) ;
%Shape: Circle [id:dp893957238679488] 
\draw  [color={rgb, 255:red, 0; green, 0; blue, 0 }  ,draw opacity=1 ][fill={rgb, 255:red, 255; green, 255; blue, 255 }  ,fill opacity=1 ] (267.69,180) .. controls (267.69,178.76) and (268.7,177.75) .. (269.94,177.75) .. controls (271.18,177.75) and (272.19,178.76) .. (272.19,180) .. controls (272.19,181.24) and (271.18,182.25) .. (269.94,182.25) .. controls (268.7,182.25) and (267.69,181.24) .. (267.69,180) -- cycle ;
%Straight Lines [id:da1342588371596778] 
\draw  [dash pattern={on 2.25pt off 1.5pt}]  (269.96,205) -- (269.99,230) ;
\draw [shift={(269.99,230)}, rotate = 89.95] [color={rgb, 255:red, 0; green, 0; blue, 0 }  ][fill={rgb, 255:red, 0; green, 0; blue, 0 }  ][line width=0.75]      (0, 0) circle [x radius= 2.01, y radius= 2.01]   ;
%Curve Lines [id:da9675563761754451] 
\draw  [dash pattern={on 2.25pt off 1.5pt}]  (284.96,155) .. controls (289.49,180.75) and (281.99,221.25) .. (269.99,230) ;
%Shape: Circle [id:dp029142070212295534] 
\draw  [color={rgb, 255:red, 74; green, 144; blue, 226 }  ,draw opacity=1 ] (279.96,155) .. controls (279.96,152.24) and (282.2,150) .. (284.96,150) .. controls (287.73,150) and (289.96,152.24) .. (289.96,155) .. controls (289.96,157.76) and (287.73,160) .. (284.96,160) .. controls (282.2,160) and (279.96,157.76) .. (279.96,155) -- cycle ;
%Shape: Circle [id:dp4588333979576664] 
\draw  [color={rgb, 255:red, 74; green, 144; blue, 226 }  ,draw opacity=1 ] (264.96,205) .. controls (264.96,202.24) and (267.2,200) .. (269.96,200) .. controls (272.73,200) and (274.96,202.24) .. (274.96,205) .. controls (274.96,207.76) and (272.73,210) .. (269.96,210) .. controls (267.2,210) and (264.96,207.76) .. (264.96,205) -- cycle ;
%Shape: Circle [id:dp7148613497029518] 
\draw  [color={rgb, 255:red, 74; green, 144; blue, 226 }  ,draw opacity=1 ] (264.99,230) .. controls (264.99,227.24) and (267.23,225) .. (269.99,225) .. controls (272.75,225) and (274.99,227.24) .. (274.99,230) .. controls (274.99,232.76) and (272.75,235) .. (269.99,235) .. controls (267.23,235) and (264.99,232.76) .. (264.99,230) -- cycle ;
%Shape: Square [id:dp9284825787006687] 
\draw  [draw opacity=0] (426,15) -- (440.5,15) -- (440.5,29.5) -- (426,29.5) -- cycle ;
%Straight Lines [id:da5751699485296166] 
\draw  [dash pattern={on 1.5pt off 1.5pt}]  (279.11,350) -- (276.56,352.29) ;
\draw [shift={(276.56,352.29)}, rotate = 138.09] [color={rgb, 255:red, 0; green, 0; blue, 0 }  ][line width=0.75]    (-3.35,0) -- (3.35,0)(0,3.35) -- (0,-3.35)   ;
%Straight Lines [id:da4871656219425369] 
\draw  [dash pattern={on 1.5pt off 1.5pt}]  (279.11,350) -- (278.06,350.94) ;
\draw [shift={(276.56,352.29)}, rotate = 138.09] [color={rgb, 255:red, 0; green, 0; blue, 0 }  ][line width=0.75]      (0, 0) circle [x radius= 3.02, y radius= 3.02]   ;
%Curve Lines [id:da6434648570572161] 
\draw  [dash pattern={on 0.75pt off 7.5pt}]  (274.5,350.16) .. controls (280.47,356.39) and (278.33,354.16) .. (276.96,352.72) ;
\draw [shift={(276.56,352.29)}, rotate = 45.79] [color={rgb, 255:red, 0; green, 0; blue, 0 }  ][line width=0.75]      (0, 0) circle [x radius= 2.68, y radius= 2.68]   ;

%Flowchart: Connector [id:dp723257361432084] 
\draw   (224,352.35) .. controls (224,351.05) and (225.05,350) .. (226.35,350) .. controls (227.65,350) and (228.7,351.05) .. (228.7,352.35) .. controls (228.7,353.65) and (227.65,354.7) .. (226.35,354.7) .. controls (225.05,354.7) and (224,353.65) .. (224,352.35) -- cycle ;
%Flowchart: Connector [id:dp5126633886819172] 
\draw  [fill={rgb, 255:red, 0; green, 0; blue, 0 }  ,fill opacity=1 ] (174,352.35) .. controls (174,351.05) and (175.05,350) .. (176.35,350) .. controls (177.65,350) and (178.7,351.05) .. (178.7,352.35) .. controls (178.7,353.65) and (177.65,354.7) .. (176.35,354.7) .. controls (175.05,354.7) and (174,353.65) .. (174,352.35) -- cycle ;
%Straight Lines [id:da14673245903040066] 
\draw    (109.98,280) -- (100.04,305) ;
\draw [shift={(100.04,305)}, rotate = 111.69] [color={rgb, 255:red, 0; green, 0; blue, 0 }  ][fill={rgb, 255:red, 0; green, 0; blue, 0 }  ][line width=0.75]      (0, 0) circle [x radius= 2.01, y radius= 2.01]   ;
%Straight Lines [id:da4850739260649175] 
\draw    (95,255) -- (109.98,280) ;
\draw [shift={(109.98,280)}, rotate = 59.07] [color={rgb, 255:red, 0; green, 0; blue, 0 }  ][fill={rgb, 255:red, 0; green, 0; blue, 0 }  ][line width=0.75]      (0, 0) circle [x radius= 2.01, y radius= 2.01]   ;
\draw [shift={(95,255)}, rotate = 59.07] [color={rgb, 255:red, 0; green, 0; blue, 0 }  ][fill={rgb, 255:red, 0; green, 0; blue, 0 }  ][line width=0.75]      (0, 0) circle [x radius= 2.01, y radius= 2.01]   ;
%Straight Lines [id:da1568759529562339] 
\draw    (125,255) -- (121.4,260.99) -- (110.5,279.13) ;
\draw [shift={(109.98,280)}, rotate = 121] [color={rgb, 255:red, 0; green, 0; blue, 0 }  ][line width=0.75]      (0, 0) circle [x radius= 2.01, y radius= 2.01]   ;
\draw [shift={(125,255)}, rotate = 121] [color={rgb, 255:red, 0; green, 0; blue, 0 }  ][fill={rgb, 255:red, 0; green, 0; blue, 0 }  ][line width=0.75]      (0, 0) circle [x radius= 2.01, y radius= 2.01]   ;
%Shape: Circle [id:dp18630348862130042] 
\draw  [color={rgb, 255:red, 0; green, 0; blue, 0 }  ,draw opacity=1 ][fill={rgb, 255:red, 255; green, 255; blue, 255 }  ,fill opacity=1 ] (107.73,280) .. controls (107.73,278.76) and (108.73,277.75) .. (109.98,277.75) .. controls (111.22,277.75) and (112.23,278.76) .. (112.23,280) .. controls (112.23,281.24) and (111.22,282.25) .. (109.98,282.25) .. controls (108.73,282.25) and (107.73,281.24) .. (107.73,280) -- cycle ;
%Straight Lines [id:da3895789514535093] 
\draw    (154.98,255) -- (169.95,280) ;
\draw [shift={(169.95,280)}, rotate = 59.07] [color={rgb, 255:red, 0; green, 0; blue, 0 }  ][fill={rgb, 255:red, 0; green, 0; blue, 0 }  ][line width=0.75]      (0, 0) circle [x radius= 2.01, y radius= 2.01]   ;
\draw [shift={(154.98,255)}, rotate = 59.07] [color={rgb, 255:red, 0; green, 0; blue, 0 }  ][fill={rgb, 255:red, 0; green, 0; blue, 0 }  ][line width=0.75]      (0, 0) circle [x radius= 2.01, y radius= 2.01]   ;
%Straight Lines [id:da8615324528243262] 
\draw    (184.98,255) -- (170.47,279.13) ;
\draw [shift={(169.95,280)}, rotate = 121] [color={rgb, 255:red, 0; green, 0; blue, 0 }  ][line width=0.75]      (0, 0) circle [x radius= 2.01, y radius= 2.01]   ;
\draw [shift={(184.98,255)}, rotate = 121] [color={rgb, 255:red, 0; green, 0; blue, 0 }  ][fill={rgb, 255:red, 0; green, 0; blue, 0 }  ][line width=0.75]      (0, 0) circle [x radius= 2.01, y radius= 2.01]   ;
%Shape: Circle [id:dp24893563332999724] 
\draw  [color={rgb, 255:red, 0; green, 0; blue, 0 }  ,draw opacity=1 ][fill={rgb, 255:red, 255; green, 255; blue, 255 }  ,fill opacity=1 ] (167.7,280) .. controls (167.7,278.76) and (168.71,277.75) .. (169.95,277.75) .. controls (171.2,277.75) and (172.2,278.76) .. (172.2,280) .. controls (172.2,281.24) and (171.2,282.25) .. (169.95,282.25) .. controls (168.71,282.25) and (167.7,281.24) .. (167.7,280) -- cycle ;
%Straight Lines [id:da16692792982492155] 
\draw    (389.94,180) -- (389.96,205) ;
\draw [shift={(389.96,205)}, rotate = 89.95] [color={rgb, 255:red, 0; green, 0; blue, 0 }  ][fill={rgb, 255:red, 0; green, 0; blue, 0 }  ][line width=0.75]      (0, 0) circle [x radius= 2.01, y radius= 2.01]   ;
%Straight Lines [id:da06331291609702838] 
\draw    (374.96,155) -- (389.94,180) ;
\draw [shift={(389.94,180)}, rotate = 59.07] [color={rgb, 255:red, 0; green, 0; blue, 0 }  ][fill={rgb, 255:red, 0; green, 0; blue, 0 }  ][line width=0.75]      (0, 0) circle [x radius= 2.01, y radius= 2.01]   ;
\draw [shift={(374.96,155)}, rotate = 59.07] [color={rgb, 255:red, 0; green, 0; blue, 0 }  ][fill={rgb, 255:red, 0; green, 0; blue, 0 }  ][line width=0.75]      (0, 0) circle [x radius= 2.01, y radius= 2.01]   ;
%Straight Lines [id:da9289080165397942] 
\draw    (404.96,155) -- (390.46,179.13) ;
\draw [shift={(389.94,180)}, rotate = 121] [color={rgb, 255:red, 0; green, 0; blue, 0 }  ][line width=0.75]      (0, 0) circle [x radius= 2.01, y radius= 2.01]   ;
\draw [shift={(404.96,155)}, rotate = 121] [color={rgb, 255:red, 0; green, 0; blue, 0 }  ][fill={rgb, 255:red, 0; green, 0; blue, 0 }  ][line width=0.75]      (0, 0) circle [x radius= 2.01, y radius= 2.01]   ;
%Curve Lines [id:da6305363516726149] 
\draw    (374.96,155) .. controls (368.96,170.75) and (377.49,197.25) .. (389.96,205) ;
%Shape: Circle [id:dp9221131161830648] 
\draw  [color={rgb, 255:red, 0; green, 0; blue, 0 }  ,draw opacity=1 ][fill={rgb, 255:red, 255; green, 255; blue, 255 }  ,fill opacity=1 ] (387.69,180) .. controls (387.69,178.76) and (388.7,177.75) .. (389.94,177.75) .. controls (391.18,177.75) and (392.19,178.76) .. (392.19,180) .. controls (392.19,181.24) and (391.18,182.25) .. (389.94,182.25) .. controls (388.7,182.25) and (387.69,181.24) .. (387.69,180) -- cycle ;
%Straight Lines [id:da13970928603339317] 
\draw    (389.96,205) -- (389.99,230) ;
\draw [shift={(389.99,230)}, rotate = 89.95] [color={rgb, 255:red, 0; green, 0; blue, 0 }  ][fill={rgb, 255:red, 0; green, 0; blue, 0 }  ][line width=0.75]      (0, 0) circle [x radius= 2.01, y radius= 2.01]   ;
%Curve Lines [id:da5064076956685393] 
\draw    (404.96,155) .. controls (409.49,180.75) and (401.99,221.25) .. (389.99,230) ;
%Straight Lines [id:da36951575067082854] 
\draw    (100.04,305) -- (100.06,330) ;
\draw [shift={(100.06,330)}, rotate = 89.95] [color={rgb, 255:red, 0; green, 0; blue, 0 }  ][fill={rgb, 255:red, 0; green, 0; blue, 0 }  ][line width=0.75]      (0, 0) circle [x radius= 2.01, y radius= 2.01]   ;
%Curve Lines [id:da45436148435090473] 
\draw [color={rgb, 255:red, 0; green, 0; blue, 0 }  ,draw opacity=1 ]   (95,255) .. controls (91.04,272.35) and (93.04,292.35) .. (100.04,305) ;
%Curve Lines [id:da9224744313559448] 
\draw [color={rgb, 255:red, 0; green, 0; blue, 0 }  ,draw opacity=1 ]   (125,255) .. controls (124.04,284.35) and (112.06,321.25) .. (100.06,330) ;
%Curve Lines [id:da8534585914152599] 
\draw [color={rgb, 255:red, 0; green, 0; blue, 0 }  ,draw opacity=1 ]   (154.98,255) .. controls (151.01,272.35) and (153.01,292.35) .. (160.01,305) ;
%Straight Lines [id:da9929800730839649] 
\draw    (160.01,305) -- (160.04,330) ;
\draw [shift={(160.04,330)}, rotate = 89.95] [color={rgb, 255:red, 0; green, 0; blue, 0 }  ][fill={rgb, 255:red, 0; green, 0; blue, 0 }  ][line width=0.75]      (0, 0) circle [x radius= 2.01, y radius= 2.01]   ;
%Curve Lines [id:da3645718747068992] 
\draw [color={rgb, 255:red, 0; green, 0; blue, 0 }  ,draw opacity=1 ]   (184.98,255) .. controls (184.01,284.35) and (172.04,321.25) .. (160.04,330) ;
%Curve Lines [id:da20017594707190767] 
\draw  [dash pattern={on 2.25pt off 1.5pt}]  (184.98,255) .. controls (192.33,269.6) and (191.67,287.6) .. (185,305) ;
\draw [shift={(185,305)}, rotate = 110.96] [color={rgb, 255:red, 0; green, 0; blue, 0 }  ][fill={rgb, 255:red, 0; green, 0; blue, 0 }  ][line width=0.75]      (0, 0) circle [x radius= 2.01, y radius= 2.01]   ;
\draw [shift={(184.98,255)}, rotate = 63.26] [color={rgb, 255:red, 0; green, 0; blue, 0 }  ][fill={rgb, 255:red, 0; green, 0; blue, 0 }  ][line width=0.75]      (0, 0) circle [x radius= 2.01, y radius= 2.01]   ;
%Straight Lines [id:da5666328048191869] 
\draw  [dash pattern={on 2.25pt off 1.5pt}]  (269.99,280) -- (285.04,305) ;
\draw [shift={(285.04,305)}, rotate = 58.96] [color={rgb, 255:red, 0; green, 0; blue, 0 }  ][fill={rgb, 255:red, 0; green, 0; blue, 0 }  ][line width=0.75]      (0, 0) circle [x radius= 2.01, y radius= 2.01]   ;
%Straight Lines [id:da7841145011492939] 
\draw    (269.99,280) -- (260.05,305) ;
\draw [shift={(260.05,305)}, rotate = 111.69] [color={rgb, 255:red, 0; green, 0; blue, 0 }  ][fill={rgb, 255:red, 0; green, 0; blue, 0 }  ][line width=0.75]      (0, 0) circle [x radius= 2.01, y radius= 2.01]   ;
%Straight Lines [id:da6881230627363685] 
\draw    (255.01,255) -- (269.99,280) ;
\draw [shift={(269.99,280)}, rotate = 59.07] [color={rgb, 255:red, 0; green, 0; blue, 0 }  ][fill={rgb, 255:red, 0; green, 0; blue, 0 }  ][line width=0.75]      (0, 0) circle [x radius= 2.01, y radius= 2.01]   ;
\draw [shift={(255.01,255)}, rotate = 59.07] [color={rgb, 255:red, 0; green, 0; blue, 0 }  ][fill={rgb, 255:red, 0; green, 0; blue, 0 }  ][line width=0.75]      (0, 0) circle [x radius= 2.01, y radius= 2.01]   ;
%Straight Lines [id:da3494206996301745] 
\draw    (285.01,255) -- (270.51,279.13) ;
\draw [shift={(269.99,280)}, rotate = 121] [color={rgb, 255:red, 0; green, 0; blue, 0 }  ][line width=0.75]      (0, 0) circle [x radius= 2.01, y radius= 2.01]   ;
\draw [shift={(285.01,255)}, rotate = 121] [color={rgb, 255:red, 0; green, 0; blue, 0 }  ][fill={rgb, 255:red, 0; green, 0; blue, 0 }  ][line width=0.75]      (0, 0) circle [x radius= 2.01, y radius= 2.01]   ;
%Shape: Circle [id:dp28301205493489046] 
\draw  [color={rgb, 255:red, 0; green, 0; blue, 0 }  ,draw opacity=1 ][fill={rgb, 255:red, 255; green, 255; blue, 255 }  ,fill opacity=1 ] (267.74,280) .. controls (267.74,278.76) and (268.75,277.75) .. (269.99,277.75) .. controls (271.23,277.75) and (272.24,278.76) .. (272.24,280) .. controls (272.24,281.24) and (271.23,282.25) .. (269.99,282.25) .. controls (268.75,282.25) and (267.74,281.24) .. (267.74,280) -- cycle ;
%Curve Lines [id:da12750274026632957] 
\draw [color={rgb, 255:red, 0; green, 0; blue, 0 }  ,draw opacity=1 ]   (255.01,255) .. controls (251.05,272.35) and (253.05,292.35) .. (260.05,305) ;
%Straight Lines [id:da026416728939982037] 
\draw    (260.05,305) -- (260.07,330) ;
\draw [shift={(260.07,330)}, rotate = 89.95] [color={rgb, 255:red, 0; green, 0; blue, 0 }  ][fill={rgb, 255:red, 0; green, 0; blue, 0 }  ][line width=0.75]      (0, 0) circle [x radius= 2.01, y radius= 2.01]   ;
%Curve Lines [id:da37864441832333073] 
\draw [color={rgb, 255:red, 0; green, 0; blue, 0 }  ,draw opacity=1 ]   (285.01,255) .. controls (284.05,284.35) and (272.07,321.25) .. (260.07,330) ;
%Curve Lines [id:da00802902170235098] 
\draw  [dash pattern={on 2.25pt off 1.5pt}]  (285.01,255) .. controls (292.37,269.6) and (291.7,287.6) .. (285.04,305) ;
\draw [shift={(285.04,305)}, rotate = 110.96] [color={rgb, 255:red, 0; green, 0; blue, 0 }  ][fill={rgb, 255:red, 0; green, 0; blue, 0 }  ][line width=0.75]      (0, 0) circle [x radius= 2.01, y radius= 2.01]   ;
\draw [shift={(285.01,255)}, rotate = 63.26] [color={rgb, 255:red, 0; green, 0; blue, 0 }  ][fill={rgb, 255:red, 0; green, 0; blue, 0 }  ][line width=0.75]      (0, 0) circle [x radius= 2.01, y radius= 2.01]   ;
%Shape: Circle [id:dp4505063302964082] 
\draw  [color={rgb, 255:red, 74; green, 144; blue, 226 }  ,draw opacity=1 ] (280.04,305) .. controls (280.04,302.24) and (282.27,300) .. (285.04,300) .. controls (287.8,300) and (290.04,302.24) .. (290.04,305) .. controls (290.04,307.76) and (287.8,310) .. (285.04,310) .. controls (282.27,310) and (280.04,307.76) .. (280.04,305) -- cycle ;
%Shape: Circle [id:dp5037794782130628] 
\draw  [color={rgb, 255:red, 74; green, 144; blue, 226 }  ,draw opacity=1 ] (280.01,255) .. controls (280.01,252.24) and (282.25,250) .. (285.01,250) .. controls (287.77,250) and (290.01,252.24) .. (290.01,255) .. controls (290.01,257.76) and (287.77,260) .. (285.01,260) .. controls (282.25,260) and (280.01,257.76) .. (280.01,255) -- cycle ;
%Shape: Circle [id:dp19923878936969652] 
\draw  [color={rgb, 255:red, 74; green, 144; blue, 226 }  ,draw opacity=1 ] (264.99,280) .. controls (264.99,277.24) and (267.23,275) .. (269.99,275) .. controls (272.75,275) and (274.99,277.24) .. (274.99,280) .. controls (274.99,282.76) and (272.75,285) .. (269.99,285) .. controls (267.23,285) and (264.99,282.76) .. (264.99,280) -- cycle ;
%Straight Lines [id:da3446944671561639] 
\draw    (329.99,280) -- (320.05,305) ;
\draw [shift={(320.05,305)}, rotate = 111.69] [color={rgb, 255:red, 0; green, 0; blue, 0 }  ][fill={rgb, 255:red, 0; green, 0; blue, 0 }  ][line width=0.75]      (0, 0) circle [x radius= 2.01, y radius= 2.01]   ;
%Straight Lines [id:da11297828785777142] 
\draw    (345.01,255) -- (329.99,280) ;
\draw [shift={(345.01,255)}, rotate = 121] [color={rgb, 255:red, 0; green, 0; blue, 0 }  ][fill={rgb, 255:red, 0; green, 0; blue, 0 }  ][line width=0.75]      (0, 0) circle [x radius= 2.01, y radius= 2.01]   ;
%Curve Lines [id:da769350140961964] 
\draw [color={rgb, 255:red, 0; green, 0; blue, 0 }  ,draw opacity=1 ]   (315.01,255) .. controls (311.05,272.35) and (313.05,292.35) .. (320.05,305) ;
%Straight Lines [id:da17178797410856572] 
\draw    (320.05,305) -- (320.07,330) ;
\draw [shift={(320.07,330)}, rotate = 89.95] [color={rgb, 255:red, 0; green, 0; blue, 0 }  ][fill={rgb, 255:red, 0; green, 0; blue, 0 }  ][line width=0.75]      (0, 0) circle [x radius= 2.01, y radius= 2.01]   ;
%Curve Lines [id:da14101415306179765] 
\draw [color={rgb, 255:red, 0; green, 0; blue, 0 }  ,draw opacity=1 ]   (345.01,255) .. controls (344.05,284.35) and (332.07,321.25) .. (320.07,330) ;
%Shape: Circle [id:dp6131349395703393] 
\draw  [color={rgb, 255:red, 0; green, 0; blue, 0 }  ,draw opacity=1 ][fill={rgb, 255:red, 255; green, 255; blue, 255 }  ,fill opacity=1 ] (342.04,305) .. controls (342.04,303.34) and (343.38,302) .. (345.04,302) .. controls (346.69,302) and (348.04,303.34) .. (348.04,305) .. controls (348.04,306.66) and (346.69,308) .. (345.04,308) .. controls (343.38,308) and (342.04,306.66) .. (342.04,305) -- cycle ;
%Straight Lines [id:da9997185659047142] 
\draw  [dash pattern={on 2.25pt off 1.5pt}]  (329.99,280) -- (345.04,305) ;
\draw [shift={(345.04,305)}, rotate = 58.96] [color={rgb, 255:red, 0; green, 0; blue, 0 }  ][line width=0.75]    (-3.35,0) -- (3.35,0)(0,3.35) -- (0,-3.35)   ;
%Curve Lines [id:da5907410039405125] 
\draw [color={rgb, 255:red, 74; green, 144; blue, 226 }  ,draw opacity=1 ] [dash pattern={on 2.25pt off 1.5pt}]  (285.01,255) .. controls (292.37,269.6) and (291.7,287.6) .. (285.04,305) ;
%Straight Lines [id:da603602124783561] 
\draw [color={rgb, 255:red, 74; green, 144; blue, 226 }  ,draw opacity=1 ] [dash pattern={on 2.25pt off 1.5pt}]  (269.99,280) -- (285.04,305) ;
%Straight Lines [id:da53162107268067] 
\draw [color={rgb, 255:red, 74; green, 144; blue, 226 }  ,draw opacity=1 ] [dash pattern={on 2.25pt off 1.5pt}]  (269.96,205) -- (269.99,230) ;
%Curve Lines [id:da4672724064042737] 
\draw [color={rgb, 255:red, 74; green, 144; blue, 226 }  ,draw opacity=1 ] [dash pattern={on 2.25pt off 1.5pt}]  (284.96,155) .. controls (289.49,180.75) and (281.99,221.25) .. (269.99,230) ;
%Curve Lines [id:da37494492805849144] 
\draw [color={rgb, 255:red, 155; green, 155; blue, 155 }  ,draw opacity=1 ]   (405.01,255) .. controls (412.37,269.6) and (411.7,287.6) .. (405.04,305) ;
%Straight Lines [id:da011623785234614825] 
\draw    (389.99,280) -- (380.05,305) ;
\draw [shift={(380.05,305)}, rotate = 111.69] [color={rgb, 255:red, 0; green, 0; blue, 0 }  ][fill={rgb, 255:red, 0; green, 0; blue, 0 }  ][line width=0.75]      (0, 0) circle [x radius= 2.01, y radius= 2.01]   ;
%Straight Lines [id:da17891077789526466] 
\draw    (380.05,305) -- (380.07,330) ;
\draw [shift={(380.07,330)}, rotate = 89.95] [color={rgb, 255:red, 0; green, 0; blue, 0 }  ][fill={rgb, 255:red, 0; green, 0; blue, 0 }  ][line width=0.75]      (0, 0) circle [x radius= 2.01, y radius= 2.01]   ;
%Curve Lines [id:da7497383663722775] 
\draw [color={rgb, 255:red, 0; green, 0; blue, 0 }  ,draw opacity=1 ]   (405.01,255) .. controls (404.05,284.35) and (392.07,321.25) .. (380.07,330) ;
\draw [shift={(405.01,255)}, rotate = 91.88] [color={rgb, 255:red, 0; green, 0; blue, 0 }  ,draw opacity=1 ][fill={rgb, 255:red, 0; green, 0; blue, 0 }  ,fill opacity=1 ][line width=0.75]      (0, 0) circle [x radius= 2.01, y radius= 2.01]   ;
%Shape: Circle [id:dp171411106866321] 
\draw  [color={rgb, 255:red, 155; green, 155; blue, 155 }  ,draw opacity=1 ][fill={rgb, 255:red, 255; green, 255; blue, 255 }  ,fill opacity=1 ] (402.04,305) .. controls (402.04,303.34) and (403.38,302) .. (405.04,302) .. controls (406.69,302) and (408.04,303.34) .. (408.04,305) .. controls (408.04,306.66) and (406.69,308) .. (405.04,308) .. controls (403.38,308) and (402.04,306.66) .. (402.04,305) -- cycle ;
%Straight Lines [id:da17683933853799683] 
\draw [color={rgb, 255:red, 155; green, 155; blue, 155 }  ,draw opacity=1 ]   (389.99,280) -- (405.04,305) ;
\draw [shift={(405.04,305)}, rotate = 58.96] [color={rgb, 255:red, 155; green, 155; blue, 155 }  ,draw opacity=1 ][line width=0.75]    (-3.35,0) -- (3.35,0)(0,3.35) -- (0,-3.35)   ;
\draw [shift={(389.99,280)}, rotate = 58.96] [color={rgb, 255:red, 155; green, 155; blue, 155 }  ,draw opacity=1 ][line width=0.75]    (-3.35,0) -- (3.35,0)(0,3.35) -- (0,-3.35)   ;
%Shape: Circle [id:dp9579778173359241] 
\draw  [color={rgb, 255:red, 0; green, 0; blue, 0 }  ,draw opacity=1 ][fill={rgb, 255:red, 255; green, 255; blue, 255 }  ,fill opacity=1 ] (326.99,280) .. controls (326.99,278.34) and (328.33,277) .. (329.99,277) .. controls (331.65,277) and (332.99,278.34) .. (332.99,280) .. controls (332.99,281.66) and (331.65,283) .. (329.99,283) .. controls (328.33,283) and (326.99,281.66) .. (326.99,280) -- cycle ;
%Straight Lines [id:da16684339837824957] 
\draw    (315.01,255) -- (329.99,280) ;
\draw [shift={(329.99,280)}, rotate = 59.07] [color={rgb, 255:red, 0; green, 0; blue, 0 }  ][line width=0.75]    (-3.35,0) -- (3.35,0)(0,3.35) -- (0,-3.35)   ;
\draw [shift={(315.01,255)}, rotate = 59.07] [color={rgb, 255:red, 0; green, 0; blue, 0 }  ][fill={rgb, 255:red, 0; green, 0; blue, 0 }  ][line width=0.75]      (0, 0) circle [x radius= 2.01, y radius= 2.01]   ;
%Shape: Circle [id:dp4922651107547389] 
\draw  [color={rgb, 255:red, 155; green, 155; blue, 155 }  ,draw opacity=1 ][fill={rgb, 255:red, 255; green, 255; blue, 255 }  ,fill opacity=1 ] (386.99,280) .. controls (386.99,278.34) and (388.33,277) .. (389.99,277) .. controls (391.65,277) and (392.99,278.34) .. (392.99,280) .. controls (392.99,281.66) and (391.65,283) .. (389.99,283) .. controls (388.33,283) and (386.99,281.66) .. (386.99,280) -- cycle ;
%Straight Lines [id:da6438818614261652] 
\draw [color={rgb, 255:red, 155; green, 155; blue, 155 }  ,draw opacity=1 ]   (375.01,255) -- (389.99,280) ;
\draw [shift={(389.99,280)}, rotate = 59.07] [color={rgb, 255:red, 155; green, 155; blue, 155 }  ,draw opacity=1 ][line width=0.75]    (-3.35,0) -- (3.35,0)(0,3.35) -- (0,-3.35)   ;
%Curve Lines [id:da5901809892630511] 
\draw [color={rgb, 255:red, 0; green, 0; blue, 0 }  ,draw opacity=1 ]   (375.01,255) .. controls (371.05,272.35) and (373.05,292.35) .. (380.05,305) ;
\draw [shift={(375.01,255)}, rotate = 102.87] [color={rgb, 255:red, 0; green, 0; blue, 0 }  ,draw opacity=1 ][fill={rgb, 255:red, 0; green, 0; blue, 0 }  ,fill opacity=1 ][line width=0.75]      (0, 0) circle [x radius= 2.01, y radius= 2.01]   ;

% Text Node
\draw (96,3) node [anchor=north west][inner sep=0.75pt]  [font=\small] [align=left] {start};
% Text Node
\draw (155,3) node [anchor=north west][inner sep=0.75pt]  [font=\small] [align=left] {new node};
% Text Node
\draw (253,3) node [anchor=north west][inner sep=0.75pt]  [font=\small] [align=left] {check};
% Text Node
\draw (191,41) node [anchor=north west][inner sep=0.75pt]  [font=\scriptsize] [align=left] {(CF node)};
% Text Node
\draw (183,104) node [anchor=north west][inner sep=0.75pt]  [font=\scriptsize] [align=left] {\begin{minipage}[lt]{40.01pt}\setlength\topsep{0pt}
(adversarial
\begin{flushright}
node)
\end{flushright}

\end{minipage}};
% Text Node
\draw (255,41) node [anchor=north west][inner sep=0.75pt]  [font=\scriptsize] [align=left] {(none)};
% Text Node
\draw (309,3) node [anchor=north west][inner sep=0.75pt]  [font=\small] [align=left] {relabel};
% Text Node
\draw (314,41) node [anchor=north west][inner sep=0.75pt]  [font=\scriptsize] [align=left] {(none)};
% Text Node
\draw (376,3) node [anchor=north west][inner sep=0.75pt]  [font=\small] [align=left] {end};
% Text Node
\draw (255,104) node [anchor=north west][inner sep=0.75pt]  [font=\scriptsize] [align=left] {(none)};
% Text Node
\draw (314,104) node [anchor=north west][inner sep=0.75pt]  [font=\scriptsize] [align=left] {(none)};
% Text Node
\draw (191,186) node [anchor=north west][inner sep=0.75pt]  [font=\scriptsize] [align=left] {(CT node)};
% Text Node
\draw (57,33) node [anchor=north west][inner sep=0.75pt]  [font=\small] [align=left] {(a)};
% Text Node
\draw (57,83) node [anchor=north west][inner sep=0.75pt]  [font=\small] [align=left] {(b)};
% Text Node
\draw (57,173) node [anchor=north west][inner sep=0.75pt]  [font=\small] [align=left] {(c)};
% Text Node
\draw (184,347) node [anchor=north west][inner sep=0.75pt]  [font=\small] [align=left] {CT};
% Text Node
\draw (234,347) node [anchor=north west][inner sep=0.75pt]  [font=\small] [align=left] {CF};
% Text Node
\draw (284,347) node [anchor=north west][inner sep=0.75pt]  [font=\small] [align=left] {PF};
% Text Node
\draw (191,306) node [anchor=north west][inner sep=0.75pt]  [font=\scriptsize] [align=left] {(CT node)};
% Text Node
\draw (57,273) node [anchor=north west][inner sep=0.75pt]  [font=\small] [align=left] {(d)};
% Text Node
\draw (314,186) node [anchor=north west][inner sep=0.75pt]  [font=\scriptsize] [align=left] {(none)};

\end{tikzpicture}

%% file: Visuals/evolution.tex
\begin{tikzpicture}[x=0.75pt,y=0.75pt,yscale=-1,xscale=1]
%uncomment if require: \path (0,185); %set diagram left start at 0, and has height of 185

%Curve Lines [id:da6168173348428867] 
\draw [color={rgb, 255:red, 155; green, 155; blue, 155 }  ,draw opacity=1 ] [dash pattern={on 3.75pt off 2.25pt}]  (340.02,10) .. controls (330.2,38) and (345.8,151.2) .. (355,160) ;
\draw [shift={(355,160)}, rotate = 43.73] [color={rgb, 255:red, 155; green, 155; blue, 155 }  ,draw opacity=1 ][fill={rgb, 255:red, 155; green, 155; blue, 155 }  ,fill opacity=1 ][line width=0.75]      (0, 0) circle [x radius= 2.01, y radius= 2.01]   ;
%Curve Lines [id:da4715071566950376] 
\draw [color={rgb, 255:red, 74; green, 144; blue, 226 }  ,draw opacity=1 ] [dash pattern={on 3.75pt off 2.25pt}]  (220.02,10) .. controls (210.2,38) and (225.8,151.2) .. (235,160) ;
\draw [shift={(235,160)}, rotate = 43.73] [color={rgb, 255:red, 74; green, 144; blue, 226 }  ,draw opacity=1 ][fill={rgb, 255:red, 74; green, 144; blue, 226 }  ,fill opacity=1 ][line width=0.75]      (0, 0) circle [x radius= 2.01, y radius= 2.01]   ;
%Straight Lines [id:da07768866415100839] 
\draw [color={rgb, 255:red, 74; green, 144; blue, 226 }  ,draw opacity=1 ] [dash pattern={on 3.75pt off 2.25pt}]  (235.02,100) -- (235,160) ;
%Straight Lines [id:da7509046045488378] 
\draw    (520,40) -- (520,70) ;
\draw [shift={(520,70)}, rotate = 90] [color={rgb, 255:red, 0; green, 0; blue, 0 }  ][fill={rgb, 255:red, 0; green, 0; blue, 0 }  ][line width=0.75]      (0, 0) circle [x radius= 2.01, y radius= 2.01]   ;
%Straight Lines [id:da41992057808588723] 
\draw [color={rgb, 255:red, 74; green, 144; blue, 226 }  ,draw opacity=1 ] [dash pattern={on 3.75pt off 2.25pt}]  (265.02,130) -- (235,160) ;
%Straight Lines [id:da28411391757893933] 
\draw  [dash pattern={on 1.5pt off 1.5pt}]  (490.9,71.8) -- (505,100) ;
\draw [shift={(490,70)}, rotate = 63.43] [color={rgb, 255:red, 0; green, 0; blue, 0 }  ][line width=0.75]      (0, 0) circle [x radius= 3.02, y radius= 3.02]   ;
%Straight Lines [id:da00020455468867019988] 
\draw    (520,10) -- (520,40) ;
\draw [shift={(520,40)}, rotate = 90] [color={rgb, 255:red, 0; green, 0; blue, 0 }  ][fill={rgb, 255:red, 0; green, 0; blue, 0 }  ][line width=0.75]      (0, 0) circle [x radius= 2.01, y radius= 2.01]   ;
\draw [shift={(520,10)}, rotate = 90] [color={rgb, 255:red, 0; green, 0; blue, 0 }  ][fill={rgb, 255:red, 0; green, 0; blue, 0 }  ][line width=0.75]      (0, 0) circle [x radius= 2.01, y radius= 2.01]   ;
%Curve Lines [id:da4324117319128139] 
\draw    (490,10) .. controls (502.6,13.2) and (514.6,25.6) .. (520,40) ;
\draw [shift={(490,10)}, rotate = 14.25] [color={rgb, 255:red, 0; green, 0; blue, 0 }  ][fill={rgb, 255:red, 0; green, 0; blue, 0 }  ][line width=0.75]      (0, 0) circle [x radius= 2.01, y radius= 2.01]   ;
%Straight Lines [id:da21613767722884014] 
\draw  [dash pattern={on 1.5pt off 1.5pt}]  (460,10) -- (490,40) ;
\draw [shift={(490,40)}, rotate = 45] [color={rgb, 255:red, 0; green, 0; blue, 0 }  ][line width=0.75]    (-3.35,0) -- (3.35,0)(0,3.35) -- (0,-3.35)   ;
\draw [shift={(460,10)}, rotate = 45] [color={rgb, 255:red, 0; green, 0; blue, 0 }  ][fill={rgb, 255:red, 0; green, 0; blue, 0 }  ][line width=0.75]      (0, 0) circle [x radius= 2.01, y radius= 2.01]   ;
%Straight Lines [id:da3377330393640985] 
\draw  [dash pattern={on 1.5pt off 1.5pt}]  (520,10) -- (491.19,38.81) ;
\draw [shift={(490,40)}, rotate = 135] [color={rgb, 255:red, 0; green, 0; blue, 0 }  ][line width=0.75]      (0, 0) circle [x radius= 2.68, y radius= 2.68]   ;
%Curve Lines [id:da8833343501684076] 
\draw  [dash pattern={on 1.5pt off 1.5pt}]  (460,10) .. controls (460.98,33.91) and (470.6,53.98) .. (488.87,69.08) ;
\draw [shift={(490,70)}, rotate = 38.66] [color={rgb, 255:red, 0; green, 0; blue, 0 }  ][line width=0.75]      (0, 0) circle [x radius= 2.68, y radius= 2.68]   ;
%Straight Lines [id:da464070012467789] 
\draw  [dash pattern={on 1.5pt off 1.5pt}]  (460,10) -- (489.1,68.2) ;
\draw [shift={(490,70)}, rotate = 63.43] [color={rgb, 255:red, 0; green, 0; blue, 0 }  ][line width=0.75]      (0, 0) circle [x radius= 3.02, y radius= 3.02]   ;
%Curve Lines [id:da05233488037073619] 
\draw    (520,10) .. controls (531.8,21.6) and (533.8,53.2) .. (520,70) ;
%Straight Lines [id:da9393740799018679] 
\draw    (520,70) -- (505,100) ;
\draw [shift={(505,100)}, rotate = 116.57] [color={rgb, 255:red, 0; green, 0; blue, 0 }  ][fill={rgb, 255:red, 0; green, 0; blue, 0 }  ][line width=0.75]      (0, 0) circle [x radius= 2.01, y radius= 2.01]   ;
%Curve Lines [id:da08681294502439729] 
\draw  [dash pattern={on 1.5pt off 1.5pt}]  (476.7,100.49) .. controls (488.67,104.32) and (499.84,116.25) .. (505,130) ;
\draw [shift={(475,100)}, rotate = 14.25] [color={rgb, 255:red, 0; green, 0; blue, 0 }  ][line width=0.75]      (0, 0) circle [x radius= 2.68, y radius= 2.68]   ;
%Curve Lines [id:da04643113895519546] 
\draw  [dash pattern={on 1.5pt off 1.5pt}]  (475.73,101.97) .. controls (480.61,114.34) and (490.37,125.06) .. (505,130) ;
\draw [shift={(475,100)}, rotate = 70.79] [color={rgb, 255:red, 0; green, 0; blue, 0 }  ][line width=0.75]      (0, 0) circle [x radius= 3.02, y radius= 3.02]   ;
%Straight Lines [id:da2195744271897333] 
\draw [fill={rgb, 255:red, 0; green, 0; blue, 0 }  ,fill opacity=1 ] [dash pattern={on 1.5pt off 1.5pt}]  (490,70) -- (475,100) ;
\draw [shift={(475,100)}, rotate = 116.57] [color={rgb, 255:red, 0; green, 0; blue, 0 }  ][line width=0.75]    (-3.35,0) -- (3.35,0)(0,3.35) -- (0,-3.35)   ;
\draw [shift={(490,70)}, rotate = 116.57] [color={rgb, 255:red, 0; green, 0; blue, 0 }  ][line width=0.75]    (-3.35,0) -- (3.35,0)(0,3.35) -- (0,-3.35)   ;
%Curve Lines [id:da29170910807685124] 
\draw  [dash pattern={on 1.5pt off 1.5pt}]  (460,10) .. controls (456.75,28.34) and (465,85.71) .. (473.76,98.5) ;
\draw [shift={(475,100)}, rotate = 43.73] [color={rgb, 255:red, 0; green, 0; blue, 0 }  ][line width=0.75]      (0, 0) circle [x radius= 3.02, y radius= 3.02]   ;
%Straight Lines [id:da2579324074541084] 
\draw  [dash pattern={on 1.5pt off 1.5pt}]  (490,42.02) -- (490,67.99) ;
\draw [shift={(490,70)}, rotate = 90] [color={rgb, 255:red, 0; green, 0; blue, 0 }  ][line width=0.75]      (0, 0) circle [x radius= 3.02, y radius= 3.02]   ;
\draw [shift={(490,40)}, rotate = 90] [color={rgb, 255:red, 0; green, 0; blue, 0 }  ][line width=0.75]      (0, 0) circle [x radius= 3.02, y radius= 3.02]   ;
%Straight Lines [id:da6053974956469127] 
\draw  [dash pattern={on 1.5pt off 1.5pt}]  (490,10) -- (490,37.99) ;
\draw [shift={(490,40)}, rotate = 90] [color={rgb, 255:red, 0; green, 0; blue, 0 }  ][line width=0.75]      (0, 0) circle [x radius= 3.02, y radius= 3.02]   ;
%Straight Lines [id:da2386256776500264] 
\draw    (160.02,40) -- (160.02,70) ;
\draw [shift={(160.02,70)}, rotate = 90] [color={rgb, 255:red, 0; green, 0; blue, 0 }  ][fill={rgb, 255:red, 0; green, 0; blue, 0 }  ][line width=0.75]      (0, 0) circle [x radius= 2.01, y radius= 2.01]   ;
%Straight Lines [id:da29096912497350724] 
\draw    (130.02,70) -- (145.02,100) ;
\draw [shift={(130.02,70)}, rotate = 63.43] [color={rgb, 255:red, 0; green, 0; blue, 0 }  ][fill={rgb, 255:red, 0; green, 0; blue, 0 }  ][line width=0.75]      (0, 0) circle [x radius= 2.01, y radius= 2.01]   ;
%Straight Lines [id:da3272699457687861] 
\draw    (160.02,10) -- (160.02,40) ;
\draw [shift={(160.02,40)}, rotate = 90] [color={rgb, 255:red, 0; green, 0; blue, 0 }  ][fill={rgb, 255:red, 0; green, 0; blue, 0 }  ][line width=0.75]      (0, 0) circle [x radius= 2.01, y radius= 2.01]   ;
\draw [shift={(160.02,10)}, rotate = 90] [color={rgb, 255:red, 0; green, 0; blue, 0 }  ][fill={rgb, 255:red, 0; green, 0; blue, 0 }  ][line width=0.75]      (0, 0) circle [x radius= 2.01, y radius= 2.01]   ;
%Curve Lines [id:da6489819231850807] 
\draw    (130.02,10) .. controls (142.62,13.2) and (154.62,25.6) .. (160.02,40) ;
\draw [shift={(130.02,10)}, rotate = 14.25] [color={rgb, 255:red, 0; green, 0; blue, 0 }  ][fill={rgb, 255:red, 0; green, 0; blue, 0 }  ][line width=0.75]      (0, 0) circle [x radius= 2.01, y radius= 2.01]   ;
%Straight Lines [id:da5407305848542356] 
\draw    (100.02,10) -- (130.02,40) ;
\draw [shift={(130.02,40)}, rotate = 45] [color={rgb, 255:red, 0; green, 0; blue, 0 }  ][fill={rgb, 255:red, 0; green, 0; blue, 0 }  ][line width=0.75]      (0, 0) circle [x radius= 2.01, y radius= 2.01]   ;
\draw [shift={(100.02,10)}, rotate = 45] [color={rgb, 255:red, 0; green, 0; blue, 0 }  ][fill={rgb, 255:red, 0; green, 0; blue, 0 }  ][line width=0.75]      (0, 0) circle [x radius= 2.01, y radius= 2.01]   ;
%Straight Lines [id:da8866365309796828] 
\draw    (160.02,10) -- (130.74,39.29) ;
\draw [shift={(130.02,40)}, rotate = 135] [color={rgb, 255:red, 0; green, 0; blue, 0 }  ][line width=0.75]      (0, 0) circle [x radius= 2.01, y radius= 2.01]   ;
%Curve Lines [id:da19546802671688368] 
\draw    (100.02,10) .. controls (101.02,34.4) and (111.02,54.8) .. (130.02,70) ;
%Straight Lines [id:da5136770387555476] 
\draw    (100.02,10) -- (130.02,70) ;
%Curve Lines [id:da7361807595418063] 
\draw    (160.02,10) .. controls (171.82,21.6) and (173.82,53.2) .. (160.02,70) ;
%Straight Lines [id:da41379161633861183] 
\draw    (160.02,70) -- (145.02,100) ;
\draw [shift={(145.02,100)}, rotate = 116.57] [color={rgb, 255:red, 0; green, 0; blue, 0 }  ][fill={rgb, 255:red, 0; green, 0; blue, 0 }  ][line width=0.75]      (0, 0) circle [x radius= 2.01, y radius= 2.01]   ;
%Curve Lines [id:da056740334957146255] 
\draw    (115.02,100) .. controls (127.62,103.2) and (139.62,115.6) .. (145.02,130) ;
%Curve Lines [id:da8623255781771918] 
\draw    (115.02,100) .. controls (119.62,113.2) and (129.62,124.8) .. (145.02,130) ;
%Straight Lines [id:da3876857443282654] 
\draw    (130.02,70) -- (115.02,100) ;
\draw [shift={(115.02,100)}, rotate = 116.57] [color={rgb, 255:red, 0; green, 0; blue, 0 }  ][fill={rgb, 255:red, 0; green, 0; blue, 0 }  ][line width=0.75]      (0, 0) circle [x radius= 2.01, y radius= 2.01]   ;
%Curve Lines [id:da7502563840808028] 
\draw    (100.02,10) .. controls (96.62,29.2) and (105.82,91.2) .. (115.02,100) ;
%Straight Lines [id:da12302403095486403] 
\draw    (130.02,41.01) -- (130.02,70) ;
\draw [shift={(130.02,40)}, rotate = 90] [color={rgb, 255:red, 0; green, 0; blue, 0 }  ][line width=0.75]      (0, 0) circle [x radius= 2.01, y radius= 2.01]   ;
%Straight Lines [id:da1773454501223859] 
\draw    (130.02,10) -- (130.02,38.99) ;
\draw [shift={(130.02,40)}, rotate = 90] [color={rgb, 255:red, 0; green, 0; blue, 0 }  ][line width=0.75]      (0, 0) circle [x radius= 2.01, y radius= 2.01]   ;
%Shape: Circle [id:dp41879959928558663] 
\draw  [color={rgb, 255:red, 0; green, 0; blue, 0 }  ,draw opacity=1 ][fill={rgb, 255:red, 255; green, 255; blue, 255 }  ,fill opacity=1 ] (127.77,40) .. controls (127.77,38.76) and (128.78,37.75) .. (130.02,37.75) .. controls (131.27,37.75) and (132.27,38.76) .. (132.27,40) .. controls (132.27,41.24) and (131.27,42.25) .. (130.02,42.25) .. controls (128.78,42.25) and (127.77,41.24) .. (127.77,40) -- cycle ;
%Straight Lines [id:da71824712603412] 
\draw    (160.02,70) -- (115.02,100) ;
%Curve Lines [id:da13576826993754465] 
\draw    (160.02,70) .. controls (160.7,88.8) and (157.1,113.2) .. (145.02,130) ;
\draw [shift={(145.02,130)}, rotate = 125.71] [color={rgb, 255:red, 0; green, 0; blue, 0 }  ][fill={rgb, 255:red, 0; green, 0; blue, 0 }  ][line width=0.75]      (0, 0) circle [x radius= 2.01, y radius= 2.01]   ;
%Curve Lines [id:da5849789006066418] 
\draw    (520,70) .. controls (520.68,88.8) and (517.08,113.2) .. (505,130) ;
\draw [shift={(505,130)}, rotate = 125.71] [color={rgb, 255:red, 0; green, 0; blue, 0 }  ][fill={rgb, 255:red, 0; green, 0; blue, 0 }  ][line width=0.75]      (0, 0) circle [x radius= 2.01, y radius= 2.01]   ;
%Straight Lines [id:da680162199936584] 
\draw    (280.02,40) -- (280.02,70) ;
\draw [shift={(280.02,70)}, rotate = 90] [color={rgb, 255:red, 0; green, 0; blue, 0 }  ][fill={rgb, 255:red, 0; green, 0; blue, 0 }  ][line width=0.75]      (0, 0) circle [x radius= 2.01, y radius= 2.01]   ;
%Straight Lines [id:da7831438058264216] 
\draw    (250.02,70) -- (265.02,100) ;
\draw [shift={(250.02,70)}, rotate = 63.43] [color={rgb, 255:red, 0; green, 0; blue, 0 }  ][fill={rgb, 255:red, 0; green, 0; blue, 0 }  ][line width=0.75]      (0, 0) circle [x radius= 2.01, y radius= 2.01]   ;
%Straight Lines [id:da9595112112133004] 
\draw    (280.02,10) -- (280.02,40) ;
\draw [shift={(280.02,40)}, rotate = 90] [color={rgb, 255:red, 0; green, 0; blue, 0 }  ][fill={rgb, 255:red, 0; green, 0; blue, 0 }  ][line width=0.75]      (0, 0) circle [x radius= 2.01, y radius= 2.01]   ;
\draw [shift={(280.02,10)}, rotate = 90] [color={rgb, 255:red, 0; green, 0; blue, 0 }  ][fill={rgb, 255:red, 0; green, 0; blue, 0 }  ][line width=0.75]      (0, 0) circle [x radius= 2.01, y radius= 2.01]   ;
%Curve Lines [id:da7171993591518285] 
\draw    (250.02,10) .. controls (262.62,13.2) and (274.62,25.6) .. (280.02,40) ;
\draw [shift={(250.02,10)}, rotate = 14.25] [color={rgb, 255:red, 0; green, 0; blue, 0 }  ][fill={rgb, 255:red, 0; green, 0; blue, 0 }  ][line width=0.75]      (0, 0) circle [x radius= 2.01, y radius= 2.01]   ;
%Straight Lines [id:da3109638722961878] 
\draw    (220.02,10) -- (250.02,40) ;
\draw [shift={(250.02,40)}, rotate = 45] [color={rgb, 255:red, 0; green, 0; blue, 0 }  ][fill={rgb, 255:red, 0; green, 0; blue, 0 }  ][line width=0.75]      (0, 0) circle [x radius= 2.01, y radius= 2.01]   ;
\draw [shift={(220.02,10)}, rotate = 45] [color={rgb, 255:red, 0; green, 0; blue, 0 }  ][fill={rgb, 255:red, 0; green, 0; blue, 0 }  ][line width=0.75]      (0, 0) circle [x radius= 2.01, y radius= 2.01]   ;
%Straight Lines [id:da6665557993571607] 
\draw    (280.02,10) -- (250.74,39.29) ;
\draw [shift={(250.02,40)}, rotate = 135] [color={rgb, 255:red, 0; green, 0; blue, 0 }  ][line width=0.75]      (0, 0) circle [x radius= 2.01, y radius= 2.01]   ;
%Curve Lines [id:da5191711847948309] 
\draw    (220.02,10) .. controls (221.02,34.4) and (231.02,54.8) .. (250.02,70) ;
%Straight Lines [id:da5316916312694654] 
\draw    (220.02,10) -- (250.02,70) ;
%Curve Lines [id:da32166376709805544] 
\draw    (280.02,10) .. controls (291.82,21.6) and (293.82,53.2) .. (280.02,70) ;
%Straight Lines [id:da09035231832440926] 
\draw    (280.02,70) -- (265.02,100) ;
\draw [shift={(265.02,100)}, rotate = 116.57] [color={rgb, 255:red, 0; green, 0; blue, 0 }  ][fill={rgb, 255:red, 0; green, 0; blue, 0 }  ][line width=0.75]      (0, 0) circle [x radius= 2.01, y radius= 2.01]   ;
%Curve Lines [id:da015711270976944247] 
\draw    (235.02,100) .. controls (247.62,103.2) and (259.62,115.6) .. (265.02,130) ;
%Curve Lines [id:da7388638971486765] 
\draw    (235.02,100) .. controls (239.62,113.2) and (249.62,124.8) .. (265.02,130) ;
%Straight Lines [id:da17127039121184262] 
\draw    (250.02,70) -- (235.02,100) ;
\draw [shift={(235.02,100)}, rotate = 116.57] [color={rgb, 255:red, 0; green, 0; blue, 0 }  ][fill={rgb, 255:red, 0; green, 0; blue, 0 }  ][line width=0.75]      (0, 0) circle [x radius= 2.01, y radius= 2.01]   ;
%Curve Lines [id:da9122532080025894] 
\draw    (220.02,10) .. controls (216.62,29.2) and (225.82,91.2) .. (235.02,100) ;
%Straight Lines [id:da4254029870756003] 
\draw    (250.02,41.01) -- (250.02,70) ;
\draw [shift={(250.02,40)}, rotate = 90] [color={rgb, 255:red, 0; green, 0; blue, 0 }  ][line width=0.75]      (0, 0) circle [x radius= 2.01, y radius= 2.01]   ;
%Straight Lines [id:da05713499241730091] 
\draw    (250.02,10) -- (250.02,38.99) ;
\draw [shift={(250.02,40)}, rotate = 90] [color={rgb, 255:red, 0; green, 0; blue, 0 }  ][line width=0.75]      (0, 0) circle [x radius= 2.01, y radius= 2.01]   ;
%Shape: Circle [id:dp7624052996010411] 
\draw  [color={rgb, 255:red, 0; green, 0; blue, 0 }  ,draw opacity=1 ][fill={rgb, 255:red, 255; green, 255; blue, 255 }  ,fill opacity=1 ] (247.77,40) .. controls (247.77,38.76) and (248.78,37.75) .. (250.02,37.75) .. controls (251.27,37.75) and (252.27,38.76) .. (252.27,40) .. controls (252.27,41.24) and (251.27,42.25) .. (250.02,42.25) .. controls (248.78,42.25) and (247.77,41.24) .. (247.77,40) -- cycle ;
%Straight Lines [id:da7004062730191072] 
\draw    (280.02,70) -- (235.02,100) ;
%Curve Lines [id:da12802971771203098] 
\draw    (280.02,70) .. controls (280.7,88.8) and (277.1,113.2) .. (265.02,130) ;
\draw [shift={(265.02,130)}, rotate = 125.71] [color={rgb, 255:red, 0; green, 0; blue, 0 }  ][fill={rgb, 255:red, 0; green, 0; blue, 0 }  ][line width=0.75]      (0, 0) circle [x radius= 2.01, y radius= 2.01]   ;
%Straight Lines [id:da02007539070929809] 
\draw [color={rgb, 255:red, 155; green, 155; blue, 155 }  ,draw opacity=1 ] [dash pattern={on 3.75pt off 2.25pt}]  (385.02,130) -- (355,160) ;
%Straight Lines [id:da009902273037593723] 
\draw    (400.02,40) -- (400.02,70) ;
\draw [shift={(400.02,70)}, rotate = 90] [color={rgb, 255:red, 0; green, 0; blue, 0 }  ][fill={rgb, 255:red, 0; green, 0; blue, 0 }  ][line width=0.75]      (0, 0) circle [x radius= 2.01, y radius= 2.01]   ;
%Straight Lines [id:da4109502243102825] 
\draw    (370.02,70) -- (385.02,100) ;
\draw [shift={(370.02,70)}, rotate = 63.43] [color={rgb, 255:red, 0; green, 0; blue, 0 }  ][fill={rgb, 255:red, 0; green, 0; blue, 0 }  ][line width=0.75]      (0, 0) circle [x radius= 2.01, y radius= 2.01]   ;
%Straight Lines [id:da5602132058912654] 
\draw    (400.02,10) -- (400.02,40) ;
\draw [shift={(400.02,40)}, rotate = 90] [color={rgb, 255:red, 0; green, 0; blue, 0 }  ][fill={rgb, 255:red, 0; green, 0; blue, 0 }  ][line width=0.75]      (0, 0) circle [x radius= 2.01, y radius= 2.01]   ;
\draw [shift={(400.02,10)}, rotate = 90] [color={rgb, 255:red, 0; green, 0; blue, 0 }  ][fill={rgb, 255:red, 0; green, 0; blue, 0 }  ][line width=0.75]      (0, 0) circle [x radius= 2.01, y radius= 2.01]   ;
%Curve Lines [id:da06465050872377409] 
\draw    (370.02,10) .. controls (382.62,13.2) and (394.62,25.6) .. (400.02,40) ;
\draw [shift={(370.02,10)}, rotate = 14.25] [color={rgb, 255:red, 0; green, 0; blue, 0 }  ][fill={rgb, 255:red, 0; green, 0; blue, 0 }  ][line width=0.75]      (0, 0) circle [x radius= 2.01, y radius= 2.01]   ;
%Straight Lines [id:da523293012361169] 
\draw    (340.02,10) -- (370.02,40) ;
\draw [shift={(370.02,40)}, rotate = 45] [color={rgb, 255:red, 0; green, 0; blue, 0 }  ][fill={rgb, 255:red, 0; green, 0; blue, 0 }  ][line width=0.75]      (0, 0) circle [x radius= 2.01, y radius= 2.01]   ;
\draw [shift={(340.02,10)}, rotate = 45] [color={rgb, 255:red, 0; green, 0; blue, 0 }  ][fill={rgb, 255:red, 0; green, 0; blue, 0 }  ][line width=0.75]      (0, 0) circle [x radius= 2.01, y radius= 2.01]   ;
%Straight Lines [id:da7261602638255887] 
\draw    (400.02,10) -- (370.74,39.29) ;
\draw [shift={(370.02,40)}, rotate = 135] [color={rgb, 255:red, 0; green, 0; blue, 0 }  ][line width=0.75]      (0, 0) circle [x radius= 2.01, y radius= 2.01]   ;
%Curve Lines [id:da1381485967106545] 
\draw    (340.02,10) .. controls (341.02,34.4) and (351.02,54.8) .. (370.02,70) ;
%Straight Lines [id:da3350095464782803] 
\draw    (340.02,10) -- (370.02,70) ;
%Curve Lines [id:da7765340082382192] 
\draw    (400.02,10) .. controls (411.82,21.6) and (413.82,53.2) .. (400.02,70) ;
%Straight Lines [id:da1798903183229471] 
\draw    (400.02,70) -- (385.02,100) ;
\draw [shift={(385.02,100)}, rotate = 116.57] [color={rgb, 255:red, 0; green, 0; blue, 0 }  ][fill={rgb, 255:red, 0; green, 0; blue, 0 }  ][line width=0.75]      (0, 0) circle [x radius= 2.01, y radius= 2.01]   ;
%Curve Lines [id:da7562869951985436] 
\draw    (355.02,100) .. controls (367.62,103.2) and (379.62,115.6) .. (385.02,130) ;
%Curve Lines [id:da17591481838688305] 
\draw    (355.02,100) .. controls (359.62,113.2) and (369.62,124.8) .. (385.02,130) ;
%Straight Lines [id:da7836312865172809] 
\draw    (370.02,70) -- (355.02,100) ;
\draw [shift={(355.02,100)}, rotate = 116.57] [color={rgb, 255:red, 0; green, 0; blue, 0 }  ][fill={rgb, 255:red, 0; green, 0; blue, 0 }  ][line width=0.75]      (0, 0) circle [x radius= 2.01, y radius= 2.01]   ;
%Curve Lines [id:da05133450830594366] 
\draw    (340.02,10) .. controls (336.62,29.2) and (345.82,91.2) .. (355.02,100) ;
%Straight Lines [id:da4502647644137243] 
\draw    (370.02,41.01) -- (370.02,70) ;
\draw [shift={(370.02,40)}, rotate = 90] [color={rgb, 255:red, 0; green, 0; blue, 0 }  ][line width=0.75]      (0, 0) circle [x radius= 2.01, y radius= 2.01]   ;
%Straight Lines [id:da19949361161393664] 
\draw    (370.02,10) -- (370.02,38.99) ;
\draw [shift={(370.02,40)}, rotate = 90] [color={rgb, 255:red, 0; green, 0; blue, 0 }  ][line width=0.75]      (0, 0) circle [x radius= 2.01, y radius= 2.01]   ;
%Shape: Circle [id:dp5860794611316685] 
\draw  [color={rgb, 255:red, 0; green, 0; blue, 0 }  ,draw opacity=1 ][fill={rgb, 255:red, 255; green, 255; blue, 255 }  ,fill opacity=1 ] (367.77,40) .. controls (367.77,38.76) and (368.78,37.75) .. (370.02,37.75) .. controls (371.27,37.75) and (372.27,38.76) .. (372.27,40) .. controls (372.27,41.24) and (371.27,42.25) .. (370.02,42.25) .. controls (368.78,42.25) and (367.77,41.24) .. (367.77,40) -- cycle ;
%Straight Lines [id:da5617384754936268] 
\draw    (400.02,70) -- (355.02,100) ;
%Curve Lines [id:da35976137003427144] 
\draw    (400.02,70) .. controls (400.7,88.8) and (397.1,113.2) .. (385.02,130) ;
\draw [shift={(385.02,130)}, rotate = 125.71] [color={rgb, 255:red, 0; green, 0; blue, 0 }  ][fill={rgb, 255:red, 0; green, 0; blue, 0 }  ][line width=0.75]      (0, 0) circle [x radius= 2.01, y radius= 2.01]   ;
%Shape: Circle [id:dp6733666811947739] 
\draw  [color={rgb, 255:red, 74; green, 144; blue, 226 }  ,draw opacity=1 ] (351.02,100) .. controls (351.02,97.79) and (352.81,96) .. (355.02,96) .. controls (357.23,96) and (359.02,97.79) .. (359.02,100) .. controls (359.02,102.21) and (357.23,104) .. (355.02,104) .. controls (352.81,104) and (351.02,102.21) .. (351.02,100) -- cycle ;
%Shape: Circle [id:dp42522890994508233] 
\draw  [color={rgb, 255:red, 74; green, 144; blue, 226 }  ,draw opacity=1 ] (366.02,70) .. controls (366.02,67.79) and (367.81,66) .. (370.02,66) .. controls (372.23,66) and (374.02,67.79) .. (374.02,70) .. controls (374.02,72.21) and (372.23,74) .. (370.02,74) .. controls (367.81,74) and (366.02,72.21) .. (366.02,70) -- cycle ;
%Shape: Circle [id:dp8819382021528412] 
\draw  [color={rgb, 255:red, 74; green, 144; blue, 226 }  ,draw opacity=1 ] (366.02,40) .. controls (366.02,37.79) and (367.81,36) .. (370.02,36) .. controls (372.23,36) and (374.02,37.79) .. (374.02,40) .. controls (374.02,42.21) and (372.23,44) .. (370.02,44) .. controls (367.81,44) and (366.02,42.21) .. (366.02,40) -- cycle ;
%Straight Lines [id:da07633324590160262] 
\draw [color={rgb, 255:red, 74; green, 144; blue, 226 }  ,draw opacity=1 ]   (370.02,70) -- (355.02,100) ;
%Straight Lines [id:da5851922069423576] 
\draw [color={rgb, 255:red, 74; green, 144; blue, 226 }  ,draw opacity=1 ]   (370.02,44) -- (370.02,70) ;
%Straight Lines [id:da6883247172504803] 
\draw  [dash pattern={on 1.5pt off 1.5pt}]  (505,130) -- (474.98,160) ;
\draw [shift={(474.98,160)}, rotate = 135.02] [color={rgb, 255:red, 0; green, 0; blue, 0 }  ][line width=0.75]    (-3.35,0) -- (3.35,0)(0,3.35) -- (0,-3.35)   ;
%Straight Lines [id:da4542181599393672] 
\draw  [dash pattern={on 1.5pt off 1.5pt}]  (475,102.01) -- (474.98,157.99) ;
\draw [shift={(474.98,160)}, rotate = 90.02] [color={rgb, 255:red, 0; green, 0; blue, 0 }  ][line width=0.75]      (0, 0) circle [x radius= 3.02, y radius= 3.02]   ;
\draw [shift={(475,100)}, rotate = 90.02] [color={rgb, 255:red, 0; green, 0; blue, 0 }  ][line width=0.75]      (0, 0) circle [x radius= 3.02, y radius= 3.02]   ;
%Curve Lines [id:da4963890590553135] 
\draw  [dash pattern={on 1.5pt off 1.5pt}]  (460.02,9.67) .. controls (450.54,36.69) and (464.72,143.36) .. (473.99,158.7) ;
\draw [shift={(474.98,160)}, rotate = 43.73] [color={rgb, 255:red, 0; green, 0; blue, 0 }  ][line width=0.75]      (0, 0) circle [x radius= 2.68, y radius= 2.68]   ;
%Shape: Circle [id:dp9443616764711625] 
\draw  [color={rgb, 255:red, 74; green, 144; blue, 226 }  ,draw opacity=1 ] (351,160) .. controls (351,157.79) and (352.79,156) .. (355,156) .. controls (357.21,156) and (359,157.79) .. (359,160) .. controls (359,162.21) and (357.21,164) .. (355,164) .. controls (352.79,164) and (351,162.21) .. (351,160) -- cycle ;
%Straight Lines [id:da30763997348572725] 
\draw    (130.02,10) -- (160.02,40) ;
%Curve Lines [id:da18878419159568727] 
\draw    (130.02,10) .. controls (136.5,31.25) and (145.5,71.75) .. (145.02,100) ;
%Straight Lines [id:da2182330403191699] 
\draw    (130.02,10) -- (160.02,70) ;
%Straight Lines [id:da4986814241419648] 
\draw    (250.02,10) -- (280.02,40) ;
%Curve Lines [id:da8436631725123594] 
\draw    (250.02,10) .. controls (256.5,31.25) and (265.5,71.75) .. (265.02,100) ;
%Straight Lines [id:da36659994784575156] 
\draw    (250.02,10) -- (280.02,70) ;
%Straight Lines [id:da6684654117078115] 
\draw    (370.02,10) -- (400.02,40) ;
%Curve Lines [id:da60501556306739] 
\draw    (370.02,10) .. controls (376.5,31.25) and (385.5,71.75) .. (385.02,100) ;
%Straight Lines [id:da49921243751539224] 
\draw    (370,10) -- (400,70) ;
%Straight Lines [id:da4935563358863103] 
\draw    (490,10) -- (520,40) ;
%Curve Lines [id:da5475506562798795] 
\draw    (490,10) .. controls (496.48,31.25) and (505.48,71.75) .. (505,100) ;
%Straight Lines [id:da41784875636028396] 
\draw    (490,10) -- (520,70) ;
%Straight Lines [id:da3081345419099505] 
\draw [color={rgb, 255:red, 74; green, 144; blue, 226 }  ,draw opacity=1 ]   (355.02,100) -- (355,160) ;
%Straight Lines [id:da3022857367859072] 
\draw  [dash pattern={on 1.5pt off 1.5pt}]  (570.11,92) -- (567.56,94.29) ;
\draw [shift={(567.56,94.29)}, rotate = 138.09] [color={rgb, 255:red, 0; green, 0; blue, 0 }  ][line width=0.75]    (-3.35,0) -- (3.35,0)(0,3.35) -- (0,-3.35)   ;
%Straight Lines [id:da31037894145900324] 
\draw  [dash pattern={on 1.5pt off 1.5pt}]  (570.11,92) -- (569.06,92.94) ;
\draw [shift={(567.56,94.29)}, rotate = 138.09] [color={rgb, 255:red, 0; green, 0; blue, 0 }  ][line width=0.75]      (0, 0) circle [x radius= 3.02, y radius= 3.02]   ;
%Curve Lines [id:da4576874620020065] 
\draw  [dash pattern={on 0.75pt off 7.5pt}]  (565.5,92.16) .. controls (571.47,98.39) and (569.33,96.16) .. (567.96,94.72) ;
\draw [shift={(567.56,94.29)}, rotate = 45.79] [color={rgb, 255:red, 0; green, 0; blue, 0 }  ][line width=0.75]      (0, 0) circle [x radius= 2.68, y radius= 2.68]   ;

%Flowchart: Connector [id:dp9167679648919921] 
\draw   (565,79.35) .. controls (565,78.05) and (566.05,77) .. (567.35,77) .. controls (568.65,77) and (569.7,78.05) .. (569.7,79.35) .. controls (569.7,80.65) and (568.65,81.7) .. (567.35,81.7) .. controls (566.05,81.7) and (565,80.65) .. (565,79.35) -- cycle ;
%Flowchart: Connector [id:dp8826102425660151] 
\draw  [fill={rgb, 255:red, 0; green, 0; blue, 0 }  ,fill opacity=1 ] (565,64.35) .. controls (565,63.05) and (566.05,62) .. (567.35,62) .. controls (568.65,62) and (569.7,63.05) .. (569.7,64.35) .. controls (569.7,65.65) and (568.65,66.7) .. (567.35,66.7) .. controls (566.05,66.7) and (565,65.65) .. (565,64.35) -- cycle ;

% Text Node
\draw (74,10.5) node [anchor=north west][inner sep=0.75pt]  [font=\small] [align=left] {(a)};
% Text Node
\draw (194,10.5) node [anchor=north west][inner sep=0.75pt]  [font=\small] [align=left] {(b)};
% Text Node
\draw (314,10.5) node [anchor=north west][inner sep=0.75pt]  [font=\small] [align=left] {(c)};
% Text Node
\draw (434,10.5) node [anchor=north west][inner sep=0.75pt]  [font=\small] [align=left] {(d)};
% Text Node
\draw (575,59) node [anchor=north west][inner sep=0.75pt]  [font=\small] [align=left] {\CT};
% Text Node
\draw (575,74) node [anchor=north west][inner sep=0.75pt]  [font=\small] [align=left] {\CF};
% Text Node
\draw (575,89) node [anchor=north west][inner sep=0.75pt]  [font=\small] [align=left] {\PF};

\end{tikzpicture}

%% file: Visuals/mechanisms.tex
\begin{tikzpicture}[x=0.75pt,y=0.75pt,yscale=-1,xscale=1]
%uncomment if require: \path (0,185); %set diagram left start at 0, and has height of 185

%Straight Lines [id:da99994259563804] 
\draw [color={rgb, 255:red, 0; green, 0; blue, 0 }  ,draw opacity=1 ]   (547.98,40) -- (532.98,70) ;
%Curve Lines [id:da49033245298049455] 
\draw [color={rgb, 255:red, 74; green, 144; blue, 226 }  ,draw opacity=1 ]   (160.02,100) .. controls (172.62,103.2) and (184.62,115.6) .. (190.02,130) ;
%Curve Lines [id:da2439801366868516] 
\draw [color={rgb, 255:red, 74; green, 144; blue, 226 }  ,draw opacity=1 ]   (160.02,100) .. controls (164.62,113.2) and (174.62,124.8) .. (190.02,130) ;
%Curve Lines [id:da5325379951436591] 
\draw    (70.05,10.05) .. controls (80.7,17.88) and (85.37,57.88) .. (85.04,70.05) ;
%Curve Lines [id:da728457685395115] 
\draw    (70.05,10.05) .. controls (52.37,9.22) and (42.03,55.22) .. (55.05,70.05) ;
%Curve Lines [id:da6728970729736169] 
\draw    (70.05,10.05) .. controls (56.7,14.55) and (50.7,55.88) .. (55.05,70.05) ;
%Curve Lines [id:da028577583210633417] 
\draw    (70.05,10.05) .. controls (82.33,15.5) and (94.65,25.65) .. (100.05,40.05) ;
%Straight Lines [id:da26141495554647387] 
\draw    (70.05,10.05) -- (100.05,40.05) ;
\draw [shift={(100.05,40.05)}, rotate = 45] [color={rgb, 255:red, 0; green, 0; blue, 0 }  ][fill={rgb, 255:red, 0; green, 0; blue, 0 }  ][line width=0.75]      (0, 0) circle [x radius= 2.01, y radius= 2.01]   ;
%Curve Lines [id:da055569745669243] 
\draw    (70.05,10.05) .. controls (62.03,17.88) and (66.03,36.88) .. (70.05,40.05) ;
%Straight Lines [id:da45704641259752277] 
\draw    (70.05,40.05) -- (85.05,70.05) ;
\draw [shift={(70.05,40.05)}, rotate = 63.43] [color={rgb, 255:red, 0; green, 0; blue, 0 }  ][fill={rgb, 255:red, 0; green, 0; blue, 0 }  ][line width=0.75]      (0, 0) circle [x radius= 2.01, y radius= 2.01]   ;
%Straight Lines [id:da3615558251870882] 
\draw    (100.05,40.05) -- (85.05,70.05) ;
\draw [shift={(85.05,70.05)}, rotate = 116.57] [color={rgb, 255:red, 0; green, 0; blue, 0 }  ][fill={rgb, 255:red, 0; green, 0; blue, 0 }  ][line width=0.75]      (0, 0) circle [x radius= 2.01, y radius= 2.01]   ;
%Curve Lines [id:da21680969987273502] 
\draw    (85.05,70.05) .. controls (93.33,73.83) and (100,87.83) .. (100.05,100.05) ;
%Curve Lines [id:da2028822689003883] 
\draw    (85.05,70.05) .. controls (87.67,82.17) and (88.67,90.5) .. (100.05,100.05) ;
%Straight Lines [id:da8273981323731896] 
\draw    (70.05,40.05) -- (55.05,70.05) ;
\draw [shift={(55.05,70.05)}, rotate = 116.57] [color={rgb, 255:red, 0; green, 0; blue, 0 }  ][fill={rgb, 255:red, 0; green, 0; blue, 0 }  ][line width=0.75]      (0, 0) circle [x radius= 2.01, y radius= 2.01]   ;
%Straight Lines [id:da45426576665284835] 
\draw    (70.05,11.06) -- (70.05,40.05) ;
\draw [shift={(70.05,10.05)}, rotate = 90] [color={rgb, 255:red, 0; green, 0; blue, 0 }  ][line width=0.75]      (0, 0) circle [x radius= 2.01, y radius= 2.01]   ;
%Shape: Circle [id:dp9653012169061176] 
\draw  [color={rgb, 255:red, 0; green, 0; blue, 0 }  ,draw opacity=1 ][fill={rgb, 255:red, 255; green, 255; blue, 255 }  ,fill opacity=1 ] (67.8,10.05) .. controls (67.8,8.81) and (68.81,7.8) .. (70.05,7.8) .. controls (71.3,7.8) and (72.3,8.81) .. (72.3,10.05) .. controls (72.3,11.29) and (71.3,12.3) .. (70.05,12.3) .. controls (68.81,12.3) and (67.8,11.29) .. (67.8,10.05) -- cycle ;
%Straight Lines [id:da05894428425751397] 
\draw    (100.05,40.05) -- (70.03,100.05) ;
%Curve Lines [id:da5837716133819082] 
\draw    (100.05,40.05) .. controls (106.37,43.55) and (112.13,83.25) .. (100.05,100.05) ;
\draw [shift={(100.05,100.05)}, rotate = 125.71] [color={rgb, 255:red, 0; green, 0; blue, 0 }  ][fill={rgb, 255:red, 0; green, 0; blue, 0 }  ][line width=0.75]      (0, 0) circle [x radius= 2.01, y radius= 2.01]   ;
%Flowchart: Connector [id:dp03246018743880308] 
\draw   (615,74.35) .. controls (615,73.05) and (616.05,72) .. (617.35,72) .. controls (618.65,72) and (619.7,73.05) .. (619.7,74.35) .. controls (619.7,75.65) and (618.65,76.7) .. (617.35,76.7) .. controls (616.05,76.7) and (615,75.65) .. (615,74.35) -- cycle ;
%Flowchart: Connector [id:dp12105795526861252] 
\draw  [fill={rgb, 255:red, 0; green, 0; blue, 0 }  ,fill opacity=1 ] (615,59.35) .. controls (615,58.05) and (616.05,57) .. (617.35,57) .. controls (618.65,57) and (619.7,58.05) .. (619.7,59.35) .. controls (619.7,60.65) and (618.65,61.7) .. (617.35,61.7) .. controls (616.05,61.7) and (615,60.65) .. (615,59.35) -- cycle ;
%Curve Lines [id:da05687228466293581] 
\draw    (70.05,10.05) .. controls (77.7,16.22) and (75.03,37.22) .. (70.05,40.05) ;
%Shape: Circle [id:dp2091260588205337] 
\draw  [color={rgb, 255:red, 0; green, 0; blue, 0 }  ,draw opacity=1 ][fill={rgb, 255:red, 255; green, 255; blue, 255 }  ,fill opacity=1 ] (67.8,10.05) .. controls (67.8,8.81) and (68.81,7.8) .. (70.05,7.8) .. controls (71.3,7.8) and (72.3,8.81) .. (72.3,10.05) .. controls (72.3,11.29) and (71.3,12.3) .. (70.05,12.3) .. controls (68.81,12.3) and (67.8,11.29) .. (67.8,10.05) -- cycle ;
%Straight Lines [id:da9519637650481628] 
\draw    (55.05,70.05) -- (70.05,100.05) ;
\draw [shift={(70.05,100.05)}, rotate = 63.43] [color={rgb, 255:red, 0; green, 0; blue, 0 }  ][fill={rgb, 255:red, 0; green, 0; blue, 0 }  ][line width=0.75]      (0, 0) circle [x radius= 2.01, y radius= 2.01]   ;
%Straight Lines [id:da3323323186290018] 
\draw    (70.05,40.05) -- (70.05,100.05) ;
%Straight Lines [id:da5274488264592523] 
\draw    (55.05,70.05) -- (40.05,100.05) ;
\draw [shift={(40.05,100.05)}, rotate = 116.57] [color={rgb, 255:red, 0; green, 0; blue, 0 }  ][fill={rgb, 255:red, 0; green, 0; blue, 0 }  ][line width=0.75]      (0, 0) circle [x radius= 2.01, y radius= 2.01]   ;
%Curve Lines [id:da6343867662943565] 
\draw    (40.05,100.05) .. controls (52.65,103.25) and (64.65,115.65) .. (70.05,130.05) ;
%Curve Lines [id:da6689422269492035] 
\draw [color={rgb, 255:red, 0; green, 0; blue, 0 }  ,draw opacity=1 ]   (40.05,100.05) .. controls (44.65,113.25) and (54.65,124.85) .. (70.05,130.05) ;
%Curve Lines [id:da4352454026958529] 
\draw [color={rgb, 255:red, 74; green, 144; blue, 226 }  ,draw opacity=1 ]   (85.04,70.05) .. controls (84.33,89.17) and (81.67,113.5) .. (70.04,130.05) ;
\draw [shift={(70.04,130.05)}, rotate = 125.08] [color={rgb, 255:red, 74; green, 144; blue, 226 }  ,draw opacity=1 ][fill={rgb, 255:red, 74; green, 144; blue, 226 }  ,fill opacity=1 ][line width=0.75]      (0, 0) circle [x radius= 2.01, y radius= 2.01]   ;
%Curve Lines [id:da8808863563949116] 
\draw    (190.02,10) .. controls (200.67,17.83) and (205.33,57.83) .. (205.01,70) ;
%Curve Lines [id:da6551710934380263] 
\draw    (190.02,10) .. controls (172.33,9.17) and (162,55.17) .. (175.02,70) ;
%Curve Lines [id:da736679178157389] 
\draw    (190.02,10) .. controls (176.67,14.5) and (170.67,55.83) .. (175.02,70) ;
%Curve Lines [id:da5464706296304208] 
\draw    (190.02,10) .. controls (202.3,15.45) and (214.62,25.6) .. (220.02,40) ;
%Curve Lines [id:da03710741385529848] 
\draw    (190.02,10) .. controls (200.97,9.12) and (221,25.5) .. (220.02,40) ;
%Straight Lines [id:da6428622020791361] 
\draw    (190.02,10) -- (220.02,40) ;
\draw [shift={(220.02,40)}, rotate = 45] [color={rgb, 255:red, 0; green, 0; blue, 0 }  ][fill={rgb, 255:red, 0; green, 0; blue, 0 }  ][line width=0.75]      (0, 0) circle [x radius= 2.01, y radius= 2.01]   ;
%Curve Lines [id:da039358298966338356] 
\draw    (190.02,10) .. controls (182,17.83) and (186,36.83) .. (190.02,40) ;
%Straight Lines [id:da4772751005800495] 
\draw    (190.02,40) -- (205.02,70) ;
\draw [shift={(190.02,40)}, rotate = 63.43] [color={rgb, 255:red, 0; green, 0; blue, 0 }  ][fill={rgb, 255:red, 0; green, 0; blue, 0 }  ][line width=0.75]      (0, 0) circle [x radius= 2.01, y radius= 2.01]   ;
%Straight Lines [id:da08454582646982667] 
\draw    (220.02,40) -- (205.02,70) ;
\draw [shift={(205.02,70)}, rotate = 116.57] [color={rgb, 255:red, 0; green, 0; blue, 0 }  ][fill={rgb, 255:red, 0; green, 0; blue, 0 }  ][line width=0.75]      (0, 0) circle [x radius= 2.01, y radius= 2.01]   ;
%Curve Lines [id:da8022072523292789] 
\draw    (205.02,70) .. controls (213.3,73.78) and (219.97,87.78) .. (220.02,100) ;
%Curve Lines [id:da5376043492520477] 
\draw    (205.02,70) .. controls (207.63,82.12) and (208.63,90.45) .. (220.02,100) ;
%Straight Lines [id:da7087683227969146] 
\draw    (190.02,40) -- (175.02,70) ;
\draw [shift={(175.02,70)}, rotate = 116.57] [color={rgb, 255:red, 0; green, 0; blue, 0 }  ][fill={rgb, 255:red, 0; green, 0; blue, 0 }  ][line width=0.75]      (0, 0) circle [x radius= 2.01, y radius= 2.01]   ;
%Straight Lines [id:da774207118747696] 
\draw    (190.02,11.01) -- (190.02,40) ;
\draw [shift={(190.02,10)}, rotate = 90] [color={rgb, 255:red, 0; green, 0; blue, 0 }  ][line width=0.75]      (0, 0) circle [x radius= 2.01, y radius= 2.01]   ;
%Straight Lines [id:da929061463652023] 
\draw    (205.02,70) -- (190.01,100) ;
%Curve Lines [id:da23704131198173117] 
\draw    (220.02,40) .. controls (226.33,43.5) and (232.1,83.2) .. (220.02,100) ;
\draw [shift={(220.02,100)}, rotate = 125.71] [color={rgb, 255:red, 0; green, 0; blue, 0 }  ][fill={rgb, 255:red, 0; green, 0; blue, 0 }  ][line width=0.75]      (0, 0) circle [x radius= 2.01, y radius= 2.01]   ;
%Curve Lines [id:da37811915824823294] 
\draw    (190.02,10) .. controls (197.67,16.17) and (195,37.17) .. (190.02,40) ;
%Shape: Circle [id:dp7693863374441471] 
\draw  [color={rgb, 255:red, 0; green, 0; blue, 0 }  ,draw opacity=1 ][fill={rgb, 255:red, 255; green, 255; blue, 255 }  ,fill opacity=1 ] (187.77,10) .. controls (187.77,8.76) and (188.78,7.75) .. (190.02,7.75) .. controls (191.27,7.75) and (192.27,8.76) .. (192.27,10) .. controls (192.27,11.24) and (191.27,12.25) .. (190.02,12.25) .. controls (188.78,12.25) and (187.77,11.24) .. (187.77,10) -- cycle ;
%Straight Lines [id:da4954969157644824] 
\draw    (175.02,70) -- (190.02,100) ;
\draw [shift={(190.02,100)}, rotate = 63.43] [color={rgb, 255:red, 0; green, 0; blue, 0 }  ][fill={rgb, 255:red, 0; green, 0; blue, 0 }  ][line width=0.75]      (0, 0) circle [x radius= 2.01, y radius= 2.01]   ;
%Straight Lines [id:da3318607737666768] 
\draw    (190.02,40) -- (190.02,100) ;
%Straight Lines [id:da7186813790708736] 
\draw    (175.02,70) -- (160.02,100) ;
\draw [shift={(160.02,100)}, rotate = 116.57] [color={rgb, 255:red, 0; green, 0; blue, 0 }  ][fill={rgb, 255:red, 0; green, 0; blue, 0 }  ][line width=0.75]      (0, 0) circle [x radius= 2.01, y radius= 2.01]   ;
%Curve Lines [id:da950949896835961] 
\draw [color={rgb, 255:red, 74; green, 144; blue, 226 }  ,draw opacity=1 ]   (205.02,70) .. controls (204.31,89.12) and (201.65,113.45) .. (190.02,130) ;
\draw [shift={(190.02,130)}, rotate = 125.08] [color={rgb, 255:red, 74; green, 144; blue, 226 }  ,draw opacity=1 ][fill={rgb, 255:red, 74; green, 144; blue, 226 }  ,fill opacity=1 ][line width=0.75]      (0, 0) circle [x radius= 2.01, y radius= 2.01]   ;
%Curve Lines [id:da8528143936964276] 
\draw    (277.98,100.05) .. controls (290.58,103.25) and (302.58,115.65) .. (307.98,130.05) ;
%Curve Lines [id:da11491440133406494] 
\draw    (307.98,10.05) .. controls (318.63,17.88) and (323.29,57.88) .. (322.97,70.05) ;
%Curve Lines [id:da10733947128383758] 
\draw [color={rgb, 255:red, 0; green, 0; blue, 0 }  ,draw opacity=1 ]   (307.98,10.05) .. controls (290.29,9.22) and (279.96,55.22) .. (292.98,70.05) ;
%Curve Lines [id:da8571583172028681] 
\draw [color={rgb, 255:red, 0; green, 0; blue, 0 }  ,draw opacity=1 ]   (307.98,10.05) .. controls (294.63,14.55) and (288.63,55.88) .. (292.98,70.05) ;
%Curve Lines [id:da004409284263093061] 
\draw    (307.98,10.05) .. controls (320.26,15.5) and (332.58,25.65) .. (337.98,40.05) ;
%Curve Lines [id:da02704217608194137] 
\draw    (307.98,10.05) .. controls (318.93,9.17) and (338.96,25.55) .. (337.98,40.05) ;
%Straight Lines [id:da2851413054675501] 
\draw    (307.98,10.05) -- (337.98,40.05) ;
\draw [shift={(337.98,40.05)}, rotate = 45] [color={rgb, 255:red, 0; green, 0; blue, 0 }  ][fill={rgb, 255:red, 0; green, 0; blue, 0 }  ][line width=0.75]      (0, 0) circle [x radius= 2.01, y radius= 2.01]   ;
%Curve Lines [id:da6020654667763782] 
\draw [color={rgb, 255:red, 74; green, 144; blue, 226 }  ,draw opacity=1 ]   (307.98,10.05) .. controls (299.96,17.88) and (303.96,36.88) .. (307.98,40.05) ;
%Straight Lines [id:da45568256008456165] 
\draw    (307.98,40.05) -- (322.98,70.05) ;
\draw [shift={(307.98,40.05)}, rotate = 63.43] [color={rgb, 255:red, 0; green, 0; blue, 0 }  ][fill={rgb, 255:red, 0; green, 0; blue, 0 }  ][line width=0.75]      (0, 0) circle [x radius= 2.01, y radius= 2.01]   ;
%Straight Lines [id:da5475516471290092] 
\draw    (337.98,40.05) -- (322.98,70.05) ;
\draw [shift={(322.98,70.05)}, rotate = 116.57] [color={rgb, 255:red, 0; green, 0; blue, 0 }  ][fill={rgb, 255:red, 0; green, 0; blue, 0 }  ][line width=0.75]      (0, 0) circle [x radius= 2.01, y radius= 2.01]   ;
%Curve Lines [id:da5055546978416163] 
\draw    (322.98,70.05) .. controls (331.26,73.83) and (337.93,87.83) .. (337.98,100.05) ;
%Curve Lines [id:da82717360361885] 
\draw    (322.98,70.05) .. controls (325.59,82.17) and (326.59,90.5) .. (337.98,100.05) ;
%Straight Lines [id:da5218876446930564] 
\draw    (307.98,40.05) -- (292.98,70.05) ;
\draw [shift={(292.98,70.05)}, rotate = 116.57] [color={rgb, 255:red, 0; green, 0; blue, 0 }  ][fill={rgb, 255:red, 0; green, 0; blue, 0 }  ][line width=0.75]      (0, 0) circle [x radius= 2.01, y radius= 2.01]   ;
%Straight Lines [id:da6200214188444505] 
\draw    (322.98,70.05) -- (307.97,100.05) ;
%Curve Lines [id:da9721727455901619] 
\draw    (337.98,40.05) .. controls (344.29,43.55) and (350.06,83.25) .. (337.98,100.05) ;
\draw [shift={(337.98,100.05)}, rotate = 125.71] [color={rgb, 255:red, 0; green, 0; blue, 0 }  ][fill={rgb, 255:red, 0; green, 0; blue, 0 }  ][line width=0.75]      (0, 0) circle [x radius= 2.01, y radius= 2.01]   ;
%Curve Lines [id:da7727745831540176] 
\draw [color={rgb, 255:red, 0; green, 0; blue, 0 }  ,draw opacity=1 ]   (307.98,10.05) .. controls (315.63,16.22) and (312.96,37.22) .. (307.98,40.05) ;
%Straight Lines [id:da73803810126338] 
\draw    (292.98,70.05) -- (307.98,100.05) ;
\draw [shift={(307.98,100.05)}, rotate = 63.43] [color={rgb, 255:red, 0; green, 0; blue, 0 }  ][fill={rgb, 255:red, 0; green, 0; blue, 0 }  ][line width=0.75]      (0, 0) circle [x radius= 2.01, y radius= 2.01]   ;
%Straight Lines [id:da9215499879065873] 
\draw    (307.98,40.05) -- (307.98,100.05) ;
%Straight Lines [id:da000637828777881766] 
\draw    (292.98,70.05) -- (277.98,100.05) ;
\draw [shift={(277.98,100.05)}, rotate = 116.57] [color={rgb, 255:red, 0; green, 0; blue, 0 }  ][fill={rgb, 255:red, 0; green, 0; blue, 0 }  ][line width=0.75]      (0, 0) circle [x radius= 2.01, y radius= 2.01]   ;
%Curve Lines [id:da46720618107481215] 
\draw    (322.98,70.05) .. controls (322.27,89.17) and (319.61,113.5) .. (307.98,130.05) ;
%Shape: Circle [id:dp006475294465778436] 
\draw  [color={rgb, 255:red, 74; green, 144; blue, 226 }  ,draw opacity=1 ] (303.98,10.05) .. controls (303.98,7.84) and (305.77,6.05) .. (307.98,6.05) .. controls (310.19,6.05) and (311.98,7.84) .. (311.98,10.05) .. controls (311.98,12.26) and (310.19,14.05) .. (307.98,14.05) .. controls (305.77,14.05) and (303.98,12.26) .. (303.98,10.05) -- cycle ;
%Curve Lines [id:da9405250922414079] 
\draw [color={rgb, 255:red, 0; green, 0; blue, 0 }  ,draw opacity=1 ]   (307.98,10.05) .. controls (294.63,14.55) and (288.63,55.88) .. (292.98,70.05) ;
%Straight Lines [id:da8382521107129323] 
\draw [color={rgb, 255:red, 0; green, 0; blue, 0 }  ,draw opacity=1 ]   (307.98,10.05) -- (307.98,40.05) ;
%Curve Lines [id:da3649575787228114] 
\draw [color={rgb, 255:red, 0; green, 0; blue, 0 }  ,draw opacity=1 ]   (70.05,40.05) .. controls (53,42.86) and (34.71,76.57) .. (40.05,100.05) ;
%Curve Lines [id:da8823186561537485] 
\draw [color={rgb, 255:red, 0; green, 0; blue, 0 }  ,draw opacity=1 ]   (190.02,40) .. controls (172.97,42.81) and (154.68,76.52) .. (160.02,100) ;
%Curve Lines [id:da4556966985246603] 
\draw [color={rgb, 255:red, 74; green, 144; blue, 226 }  ,draw opacity=1 ]   (307.98,40.05) .. controls (290.93,42.86) and (272.64,76.57) .. (277.98,100.05) ;
%Shape: Circle [id:dp9285183366930082] 
\draw  [color={rgb, 255:red, 74; green, 144; blue, 226 }  ,draw opacity=1 ] (273.98,100.05) .. controls (273.98,97.84) and (275.77,96.05) .. (277.98,96.05) .. controls (280.19,96.05) and (281.98,97.84) .. (281.98,100.05) .. controls (281.98,102.26) and (280.19,104.05) .. (277.98,104.05) .. controls (275.77,104.05) and (273.98,102.26) .. (273.98,100.05) -- cycle ;
%Straight Lines [id:da1523057195571632] 
\draw [color={rgb, 255:red, 74; green, 144; blue, 226 }  ,draw opacity=1 ]   (292.98,70.05) -- (277.98,100.05) ;
%Curve Lines [id:da5580449093681882] 
\draw [color={rgb, 255:red, 0; green, 0; blue, 0 }  ,draw opacity=1 ]   (70.05,40.05) .. controls (53.57,51.14) and (43,78.57) .. (40.05,100.05) ;
%Curve Lines [id:da43501736513517386] 
\draw [color={rgb, 255:red, 0; green, 0; blue, 0 }  ,draw opacity=1 ]   (190.02,40) .. controls (173.54,51.09) and (162.97,78.52) .. (160.02,100) ;
%Curve Lines [id:da9002612100931846] 
\draw [color={rgb, 255:red, 74; green, 144; blue, 226 }  ,draw opacity=1 ]   (307.98,40.05) .. controls (291.5,51.14) and (280.93,78.57) .. (277.98,100.05) ;
%Shape: Circle [id:dp5957421188550233] 
\draw  [color={rgb, 255:red, 74; green, 144; blue, 226 }  ,draw opacity=1 ] (288.98,70.05) .. controls (288.98,67.84) and (290.77,66.05) .. (292.98,66.05) .. controls (295.19,66.05) and (296.98,67.84) .. (296.98,70.05) .. controls (296.98,72.26) and (295.19,74.05) .. (292.98,74.05) .. controls (290.77,74.05) and (288.98,72.26) .. (288.98,70.05) -- cycle ;
%Shape: Circle [id:dp3125219265170688] 
\draw  [color={rgb, 255:red, 74; green, 144; blue, 226 }  ,draw opacity=1 ] (303.98,40.05) .. controls (303.98,37.84) and (305.77,36.05) .. (307.98,36.05) .. controls (310.19,36.05) and (311.98,37.84) .. (311.98,40.05) .. controls (311.98,42.26) and (310.19,44.05) .. (307.98,44.05) .. controls (305.77,44.05) and (303.98,42.26) .. (303.98,40.05) -- cycle ;
%Shape: Circle [id:dp10802198908513982] 
\draw  [color={rgb, 255:red, 0; green, 0; blue, 0 }  ,draw opacity=1 ][fill={rgb, 255:red, 255; green, 255; blue, 255 }  ,fill opacity=1 ] (305.73,10.05) .. controls (305.73,8.81) and (306.74,7.8) .. (307.98,7.8) .. controls (309.23,7.8) and (310.23,8.81) .. (310.23,10.05) .. controls (310.23,11.29) and (309.23,12.3) .. (307.98,12.3) .. controls (306.74,12.3) and (305.73,11.29) .. (305.73,10.05) -- cycle ;
%Shape: Circle [id:dp043233232314815684] 
\draw  [color={rgb, 255:red, 74; green, 144; blue, 226 }  ,draw opacity=1 ] (81.04,70.05) .. controls (81.04,67.84) and (82.83,66.05) .. (85.04,66.05) .. controls (87.25,66.05) and (89.04,67.84) .. (89.04,70.05) .. controls (89.04,72.26) and (87.25,74.05) .. (85.04,74.05) .. controls (82.83,74.05) and (81.04,72.26) .. (81.04,70.05) -- cycle ;
%Shape: Circle [id:dp48026625309244786] 
\draw  [color={rgb, 255:red, 74; green, 144; blue, 226 }  ,draw opacity=1 ] (66.05,130.05) .. controls (66.05,127.84) and (67.85,126.05) .. (70.05,126.05) .. controls (72.26,126.05) and (74.05,127.84) .. (74.05,130.05) .. controls (74.05,132.26) and (72.26,134.05) .. (70.05,134.05) .. controls (67.85,134.05) and (66.05,132.26) .. (66.05,130.05) -- cycle ;
%Shape: Circle [id:dp35079169193730797] 
\draw  [color={rgb, 255:red, 74; green, 144; blue, 226 }  ,draw opacity=1 ] (66.05,10.05) .. controls (66.05,7.84) and (67.85,6.05) .. (70.05,6.05) .. controls (72.26,6.05) and (74.05,7.84) .. (74.05,10.05) .. controls (74.05,12.26) and (72.26,14.05) .. (70.05,14.05) .. controls (67.85,14.05) and (66.05,12.26) .. (66.05,10.05) -- cycle ;
%Straight Lines [id:da9241802235079244] 
\draw [color={rgb, 255:red, 74; green, 144; blue, 226 }  ,draw opacity=1 ]   (100.05,40.05) -- (85.05,70.05) ;
%Shape: Circle [id:dp3725644892377368] 
\draw  [color={rgb, 255:red, 74; green, 144; blue, 226 }  ,draw opacity=1 ] (96.05,40.05) .. controls (96.05,37.84) and (97.85,36.05) .. (100.05,36.05) .. controls (102.26,36.05) and (104.05,37.84) .. (104.05,40.05) .. controls (104.05,42.26) and (102.26,44.05) .. (100.05,44.05) .. controls (97.85,44.05) and (96.05,42.26) .. (96.05,40.05) -- cycle ;
%Curve Lines [id:da6355691018378326] 
\draw [color={rgb, 255:red, 74; green, 144; blue, 226 }  ,draw opacity=1 ]   (70.05,10.05) .. controls (81,9.17) and (101.03,25.55) .. (100.05,40.05) ;
%Shape: Circle [id:dp7772592695174172] 
\draw  [color={rgb, 255:red, 74; green, 144; blue, 226 }  ,draw opacity=1 ] (186.02,130) .. controls (186.02,127.79) and (187.81,126) .. (190.02,126) .. controls (192.23,126) and (194.02,127.79) .. (194.02,130) .. controls (194.02,132.21) and (192.23,134) .. (190.02,134) .. controls (187.81,134) and (186.02,132.21) .. (186.02,130) -- cycle ;
%Shape: Circle [id:dp12487017013629964] 
\draw  [color={rgb, 255:red, 74; green, 144; blue, 226 }  ,draw opacity=1 ] (156.02,100) .. controls (156.02,97.79) and (157.81,96) .. (160.02,96) .. controls (162.23,96) and (164.02,97.79) .. (164.02,100) .. controls (164.02,102.21) and (162.23,104) .. (160.02,104) .. controls (157.81,104) and (156.02,102.21) .. (156.02,100) -- cycle ;
%Shape: Circle [id:dp3988561207838638] 
\draw  [color={rgb, 255:red, 74; green, 144; blue, 226 }  ,draw opacity=1 ] (201.01,70) .. controls (201.01,67.79) and (202.8,66) .. (205.01,66) .. controls (207.22,66) and (209.01,67.79) .. (209.01,70) .. controls (209.01,72.21) and (207.22,74) .. (205.01,74) .. controls (202.8,74) and (201.01,72.21) .. (201.01,70) -- cycle ;
%Curve Lines [id:da8286884983822526] 
\draw [color={rgb, 255:red, 74; green, 144; blue, 226 }  ,draw opacity=1 ]   (277.98,100.05) .. controls (282.58,113.25) and (292.58,124.85) .. (307.98,130.05) ;
\draw [shift={(307.98,130.05)}, rotate = 18.66] [color={rgb, 255:red, 74; green, 144; blue, 226 }  ,draw opacity=1 ][fill={rgb, 255:red, 74; green, 144; blue, 226 }  ,fill opacity=1 ][line width=0.75]      (0, 0) circle [x radius= 2.01, y radius= 2.01]   ;
%Shape: Circle [id:dp3186300860790279] 
\draw  [color={rgb, 255:red, 74; green, 144; blue, 226 }  ,draw opacity=1 ] (303.98,130.05) .. controls (303.98,127.84) and (305.77,126.05) .. (307.98,126.05) .. controls (310.19,126.05) and (311.98,127.84) .. (311.98,130.05) .. controls (311.98,132.26) and (310.19,134.05) .. (307.98,134.05) .. controls (305.77,134.05) and (303.98,132.26) .. (303.98,130.05) -- cycle ;
%Curve Lines [id:da05895122789904794] 
\draw [color={rgb, 255:red, 23; green, 95; blue, 211 }  ,draw opacity=1 ]   (397.98,100) .. controls (410.58,103.2) and (422.58,115.6) .. (427.98,130) ;
%Curve Lines [id:da6981869336046194] 
\draw [color={rgb, 255:red, 6; green, 65; blue, 177 }  ,draw opacity=1 ]   (427.98,10) .. controls (438.63,17.83) and (443.29,57.83) .. (442.97,70) ;
%Curve Lines [id:da507397222857861] 
\draw [color={rgb, 255:red, 0; green, 0; blue, 0 }  ,draw opacity=1 ]   (427.98,10) .. controls (410.29,9.17) and (399.96,55.17) .. (412.98,70) ;
%Curve Lines [id:da32047872113100284] 
\draw    (427.98,10) .. controls (414.63,14.5) and (408.63,55.83) .. (412.98,70) ;
%Curve Lines [id:da925148619059831] 
\draw    (427.98,10) .. controls (440.26,15.45) and (452.58,25.6) .. (457.98,40) ;
%Curve Lines [id:da538839139985211] 
\draw    (427.98,10) .. controls (438.93,9.12) and (458.96,25.5) .. (457.98,40) ;
%Straight Lines [id:da4422238516899071] 
\draw    (427.98,10) -- (457.98,40) ;
\draw [shift={(457.98,40)}, rotate = 45] [color={rgb, 255:red, 0; green, 0; blue, 0 }  ][fill={rgb, 255:red, 0; green, 0; blue, 0 }  ][line width=0.75]      (0, 0) circle [x radius= 2.01, y radius= 2.01]   ;
%Curve Lines [id:da6996704676598461] 
\draw [color={rgb, 255:red, 74; green, 144; blue, 226 }  ,draw opacity=1 ]   (427.98,10) .. controls (419.96,17.83) and (423.96,36.83) .. (427.98,40) ;
%Straight Lines [id:da8080855103256304] 
\draw    (427.98,40) -- (442.98,70) ;
\draw [shift={(427.98,40)}, rotate = 63.43] [color={rgb, 255:red, 0; green, 0; blue, 0 }  ][fill={rgb, 255:red, 0; green, 0; blue, 0 }  ][line width=0.75]      (0, 0) circle [x radius= 2.01, y radius= 2.01]   ;
%Straight Lines [id:da3244514468260947] 
\draw    (457.98,40) -- (442.98,70) ;
\draw [shift={(442.98,70)}, rotate = 116.57] [color={rgb, 255:red, 0; green, 0; blue, 0 }  ][fill={rgb, 255:red, 0; green, 0; blue, 0 }  ][line width=0.75]      (0, 0) circle [x radius= 2.01, y radius= 2.01]   ;
%Curve Lines [id:da023041608118340884] 
\draw    (442.98,70) .. controls (451.26,73.78) and (457.93,87.78) .. (457.98,100) ;
%Curve Lines [id:da6964143331493441] 
\draw    (442.98,70) .. controls (445.59,82.12) and (446.59,90.45) .. (457.98,100) ;
%Straight Lines [id:da6843353537990337] 
\draw    (427.98,40) -- (412.98,70) ;
\draw [shift={(412.98,70)}, rotate = 116.57] [color={rgb, 255:red, 0; green, 0; blue, 0 }  ][fill={rgb, 255:red, 0; green, 0; blue, 0 }  ][line width=0.75]      (0, 0) circle [x radius= 2.01, y radius= 2.01]   ;
%Straight Lines [id:da7637327697712846] 
\draw    (442.98,70) -- (427.97,100) ;
%Curve Lines [id:da03519963454727959] 
\draw    (457.98,40) .. controls (464.29,43.5) and (470.06,83.2) .. (457.98,100) ;
\draw [shift={(457.98,100)}, rotate = 125.71] [color={rgb, 255:red, 0; green, 0; blue, 0 }  ][fill={rgb, 255:red, 0; green, 0; blue, 0 }  ][line width=0.75]      (0, 0) circle [x radius= 2.01, y radius= 2.01]   ;
%Curve Lines [id:da38422756318998963] 
\draw [color={rgb, 255:red, 0; green, 0; blue, 0 }  ,draw opacity=1 ]   (427.98,10) .. controls (435.63,16.17) and (432.96,37.17) .. (427.98,40) ;
%Straight Lines [id:da410601986224971] 
\draw    (412.98,70) -- (427.98,100) ;
\draw [shift={(427.98,100)}, rotate = 63.43] [color={rgb, 255:red, 0; green, 0; blue, 0 }  ][fill={rgb, 255:red, 0; green, 0; blue, 0 }  ][line width=0.75]      (0, 0) circle [x radius= 2.01, y radius= 2.01]   ;
%Straight Lines [id:da5659772418816] 
\draw    (427.98,40) -- (427.98,100) ;
%Straight Lines [id:da046126437238496965] 
\draw    (412.98,70) -- (397.98,100) ;
\draw [shift={(397.98,100)}, rotate = 116.57] [color={rgb, 255:red, 0; green, 0; blue, 0 }  ][fill={rgb, 255:red, 0; green, 0; blue, 0 }  ][line width=0.75]      (0, 0) circle [x radius= 2.01, y radius= 2.01]   ;
%Curve Lines [id:da4413996136067717] 
\draw [color={rgb, 255:red, 6; green, 65; blue, 177 }  ,draw opacity=1 ]   (442.98,70) .. controls (442.27,89.12) and (439.61,113.45) .. (427.98,130) ;
%Shape: Circle [id:dp9363707347919057] 
\draw  [color={rgb, 255:red, 74; green, 144; blue, 226 }  ,draw opacity=1 ] (423.98,10) .. controls (423.98,7.79) and (425.77,6) .. (427.98,6) .. controls (430.19,6) and (431.98,7.79) .. (431.98,10) .. controls (431.98,12.21) and (430.19,14) .. (427.98,14) .. controls (425.77,14) and (423.98,12.21) .. (423.98,10) -- cycle ;
%Curve Lines [id:da27742164193334184] 
\draw [color={rgb, 255:red, 0; green, 0; blue, 0 }  ,draw opacity=1 ]   (427.98,10) .. controls (414.63,14.5) and (408.63,55.83) .. (412.98,70) ;
%Straight Lines [id:da5526506891195337] 
\draw [color={rgb, 255:red, 0; green, 0; blue, 0 }  ,draw opacity=1 ]   (427.98,10) -- (427.98,40) ;
%Curve Lines [id:da4824863878816479] 
\draw [color={rgb, 255:red, 74; green, 144; blue, 226 }  ,draw opacity=1 ]   (427.98,40) .. controls (410.93,42.81) and (392.64,76.52) .. (397.98,100) ;
%Shape: Circle [id:dp9410619999793163] 
\draw  [color={rgb, 255:red, 74; green, 144; blue, 226 }  ,draw opacity=1 ] (393.98,100) .. controls (393.98,97.79) and (395.77,96) .. (397.98,96) .. controls (400.19,96) and (401.98,97.79) .. (401.98,100) .. controls (401.98,102.21) and (400.19,104) .. (397.98,104) .. controls (395.77,104) and (393.98,102.21) .. (393.98,100) -- cycle ;
%Straight Lines [id:da41533091016303814] 
\draw [color={rgb, 255:red, 74; green, 144; blue, 226 }  ,draw opacity=1 ]   (412.98,70) -- (397.98,100) ;
%Curve Lines [id:da20274950733876906] 
\draw [color={rgb, 255:red, 74; green, 144; blue, 226 }  ,draw opacity=1 ]   (427.98,40) .. controls (411.5,51.09) and (400.93,78.52) .. (397.98,100) ;
%Shape: Circle [id:dp9785512005623174] 
\draw  [color={rgb, 255:red, 74; green, 144; blue, 226 }  ,draw opacity=1 ] (408.98,70) .. controls (408.98,67.79) and (410.77,66) .. (412.98,66) .. controls (415.19,66) and (416.98,67.79) .. (416.98,70) .. controls (416.98,72.21) and (415.19,74) .. (412.98,74) .. controls (410.77,74) and (408.98,72.21) .. (408.98,70) -- cycle ;
%Shape: Circle [id:dp7010593874546049] 
\draw  [color={rgb, 255:red, 74; green, 144; blue, 226 }  ,draw opacity=1 ] (423.98,40) .. controls (423.98,37.79) and (425.77,36) .. (427.98,36) .. controls (430.19,36) and (431.98,37.79) .. (431.98,40) .. controls (431.98,42.21) and (430.19,44) .. (427.98,44) .. controls (425.77,44) and (423.98,42.21) .. (423.98,40) -- cycle ;
%Shape: Circle [id:dp9161870156529757] 
\draw  [color={rgb, 255:red, 0; green, 0; blue, 0 }  ,draw opacity=1 ][fill={rgb, 255:red, 255; green, 255; blue, 255 }  ,fill opacity=1 ] (425.73,10) .. controls (425.73,8.76) and (426.74,7.75) .. (427.98,7.75) .. controls (429.23,7.75) and (430.23,8.76) .. (430.23,10) .. controls (430.23,11.24) and (429.23,12.25) .. (427.98,12.25) .. controls (426.74,12.25) and (425.73,11.24) .. (425.73,10) -- cycle ;
%Curve Lines [id:da07421075195078086] 
\draw [color={rgb, 255:red, 74; green, 144; blue, 226 }  ,draw opacity=1 ]   (397.98,100) .. controls (402.58,113.2) and (412.58,124.8) .. (427.98,130) ;
\draw [shift={(427.98,130)}, rotate = 18.66] [color={rgb, 255:red, 74; green, 144; blue, 226 }  ,draw opacity=1 ][fill={rgb, 255:red, 74; green, 144; blue, 226 }  ,fill opacity=1 ][line width=0.75]      (0, 0) circle [x radius= 2.01, y radius= 2.01]   ;
%Shape: Circle [id:dp3929586062421013] 
\draw  [color={rgb, 255:red, 74; green, 144; blue, 226 }  ,draw opacity=1 ] (423.98,130) .. controls (423.98,127.79) and (425.77,126) .. (427.98,126) .. controls (430.19,126) and (431.98,127.79) .. (431.98,130) .. controls (431.98,132.21) and (430.19,134) .. (427.98,134) .. controls (425.77,134) and (423.98,132.21) .. (423.98,130) -- cycle ;
%Shape: Circle [id:dp6403898136259163] 
\draw  [color={rgb, 255:red, 6; green, 65; blue, 177 }  ,draw opacity=1 ] (438.97,70) .. controls (438.97,67.79) and (440.76,66) .. (442.97,66) .. controls (445.18,66) and (446.97,67.79) .. (446.97,70) .. controls (446.97,72.21) and (445.18,74) .. (442.97,74) .. controls (440.76,74) and (438.97,72.21) .. (438.97,70) -- cycle ;
%Straight Lines [id:da11337166459585146] 
\draw [color={rgb, 255:red, 6; green, 65; blue, 177 }  ,draw opacity=1 ]   (427.98,40) -- (442.98,70) ;
%Curve Lines [id:da5362596107704657] 
\draw [color={rgb, 255:red, 74; green, 144; blue, 226 }  ,draw opacity=1 ]   (517.98,100) .. controls (530.58,103.2) and (542.58,115.6) .. (547.98,130) ;
%Curve Lines [id:da42674943985102853] 
\draw [color={rgb, 255:red, 11; green, 73; blue, 214 }  ,draw opacity=1 ]   (547.98,10) .. controls (558.63,17.83) and (563.29,57.83) .. (562.97,70) ;
%Curve Lines [id:da221656588593524] 
\draw [color={rgb, 255:red, 0; green, 0; blue, 0 }  ,draw opacity=1 ]   (547.98,10) .. controls (530.29,9.17) and (519.96,55.17) .. (532.98,70) ;
%Curve Lines [id:da557583618627667] 
\draw    (547.98,10) .. controls (534.63,14.5) and (528.63,55.83) .. (532.98,70) ;
%Curve Lines [id:da08985507903753209] 
\draw    (547.98,10) .. controls (560.26,15.45) and (572.58,25.6) .. (577.98,40) ;
%Curve Lines [id:da17716031637397656] 
\draw    (547.98,10) .. controls (558.93,9.12) and (578.96,25.5) .. (577.98,40) ;
%Straight Lines [id:da1741554965025066] 
\draw    (547.98,10) -- (577.98,40) ;
\draw [shift={(577.98,40)}, rotate = 45] [color={rgb, 255:red, 0; green, 0; blue, 0 }  ][fill={rgb, 255:red, 0; green, 0; blue, 0 }  ][line width=0.75]      (0, 0) circle [x radius= 2.01, y radius= 2.01]   ;
%Curve Lines [id:da9308015737141591] 
\draw [color={rgb, 255:red, 0; green, 0; blue, 0 }  ,draw opacity=1 ]   (547.98,10) .. controls (539.96,17.83) and (543.96,36.83) .. (547.98,40) ;
%Straight Lines [id:da9352421323187127] 
\draw    (547.98,40) -- (562.98,70) ;
\draw [shift={(547.98,40)}, rotate = 63.43] [color={rgb, 255:red, 0; green, 0; blue, 0 }  ][fill={rgb, 255:red, 0; green, 0; blue, 0 }  ][line width=0.75]      (0, 0) circle [x radius= 2.01, y radius= 2.01]   ;
%Straight Lines [id:da5852788525064229] 
\draw    (577.98,40) -- (562.98,70) ;
%Curve Lines [id:da07128588396588409] 
\draw    (562.98,70) .. controls (571.26,73.78) and (577.93,87.78) .. (577.98,100) ;
\draw [shift={(562.98,70)}, rotate = 24.56] [color={rgb, 255:red, 0; green, 0; blue, 0 }  ][fill={rgb, 255:red, 0; green, 0; blue, 0 }  ][line width=0.75]      (0, 0) circle [x radius= 2.01, y radius= 2.01]   ;
%Curve Lines [id:da6063768652328237] 
\draw    (562.98,70) .. controls (565.59,82.12) and (566.59,90.45) .. (577.98,100) ;
%Straight Lines [id:da5669626421595821] 
\draw    (562.98,70) -- (547.97,100) ;
%Curve Lines [id:da9651103008104494] 
\draw    (577.98,40) .. controls (584.29,43.5) and (590.06,83.2) .. (577.98,100) ;
\draw [shift={(577.98,100)}, rotate = 125.71] [color={rgb, 255:red, 0; green, 0; blue, 0 }  ][fill={rgb, 255:red, 0; green, 0; blue, 0 }  ][line width=0.75]      (0, 0) circle [x radius= 2.01, y radius= 2.01]   ;
\draw [shift={(577.98,40)}, rotate = 29.01] [color={rgb, 255:red, 0; green, 0; blue, 0 }  ][fill={rgb, 255:red, 0; green, 0; blue, 0 }  ][line width=0.75]      (0, 0) circle [x radius= 2.01, y radius= 2.01]   ;
%Curve Lines [id:da10147340254117287] 
\draw [color={rgb, 255:red, 0; green, 0; blue, 0 }  ,draw opacity=1 ]   (547.98,10) .. controls (555.63,16.17) and (552.96,37.17) .. (547.98,40) ;
%Straight Lines [id:da18239412958702217] 
\draw    (532.98,70) -- (547.98,100) ;
\draw [shift={(547.98,100)}, rotate = 63.43] [color={rgb, 255:red, 0; green, 0; blue, 0 }  ][fill={rgb, 255:red, 0; green, 0; blue, 0 }  ][line width=0.75]      (0, 0) circle [x radius= 2.01, y radius= 2.01]   ;
\draw [shift={(532.98,70)}, rotate = 63.43] [color={rgb, 255:red, 0; green, 0; blue, 0 }  ][fill={rgb, 255:red, 0; green, 0; blue, 0 }  ][line width=0.75]      (0, 0) circle [x radius= 2.01, y radius= 2.01]   ;
%Straight Lines [id:da6129777082140695] 
\draw    (547.98,40) -- (547.98,100) ;
%Straight Lines [id:da2853910392305602] 
\draw    (532.98,70) -- (517.98,100) ;
\draw [shift={(517.98,100)}, rotate = 116.57] [color={rgb, 255:red, 0; green, 0; blue, 0 }  ][fill={rgb, 255:red, 0; green, 0; blue, 0 }  ][line width=0.75]      (0, 0) circle [x radius= 2.01, y radius= 2.01]   ;
%Curve Lines [id:da8942654935265979] 
\draw [color={rgb, 255:red, 74; green, 144; blue, 226 }  ,draw opacity=1 ]   (562.98,70) .. controls (562.27,89.12) and (559.61,113.45) .. (547.98,130) ;
%Shape: Circle [id:dp932253827540449] 
\draw  [color={rgb, 255:red, 23; green, 95; blue, 211 }  ,draw opacity=1 ] (543.98,10) .. controls (543.98,7.79) and (545.77,6) .. (547.98,6) .. controls (550.19,6) and (551.98,7.79) .. (551.98,10) .. controls (551.98,12.21) and (550.19,14) .. (547.98,14) .. controls (545.77,14) and (543.98,12.21) .. (543.98,10) -- cycle ;
%Curve Lines [id:da20174422112714652] 
\draw [color={rgb, 255:red, 0; green, 0; blue, 0 }  ,draw opacity=1 ]   (547.98,10) .. controls (534.63,14.5) and (528.63,55.83) .. (532.98,70) ;
%Straight Lines [id:da18250858066498765] 
\draw [color={rgb, 255:red, 0; green, 0; blue, 0 }  ,draw opacity=1 ]   (547.98,10) -- (547.98,40) ;
%Curve Lines [id:da18446545366793554] 
\draw [color={rgb, 255:red, 23; green, 95; blue, 211 }  ,draw opacity=1 ]   (547.98,40) .. controls (530.93,42.81) and (512.64,76.52) .. (517.98,100) ;
%Straight Lines [id:da36463803607089096] 
\draw [color={rgb, 255:red, 23; green, 95; blue, 211 }  ,draw opacity=1 ]   (532.98,70) -- (517.98,100) ;
%Curve Lines [id:da660903757898704] 
\draw [color={rgb, 255:red, 23; green, 95; blue, 211 }  ,draw opacity=1 ]   (547.98,40) .. controls (531.5,51.09) and (520.93,78.52) .. (517.98,100) ;
%Shape: Circle [id:dp8663108458682149] 
\draw  [color={rgb, 255:red, 23; green, 95; blue, 211 }  ,draw opacity=1 ] (528.98,70) .. controls (528.98,67.79) and (530.77,66) .. (532.98,66) .. controls (535.19,66) and (536.98,67.79) .. (536.98,70) .. controls (536.98,72.21) and (535.19,74) .. (532.98,74) .. controls (530.77,74) and (528.98,72.21) .. (528.98,70) -- cycle ;
%Shape: Circle [id:dp4132603871424567] 
\draw  [color={rgb, 255:red, 23; green, 95; blue, 211 }  ,draw opacity=1 ] (543.98,40) .. controls (543.98,37.79) and (545.77,36) .. (547.98,36) .. controls (550.19,36) and (551.98,37.79) .. (551.98,40) .. controls (551.98,42.21) and (550.19,44) .. (547.98,44) .. controls (545.77,44) and (543.98,42.21) .. (543.98,40) -- cycle ;
%Shape: Circle [id:dp13788966722918594] 
\draw  [color={rgb, 255:red, 0; green, 0; blue, 0 }  ,draw opacity=1 ][fill={rgb, 255:red, 255; green, 255; blue, 255 }  ,fill opacity=1 ] (545.73,10) .. controls (545.73,8.76) and (546.74,7.75) .. (547.98,7.75) .. controls (549.23,7.75) and (550.23,8.76) .. (550.23,10) .. controls (550.23,11.24) and (549.23,12.25) .. (547.98,12.25) .. controls (546.74,12.25) and (545.73,11.24) .. (545.73,10) -- cycle ;
%Curve Lines [id:da10612229622886593] 
\draw [color={rgb, 255:red, 74; green, 144; blue, 226 }  ,draw opacity=1 ]   (517.98,100) .. controls (522.58,113.2) and (532.58,124.8) .. (547.98,130) ;
\draw [shift={(547.98,130)}, rotate = 18.66] [color={rgb, 255:red, 74; green, 144; blue, 226 }  ,draw opacity=1 ][fill={rgb, 255:red, 74; green, 144; blue, 226 }  ,fill opacity=1 ][line width=0.75]      (0, 0) circle [x radius= 2.01, y radius= 2.01]   ;
%Shape: Circle [id:dp4526002266131778] 
\draw  [color={rgb, 255:red, 74; green, 144; blue, 226 }  ,draw opacity=1 ] (543.98,130) .. controls (543.98,127.79) and (545.77,126) .. (547.98,126) .. controls (550.19,126) and (551.98,127.79) .. (551.98,130) .. controls (551.98,132.21) and (550.19,134) .. (547.98,134) .. controls (545.77,134) and (543.98,132.21) .. (543.98,130) -- cycle ;
%Straight Lines [id:da036645944847443546] 
\draw [color={rgb, 255:red, 11; green, 73; blue, 214 }  ,draw opacity=1 ]   (547.98,40) -- (562.98,70) ;
%Shape: Circle [id:dp9332997469618625] 
\draw  [color={rgb, 255:red, 74; green, 144; blue, 226 }  ,draw opacity=1 ] (513.98,100) .. controls (513.98,97.79) and (515.77,96) .. (517.98,96) .. controls (520.19,96) and (521.98,97.79) .. (521.98,100) .. controls (521.98,102.21) and (520.19,104) .. (517.98,104) .. controls (515.77,104) and (513.98,102.21) .. (513.98,100) -- cycle ;
%Shape: Circle [id:dp6283631608709122] 
\draw  [color={rgb, 255:red, 74; green, 144; blue, 226 }  ,draw opacity=1 ] (558.97,70) .. controls (558.97,67.79) and (560.76,66) .. (562.97,66) .. controls (565.18,66) and (566.97,67.79) .. (566.97,70) .. controls (566.97,72.21) and (565.18,74) .. (562.97,74) .. controls (560.76,74) and (558.97,72.21) .. (558.97,70) -- cycle ;

% Text Node
\draw (28,10.5) node [anchor=north west][inner sep=0.75pt]  [font=\small] [align=left] {(i)};
% Text Node
\draw (148,10.5) node [anchor=north west][inner sep=0.75pt]  [font=\small] [align=left] {(ii)};
% Text Node
\draw (265,10.5) node [anchor=north west][inner sep=0.75pt]  [font=\small] [align=left] {(iii)};
% Text Node
\draw (625,54) node [anchor=north west][inner sep=0.75pt]  [font=\small] [align=left] {\CT};
% Text Node
\draw (625,69) node [anchor=north west][inner sep=0.75pt]  [font=\small] [align=left] {\CF};
% Text Node
\draw (385,10.45) node [anchor=north west][inner sep=0.75pt]  [font=\small] [align=left] {(iv)};
% Text Node
\draw (507,10.45) node [anchor=north west][inner sep=0.75pt]  [font=\small] [align=left] {(v)};

\end{tikzpicture}

%% file: Visuals/evolution-simple.tex
\begin{tikzpicture}[x=0.75pt,y=0.75pt,yscale=-1,xscale=1]
%uncomment if require: \path (0,185); %set diagram left start at 0, and has height of 185

%Curve Lines [id:da27016025309381064] 
\draw [color={rgb, 255:red, 74; green, 144; blue, 226 }  ,draw opacity=1 ] [dash pattern={on 3.75pt off 2.25pt}]  (220.02,100) .. controls (214.2,106.4) and (222.8,141.8) .. (235,155) ;
%Curve Lines [id:da7095915129532866] 
\draw    (220.02,100) .. controls (232.62,103.2) and (244.62,115.6) .. (250.02,130) ;
%Curve Lines [id:da09226070801303043] 
\draw    (220.02,100) .. controls (224.62,113.2) and (234.62,124.8) .. (250.02,130) ;
%Straight Lines [id:da1832525707428153] 
\draw [color={rgb, 255:red, 74; green, 144; blue, 226 }  ,draw opacity=1 ] [dash pattern={on 3.75pt off 2.25pt}]  (235.02,70) -- (235,155) ;
\draw [shift={(235,155)}, rotate = 90.02] [color={rgb, 255:red, 74; green, 144; blue, 226 }  ,draw opacity=1 ][fill={rgb, 255:red, 74; green, 144; blue, 226 }  ,fill opacity=1 ][line width=0.75]      (0, 0) circle [x radius= 2.01, y radius= 2.01]   ;
%Curve Lines [id:da13632575590207707] 
\draw    (130.05,10.05) .. controls (140.7,17.88) and (145.37,57.88) .. (145.04,70.05) ;
%Curve Lines [id:da48790780059635264] 
\draw    (130.05,10.05) .. controls (112.37,9.22) and (102.03,55.22) .. (115.05,70.05) ;
%Curve Lines [id:da38912585590639504] 
\draw    (130.05,10.05) .. controls (116.7,14.55) and (110.7,55.88) .. (115.05,70.05) ;
%Curve Lines [id:da08437934388897639] 
\draw    (130.05,10.05) .. controls (142.33,15.5) and (154.65,25.65) .. (160.05,40.05) ;
%Curve Lines [id:da9004889004440765] 
\draw    (130.05,10.05) .. controls (141,9.17) and (161.03,25.55) .. (160.05,40.05) ;
%Straight Lines [id:da7205787291392504] 
\draw    (130.05,10.05) -- (160.05,40.05) ;
\draw [shift={(160.05,40.05)}, rotate = 45] [color={rgb, 255:red, 0; green, 0; blue, 0 }  ][fill={rgb, 255:red, 0; green, 0; blue, 0 }  ][line width=0.75]      (0, 0) circle [x radius= 2.01, y radius= 2.01]   ;
%Curve Lines [id:da17942975313096132] 
\draw    (130.05,10.05) .. controls (122.03,17.88) and (126.03,36.88) .. (130.05,40.05) ;
%Straight Lines [id:da42585338799977757] 
\draw [color={rgb, 255:red, 74; green, 144; blue, 226 }  ,draw opacity=1 ] [dash pattern={on 3.75pt off 2.25pt}]  (220.02,100) -- (235,155) ;
%Straight Lines [id:da3286009903481659] 
\draw    (130.05,40.05) -- (145.05,70.05) ;
\draw [shift={(130.05,40.05)}, rotate = 63.43] [color={rgb, 255:red, 0; green, 0; blue, 0 }  ][fill={rgb, 255:red, 0; green, 0; blue, 0 }  ][line width=0.75]      (0, 0) circle [x radius= 2.01, y radius= 2.01]   ;
%Straight Lines [id:da15326589386906408] 
\draw    (160.05,40.05) -- (145.05,70.05) ;
\draw [shift={(145.05,70.05)}, rotate = 116.57] [color={rgb, 255:red, 0; green, 0; blue, 0 }  ][fill={rgb, 255:red, 0; green, 0; blue, 0 }  ][line width=0.75]      (0, 0) circle [x radius= 2.01, y radius= 2.01]   ;
%Curve Lines [id:da08097857781336126] 
\draw    (145.05,70.05) .. controls (153.33,73.83) and (160,87.83) .. (160.05,100.05) ;
%Curve Lines [id:da04583686207930704] 
\draw    (145.05,70.05) .. controls (147.67,82.17) and (148.67,90.5) .. (160.05,100.05) ;
%Straight Lines [id:da3980046244655723] 
\draw    (130.05,40.05) -- (115.05,70.05) ;
\draw [shift={(115.05,70.05)}, rotate = 116.57] [color={rgb, 255:red, 0; green, 0; blue, 0 }  ][fill={rgb, 255:red, 0; green, 0; blue, 0 }  ][line width=0.75]      (0, 0) circle [x radius= 2.01, y radius= 2.01]   ;
%Straight Lines [id:da19708099110883892] 
\draw    (130.05,11.06) -- (130.05,40.05) ;
\draw [shift={(130.05,10.05)}, rotate = 90] [color={rgb, 255:red, 0; green, 0; blue, 0 }  ][line width=0.75]      (0, 0) circle [x radius= 2.01, y radius= 2.01]   ;
%Shape: Circle [id:dp8354377432186586] 
\draw  [color={rgb, 255:red, 0; green, 0; blue, 0 }  ,draw opacity=1 ][fill={rgb, 255:red, 255; green, 255; blue, 255 }  ,fill opacity=1 ] (127.8,10.05) .. controls (127.8,8.81) and (128.81,7.8) .. (130.05,7.8) .. controls (131.3,7.8) and (132.3,8.81) .. (132.3,10.05) .. controls (132.3,11.29) and (131.3,12.3) .. (130.05,12.3) .. controls (128.81,12.3) and (127.8,11.29) .. (127.8,10.05) -- cycle ;
%Straight Lines [id:da10004582212314395] 
\draw    (160.05,40.05) -- (130.03,100.05) ;
%Curve Lines [id:da4274132900477702] 
\draw    (160.05,40.05) .. controls (166.37,43.55) and (172.13,83.25) .. (160.05,100.05) ;
\draw [shift={(160.05,100.05)}, rotate = 125.71] [color={rgb, 255:red, 0; green, 0; blue, 0 }  ][fill={rgb, 255:red, 0; green, 0; blue, 0 }  ][line width=0.75]      (0, 0) circle [x radius= 2.01, y radius= 2.01]   ;
%Straight Lines [id:da8320207840029902] 
\draw  [dash pattern={on 1.5pt off 1.5pt}]  (570.11,87) -- (567.56,89.29) ;
\draw [shift={(567.56,89.29)}, rotate = 138.09] [color={rgb, 255:red, 0; green, 0; blue, 0 }  ][line width=0.75]    (-3.35,0) -- (3.35,0)(0,3.35) -- (0,-3.35)   ;
%Straight Lines [id:da38846863033301393] 
\draw  [dash pattern={on 1.5pt off 1.5pt}]  (570.11,87) -- (569.06,87.94) ;
\draw [shift={(567.56,89.29)}, rotate = 138.09] [color={rgb, 255:red, 0; green, 0; blue, 0 }  ][line width=0.75]      (0, 0) circle [x radius= 3.02, y radius= 3.02]   ;
%Curve Lines [id:da5646671304830372] 
\draw  [dash pattern={on 0.75pt off 7.5pt}]  (565.5,87.16) .. controls (571.47,93.39) and (569.33,91.16) .. (567.96,89.72) ;
\draw [shift={(567.56,89.29)}, rotate = 45.79] [color={rgb, 255:red, 0; green, 0; blue, 0 }  ][line width=0.75]      (0, 0) circle [x radius= 2.68, y radius= 2.68]   ;

%Flowchart: Connector [id:dp23275110801811627] 
\draw   (565,74.35) .. controls (565,73.05) and (566.05,72) .. (567.35,72) .. controls (568.65,72) and (569.7,73.05) .. (569.7,74.35) .. controls (569.7,75.65) and (568.65,76.7) .. (567.35,76.7) .. controls (566.05,76.7) and (565,75.65) .. (565,74.35) -- cycle ;
%Flowchart: Connector [id:dp002532435089975138] 
\draw  [fill={rgb, 255:red, 0; green, 0; blue, 0 }  ,fill opacity=1 ] (565,59.35) .. controls (565,58.05) and (566.05,57) .. (567.35,57) .. controls (568.65,57) and (569.7,58.05) .. (569.7,59.35) .. controls (569.7,60.65) and (568.65,61.7) .. (567.35,61.7) .. controls (566.05,61.7) and (565,60.65) .. (565,59.35) -- cycle ;
%Curve Lines [id:da6128207066606559] 
\draw    (130.05,10.05) .. controls (137.7,16.22) and (135.03,37.22) .. (130.05,40.05) ;
%Shape: Circle [id:dp21659872017020232] 
\draw  [color={rgb, 255:red, 0; green, 0; blue, 0 }  ,draw opacity=1 ][fill={rgb, 255:red, 255; green, 255; blue, 255 }  ,fill opacity=1 ] (127.8,10.05) .. controls (127.8,8.81) and (128.81,7.8) .. (130.05,7.8) .. controls (131.3,7.8) and (132.3,8.81) .. (132.3,10.05) .. controls (132.3,11.29) and (131.3,12.3) .. (130.05,12.3) .. controls (128.81,12.3) and (127.8,11.29) .. (127.8,10.05) -- cycle ;
%Straight Lines [id:da5970478399698377] 
\draw    (115.05,70.05) -- (130.05,100.05) ;
\draw [shift={(130.05,100.05)}, rotate = 63.43] [color={rgb, 255:red, 0; green, 0; blue, 0 }  ][fill={rgb, 255:red, 0; green, 0; blue, 0 }  ][line width=0.75]      (0, 0) circle [x radius= 2.01, y radius= 2.01]   ;
%Straight Lines [id:da45984243337751296] 
\draw    (130.05,40.05) -- (130.05,100.05) ;
%Straight Lines [id:da1680327492686795] 
\draw    (115.05,70.05) -- (100.05,100.05) ;
\draw [shift={(100.05,100.05)}, rotate = 116.57] [color={rgb, 255:red, 0; green, 0; blue, 0 }  ][fill={rgb, 255:red, 0; green, 0; blue, 0 }  ][line width=0.75]      (0, 0) circle [x radius= 2.01, y radius= 2.01]   ;
%Curve Lines [id:da10424655436422625] 
\draw    (100.05,100.05) .. controls (112.65,103.25) and (124.65,115.65) .. (130.05,130.05) ;
%Curve Lines [id:da5209328448035667] 
\draw    (100.05,100.05) .. controls (104.65,113.25) and (114.65,124.85) .. (130.05,130.05) ;
%Curve Lines [id:da8941644358625883] 
\draw    (145.04,70.05) .. controls (144.33,89.17) and (141.67,113.5) .. (130.04,130.05) ;
\draw [shift={(130.04,130.05)}, rotate = 125.08] [color={rgb, 255:red, 0; green, 0; blue, 0 }  ][fill={rgb, 255:red, 0; green, 0; blue, 0 }  ][line width=0.75]      (0, 0) circle [x radius= 2.01, y radius= 2.01]   ;
%Curve Lines [id:da7260407758878814] 
\draw    (250.02,10) .. controls (260.67,17.83) and (265.33,57.83) .. (265.01,70) ;
%Curve Lines [id:da12762087422063118] 
\draw    (250.02,10) .. controls (232.33,9.17) and (222,55.17) .. (235.02,70) ;
%Curve Lines [id:da9421840236170799] 
\draw    (250.02,10) .. controls (236.67,14.5) and (230.67,55.83) .. (235.02,70) ;
%Curve Lines [id:da4231731264325873] 
\draw    (250.02,10) .. controls (262.3,15.45) and (274.62,25.6) .. (280.02,40) ;
%Curve Lines [id:da9104579704734496] 
\draw    (250.02,10) .. controls (260.97,9.12) and (281,25.5) .. (280.02,40) ;
%Straight Lines [id:da9535441209949141] 
\draw    (250.02,10) -- (280.02,40) ;
\draw [shift={(280.02,40)}, rotate = 45] [color={rgb, 255:red, 0; green, 0; blue, 0 }  ][fill={rgb, 255:red, 0; green, 0; blue, 0 }  ][line width=0.75]      (0, 0) circle [x radius= 2.01, y radius= 2.01]   ;
%Curve Lines [id:da04782543255633076] 
\draw    (250.02,10) .. controls (242,17.83) and (246,36.83) .. (250.02,40) ;
%Straight Lines [id:da9410233746193143] 
\draw    (250.02,40) -- (265.02,70) ;
\draw [shift={(250.02,40)}, rotate = 63.43] [color={rgb, 255:red, 0; green, 0; blue, 0 }  ][fill={rgb, 255:red, 0; green, 0; blue, 0 }  ][line width=0.75]      (0, 0) circle [x radius= 2.01, y radius= 2.01]   ;
%Straight Lines [id:da9706143249718712] 
\draw    (280.02,40) -- (265.02,70) ;
\draw [shift={(265.02,70)}, rotate = 116.57] [color={rgb, 255:red, 0; green, 0; blue, 0 }  ][fill={rgb, 255:red, 0; green, 0; blue, 0 }  ][line width=0.75]      (0, 0) circle [x radius= 2.01, y radius= 2.01]   ;
%Curve Lines [id:da7546825716596384] 
\draw    (265.02,70) .. controls (273.3,73.78) and (279.97,87.78) .. (280.02,100) ;
%Curve Lines [id:da19776274509669933] 
\draw    (265.02,70) .. controls (267.63,82.12) and (268.63,90.45) .. (280.02,100) ;
%Straight Lines [id:da19496908965967574] 
\draw    (250.02,40) -- (235.02,70) ;
\draw [shift={(235.02,70)}, rotate = 116.57] [color={rgb, 255:red, 0; green, 0; blue, 0 }  ][fill={rgb, 255:red, 0; green, 0; blue, 0 }  ][line width=0.75]      (0, 0) circle [x radius= 2.01, y radius= 2.01]   ;
%Straight Lines [id:da0905276852727196] 
\draw    (250.02,11.01) -- (250.02,40) ;
\draw [shift={(250.02,10)}, rotate = 90] [color={rgb, 255:red, 0; green, 0; blue, 0 }  ][line width=0.75]      (0, 0) circle [x radius= 2.01, y radius= 2.01]   ;
%Straight Lines [id:da8365582517894524] 
\draw    (265.02,70) -- (250.01,100) ;
%Curve Lines [id:da5821114779961544] 
\draw    (280.02,40) .. controls (286.33,43.5) and (292.1,83.2) .. (280.02,100) ;
\draw [shift={(280.02,100)}, rotate = 125.71] [color={rgb, 255:red, 0; green, 0; blue, 0 }  ][fill={rgb, 255:red, 0; green, 0; blue, 0 }  ][line width=0.75]      (0, 0) circle [x radius= 2.01, y radius= 2.01]   ;
%Curve Lines [id:da9908984435101427] 
\draw    (250.02,10) .. controls (257.67,16.17) and (255,37.17) .. (250.02,40) ;
%Shape: Circle [id:dp4159017935056025] 
\draw  [color={rgb, 255:red, 0; green, 0; blue, 0 }  ,draw opacity=1 ][fill={rgb, 255:red, 255; green, 255; blue, 255 }  ,fill opacity=1 ] (247.77,10) .. controls (247.77,8.76) and (248.78,7.75) .. (250.02,7.75) .. controls (251.27,7.75) and (252.27,8.76) .. (252.27,10) .. controls (252.27,11.24) and (251.27,12.25) .. (250.02,12.25) .. controls (248.78,12.25) and (247.77,11.24) .. (247.77,10) -- cycle ;
%Straight Lines [id:da10092832425471532] 
\draw    (235.02,70) -- (250.02,100) ;
\draw [shift={(250.02,100)}, rotate = 63.43] [color={rgb, 255:red, 0; green, 0; blue, 0 }  ][fill={rgb, 255:red, 0; green, 0; blue, 0 }  ][line width=0.75]      (0, 0) circle [x radius= 2.01, y radius= 2.01]   ;
%Straight Lines [id:da4997309439698038] 
\draw    (250.02,40) -- (250.02,100) ;
%Straight Lines [id:da25930681370006314] 
\draw    (235.02,70) -- (220.02,100) ;
\draw [shift={(220.02,100)}, rotate = 116.57] [color={rgb, 255:red, 0; green, 0; blue, 0 }  ][fill={rgb, 255:red, 0; green, 0; blue, 0 }  ][line width=0.75]      (0, 0) circle [x radius= 2.01, y radius= 2.01]   ;
%Curve Lines [id:da430635647564073] 
\draw    (265.02,70) .. controls (264.31,89.12) and (261.65,113.45) .. (250.02,130) ;
\draw [shift={(250.02,130)}, rotate = 125.08] [color={rgb, 255:red, 0; green, 0; blue, 0 }  ][fill={rgb, 255:red, 0; green, 0; blue, 0 }  ][line width=0.75]      (0, 0) circle [x radius= 2.01, y radius= 2.01]   ;
%Curve Lines [id:da13123203956323015] 
\draw    (337.98,100.05) .. controls (350.58,103.25) and (362.58,115.65) .. (367.98,130.05) ;
%Curve Lines [id:da6493026271581913] 
\draw    (337.98,100.05) .. controls (342.58,113.25) and (352.58,124.85) .. (367.98,130.05) ;
%Curve Lines [id:da5731399778998083] 
\draw    (367.98,10.05) .. controls (378.63,17.88) and (383.29,57.88) .. (382.97,70.05) ;
%Curve Lines [id:da8359697902577792] 
\draw [color={rgb, 255:red, 74; green, 144; blue, 226 }  ,draw opacity=1 ]   (367.98,10.05) .. controls (350.29,9.22) and (339.96,55.22) .. (352.98,70.05) ;
%Curve Lines [id:da38426153812547703] 
\draw    (367.98,10.05) .. controls (354.63,14.55) and (348.63,55.88) .. (352.98,70.05) ;
%Curve Lines [id:da43220280495442853] 
\draw    (367.98,10.05) .. controls (380.26,15.5) and (392.58,25.65) .. (397.98,40.05) ;
%Curve Lines [id:da4637103929423254] 
\draw    (367.98,10.05) .. controls (378.93,9.17) and (398.96,25.55) .. (397.98,40.05) ;
%Straight Lines [id:da8729156545570786] 
\draw    (367.98,10.05) -- (397.98,40.05) ;
\draw [shift={(397.98,40.05)}, rotate = 45] [color={rgb, 255:red, 0; green, 0; blue, 0 }  ][fill={rgb, 255:red, 0; green, 0; blue, 0 }  ][line width=0.75]      (0, 0) circle [x radius= 2.01, y radius= 2.01]   ;
%Curve Lines [id:da7018944661397526] 
\draw [color={rgb, 255:red, 74; green, 144; blue, 226 }  ,draw opacity=1 ]   (367.98,10.05) .. controls (359.96,17.88) and (363.96,36.88) .. (367.98,40.05) ;
%Straight Lines [id:da08980623536254517] 
\draw    (367.98,40.05) -- (382.98,70.05) ;
\draw [shift={(367.98,40.05)}, rotate = 63.43] [color={rgb, 255:red, 0; green, 0; blue, 0 }  ][fill={rgb, 255:red, 0; green, 0; blue, 0 }  ][line width=0.75]      (0, 0) circle [x radius= 2.01, y radius= 2.01]   ;
%Straight Lines [id:da8044675750531666] 
\draw    (397.98,40.05) -- (382.98,70.05) ;
\draw [shift={(382.98,70.05)}, rotate = 116.57] [color={rgb, 255:red, 0; green, 0; blue, 0 }  ][fill={rgb, 255:red, 0; green, 0; blue, 0 }  ][line width=0.75]      (0, 0) circle [x radius= 2.01, y radius= 2.01]   ;
%Curve Lines [id:da03080794966423772] 
\draw    (382.98,70.05) .. controls (391.26,73.83) and (397.93,87.83) .. (397.98,100.05) ;
%Curve Lines [id:da09428839266743438] 
\draw    (382.98,70.05) .. controls (385.59,82.17) and (386.59,90.5) .. (397.98,100.05) ;
%Straight Lines [id:da1304394095408531] 
\draw    (367.98,40.05) -- (352.98,70.05) ;
\draw [shift={(352.98,70.05)}, rotate = 116.57] [color={rgb, 255:red, 0; green, 0; blue, 0 }  ][fill={rgb, 255:red, 0; green, 0; blue, 0 }  ][line width=0.75]      (0, 0) circle [x radius= 2.01, y radius= 2.01]   ;
%Straight Lines [id:da376631629196577] 
\draw    (382.98,70.05) -- (367.97,100.05) ;
%Curve Lines [id:da22356533534415313] 
\draw    (397.98,40.05) .. controls (404.29,43.55) and (410.06,83.25) .. (397.98,100.05) ;
\draw [shift={(397.98,100.05)}, rotate = 125.71] [color={rgb, 255:red, 0; green, 0; blue, 0 }  ][fill={rgb, 255:red, 0; green, 0; blue, 0 }  ][line width=0.75]      (0, 0) circle [x radius= 2.01, y radius= 2.01]   ;
%Curve Lines [id:da7641861179121938] 
\draw [color={rgb, 255:red, 74; green, 144; blue, 226 }  ,draw opacity=1 ]   (367.98,10.05) .. controls (375.63,16.22) and (372.96,37.22) .. (367.98,40.05) ;
%Straight Lines [id:da9262764345987208] 
\draw    (352.98,70.05) -- (367.98,100.05) ;
\draw [shift={(367.98,100.05)}, rotate = 63.43] [color={rgb, 255:red, 0; green, 0; blue, 0 }  ][fill={rgb, 255:red, 0; green, 0; blue, 0 }  ][line width=0.75]      (0, 0) circle [x radius= 2.01, y radius= 2.01]   ;
%Straight Lines [id:da1471936460679576] 
\draw    (367.98,40.05) -- (367.98,100.05) ;
%Straight Lines [id:da12225213058978268] 
\draw    (352.98,70.05) -- (337.98,100.05) ;
\draw [shift={(337.98,100.05)}, rotate = 116.57] [color={rgb, 255:red, 0; green, 0; blue, 0 }  ][fill={rgb, 255:red, 0; green, 0; blue, 0 }  ][line width=0.75]      (0, 0) circle [x radius= 2.01, y radius= 2.01]   ;
%Curve Lines [id:da4638913613571485] 
\draw    (382.98,70.05) .. controls (382.27,89.17) and (379.61,113.5) .. (367.98,130.05) ;
\draw [shift={(367.98,130.05)}, rotate = 125.08] [color={rgb, 255:red, 0; green, 0; blue, 0 }  ][fill={rgb, 255:red, 0; green, 0; blue, 0 }  ][line width=0.75]      (0, 0) circle [x radius= 2.01, y radius= 2.01]   ;
%Straight Lines [id:da8152349014760132] 
\draw [color={rgb, 255:red, 0; green, 0; blue, 0 }  ,draw opacity=1 ]   (352.98,70.05) -- (352.96,155.05) ;
%Shape: Circle [id:dp351210825557022] 
\draw  [color={rgb, 255:red, 74; green, 144; blue, 226 }  ,draw opacity=1 ] (348.96,155.05) .. controls (348.96,152.84) and (350.75,151.05) .. (352.96,151.05) .. controls (355.17,151.05) and (356.96,152.84) .. (356.96,155.05) .. controls (356.96,157.26) and (355.17,159.05) .. (352.96,159.05) .. controls (350.75,159.05) and (348.96,157.26) .. (348.96,155.05) -- cycle ;
%Shape: Circle [id:dp21822581627461612] 
\draw  [color={rgb, 255:red, 74; green, 144; blue, 226 }  ,draw opacity=1 ] (363.98,10.05) .. controls (363.98,7.84) and (365.77,6.05) .. (367.98,6.05) .. controls (370.19,6.05) and (371.98,7.84) .. (371.98,10.05) .. controls (371.98,12.26) and (370.19,14.05) .. (367.98,14.05) .. controls (365.77,14.05) and (363.98,12.26) .. (363.98,10.05) -- cycle ;
%Straight Lines [id:da7447499040344979] 
\draw [color={rgb, 255:red, 0; green, 0; blue, 0 }  ,draw opacity=1 ]   (337.98,100.05) -- (352.96,155.05) ;
%Curve Lines [id:da75420740518595] 
\draw [color={rgb, 255:red, 74; green, 144; blue, 226 }  ,draw opacity=1 ]   (367.98,10.05) .. controls (354.63,14.55) and (348.63,55.88) .. (352.98,70.05) ;
%Straight Lines [id:da09506285719622054] 
\draw [color={rgb, 255:red, 74; green, 144; blue, 226 }  ,draw opacity=1 ]   (367.98,40.05) -- (352.98,70.05) ;
%Straight Lines [id:da4305889566575002] 
\draw [color={rgb, 255:red, 74; green, 144; blue, 226 }  ,draw opacity=1 ]   (367.98,10.05) -- (367.98,40.05) ;
%Curve Lines [id:da3663557506669176] 
\draw [color={rgb, 255:red, 0; green, 0; blue, 0 }  ,draw opacity=1 ]   (130.05,40.05) .. controls (113,42.86) and (94.71,76.57) .. (100.05,100.05) ;
%Curve Lines [id:da7495888954041533] 
\draw [color={rgb, 255:red, 0; green, 0; blue, 0 }  ,draw opacity=1 ]   (250.02,40) .. controls (232.97,42.81) and (214.68,76.52) .. (220.02,100) ;
%Curve Lines [id:da30131613833820303] 
\draw [color={rgb, 255:red, 74; green, 144; blue, 226 }  ,draw opacity=1 ]   (367.98,40.05) .. controls (350.93,42.86) and (332.64,76.57) .. (337.98,100.05) ;
%Curve Lines [id:da6768069706979684] 
\draw [color={rgb, 255:red, 74; green, 144; blue, 226 }  ,draw opacity=1 ]   (337.98,100.05) .. controls (332.16,106.45) and (340.76,141.85) .. (352.96,155.05) ;
\draw [shift={(352.96,155.05)}, rotate = 47.25] [color={rgb, 255:red, 74; green, 144; blue, 226 }  ,draw opacity=1 ][fill={rgb, 255:red, 74; green, 144; blue, 226 }  ,fill opacity=1 ][line width=0.75]      (0, 0) circle [x radius= 2.01, y radius= 2.01]   ;
%Shape: Circle [id:dp5910621468147098] 
\draw  [color={rgb, 255:red, 74; green, 144; blue, 226 }  ,draw opacity=1 ] (333.98,100.05) .. controls (333.98,97.84) and (335.77,96.05) .. (337.98,96.05) .. controls (340.19,96.05) and (341.98,97.84) .. (341.98,100.05) .. controls (341.98,102.26) and (340.19,104.05) .. (337.98,104.05) .. controls (335.77,104.05) and (333.98,102.26) .. (333.98,100.05) -- cycle ;
%Straight Lines [id:da26560794205229055] 
\draw [color={rgb, 255:red, 74; green, 144; blue, 226 }  ,draw opacity=1 ]   (352.98,70.05) -- (337.98,100.05) ;
%Curve Lines [id:da917529038871119] 
\draw [color={rgb, 255:red, 0; green, 0; blue, 0 }  ,draw opacity=1 ]   (130.05,40.05) .. controls (113.57,51.14) and (103,78.57) .. (100.05,100.05) ;
%Curve Lines [id:da7298727536250964] 
\draw [color={rgb, 255:red, 0; green, 0; blue, 0 }  ,draw opacity=1 ]   (250.02,40) .. controls (233.54,51.09) and (222.97,78.52) .. (220.02,100) ;
%Curve Lines [id:da39566359849713684] 
\draw [color={rgb, 255:red, 74; green, 144; blue, 226 }  ,draw opacity=1 ]   (367.98,40.05) .. controls (351.5,51.14) and (340.93,78.57) .. (337.98,100.05) ;
%Shape: Circle [id:dp6517846403161589] 
\draw  [color={rgb, 255:red, 74; green, 144; blue, 226 }  ,draw opacity=1 ] (348.98,70.05) .. controls (348.98,67.84) and (350.77,66.05) .. (352.98,66.05) .. controls (355.19,66.05) and (356.98,67.84) .. (356.98,70.05) .. controls (356.98,72.26) and (355.19,74.05) .. (352.98,74.05) .. controls (350.77,74.05) and (348.98,72.26) .. (348.98,70.05) -- cycle ;
%Shape: Circle [id:dp4156168933233688] 
\draw  [color={rgb, 255:red, 74; green, 144; blue, 226 }  ,draw opacity=1 ] (363.98,40.05) .. controls (363.98,37.84) and (365.77,36.05) .. (367.98,36.05) .. controls (370.19,36.05) and (371.98,37.84) .. (371.98,40.05) .. controls (371.98,42.26) and (370.19,44.05) .. (367.98,44.05) .. controls (365.77,44.05) and (363.98,42.26) .. (363.98,40.05) -- cycle ;
%Shape: Circle [id:dp4832840322195707] 
\draw  [color={rgb, 255:red, 0; green, 0; blue, 0 }  ,draw opacity=1 ][fill={rgb, 255:red, 255; green, 255; blue, 255 }  ,fill opacity=1 ] (365.73,10.05) .. controls (365.73,8.81) and (366.74,7.8) .. (367.98,7.8) .. controls (369.23,7.8) and (370.23,8.81) .. (370.23,10.05) .. controls (370.23,11.29) and (369.23,12.3) .. (367.98,12.3) .. controls (366.74,12.3) and (365.73,11.29) .. (365.73,10.05) -- cycle ;
%Curve Lines [id:da5978207149131632] 
\draw    (459.05,100.61) .. controls (470.88,104.57) and (481.88,116.39) .. (486.98,130) ;
\draw [shift={(456.98,100)}, rotate = 14.25] [color={rgb, 255:red, 0; green, 0; blue, 0 }  ][line width=0.75]      (0, 0) circle [x radius= 3.02, y radius= 3.02]   ;
%Curve Lines [id:da12736753389209887] 
\draw    (457.71,101.97) .. controls (462.6,114.34) and (472.35,125.06) .. (486.98,130) ;
\draw [shift={(456.98,100)}, rotate = 70.79] [color={rgb, 255:red, 0; green, 0; blue, 0 }  ][line width=0.75]      (0, 0) circle [x radius= 3.02, y radius= 3.02]   ;
%Curve Lines [id:da15467793513847727] 
\draw    (488.53,11.41) .. controls (498.08,21.76) and (502.28,58.44) .. (501.97,70) ;
\draw [shift={(486.98,10)}, rotate = 36.35] [color={rgb, 255:red, 0; green, 0; blue, 0 }  ][line width=0.75]      (0, 0) circle [x radius= 3.02, y radius= 3.02]   ;
%Curve Lines [id:da247394866777219] 
\draw    (488.82,10.85) .. controls (500.48,16.48) and (511.85,26.32) .. (516.98,40) ;
\draw [shift={(486.98,10)}, rotate = 23.93] [color={rgb, 255:red, 0; green, 0; blue, 0 }  ][line width=0.75]      (0, 0) circle [x radius= 3.02, y radius= 3.02]   ;
%Curve Lines [id:da9031850761592497] 
\draw    (486.98,10) .. controls (497.93,9.12) and (517.96,25.5) .. (516.98,40) ;
\draw [shift={(516.98,40)}, rotate = 93.86] [color={rgb, 255:red, 0; green, 0; blue, 0 }  ][fill={rgb, 255:red, 0; green, 0; blue, 0 }  ][line width=0.75]      (0, 0) circle [x radius= 2.01, y radius= 2.01]   ;
%Straight Lines [id:da10156447130622392] 
\draw    (488.41,11.42) -- (516.98,40) ;
\draw [shift={(486.98,10)}, rotate = 45] [color={rgb, 255:red, 0; green, 0; blue, 0 }  ][line width=0.75]      (0, 0) circle [x radius= 3.02, y radius= 3.02]   ;
%Curve Lines [id:da7678949855180909] 
\draw [color={rgb, 255:red, 0; green, 0; blue, 0 }  ,draw opacity=1 ] [dash pattern={on 1.5pt off 1.5pt}]  (485.67,11.52) .. controls (480.12,19.12) and (482.49,33.39) .. (485.75,38.54) ;
\draw [shift={(486.98,40)}, rotate = 38.21] [color={rgb, 255:red, 0; green, 0; blue, 0 }  ,draw opacity=1 ][line width=0.75]      (0, 0) circle [x radius= 3.02, y radius= 3.02]   ;
\draw [shift={(486.98,10)}, rotate = 135.68] [color={rgb, 255:red, 0; green, 0; blue, 0 }  ,draw opacity=1 ][line width=0.75]      (0, 0) circle [x radius= 3.02, y radius= 3.02]   ;
%Straight Lines [id:da13659117191753378] 
\draw    (487.88,41.8) -- (501.98,70) ;
\draw [shift={(486.98,40)}, rotate = 63.43] [color={rgb, 255:red, 0; green, 0; blue, 0 }  ][line width=0.75]      (0, 0) circle [x radius= 3.02, y radius= 3.02]   ;
%Straight Lines [id:da3891367008622876] 
\draw    (516.98,40) -- (501.98,70) ;
\draw [shift={(501.98,70)}, rotate = 116.57] [color={rgb, 255:red, 0; green, 0; blue, 0 }  ][fill={rgb, 255:red, 0; green, 0; blue, 0 }  ][line width=0.75]      (0, 0) circle [x radius= 2.01, y radius= 2.01]   ;
%Curve Lines [id:da4937437866602896] 
\draw    (501.98,70) .. controls (510.26,73.78) and (516.93,87.78) .. (516.98,100) ;
%Curve Lines [id:da9630096159979658] 
\draw    (501.98,70) .. controls (504.59,82.12) and (505.59,90.45) .. (516.98,100) ;
%Straight Lines [id:da5579333495831857] 
\draw  [dash pattern={on 1.5pt off 1.5pt}]  (486.08,41.8) -- (472.88,68.2) ;
\draw [shift={(471.98,70)}, rotate = 116.57] [color={rgb, 255:red, 0; green, 0; blue, 0 }  ][line width=0.75]      (0, 0) circle [x radius= 3.02, y radius= 3.02]   ;
\draw [shift={(486.98,40)}, rotate = 116.57] [color={rgb, 255:red, 0; green, 0; blue, 0 }  ][line width=0.75]      (0, 0) circle [x radius= 3.02, y radius= 3.02]   ;
%Straight Lines [id:da3959510274186331] 
\draw    (501.98,70) -- (486.97,100) ;
%Curve Lines [id:da9302129385494041] 
\draw    (516.98,40) .. controls (523.29,43.5) and (529.06,83.2) .. (516.98,100) ;
\draw [shift={(516.98,100)}, rotate = 125.71] [color={rgb, 255:red, 0; green, 0; blue, 0 }  ][fill={rgb, 255:red, 0; green, 0; blue, 0 }  ][line width=0.75]      (0, 0) circle [x radius= 2.01, y radius= 2.01]   ;
%Curve Lines [id:da58977738947374] 
\draw [color={rgb, 255:red, 0; green, 0; blue, 0 }  ,draw opacity=1 ] [dash pattern={on 1.5pt off 1.5pt}]  (488.44,11.5) .. controls (493.63,18.23) and (492.07,33.63) .. (488.37,38.68) ;
\draw [shift={(486.98,40)}, rotate = 150.35] [color={rgb, 255:red, 0; green, 0; blue, 0 }  ,draw opacity=1 ][line width=0.75]      (0, 0) circle [x radius= 3.02, y radius= 3.02]   ;
\draw [shift={(486.98,10)}, rotate = 38.89] [color={rgb, 255:red, 0; green, 0; blue, 0 }  ,draw opacity=1 ][line width=0.75]      (0, 0) circle [x radius= 3.02, y radius= 3.02]   ;
%Straight Lines [id:da38482432219483786] 
\draw    (471.98,70) -- (486.98,100) ;
\draw [shift={(486.98,100)}, rotate = 63.43] [color={rgb, 255:red, 0; green, 0; blue, 0 }  ][fill={rgb, 255:red, 0; green, 0; blue, 0 }  ][line width=0.75]      (0, 0) circle [x radius= 2.01, y radius= 2.01]   ;
\draw [shift={(471.98,70)}, rotate = 63.43] [color={rgb, 255:red, 0; green, 0; blue, 0 }  ][line width=0.75]    (-3.35,0) -- (3.35,0)(0,3.35) -- (0,-3.35)   ;
%Straight Lines [id:da18366851448533494] 
\draw    (486.98,41.68) -- (486.98,100) ;
\draw [shift={(486.98,40)}, rotate = 90] [color={rgb, 255:red, 0; green, 0; blue, 0 }  ][line width=0.75]      (0, 0) circle [x radius= 2.68, y radius= 2.68]   ;
%Straight Lines [id:da1765829800566059] 
\draw [color={rgb, 255:red, 0; green, 0; blue, 0 }  ,draw opacity=1 ] [dash pattern={on 1.5pt off 1.5pt}]  (471.08,71.8) -- (457.88,98.2) ;
\draw [shift={(456.98,100)}, rotate = 116.57] [color={rgb, 255:red, 0; green, 0; blue, 0 }  ,draw opacity=1 ][line width=0.75]      (0, 0) circle [x radius= 3.02, y radius= 3.02]   ;
\draw [shift={(471.98,70)}, rotate = 116.57] [color={rgb, 255:red, 0; green, 0; blue, 0 }  ,draw opacity=1 ][line width=0.75]      (0, 0) circle [x radius= 3.02, y radius= 3.02]   ;
%Curve Lines [id:da07538617351947818] 
\draw    (501.98,70) .. controls (501.27,89.12) and (498.61,113.45) .. (486.98,130) ;
\draw [shift={(486.98,130)}, rotate = 125.08] [color={rgb, 255:red, 0; green, 0; blue, 0 }  ][fill={rgb, 255:red, 0; green, 0; blue, 0 }  ][line width=0.75]      (0, 0) circle [x radius= 2.01, y radius= 2.01]   ;
%Straight Lines [id:da29382110985668564] 
\draw [color={rgb, 255:red, 0; green, 0; blue, 0 }  ,draw opacity=1 ] [dash pattern={on 1.5pt off 1.5pt}]  (471.67,72.89) -- (471.95,153.32) ;
\draw [shift={(471.96,155)}, rotate = 89.8] [color={rgb, 255:red, 0; green, 0; blue, 0 }  ,draw opacity=1 ][line width=0.75]      (0, 0) circle [x radius= 2.68, y radius= 2.68]   ;
%Straight Lines [id:da3883584656211725] 
\draw [color={rgb, 255:red, 0; green, 0; blue, 0 }  ,draw opacity=1 ] [dash pattern={on 1.5pt off 1.5pt}]  (456.98,100) -- (471.96,155) ;
\draw [shift={(471.96,155)}, rotate = 74.77] [color={rgb, 255:red, 0; green, 0; blue, 0 }  ,draw opacity=1 ][line width=0.75]    (-3.35,0) -- (3.35,0)(0,3.35) -- (0,-3.35)   ;
\draw [shift={(456.98,100)}, rotate = 74.77] [color={rgb, 255:red, 0; green, 0; blue, 0 }  ,draw opacity=1 ][line width=0.75]    (-3.35,0) -- (3.35,0)(0,3.35) -- (0,-3.35)   ;
%Curve Lines [id:da7688878387384285] 
\draw [color={rgb, 255:red, 0; green, 0; blue, 0 }  ,draw opacity=1 ] [dash pattern={on 0.75pt off 37.5pt}]  (485.03,10.94) .. controls (473.53,18.26) and (468.27,53.55) .. (471.53,68.27) ;
\draw [shift={(471.98,70)}, rotate = 72.91] [color={rgb, 255:red, 0; green, 0; blue, 0 }  ,draw opacity=1 ][line width=0.75]      (0, 0) circle [x radius= 3.02, y radius= 3.02]   ;
\draw [shift={(486.98,10)}, rotate = 161.38] [color={rgb, 255:red, 0; green, 0; blue, 0 }  ,draw opacity=1 ][line width=0.75]      (0, 0) circle [x radius= 3.02, y radius= 3.02]   ;
%Straight Lines [id:da6275902895879706] 
\draw [color={rgb, 255:red, 0; green, 0; blue, 0 }  ,draw opacity=1 ] [dash pattern={on 1.5pt off 1.5pt}]  (486.98,10) -- (486.98,40) ;
\draw [shift={(486.98,40)}, rotate = 90] [color={rgb, 255:red, 0; green, 0; blue, 0 }  ,draw opacity=1 ][line width=0.75]    (-3.35,0) -- (3.35,0)(0,3.35) -- (0,-3.35)   ;
\draw [shift={(486.98,10)}, rotate = 90] [color={rgb, 255:red, 0; green, 0; blue, 0 }  ,draw opacity=1 ][line width=0.75]    (-3.35,0) -- (3.35,0)(0,3.35) -- (0,-3.35)   ;
%Curve Lines [id:da03406204945675251] 
\draw [color={rgb, 255:red, 0; green, 0; blue, 0 }  ,draw opacity=1 ] [dash pattern={on 1.5pt off 1.5pt}]  (485.17,39.98) .. controls (470.81,41.24) and (452.6,76.03) .. (456.7,98.63) ;
\draw [shift={(456.98,100)}, rotate = 77.19] [color={rgb, 255:red, 0; green, 0; blue, 0 }  ,draw opacity=1 ][line width=0.75]      (0, 0) circle [x radius= 2.68, y radius= 2.68]   ;
\draw [shift={(486.98,40)}, rotate = 186.82] [color={rgb, 255:red, 0; green, 0; blue, 0 }  ,draw opacity=1 ][line width=0.75]      (0, 0) circle [x radius= 2.68, y radius= 2.68]   ;
%Curve Lines [id:da35218787011865893] 
\draw [color={rgb, 255:red, 0; green, 0; blue, 0 }  ,draw opacity=1 ] [dash pattern={on 1.5pt off 1.5pt}]  (455.97,101.75) .. controls (452.51,110.91) and (460.05,140.55) .. (470.78,153.64) ;
\draw [shift={(471.96,155)}, rotate = 47.25] [color={rgb, 255:red, 0; green, 0; blue, 0 }  ,draw opacity=1 ][line width=0.75]      (0, 0) circle [x radius= 3.02, y radius= 3.02]   ;
\draw [shift={(456.98,100)}, rotate = 132.3] [color={rgb, 255:red, 0; green, 0; blue, 0 }  ,draw opacity=1 ][line width=0.75]      (0, 0) circle [x radius= 3.02, y radius= 3.02]   ;
%Curve Lines [id:da5898054453167948] 
\draw [color={rgb, 255:red, 0; green, 0; blue, 0 }  ,draw opacity=1 ] [dash pattern={on 1.5pt off 1.5pt}]  (485.27,41.22) .. controls (470.23,52.65) and (460.43,77.71) .. (457.26,98.12) ;
\draw [shift={(456.98,100)}, rotate = 97.81] [color={rgb, 255:red, 0; green, 0; blue, 0 }  ,draw opacity=1 ][line width=0.75]      (0, 0) circle [x radius= 3.02, y radius= 3.02]   ;
\draw [shift={(486.98,40)}, rotate = 146.06] [color={rgb, 255:red, 0; green, 0; blue, 0 }  ,draw opacity=1 ][line width=0.75]      (0, 0) circle [x radius= 3.02, y radius= 3.02]   ;
%Curve Lines [id:da09644571760282883] 
\draw [color={rgb, 255:red, 0; green, 0; blue, 0 }  ,draw opacity=1 ] [dash pattern={on 1.5pt off 1.5pt}]  (485.07,9.95) .. controls (465.57,10.77) and (462.26,52.88) .. (471.13,68.62) ;
\draw [shift={(471.98,70)}, rotate = 55.61] [color={rgb, 255:red, 0; green, 0; blue, 0 }  ,draw opacity=1 ][line width=0.75]      (0, 0) circle [x radius= 2.68, y radius= 2.68]   ;
\draw [shift={(486.98,10)}, rotate = 185.23] [color={rgb, 255:red, 0; green, 0; blue, 0 }  ,draw opacity=1 ][line width=0.75]      (0, 0) circle [x radius= 2.68, y radius= 2.68]   ;

% Text Node
\draw (88,10.5) node [anchor=north west][inner sep=0.75pt]  [font=\small] [align=left] {(a)};
% Text Node
\draw (208,10.5) node [anchor=north west][inner sep=0.75pt]  [font=\small] [align=left] {(b)};
% Text Node
\draw (326,10.5) node [anchor=north west][inner sep=0.75pt]  [font=\small] [align=left] {(c)};
% Text Node
\draw (445,10.5) node [anchor=north west][inner sep=0.75pt]  [font=\small] [align=left] {(d)};
% Text Node
\draw (575,54) node [anchor=north west][inner sep=0.75pt]  [font=\small] [align=left] {\CT};
% Text Node
\draw (575,69) node [anchor=north west][inner sep=0.75pt]  [font=\small] [align=left] {\CF};
% Text Node
\draw (575,84) node [anchor=north west][inner sep=0.75pt]  [font=\small] [align=left] {\PF};

\end{tikzpicture}

%% file: Visuals/BFS-partition.tex
\begin{tikzpicture}[x=0.85pt,y=0.85pt,yscale=-1,xscale=1]
%uncomment if require: \path (0,214); %set diagram left start at 0, and has height of 214

%Curve Lines [id:da5480550902222248] 
\draw    (320.02,100) .. controls (322,111.88) and (312.25,126.13) .. (305,129.75) ;
%Curve Lines [id:da35155108068761287] 
\draw    (275,130) .. controls (285,118.63) and (305.5,103.88) .. (319.98,100) ;
%Straight Lines [id:da3476586192139437] 
\draw    (320.02,100) -- (305,129.75) ;
%Straight Lines [id:da741250867901395] 
\draw    (290.02,100) -- (275,130) ;
%Straight Lines [id:da2940953531414774] 
\draw [color={rgb, 255:red, 0; green, 0; blue, 0 }  ,draw opacity=1 ]   (320.02,100) -- (335,130) ;
%Curve Lines [id:da8011441569218617] 
\draw    (335,130) .. controls (317.77,119.87) and (304.75,86.88) .. (305.02,70) ;
%Straight Lines [id:da4333215642960052] 
\draw    (334.98,70) -- (319.98,100) ;
%Straight Lines [id:da2900224460359532] 
\draw    (304.98,40) -- (290.02,100) ;
%Straight Lines [id:da0033306825567213094] 
\draw [color={rgb, 255:red, 0; green, 0; blue, 0 }  ,draw opacity=1 ]   (275.02,40) -- (275,70) ;
%Curve Lines [id:da6940238237430157] 
\draw    (275.02,40) .. controls (283,49.38) and (279,66.63) .. (275,70) ;
%Curve Lines [id:da18983199234969483] 
\draw  [dash pattern={on 1.5pt off 1.5pt}]  (336,43) .. controls (337,50.38) and (343.25,66.63) .. (354.98,70) ;
%Straight Lines [id:da5599649948964479] 
\draw [color={rgb, 255:red, 0; green, 0; blue, 0 }  ,draw opacity=1 ]   (160.5,110) -- (216.5,130) ;
%Straight Lines [id:da3647696387718755] 
\draw    (90,10) -- (93.36,30.13) -- (105.02,100) ;
%Curve Lines [id:da9311021273235328] 
\draw    (90.02,10) .. controls (103.44,12.33) and (114.56,27.89) .. (120.02,40) ;
%Curve Lines [id:da11200335963724284] 
\draw    (196.67,85) .. controls (201.44,93.67) and (201.42,101.63) .. (196.67,110) ;
%Straight Lines [id:da1725613326687946] 
\draw [color={rgb, 255:red, 0; green, 0; blue, 0 }  ,draw opacity=1 ]   (180.02,100) -- (180,130) ;
%Straight Lines [id:da8691755921965871] 
\draw    (180.02,70) -- (196.67,85) ;
%Straight Lines [id:da09613442507468672] 
\draw    (75.02,100) -- (105.02,130) ;
%Straight Lines [id:da32618706209453074] 
\draw    (105.02,100) -- (70.2,160.2) ;
%Curve Lines [id:da17157616159067313] 
\draw    (105.02,130) .. controls (107.8,137) and (107.8,151) .. (104.22,160) ;
%Curve Lines [id:da7764205712082841] 
\draw    (60,70) .. controls (49.4,86.2) and (59.4,144.2) .. (70.2,160.2) ;
%Straight Lines [id:da604131069967448] 
\draw [color={rgb, 255:red, 144; green, 19; blue, 254 }  ,draw opacity=1 ]   (75.02,130) -- (70.2,160.2) ;
\draw [shift={(70.2,160.2)}, rotate = 99.07] [color={rgb, 255:red, 144; green, 19; blue, 254 }  ,draw opacity=1 ][fill={rgb, 255:red, 144; green, 19; blue, 254 }  ,fill opacity=1 ][line width=0.75]      (0, 0) circle [x radius= 2.01, y radius= 2.01]   ;
%Curve Lines [id:da8000816653352497] 
\draw    (74.22,130) .. controls (86.82,133.2) and (98.82,145.6) .. (104.22,160) ;
%Curve Lines [id:da9141707162565932] 
\draw [color={rgb, 255:red, 144; green, 19; blue, 254 }  ,draw opacity=1 ]   (75.02,130) .. controls (79.62,143.2) and (88.82,154.8) .. (104.22,160) ;
\draw [shift={(104.22,160)}, rotate = 18.66] [color={rgb, 255:red, 144; green, 19; blue, 254 }  ,draw opacity=1 ][fill={rgb, 255:red, 144; green, 19; blue, 254 }  ,fill opacity=1 ][line width=0.75]      (0, 0) circle [x radius= 2.01, y radius= 2.01]   ;
%Curve Lines [id:da043000915108463045] 
\draw    (120.02,40) .. controls (104.2,45) and (92.2,55.8) .. (90,70) ;
%Straight Lines [id:da4019224173881475] 
\draw    (135.02,100) -- (105.02,130) ;
%Straight Lines [id:da33498672846009614] 
\draw    (105.02,100) -- (105.02,128.99) ;
\draw [shift={(105.02,130)}, rotate = 90] [color={rgb, 255:red, 0; green, 0; blue, 0 }  ][line width=0.75]      (0, 0) circle [x radius= 2.01, y radius= 2.01]   ;
%Straight Lines [id:da5074943069673798] 
\draw    (150.02,70) -- (160.5,110) ;
%Curve Lines [id:da720108822004515] 
\draw    (60.02,70) .. controls (55.4,87.8) and (61,123.63) .. (75.02,130) ;
%Straight Lines [id:da4859822281221624] 
\draw    (75.02,100) -- (75.02,128.99) ;
\draw [shift={(75.02,130)}, rotate = 90] [color={rgb, 255:red, 0; green, 0; blue, 0 }  ][line width=0.75]      (0, 0) circle [x radius= 2.01, y radius= 2.01]   ;
%Straight Lines [id:da628108333833968] 
\draw    (180.02,100) -- (160.5,110) ;
%Straight Lines [id:da2221233053852396] 
\draw    (90,70) -- (120.02,40) ;
%Straight Lines [id:da2461387724153591] 
\draw    (180,160) -- (135.02,100) ;
%Curve Lines [id:da6347661613835701] 
\draw    (180,130) .. controls (185.75,140.63) and (184.75,151.63) .. (180,160) ;
%Curve Lines [id:da6008356375225103] 
\draw    (120.02,70) .. controls (119,87.13) and (108.75,95.63) .. (105.02,100) ;
%Straight Lines [id:da39934379911444273] 
\draw    (150.02,40) -- (150.02,70) ;
\draw [shift={(150.02,70)}, rotate = 90] [color={rgb, 255:red, 0; green, 0; blue, 0 }  ][fill={rgb, 255:red, 0; green, 0; blue, 0 }  ][line width=0.75]      (0, 0) circle [x radius= 2.01, y radius= 2.01]   ;
%Straight Lines [id:da6355063623345508] 
\draw    (150.02,10) -- (150.02,40) ;
\draw [shift={(150.02,40)}, rotate = 90] [color={rgb, 255:red, 0; green, 0; blue, 0 }  ][fill={rgb, 255:red, 0; green, 0; blue, 0 }  ][line width=0.75]      (0, 0) circle [x radius= 2.01, y radius= 2.01]   ;
\draw [shift={(150.02,10)}, rotate = 90] [color={rgb, 255:red, 0; green, 0; blue, 0 }  ][fill={rgb, 255:red, 0; green, 0; blue, 0 }  ][line width=0.75]      (0, 0) circle [x radius= 2.01, y radius= 2.01]   ;
%Curve Lines [id:da7285666957288941] 
\draw    (120.02,10) .. controls (132.62,13.2) and (144.62,25.6) .. (150.02,40) ;
\draw [shift={(120.02,10)}, rotate = 14.25] [color={rgb, 255:red, 0; green, 0; blue, 0 }  ][fill={rgb, 255:red, 0; green, 0; blue, 0 }  ][line width=0.75]      (0, 0) circle [x radius= 2.01, y radius= 2.01]   ;
%Straight Lines [id:da7130535616681685] 
\draw    (150.02,10) -- (120.74,39.29) ;
\draw [shift={(120.02,40)}, rotate = 135] [color={rgb, 255:red, 0; green, 0; blue, 0 }  ][line width=0.75]      (0, 0) circle [x radius= 2.01, y radius= 2.01]   ;
%Curve Lines [id:da7324945070806188] 
\draw    (90.02,10) .. controls (104.33,19) and (115,47.44) .. (120.02,70) ;
%Straight Lines [id:da7077584780141983] 
\draw    (90.02,10) -- (120.02,70) ;
%Curve Lines [id:da2590068029293493] 
\draw    (150.02,10) .. controls (161.82,21.6) and (163.82,53.2) .. (150.02,70) ;
%Straight Lines [id:da11426999387570991] 
\draw    (150.02,70) -- (135.02,100) ;
\draw [shift={(135.02,100)}, rotate = 116.57] [color={rgb, 255:red, 0; green, 0; blue, 0 }  ][fill={rgb, 255:red, 0; green, 0; blue, 0 }  ][line width=0.75]      (0, 0) circle [x radius= 2.01, y radius= 2.01]   ;
%Curve Lines [id:da4076953626568979] 
\draw    (105.02,100) .. controls (117.62,103.2) and (129.62,115.6) .. (135.02,130) ;
%Straight Lines [id:da7863740289000253] 
\draw [color={rgb, 255:red, 74; green, 144; blue, 226 }  ,draw opacity=1 ]   (120.02,41.01) -- (120.02,70) ;
\draw [shift={(120.02,40)}, rotate = 90] [color={rgb, 255:red, 74; green, 144; blue, 226 }  ,draw opacity=1 ][line width=0.75]      (0, 0) circle [x radius= 2.01, y radius= 2.01]   ;
%Straight Lines [id:da3968960365497921] 
\draw    (120.02,10) -- (120.02,38.99) ;
\draw [shift={(120.02,40)}, rotate = 90] [color={rgb, 255:red, 0; green, 0; blue, 0 }  ][line width=0.75]      (0, 0) circle [x radius= 2.01, y radius= 2.01]   ;
%Shape: Circle [id:dp21563519910012097] 
\draw  [color={rgb, 255:red, 0; green, 0; blue, 0 }  ,draw opacity=1 ][fill={rgb, 255:red, 255; green, 255; blue, 255 }  ,fill opacity=1 ] (117.77,40) .. controls (117.77,38.76) and (118.78,37.75) .. (120.02,37.75) .. controls (121.27,37.75) and (122.27,38.76) .. (122.27,40) .. controls (122.27,41.24) and (121.27,42.25) .. (120.02,42.25) .. controls (118.78,42.25) and (117.77,41.24) .. (117.77,40) -- cycle ;
%Straight Lines [id:da6133852584064381] 
\draw    (150.02,70) -- (105.02,100) ;
%Curve Lines [id:da2894304085900521] 
\draw    (150.02,70) .. controls (150.7,88.8) and (147.1,113.2) .. (135.02,130) ;
%Straight Lines [id:da03057126428179191] 
\draw    (120.02,10) -- (150.02,40) ;
%Curve Lines [id:da39809651092913234] 
\draw    (120.02,10) .. controls (126.5,31.25) and (135.5,71.75) .. (135.02,100) ;
%Straight Lines [id:da2064280333844586] 
\draw    (120.02,10) -- (150.02,70) ;
%Straight Lines [id:da09345250811394767] 
\draw    (60,41.01) -- (60,70) ;
\draw [shift={(60,40)}, rotate = 90] [color={rgb, 255:red, 0; green, 0; blue, 0 }  ][line width=0.75]      (0, 0) circle [x radius= 2.01, y radius= 2.01]   ;
%Straight Lines [id:da7239086125700329] 
\draw    (90,10) -- (60.71,39.29) ;
\draw [shift={(60,40)}, rotate = 135] [color={rgb, 255:red, 0; green, 0; blue, 0 }  ][line width=0.75]      (0, 0) circle [x radius= 2.01, y radius= 2.01]   ;
%Curve Lines [id:da5764192133464557] 
\draw    (90.02,10) .. controls (83.67,22.56) and (73.44,36.56) .. (60,40) ;
\draw [shift={(90.02,10)}, rotate = 116.85] [color={rgb, 255:red, 0; green, 0; blue, 0 }  ][fill={rgb, 255:red, 0; green, 0; blue, 0 }  ][line width=0.75]      (0, 0) circle [x radius= 2.01, y radius= 2.01]   ;
%Curve Lines [id:da503361444970354] 
\draw    (90,10) .. controls (78,12.5) and (63.33,25.17) .. (59.98,40) ;
\draw [shift={(90,10)}, rotate = 168.23] [color={rgb, 255:red, 0; green, 0; blue, 0 }  ][fill={rgb, 255:red, 0; green, 0; blue, 0 }  ][line width=0.75]      (0, 0) circle [x radius= 2.01, y radius= 2.01]   ;
%Straight Lines [id:da44681318341868503] 
\draw    (120.02,70) -- (75.02,100) ;
%Curve Lines [id:da10567323165136344] 
\draw    (60.02,70) .. controls (69.4,76.6) and (75,87.8) .. (75.02,100) ;
%Curve Lines [id:da631231640799846] 
\draw [color={rgb, 255:red, 208; green, 2; blue, 27 }  ,draw opacity=1 ]   (60.02,70) .. controls (61.4,83.4) and (66.33,94.83) .. (75.02,100) ;
\draw [shift={(75.02,100)}, rotate = 30.74] [color={rgb, 255:red, 208; green, 2; blue, 27 }  ,draw opacity=1 ][fill={rgb, 255:red, 208; green, 2; blue, 27 }  ,fill opacity=1 ][line width=0.75]      (0, 0) circle [x radius= 2.01, y radius= 2.01]   ;
%Straight Lines [id:da22187932814810174] 
\draw    (135.02,100) -- (180,130) ;
%Straight Lines [id:da5026111279327413] 
\draw    (150.02,40) -- (180.02,70) ;
%Curve Lines [id:da17035264979451892] 
\draw    (150.02,70) .. controls (161.5,75.88) and (173.5,86.63) .. (180.02,100) ;
%Curve Lines [id:da6480816617374814] 
\draw    (150.02,70) .. controls (155.25,78.63) and (166,93.63) .. (180.02,100) ;
%Straight Lines [id:da930674894552275] 
\draw    (208.43,12.64) -- (180.02,70) ;
\draw [shift={(180.02,70)}, rotate = 116.35] [color={rgb, 255:red, 0; green, 0; blue, 0 }  ][fill={rgb, 255:red, 0; green, 0; blue, 0 }  ][line width=0.75]      (0, 0) circle [x radius= 2.01, y radius= 2.01]   ;
%Curve Lines [id:da26265247067744824] 
\draw    (90,10) .. controls (87.44,27.22) and (72.33,58.78) .. (60.02,70) ;
%Curve Lines [id:da7343927095295157] 
\draw [color={rgb, 255:red, 208; green, 2; blue, 27 }  ,draw opacity=1 ]   (59.98,40) .. controls (52.75,46.88) and (52.25,62.63) .. (60,70) ;
\draw [shift={(60,70)}, rotate = 43.58] [color={rgb, 255:red, 208; green, 2; blue, 27 }  ,draw opacity=1 ][fill={rgb, 255:red, 208; green, 2; blue, 27 }  ,fill opacity=1 ][line width=0.75]      (0, 0) circle [x radius= 2.01, y radius= 2.01]   ;
%Straight Lines [id:da4352022984513214] 
\draw [color={rgb, 255:red, 74; green, 144; blue, 226 }  ,draw opacity=1 ]   (120.02,70) -- (135.02,100) ;
\draw [shift={(120.02,70)}, rotate = 63.43] [color={rgb, 255:red, 74; green, 144; blue, 226 }  ,draw opacity=1 ][fill={rgb, 255:red, 74; green, 144; blue, 226 }  ,fill opacity=1 ][line width=0.75]      (0, 0) circle [x radius= 2.01, y radius= 2.01]   ;
%Shape: Circle [id:dp9914029261678647] 
\draw  [color={rgb, 255:red, 74; green, 144; blue, 226 }  ,draw opacity=1 ][fill={rgb, 255:red, 74; green, 144; blue, 226 }  ,fill opacity=1 ] (132.77,100) .. controls (132.77,98.76) and (133.78,97.75) .. (135.02,97.75) .. controls (136.27,97.75) and (137.27,98.76) .. (137.27,100) .. controls (137.27,101.24) and (136.27,102.25) .. (135.02,102.25) .. controls (133.78,102.25) and (132.77,101.24) .. (132.77,100) -- cycle ;
%Straight Lines [id:da3977820797637128] 
\draw [color={rgb, 255:red, 74; green, 144; blue, 226 }  ,draw opacity=1 ]   (120.02,70) -- (105.02,100) ;
\draw [shift={(105.02,100)}, rotate = 116.57] [color={rgb, 255:red, 74; green, 144; blue, 226 }  ,draw opacity=1 ][fill={rgb, 255:red, 74; green, 144; blue, 226 }  ,fill opacity=1 ][line width=0.75]      (0, 0) circle [x radius= 2.01, y radius= 2.01]   ;
%Curve Lines [id:da3573321530971544] 
\draw [color={rgb, 255:red, 74; green, 144; blue, 226 }  ,draw opacity=1 ]   (105.02,100) .. controls (109.62,113.2) and (119.62,124.8) .. (135.02,130) ;
\draw [shift={(135.02,130)}, rotate = 18.66] [color={rgb, 255:red, 74; green, 144; blue, 226 }  ,draw opacity=1 ][fill={rgb, 255:red, 74; green, 144; blue, 226 }  ,fill opacity=1 ][line width=0.75]      (0, 0) circle [x radius= 2.01, y radius= 2.01]   ;
%Straight Lines [id:da2984405380808083] 
\draw [color={rgb, 255:red, 120; green, 190; blue, 36 }  ,draw opacity=1 ]   (180,42.63) -- (180.02,70) ;
\draw [shift={(180.02,70)}, rotate = 89.95] [color={rgb, 255:red, 120; green, 190; blue, 36 }  ,draw opacity=1 ][fill={rgb, 255:red, 120; green, 190; blue, 36 }  ,fill opacity=1 ][line width=0.75]      (0, 0) circle [x radius= 2.01, y radius= 2.01]   ;
%Straight Lines [id:da8741033820397048] 
\draw [color={rgb, 255:red, 120; green, 190; blue, 36 }  ,draw opacity=1 ]   (180.02,70) -- (180.02,100) ;
\draw [shift={(180.02,100)}, rotate = 90] [color={rgb, 255:red, 120; green, 190; blue, 36 }  ,draw opacity=1 ][fill={rgb, 255:red, 120; green, 190; blue, 36 }  ,fill opacity=1 ][line width=0.75]      (0, 0) circle [x radius= 2.01, y radius= 2.01]   ;
%Curve Lines [id:da6813077178459005] 
\draw [color={rgb, 255:red, 245; green, 166; blue, 35 }  ,draw opacity=1 ]   (196.67,85) .. controls (215.57,98.43) and (190.14,155) .. (180,160) ;
\draw [shift={(180,160)}, rotate = 153.76] [color={rgb, 255:red, 245; green, 166; blue, 35 }  ,draw opacity=1 ][fill={rgb, 255:red, 245; green, 166; blue, 35 }  ,fill opacity=1 ][line width=0.75]      (0, 0) circle [x radius= 2.01, y radius= 2.01]   ;
%Straight Lines [id:da5106475124018734] 
\draw [color={rgb, 255:red, 208; green, 2; blue, 27 }  ,draw opacity=1 ]   (60,40) -- (90,70) ;
\draw [shift={(90,70)}, rotate = 45] [color={rgb, 255:red, 208; green, 2; blue, 27 }  ,draw opacity=1 ][fill={rgb, 255:red, 208; green, 2; blue, 27 }  ,fill opacity=1 ][line width=0.75]      (0, 0) circle [x radius= 2.01, y radius= 2.01]   ;
%Shape: Circle [id:dp8231957622737605] 
\draw  [color={rgb, 255:red, 0; green, 0; blue, 0 }  ,draw opacity=1 ][fill={rgb, 255:red, 255; green, 255; blue, 255 }  ,fill opacity=1 ] (57.75,40) .. controls (57.75,38.76) and (58.76,37.75) .. (60,37.75) .. controls (61.24,37.75) and (62.25,38.76) .. (62.25,40) .. controls (62.25,41.24) and (61.24,42.25) .. (60,42.25) .. controls (58.76,42.25) and (57.75,41.24) .. (57.75,40) -- cycle ;
%Straight Lines [id:da060329360704717616] 
\draw [color={rgb, 255:red, 120; green, 190; blue, 36 }  ,draw opacity=1 ]   (180.02,70) -- (160.5,110) ;
\draw [shift={(160.5,110)}, rotate = 116.02] [color={rgb, 255:red, 120; green, 190; blue, 36 }  ,draw opacity=1 ][fill={rgb, 255:red, 120; green, 190; blue, 36 }  ,fill opacity=1 ][line width=0.75]      (0, 0) circle [x radius= 2.01, y radius= 2.01]   ;
%Shape: Circle [id:dp8357889276581049] 
\draw  [color={rgb, 255:red, 0; green, 0; blue, 0 }  ,draw opacity=1 ][fill={rgb, 255:red, 255; green, 255; blue, 255 }  ,fill opacity=1 ] (72.77,130) .. controls (72.77,128.76) and (73.78,127.75) .. (75.02,127.75) .. controls (76.27,127.75) and (77.27,128.76) .. (77.27,130) .. controls (77.27,131.24) and (76.27,132.25) .. (75.02,132.25) .. controls (73.78,132.25) and (72.77,131.24) .. (72.77,130) -- cycle ;
%Shape: Circle [id:dp7991864735761643] 
\draw  [color={rgb, 255:red, 0; green, 0; blue, 0 }  ,draw opacity=1 ][fill={rgb, 255:red, 255; green, 255; blue, 255 }  ,fill opacity=1 ] (102.77,130) .. controls (102.77,128.76) and (103.78,127.75) .. (105.02,127.75) .. controls (106.27,127.75) and (107.27,128.76) .. (107.27,130) .. controls (107.27,131.24) and (106.27,132.25) .. (105.02,132.25) .. controls (103.78,132.25) and (102.77,131.24) .. (102.77,130) -- cycle ;
%Straight Lines [id:da9684233408069973] 
\draw [color={rgb, 255:red, 0; green, 0; blue, 0 }  ,draw opacity=1 ]   (196.67,110) -- (216.5,130) ;
%Straight Lines [id:da49822969740892553] 
\draw [color={rgb, 255:red, 0; green, 0; blue, 0 }  ,draw opacity=1 ]   (196.67,85) -- (196.64,110) ;
%Curve Lines [id:da21836010117887827] 
\draw    (180.02,10) .. controls (197.6,21.3) and (201.4,69.7) .. (196.67,85) ;
%Straight Lines [id:da22652653261572386] 
\draw  [dash pattern={on 1.5pt off 1.5pt}]  (150.02,10) -- (178.6,38.58) ;
\draw [shift={(180.02,40)}, rotate = 45] [color={rgb, 255:red, 0; green, 0; blue, 0 }  ][line width=0.75]      (0, 0) circle [x radius= 3.02, y radius= 3.02]   ;
%Straight Lines [id:da11346899882996198] 
\draw  [dash pattern={on 1.5pt off 1.5pt}]  (180.02,40) -- (208.58,11.43) ;
\draw [shift={(210,10)}, rotate = 314.98] [color={rgb, 255:red, 0; green, 0; blue, 0 }  ][line width=0.75]      (0, 0) circle [x radius= 3.02, y radius= 3.02]   ;
%Straight Lines [id:da31404851417562396] 
\draw [fill={rgb, 255:red, 0; green, 0; blue, 0 }  ,fill opacity=1 ] [dash pattern={on 1.5pt off 1.5pt}]  (210,10) -- (180.02,40) ;
\draw [shift={(210,10)}, rotate = 134.98] [color={rgb, 255:red, 0; green, 0; blue, 0 }  ][line width=0.75]    (-3.35,0) -- (3.35,0)(0,3.35) -- (0,-3.35)   ;
%Straight Lines [id:da43911626572036644] 
\draw  [dash pattern={on 1.5pt off 30pt}]  (181.21,38.81) -- (208.81,11.19) ;
\draw [shift={(210,10)}, rotate = 314.98] [color={rgb, 255:red, 0; green, 0; blue, 0 }  ][line width=0.75]      (0, 0) circle [x radius= 2.68, y radius= 2.68]   ;
\draw [shift={(180.02,40)}, rotate = 314.98] [color={rgb, 255:red, 0; green, 0; blue, 0 }  ][line width=0.75]      (0, 0) circle [x radius= 2.68, y radius= 2.68]   ;
%Straight Lines [id:da25249402313448943] 
\draw  [dash pattern={on 1.5pt off 1.5pt}]  (180.02,10) -- (180,37) ;
\draw [shift={(180.02,10)}, rotate = 90.05] [color={rgb, 255:red, 0; green, 0; blue, 0 }  ][fill={rgb, 255:red, 0; green, 0; blue, 0 }  ][line width=0.75]      (0, 0) circle [x radius= 2.01, y radius= 2.01]   ;
%Shape: Circle [id:dp3404867915693012] 
\draw  [color={rgb, 255:red, 120; green, 190; blue, 36 }  ,draw opacity=1 ] (176.02,70) .. controls (176.02,67.79) and (177.81,66) .. (180.02,66) .. controls (182.23,66) and (184.02,67.79) .. (184.02,70) .. controls (184.02,72.21) and (182.23,74) .. (180.02,74) .. controls (177.81,74) and (176.02,72.21) .. (176.02,70) -- cycle ;
%Shape: Circle [id:dp3296776315747201] 
\draw  [color={rgb, 255:red, 245; green, 166; blue, 35 }  ,draw opacity=1 ] (192.67,85) .. controls (192.67,82.79) and (194.46,81) .. (196.67,81) .. controls (198.88,81) and (200.67,82.79) .. (200.67,85) .. controls (200.67,87.21) and (198.88,89) .. (196.67,89) .. controls (194.46,89) and (192.67,87.21) .. (192.67,85) -- cycle ;
%Straight Lines [id:da7401165709355885] 
\draw [color={rgb, 255:red, 120; green, 190; blue, 36 }  ,draw opacity=1 ]   (180.02,70) -- (196.64,110) ;
\draw [shift={(196.64,110)}, rotate = 67.44] [color={rgb, 255:red, 120; green, 190; blue, 36 }  ,draw opacity=1 ][fill={rgb, 255:red, 120; green, 190; blue, 36 }  ,fill opacity=1 ][line width=0.75]      (0, 0) circle [x radius= 2.01, y radius= 2.01]   ;
%Curve Lines [id:da20187376186887418] 
\draw [color={rgb, 255:red, 245; green, 166; blue, 35 }  ,draw opacity=1 ]   (196.67,85) .. controls (209.29,86.71) and (218.43,117.57) .. (216.5,130) ;
\draw [shift={(216.5,130)}, rotate = 98.82] [color={rgb, 255:red, 245; green, 166; blue, 35 }  ,draw opacity=1 ][fill={rgb, 255:red, 245; green, 166; blue, 35 }  ,fill opacity=1 ][line width=0.75]      (0, 0) circle [x radius= 2.01, y radius= 2.01]   ;
%Straight Lines [id:da8809117352779924] 
\draw [color={rgb, 255:red, 120; green, 190; blue, 36 }  ,draw opacity=1 ]   (160.5,110) -- (180,130) ;
\draw [shift={(180,130)}, rotate = 45.73] [color={rgb, 255:red, 120; green, 190; blue, 36 }  ,draw opacity=1 ][fill={rgb, 255:red, 120; green, 190; blue, 36 }  ,fill opacity=1 ][line width=0.75]      (0, 0) circle [x radius= 2.01, y radius= 2.01]   ;
%Shape: Circle [id:dp03145379588003472] 
\draw  [color={rgb, 255:red, 208; green, 2; blue, 27 }  ,draw opacity=1 ] (56,40) .. controls (56,37.79) and (57.79,36) .. (60,36) .. controls (62.21,36) and (64,37.79) .. (64,40) .. controls (64,42.21) and (62.21,44) .. (60,44) .. controls (57.79,44) and (56,42.21) .. (56,40) -- cycle ;
%Shape: Circle [id:dp8636383010290767] 
\draw  [color={rgb, 255:red, 74; green, 144; blue, 226 }  ,draw opacity=1 ] (116.02,40) .. controls (116.02,37.79) and (117.81,36) .. (120.02,36) .. controls (122.23,36) and (124.02,37.79) .. (124.02,40) .. controls (124.02,42.21) and (122.23,44) .. (120.02,44) .. controls (117.81,44) and (116.02,42.21) .. (116.02,40) -- cycle ;
%Shape: Circle [id:dp25521289619657106] 
\draw  [color={rgb, 255:red, 144; green, 19; blue, 254 }  ,draw opacity=1 ] (71.02,130) .. controls (71.02,127.79) and (72.81,126) .. (75.02,126) .. controls (77.23,126) and (79.02,127.79) .. (79.02,130) .. controls (79.02,132.21) and (77.23,134) .. (75.02,134) .. controls (72.81,134) and (71.02,132.21) .. (71.02,130) -- cycle ;
%Straight Lines [id:da6221499123914971] 
\draw  [dash pattern={on 1.5pt off 1.5pt}]  (334.98,40) -- (353.88,68.32) ;
\draw [shift={(355,70)}, rotate = 56.28] [color={rgb, 255:red, 0; green, 0; blue, 0 }  ][line width=0.75]      (0, 0) circle [x radius= 3.02, y radius= 3.02]   ;
%Straight Lines [id:da6010219090091418] 
\draw [fill={rgb, 255:red, 0; green, 0; blue, 0 }  ,fill opacity=1 ] [dash pattern={on 1.5pt off 30pt}]  (355,70) -- (334.98,40) ;
\draw [shift={(355,70)}, rotate = 236.28] [color={rgb, 255:red, 0; green, 0; blue, 0 }  ][line width=0.75]    (-3.35,0) -- (3.35,0)(0,3.35) -- (0,-3.35)   ;
%Curve Lines [id:da922184646906662] 
\draw  [dash pattern={on 1.5pt off 1.5pt}]  (180,42.63) .. controls (185.9,55.6) and (193.95,64.1) .. (208.21,69.36) ;
\draw [shift={(210.02,70)}, rotate = 18.66] [color={rgb, 255:red, 0; green, 0; blue, 0 }  ][line width=0.75]      (0, 0) circle [x radius= 3.02, y radius= 3.02]   ;
%Curve Lines [id:da5237893077852346] 
\draw [color={rgb, 255:red, 245; green, 166; blue, 35 }  ,draw opacity=1 ]   (180.02,40) .. controls (190.56,49.22) and (196.56,76.78) .. (196.67,85) ;
\draw [shift={(196.67,85)}, rotate = 89.23] [color={rgb, 255:red, 245; green, 166; blue, 35 }  ,draw opacity=1 ][fill={rgb, 255:red, 245; green, 166; blue, 35 }  ,fill opacity=1 ][line width=0.75]      (0, 0) circle [x radius= 2.01, y radius= 2.01]   ;
%Straight Lines [id:da2886225510515199] 
\draw  [dash pattern={on 1.5pt off 1.5pt}]  (210,12.01) -- (210.02,67) ;
\draw [shift={(210,10)}, rotate = 89.98] [color={rgb, 255:red, 0; green, 0; blue, 0 }  ][line width=0.75]      (0, 0) circle [x radius= 3.02, y radius= 3.02]   ;
%Straight Lines [id:da29905181085305077] 
\draw  [dash pattern={on 1.5pt off 1.5pt}]  (180.02,40) -- (210.02,70) ;
\draw [shift={(210.02,70)}, rotate = 45] [color={rgb, 255:red, 0; green, 0; blue, 0 }  ][line width=0.75]    (-3.35,0) -- (3.35,0)(0,3.35) -- (0,-3.35)   ;
\draw [shift={(180.02,40)}, rotate = 45] [color={rgb, 255:red, 0; green, 0; blue, 0 }  ][line width=0.75]    (-3.35,0) -- (3.35,0)(0,3.35) -- (0,-3.35)   ;
%Straight Lines [id:da5517279618347369] 
\draw  [dash pattern={on 1.5pt off 1.5pt}]  (180.02,40) -- (208.83,68.81) ;
\draw [shift={(210.02,70)}, rotate = 45] [color={rgb, 255:red, 0; green, 0; blue, 0 }  ][line width=0.75]      (0, 0) circle [x radius= 2.68, y radius= 2.68]   ;
%Curve Lines [id:da2922734949889415] 
\draw  [dash pattern={on 1.5pt off 1.5pt}]  (305.9,11.82) .. controls (312.21,24.11) and (319.65,34.2) .. (333.25,39.38) ;
\draw [shift={(334.98,40)}, rotate = 18.66] [color={rgb, 255:red, 0; green, 0; blue, 0 }  ][line width=0.75]      (0, 0) circle [x radius= 3.02, y radius= 3.02]   ;
\draw [shift={(304.98,10)}, rotate = 63.57] [color={rgb, 255:red, 0; green, 0; blue, 0 }  ][line width=0.75]      (0, 0) circle [x radius= 3.02, y radius= 3.02]   ;
%Straight Lines [id:da5105322657265406] 
\draw  [dash pattern={on 1.5pt off 1.5pt}]  (304.98,10) -- (334.98,40) ;
\draw [shift={(334.98,40)}, rotate = 45] [color={rgb, 255:red, 0; green, 0; blue, 0 }  ][line width=0.75]    (-3.35,0) -- (3.35,0)(0,3.35) -- (0,-3.35)   ;
\draw [shift={(304.98,10)}, rotate = 45] [color={rgb, 255:red, 0; green, 0; blue, 0 }  ][line width=0.75]    (-3.35,0) -- (3.35,0)(0,3.35) -- (0,-3.35)   ;
%Straight Lines [id:da6776523420713856] 
\draw  [dash pattern={on 0.75pt off 30pt}]  (306.17,11.19) -- (333.79,38.81) ;
\draw [shift={(334.98,40)}, rotate = 45] [color={rgb, 255:red, 0; green, 0; blue, 0 }  ][line width=0.75]      (0, 0) circle [x radius= 2.68, y radius= 2.68]   ;
\draw [shift={(304.98,10)}, rotate = 45] [color={rgb, 255:red, 0; green, 0; blue, 0 }  ][line width=0.75]      (0, 0) circle [x radius= 2.68, y radius= 2.68]   ;
%Straight Lines [id:da8520491816482626] 
\draw  [dash pattern={on 0.75pt off 30pt}]  (335.91,41.4) -- (354.05,68.6) ;
\draw [shift={(354.98,70)}, rotate = 56.31] [color={rgb, 255:red, 0; green, 0; blue, 0 }  ][line width=0.75]      (0, 0) circle [x radius= 2.68, y radius= 2.68]   ;
\draw [shift={(334.98,40)}, rotate = 56.31] [color={rgb, 255:red, 0; green, 0; blue, 0 }  ][line width=0.75]      (0, 0) circle [x radius= 2.68, y radius= 2.68]   ;
%Straight Lines [id:da8205450830151395] 
\draw    (305.02,40) -- (305.02,70) ;
\draw [shift={(305.02,70)}, rotate = 90] [color={rgb, 255:red, 0; green, 0; blue, 0 }  ][fill={rgb, 255:red, 0; green, 0; blue, 0 }  ][line width=0.75]      (0, 0) circle [x radius= 2.01, y radius= 2.01]   ;
%Curve Lines [id:da4059286792582145] 
\draw    (275.02,40) .. controls (287.62,43.2) and (299.62,55.6) .. (305.02,70) ;
\draw [shift={(275.02,40)}, rotate = 14.25] [color={rgb, 255:red, 0; green, 0; blue, 0 }  ][fill={rgb, 255:red, 0; green, 0; blue, 0 }  ][line width=0.75]      (0, 0) circle [x radius= 2.01, y radius= 2.01]   ;
%Straight Lines [id:da40311794398154366] 
\draw    (305.02,70) -- (290.02,100) ;
%Straight Lines [id:da46392773996348446] 
\draw    (305.02,13) -- (305.02,40) ;
%Curve Lines [id:da9350358821236168] 
\draw    (306.02,13) .. controls (308.5,20.88) and (311.5,29.38) .. (305.02,40) ;
%Straight Lines [id:da7799892672982128] 
\draw    (304.98,10) -- (275.02,40) ;
%Curve Lines [id:da6908748449485222] 
\draw    (302.98,13) .. controls (300.5,22.38) and (289.25,36.13) .. (275.02,40) ;
%Curve Lines [id:da9521698433840888] 
\draw  [dash pattern={on 1.5pt off 1.5pt}]  (308.02,11) .. controls (319.75,16.63) and (330,26.63) .. (335.02,40) ;
%Straight Lines [id:da7926532340251167] 
\draw    (334.98,43) -- (334.98,70) ;
%Straight Lines [id:da5662538855061081] 
\draw    (305.02,40) -- (335.02,70) ;
%Straight Lines [id:da4562004280330777] 
\draw [color={rgb, 255:red, 0; green, 0; blue, 0 }  ,draw opacity=1 ][fill={rgb, 255:red, 0; green, 0; blue, 0 }  ,fill opacity=1 ]   (305.02,70) -- (320.02,100) ;
%Straight Lines [id:da7534745352998364] 
\draw    (334.98,70) -- (349.98,100) ;
%Straight Lines [id:da6103506839202498] 
\draw [color={rgb, 255:red, 74; green, 144; blue, 226 }  ,draw opacity=1 ]   (306.5,12.63) -- (334.98,70) ;
\draw [shift={(334.98,70)}, rotate = 63.6] [color={rgb, 255:red, 74; green, 144; blue, 226 }  ,draw opacity=1 ][fill={rgb, 255:red, 74; green, 144; blue, 226 }  ,fill opacity=1 ][line width=0.75]      (0, 0) circle [x radius= 2.01, y radius= 2.01]   ;
%Curve Lines [id:da3591700364826922] 
\draw [color={rgb, 255:red, 208; green, 2; blue, 27 }  ,draw opacity=1 ]   (301.98,11) .. controls (293.33,15.17) and (282,22.88) .. (275.02,40) ;
\draw [shift={(275.02,40)}, rotate = 112.17] [color={rgb, 255:red, 208; green, 2; blue, 27 }  ,draw opacity=1 ][fill={rgb, 255:red, 208; green, 2; blue, 27 }  ,fill opacity=1 ][line width=0.75]      (0, 0) circle [x radius= 2.01, y radius= 2.01]   ;
%Straight Lines [id:da7066212015382862] 
\draw [color={rgb, 255:red, 208; green, 2; blue, 27 }  ,draw opacity=1 ]   (275.02,40) -- (290.02,100) ;
\draw [shift={(290.02,100)}, rotate = 75.96] [color={rgb, 255:red, 208; green, 2; blue, 27 }  ,draw opacity=1 ][fill={rgb, 255:red, 208; green, 2; blue, 27 }  ,fill opacity=1 ][line width=0.75]      (0, 0) circle [x radius= 2.01, y radius= 2.01]   ;
%Straight Lines [id:da9528875172306369] 
\draw [color={rgb, 255:red, 208; green, 2; blue, 27 }  ,draw opacity=1 ]   (275.02,40) -- (305.02,70) ;
\draw [shift={(305.02,70)}, rotate = 45] [color={rgb, 255:red, 208; green, 2; blue, 27 }  ,draw opacity=1 ][fill={rgb, 255:red, 208; green, 2; blue, 27 }  ,fill opacity=1 ][line width=0.75]      (0, 0) circle [x radius= 2.01, y radius= 2.01]   ;
%Curve Lines [id:da6780396188087334] 
\draw [color={rgb, 255:red, 120; green, 190; blue, 36 }  ,draw opacity=1 ]   (349.98,100) .. controls (353.5,115.13) and (340.5,127.63) .. (335,130) ;
\draw [shift={(335,130)}, rotate = 156.64] [color={rgb, 255:red, 120; green, 190; blue, 36 }  ,draw opacity=1 ][fill={rgb, 255:red, 120; green, 190; blue, 36 }  ,fill opacity=1 ][line width=0.75]      (0, 0) circle [x radius= 2.01, y radius= 2.01]   ;
%Curve Lines [id:da07628021166086896] 
\draw  [dash pattern={on 1.5pt off 1.5pt}]  (338,41) .. controls (346.25,45.38) and (353.75,57.38) .. (355,67) ;
%Curve Lines [id:da0245157900843056] 
\draw [color={rgb, 255:red, 208; green, 2; blue, 27 }  ,draw opacity=1 ]   (275.02,40) .. controls (268,48.63) and (269.75,62.63) .. (275,70) ;
\draw [shift={(275,70)}, rotate = 54.55] [color={rgb, 255:red, 208; green, 2; blue, 27 }  ,draw opacity=1 ][fill={rgb, 255:red, 208; green, 2; blue, 27 }  ,fill opacity=1 ][line width=0.75]      (0, 0) circle [x radius= 2.01, y radius= 2.01]   ;
%Curve Lines [id:da49127599506443487] 
\draw [color={rgb, 255:red, 245; green, 166; blue, 35 }  ,draw opacity=1 ]   (304.98,40) .. controls (307.5,52.88) and (318,78.88) .. (320.02,100) ;
\draw [shift={(320.02,100)}, rotate = 84.53] [color={rgb, 255:red, 245; green, 166; blue, 35 }  ,draw opacity=1 ][fill={rgb, 255:red, 245; green, 166; blue, 35 }  ,fill opacity=1 ][line width=0.75]      (0, 0) circle [x radius= 2.01, y radius= 2.01]   ;
%Curve Lines [id:da2715765100575873] 
\draw    (349.98,100) .. controls (332.75,89.88) and (312.25,53.38) .. (305.02,40) ;
%Straight Lines [id:da500315407050942] 
\draw [color={rgb, 255:red, 120; green, 190; blue, 36 }  ,draw opacity=1 ][fill={rgb, 255:red, 0; green, 0; blue, 0 }  ,fill opacity=1 ]   (336,43) -- (349.98,100) ;
\draw [shift={(349.98,100)}, rotate = 76.22] [color={rgb, 255:red, 120; green, 190; blue, 36 }  ,draw opacity=1 ][fill={rgb, 255:red, 120; green, 190; blue, 36 }  ,fill opacity=1 ][line width=0.75]      (0, 0) circle [x radius= 2.01, y radius= 2.01]   ;
%Curve Lines [id:da3730486135420702] 
\draw [color={rgb, 255:red, 245; green, 166; blue, 35 }  ,draw opacity=1 ]   (303.98,13) .. controls (300.75,22.13) and (299.75,30.38) .. (304.98,40) ;
\draw [shift={(304.98,40)}, rotate = 61.49] [color={rgb, 255:red, 245; green, 166; blue, 35 }  ,draw opacity=1 ][fill={rgb, 255:red, 245; green, 166; blue, 35 }  ,fill opacity=1 ][line width=0.75]      (0, 0) circle [x radius= 2.01, y radius= 2.01]   ;
%Straight Lines [id:da6824238652420582] 
\draw    (305,159.75) -- (275,130) ;
%Curve Lines [id:da043260635493709176] 
\draw    (305,129.75) .. controls (310.75,140.38) and (309.75,151.38) .. (305,159.75) ;
%Curve Lines [id:da14576629890952852] 
\draw [color={rgb, 255:red, 74; green, 144; blue, 226 }  ,draw opacity=1 ]   (334.98,70) .. controls (337,96.63) and (315.14,154.75) .. (305,159.75) ;
\draw [shift={(305,159.75)}, rotate = 153.76] [color={rgb, 255:red, 74; green, 144; blue, 226 }  ,draw opacity=1 ][fill={rgb, 255:red, 74; green, 144; blue, 226 }  ,fill opacity=1 ][line width=0.75]      (0, 0) circle [x radius= 2.01, y radius= 2.01]   ;
%Curve Lines [id:da7360387829025741] 
\draw [color={rgb, 255:red, 208; green, 2; blue, 27 }  ,draw opacity=1 ]   (290.02,100) .. controls (279,101.38) and (272,120.13) .. (275,130) ;
\draw [shift={(275,130)}, rotate = 73.1] [color={rgb, 255:red, 208; green, 2; blue, 27 }  ,draw opacity=1 ][fill={rgb, 255:red, 208; green, 2; blue, 27 }  ,fill opacity=1 ][line width=0.75]      (0, 0) circle [x radius= 2.01, y radius= 2.01]   ;
%Curve Lines [id:da22531874328968737] 
\draw [color={rgb, 255:red, 245; green, 166; blue, 35 }  ,draw opacity=1 ]   (319.98,100) .. controls (313.25,104.38) and (303.25,116.88) .. (305,129.75) ;
\draw [shift={(305,129.75)}, rotate = 82.26] [color={rgb, 255:red, 245; green, 166; blue, 35 }  ,draw opacity=1 ][fill={rgb, 255:red, 245; green, 166; blue, 35 }  ,fill opacity=1 ][line width=0.75]      (0, 0) circle [x radius= 2.01, y radius= 2.01]   ;
%Shape: Circle [id:dp06590085571498894] 
\draw  [color={rgb, 255:red, 208; green, 2; blue, 27 }  ,draw opacity=1 ] (271.02,40) .. controls (271.02,37.79) and (272.81,36) .. (275.02,36) .. controls (277.23,36) and (279.02,37.79) .. (279.02,40) .. controls (279.02,42.21) and (277.23,44) .. (275.02,44) .. controls (272.81,44) and (271.02,42.21) .. (271.02,40) -- cycle ;
%Shape: Circle [id:dp6492412206568494] 
\draw  [color={rgb, 255:red, 245; green, 166; blue, 35 }  ,draw opacity=1 ] (301.02,40) .. controls (301.02,37.79) and (302.81,36) .. (305.02,36) .. controls (307.23,36) and (309.02,37.79) .. (309.02,40) .. controls (309.02,42.21) and (307.23,44) .. (305.02,44) .. controls (302.81,44) and (301.02,42.21) .. (301.02,40) -- cycle ;
%Shape: Circle [id:dp27240569006518434] 
\draw  [color={rgb, 255:red, 120; green, 190; blue, 36 }  ,draw opacity=1 ] (345.98,100) .. controls (345.98,97.79) and (347.77,96) .. (349.98,96) .. controls (352.19,96) and (353.98,97.79) .. (353.98,100) .. controls (353.98,102.21) and (352.19,104) .. (349.98,104) .. controls (347.77,104) and (345.98,102.21) .. (345.98,100) -- cycle ;
%Shape: Circle [id:dp416655304541134] 
\draw  [color={rgb, 255:red, 74; green, 144; blue, 226 }  ,draw opacity=1 ] (331.02,70) .. controls (331.02,67.79) and (332.81,66) .. (335.02,66) .. controls (337.23,66) and (339.02,67.79) .. (339.02,70) .. controls (339.02,72.21) and (337.23,74) .. (335.02,74) .. controls (332.81,74) and (331.02,72.21) .. (331.02,70) -- cycle ;
%Straight Lines [id:da5455903879531323] 
\draw  [dash pattern={on 1.5pt off 1.5pt}]  (405.11,92) -- (402.56,94.29) ;
\draw [shift={(402.56,94.29)}, rotate = 138.09] [color={rgb, 255:red, 0; green, 0; blue, 0 }  ][line width=0.75]    (-3.35,0) -- (3.35,0)(0,3.35) -- (0,-3.35)   ;
%Straight Lines [id:da6819379807073557] 
\draw  [dash pattern={on 1.5pt off 1.5pt}]  (405.11,92) -- (404.06,92.94) ;
\draw [shift={(402.56,94.29)}, rotate = 138.09] [color={rgb, 255:red, 0; green, 0; blue, 0 }  ][line width=0.75]      (0, 0) circle [x radius= 3.02, y radius= 3.02]   ;
%Curve Lines [id:da249136801254636] 
\draw  [dash pattern={on 0.75pt off 7.5pt}]  (400.5,92.16) .. controls (406.47,98.39) and (404.33,96.16) .. (402.96,94.72) ;
\draw [shift={(402.56,94.29)}, rotate = 45.79] [color={rgb, 255:red, 0; green, 0; blue, 0 }  ][line width=0.75]      (0, 0) circle [x radius= 2.68, y radius= 2.68]   ;

%Flowchart: Connector [id:dp16246636862265307] 
\draw   (400,79.35) .. controls (400,78.05) and (401.05,77) .. (402.35,77) .. controls (403.65,77) and (404.7,78.05) .. (404.7,79.35) .. controls (404.7,80.65) and (403.65,81.7) .. (402.35,81.7) .. controls (401.05,81.7) and (400,80.65) .. (400,79.35) -- cycle ;
%Flowchart: Connector [id:dp868834266529244] 
\draw  [fill={rgb, 255:red, 0; green, 0; blue, 0 }  ,fill opacity=1 ] (400,64.35) .. controls (400,63.05) and (401.05,62) .. (402.35,62) .. controls (403.65,62) and (404.7,63.05) .. (404.7,64.35) .. controls (404.7,65.65) and (403.65,66.7) .. (402.35,66.7) .. controls (401.05,66.7) and (400,65.65) .. (400,64.35) -- cycle ;

% Text Node
\draw (410,59) node [anchor=north west][inner sep=0.75pt]  [font=\small] [align=left] {\CT};
% Text Node
\draw (410,74) node [anchor=north west][inner sep=0.75pt]  [font=\small] [align=left] {\CF};
% Text Node
\draw (410,89) node [anchor=north west][inner sep=0.75pt]  [font=\small] [align=left] {\PF};

\end{tikzpicture}

%% file: Visuals/monotonicity-p.tex
\begin{tikzpicture}[x=0.75pt,y=0.75pt,yscale=-1,xscale=1]
%uncomment if require: \path (0,561); %set diagram left start at 0, and has height of 561

%Straight Lines [id:da7171923072569963] 
\draw [color={rgb, 255:red, 0; green, 0; blue, 0 }  ,draw opacity=1 ]   (394,171) -- (387,199) ;
\draw [shift={(387,199)}, rotate = 104.04] [color={rgb, 255:red, 0; green, 0; blue, 0 }  ,draw opacity=1 ][fill={rgb, 255:red, 0; green, 0; blue, 0 }  ,fill opacity=1 ][line width=0.75]      (0, 0) circle [x radius= 2.01, y radius= 2.01]   ;
%Shape: Circle [id:dp9549734647222975] 
\draw  [color={rgb, 255:red, 0; green, 0; blue, 0 }  ,draw opacity=1 ][fill={rgb, 255:red, 255; green, 255; blue, 255 }  ,fill opacity=1 ] (398.5,299) .. controls (398.5,297.62) and (399.62,296.5) .. (401,296.5) .. controls (402.38,296.5) and (403.5,297.62) .. (403.5,299) .. controls (403.5,300.38) and (402.38,301.5) .. (401,301.5) .. controls (399.62,301.5) and (398.5,300.38) .. (398.5,299) -- cycle ;
%Straight Lines [id:da21035205492650988] 
\draw [color={rgb, 255:red, 0; green, 0; blue, 0 }  ,draw opacity=1 ] [dash pattern={on 1.5pt off 1.5pt}]  (394,271) -- (401,299) ;
\draw [shift={(401,299)}, rotate = 75.96] [color={rgb, 255:red, 0; green, 0; blue, 0 }  ,draw opacity=1 ][line width=0.75]    (-3.35,0) -- (3.35,0)(0,3.35) -- (0,-3.35)   ;
%Straight Lines [id:da6612808104720815] 
\draw [color={rgb, 255:red, 0; green, 0; blue, 0 }  ,draw opacity=1 ]   (394,271) -- (387,299) ;
\draw [shift={(387,299)}, rotate = 104.04] [color={rgb, 255:red, 0; green, 0; blue, 0 }  ,draw opacity=1 ][fill={rgb, 255:red, 0; green, 0; blue, 0 }  ,fill opacity=1 ][line width=0.75]      (0, 0) circle [x radius= 2.01, y radius= 2.01]   ;
%Straight Lines [id:da961813482881127] 
\draw [color={rgb, 255:red, 0; green, 0; blue, 0 }  ,draw opacity=1 ] [dash pattern={on 1.5pt off 1.5pt}]  (381,249) -- (367,271) ;
\draw [shift={(367,271)}, rotate = 122.47] [color={rgb, 255:red, 0; green, 0; blue, 0 }  ,draw opacity=1 ][fill={rgb, 255:red, 0; green, 0; blue, 0 }  ,fill opacity=1 ][line width=0.75]      (0, 0) circle [x radius= 2.01, y radius= 2.01]   ;
%Straight Lines [id:da4373773298823195] 
\draw [color={rgb, 255:red, 0; green, 0; blue, 0 }  ,draw opacity=1 ]   (314,271) -- (307,299) ;
\draw [shift={(307,299)}, rotate = 104.04] [color={rgb, 255:red, 0; green, 0; blue, 0 }  ,draw opacity=1 ][fill={rgb, 255:red, 0; green, 0; blue, 0 }  ,fill opacity=1 ][line width=0.75]      (0, 0) circle [x radius= 2.01, y radius= 2.01]   ;
%Straight Lines [id:da7922402092989718] 
\draw [color={rgb, 255:red, 0; green, 0; blue, 0 }  ,draw opacity=1 ]   (314,171) -- (307,199) ;
\draw [shift={(307,199)}, rotate = 104.04] [color={rgb, 255:red, 0; green, 0; blue, 0 }  ,draw opacity=1 ][fill={rgb, 255:red, 0; green, 0; blue, 0 }  ,fill opacity=1 ][line width=0.75]      (0, 0) circle [x radius= 2.01, y radius= 2.01]   ;
%Shape: Circle [id:dp7688612225216213] 
\draw  [color={rgb, 255:red, 0; green, 0; blue, 0 }  ,draw opacity=1 ][fill={rgb, 255:red, 255; green, 255; blue, 255 }  ,fill opacity=1 ] (391.5,271) .. controls (391.5,269.62) and (392.62,268.5) .. (394,268.5) .. controls (395.38,268.5) and (396.5,269.62) .. (396.5,271) .. controls (396.5,272.38) and (395.38,273.5) .. (394,273.5) .. controls (392.62,273.5) and (391.5,272.38) .. (391.5,271) -- cycle ;
%Shape: Circle [id:dp303923341220241] 
\draw  [color={rgb, 255:red, 0; green, 0; blue, 0 }  ,draw opacity=1 ][fill={rgb, 255:red, 255; green, 255; blue, 255 }  ,fill opacity=1 ] (378.5,249) .. controls (378.5,247.62) and (379.62,246.5) .. (381,246.5) .. controls (382.38,246.5) and (383.5,247.62) .. (383.5,249) .. controls (383.5,250.38) and (382.38,251.5) .. (381,251.5) .. controls (379.62,251.5) and (378.5,250.38) .. (378.5,249) -- cycle ;
%Straight Lines [id:da06930894894479744] 
\draw    (381,149) -- (367,171) ;
\draw [shift={(367,171)}, rotate = 122.47] [color={rgb, 255:red, 0; green, 0; blue, 0 }  ][fill={rgb, 255:red, 0; green, 0; blue, 0 }  ][line width=0.75]      (0, 0) circle [x radius= 2.01, y radius= 2.01]   ;
%Straight Lines [id:da2838810224796544] 
\draw    (221,120) -- (221,149) ;
\draw [shift={(221,149)}, rotate = 90] [color={rgb, 255:red, 0; green, 0; blue, 0 }  ][fill={rgb, 255:red, 0; green, 0; blue, 0 }  ][line width=0.75]      (0, 0) circle [x radius= 2.01, y radius= 2.01]   ;
%Straight Lines [id:da928894547585291] 
\draw    (221,149) -- (207,171) ;
\draw [shift={(207,171)}, rotate = 122.47] [color={rgb, 255:red, 0; green, 0; blue, 0 }  ][fill={rgb, 255:red, 0; green, 0; blue, 0 }  ][line width=0.75]      (0, 0) circle [x radius= 2.01, y radius= 2.01]   ;
%Straight Lines [id:da18687717144290483] 
\draw    (207,171) -- (207,199) ;
\draw [shift={(207,199)}, rotate = 90] [color={rgb, 255:red, 0; green, 0; blue, 0 }  ][fill={rgb, 255:red, 0; green, 0; blue, 0 }  ][line width=0.75]      (0, 0) circle [x radius= 2.01, y radius= 2.01]   ;
%Shape: Circle [id:dp057729434926461876] 
\draw  [fill={rgb, 255:red, 255; green, 255; blue, 255 }  ,fill opacity=1 ] (218.5,120) .. controls (218.5,118.62) and (219.62,117.5) .. (221,117.5) .. controls (222.38,117.5) and (223.5,118.62) .. (223.5,120) .. controls (223.5,121.38) and (222.38,122.5) .. (221,122.5) .. controls (219.62,122.5) and (218.5,121.38) .. (218.5,120) -- cycle ;
%Straight Lines [id:da058913028156624025] 
\draw    (221,149) -- (234,171) ;
\draw [shift={(234,171)}, rotate = 59.42] [color={rgb, 255:red, 0; green, 0; blue, 0 }  ][fill={rgb, 255:red, 0; green, 0; blue, 0 }  ][line width=0.75]      (0, 0) circle [x radius= 2.01, y radius= 2.01]   ;
%Shape: Circle [id:dp4694281273414391] 
\draw  [color={rgb, 255:red, 74; green, 144; blue, 226 }  ,draw opacity=1 ] (310,171) .. controls (310,168.79) and (311.79,167) .. (314,167) .. controls (316.21,167) and (318,168.79) .. (318,171) .. controls (318,173.21) and (316.21,175) .. (314,175) .. controls (311.79,175) and (310,173.21) .. (310,171) -- cycle ;
%Straight Lines [id:da2592497807916032] 
\draw    (301,120) -- (301,149) ;
\draw [shift={(301,149)}, rotate = 90] [color={rgb, 255:red, 0; green, 0; blue, 0 }  ][fill={rgb, 255:red, 0; green, 0; blue, 0 }  ][line width=0.75]      (0, 0) circle [x radius= 2.01, y radius= 2.01]   ;
%Straight Lines [id:da9394119936808425] 
\draw    (301,149) -- (287,171) ;
\draw [shift={(287,171)}, rotate = 122.47] [color={rgb, 255:red, 0; green, 0; blue, 0 }  ][fill={rgb, 255:red, 0; green, 0; blue, 0 }  ][line width=0.75]      (0, 0) circle [x radius= 2.01, y radius= 2.01]   ;
%Shape: Circle [id:dp5090894342889158] 
\draw  [fill={rgb, 255:red, 255; green, 255; blue, 255 }  ,fill opacity=1 ] (298.5,120) .. controls (298.5,118.62) and (299.62,117.5) .. (301,117.5) .. controls (302.38,117.5) and (303.5,118.62) .. (303.5,120) .. controls (303.5,121.38) and (302.38,122.5) .. (301,122.5) .. controls (299.62,122.5) and (298.5,121.38) .. (298.5,120) -- cycle ;
%Straight Lines [id:da6486261736517827] 
\draw    (301,149) -- (314,171) ;
\draw [shift={(314,171)}, rotate = 59.42] [color={rgb, 255:red, 0; green, 0; blue, 0 }  ][fill={rgb, 255:red, 0; green, 0; blue, 0 }  ][line width=0.75]      (0, 0) circle [x radius= 2.01, y radius= 2.01]   ;
%Shape: Circle [id:dp9729726201837542] 
\draw  [color={rgb, 255:red, 74; green, 144; blue, 226 }  ,draw opacity=1 ] (317,199) .. controls (317,196.79) and (318.79,195) .. (321,195) .. controls (323.21,195) and (325,196.79) .. (325,199) .. controls (325,201.21) and (323.21,203) .. (321,203) .. controls (318.79,203) and (317,201.21) .. (317,199) -- cycle ;
%Shape: Circle [id:dp6245965665156129] 
\draw  [color={rgb, 255:red, 74; green, 144; blue, 226 }  ,draw opacity=1 ] (297,149) .. controls (297,146.79) and (298.79,145) .. (301,145) .. controls (303.21,145) and (305,146.79) .. (305,149) .. controls (305,151.21) and (303.21,153) .. (301,153) .. controls (298.79,153) and (297,151.21) .. (297,149) -- cycle ;
%Shape: Circle [id:dp5019586296273807] 
\draw  [color={rgb, 255:red, 74; green, 144; blue, 226 }  ,draw opacity=1 ] (297,120) .. controls (297,117.79) and (298.79,116) .. (301,116) .. controls (303.21,116) and (305,117.79) .. (305,120) .. controls (305,122.21) and (303.21,124) .. (301,124) .. controls (298.79,124) and (297,122.21) .. (297,120) -- cycle ;
%Straight Lines [id:da04142653705901478] 
\draw [color={rgb, 255:red, 208; green, 2; blue, 27 }  ,draw opacity=1 ]   (381,120) -- (381,149) ;
\draw [shift={(381,149)}, rotate = 90] [color={rgb, 255:red, 208; green, 2; blue, 27 }  ,draw opacity=1 ][fill={rgb, 255:red, 208; green, 2; blue, 27 }  ,fill opacity=1 ][line width=0.75]      (0, 0) circle [x radius= 2.01, y radius= 2.01]   ;
%Shape: Circle [id:dp15712758690312112] 
\draw  [color={rgb, 255:red, 208; green, 2; blue, 27 }  ,draw opacity=1 ][fill={rgb, 255:red, 255; green, 255; blue, 255 }  ,fill opacity=1 ] (378.5,120) .. controls (378.5,118.62) and (379.62,117.5) .. (381,117.5) .. controls (382.38,117.5) and (383.5,118.62) .. (383.5,120) .. controls (383.5,121.38) and (382.38,122.5) .. (381,122.5) .. controls (379.62,122.5) and (378.5,121.38) .. (378.5,120) -- cycle ;
%Curve Lines [id:da39814781691654744] 
\draw [color={rgb, 255:red, 208; green, 2; blue, 27 }  ,draw opacity=1 ]   (394,171) .. controls (394,172) and (401,199) .. (401,199) ;
\draw [shift={(401,199)}, rotate = 0] [color={rgb, 255:red, 208; green, 2; blue, 27 }  ,draw opacity=1 ][fill={rgb, 255:red, 208; green, 2; blue, 27 }  ,fill opacity=1 ][line width=0.75]      (0, 0) circle [x radius= 2.01, y radius= 2.01]   ;
%Straight Lines [id:da3154230718745228] 
\draw [color={rgb, 255:red, 208; green, 2; blue, 27 }  ,draw opacity=1 ]   (381,149) -- (394,171) ;
\draw [shift={(394,171)}, rotate = 59.42] [color={rgb, 255:red, 208; green, 2; blue, 27 }  ,draw opacity=1 ][fill={rgb, 255:red, 208; green, 2; blue, 27 }  ,fill opacity=1 ][line width=0.75]      (0, 0) circle [x radius= 2.01, y radius= 2.01]   ;
%Straight Lines [id:da4746079644618031] 
\draw [color={rgb, 255:red, 0; green, 0; blue, 0 }  ,draw opacity=1 ]   (221,220) -- (221,249) ;
\draw [shift={(221,249)}, rotate = 90] [color={rgb, 255:red, 0; green, 0; blue, 0 }  ,draw opacity=1 ][fill={rgb, 255:red, 0; green, 0; blue, 0 }  ,fill opacity=1 ][line width=0.75]      (0, 0) circle [x radius= 2.01, y radius= 2.01]   ;
%Straight Lines [id:da4202388353640113] 
\draw [color={rgb, 255:red, 0; green, 0; blue, 0 }  ,draw opacity=1 ]   (221,249) -- (207,271) ;
\draw [shift={(207,271)}, rotate = 122.47] [color={rgb, 255:red, 0; green, 0; blue, 0 }  ,draw opacity=1 ][fill={rgb, 255:red, 0; green, 0; blue, 0 }  ,fill opacity=1 ][line width=0.75]      (0, 0) circle [x radius= 2.01, y radius= 2.01]   ;
%Shape: Circle [id:dp09282070804175746] 
\draw  [color={rgb, 255:red, 0; green, 0; blue, 0 }  ,draw opacity=1 ][fill={rgb, 255:red, 255; green, 255; blue, 255 }  ,fill opacity=1 ] (218.5,220) .. controls (218.5,218.62) and (219.62,217.5) .. (221,217.5) .. controls (222.38,217.5) and (223.5,218.62) .. (223.5,220) .. controls (223.5,221.38) and (222.38,222.5) .. (221,222.5) .. controls (219.62,222.5) and (218.5,221.38) .. (218.5,220) -- cycle ;
%Straight Lines [id:da9333486502498194] 
\draw [color={rgb, 255:red, 0; green, 0; blue, 0 }  ,draw opacity=1 ]   (221,249) -- (234,271) ;
\draw [shift={(234,271)}, rotate = 59.42] [color={rgb, 255:red, 0; green, 0; blue, 0 }  ,draw opacity=1 ][fill={rgb, 255:red, 0; green, 0; blue, 0 }  ,fill opacity=1 ][line width=0.75]      (0, 0) circle [x radius= 2.01, y radius= 2.01]   ;
%Shape: Circle [id:dp5173192341398191] 
\draw  [color={rgb, 255:red, 74; green, 144; blue, 226 }  ,draw opacity=1 ] (310,271) .. controls (310,268.79) and (311.79,267) .. (314,267) .. controls (316.21,267) and (318,268.79) .. (318,271) .. controls (318,273.21) and (316.21,275) .. (314,275) .. controls (311.79,275) and (310,273.21) .. (310,271) -- cycle ;
%Straight Lines [id:da6172893350739918] 
\draw    (301,220) -- (301,249) ;
\draw [shift={(301,249)}, rotate = 90] [color={rgb, 255:red, 0; green, 0; blue, 0 }  ][fill={rgb, 255:red, 0; green, 0; blue, 0 }  ][line width=0.75]      (0, 0) circle [x radius= 2.01, y radius= 2.01]   ;
%Straight Lines [id:da34984961135927595] 
\draw    (301,249) -- (287,271) ;
\draw [shift={(287,271)}, rotate = 122.47] [color={rgb, 255:red, 0; green, 0; blue, 0 }  ][fill={rgb, 255:red, 0; green, 0; blue, 0 }  ][line width=0.75]      (0, 0) circle [x radius= 2.01, y radius= 2.01]   ;
%Shape: Circle [id:dp8715955999386025] 
\draw  [fill={rgb, 255:red, 255; green, 255; blue, 255 }  ,fill opacity=1 ] (298.5,220) .. controls (298.5,218.62) and (299.62,217.5) .. (301,217.5) .. controls (302.38,217.5) and (303.5,218.62) .. (303.5,220) .. controls (303.5,221.38) and (302.38,222.5) .. (301,222.5) .. controls (299.62,222.5) and (298.5,221.38) .. (298.5,220) -- cycle ;
%Straight Lines [id:da11315556911782476] 
\draw    (301,249) -- (314,271) ;
\draw [shift={(314,271)}, rotate = 59.42] [color={rgb, 255:red, 0; green, 0; blue, 0 }  ][fill={rgb, 255:red, 0; green, 0; blue, 0 }  ][line width=0.75]      (0, 0) circle [x radius= 2.01, y radius= 2.01]   ;
%Shape: Circle [id:dp3196895518348998] 
\draw  [color={rgb, 255:red, 74; green, 144; blue, 226 }  ,draw opacity=1 ] (317,299) .. controls (317,296.79) and (318.79,295) .. (321,295) .. controls (323.21,295) and (325,296.79) .. (325,299) .. controls (325,301.21) and (323.21,303) .. (321,303) .. controls (318.79,303) and (317,301.21) .. (317,299) -- cycle ;
%Shape: Circle [id:dp9943588990267974] 
\draw  [color={rgb, 255:red, 74; green, 144; blue, 226 }  ,draw opacity=1 ] (297,249) .. controls (297,246.79) and (298.79,245) .. (301,245) .. controls (303.21,245) and (305,246.79) .. (305,249) .. controls (305,251.21) and (303.21,253) .. (301,253) .. controls (298.79,253) and (297,251.21) .. (297,249) -- cycle ;
%Shape: Circle [id:dp37527023199951093] 
\draw  [color={rgb, 255:red, 74; green, 144; blue, 226 }  ,draw opacity=1 ] (297,220) .. controls (297,217.79) and (298.79,216) .. (301,216) .. controls (303.21,216) and (305,217.79) .. (305,220) .. controls (305,222.21) and (303.21,224) .. (301,224) .. controls (298.79,224) and (297,222.21) .. (297,220) -- cycle ;
%Shape: Circle [id:dp547092861220334] 
\draw  [color={rgb, 255:red, 0; green, 0; blue, 0 }  ,draw opacity=1 ][fill={rgb, 255:red, 255; green, 255; blue, 255 }  ,fill opacity=1 ] (378.5,220) .. controls (378.5,218.62) and (379.62,217.5) .. (381,217.5) .. controls (382.38,217.5) and (383.5,218.62) .. (383.5,220) .. controls (383.5,221.38) and (382.38,222.5) .. (381,222.5) .. controls (379.62,222.5) and (378.5,221.38) .. (378.5,220) -- cycle ;
%Straight Lines [id:da7325923607843529] 
\draw [color={rgb, 255:red, 0; green, 0; blue, 0 }  ,draw opacity=1 ]   (381,349.5) -- (367,371.5) ;
\draw [shift={(367,371.5)}, rotate = 122.47] [color={rgb, 255:red, 0; green, 0; blue, 0 }  ,draw opacity=1 ][fill={rgb, 255:red, 0; green, 0; blue, 0 }  ,fill opacity=1 ][line width=0.75]      (0, 0) circle [x radius= 2.01, y radius= 2.01]   ;
%Straight Lines [id:da7745304859740241] 
\draw [color={rgb, 255:red, 0; green, 0; blue, 0 }  ,draw opacity=1 ]   (221,320.5) -- (221,349.5) ;
\draw [shift={(221,349.5)}, rotate = 90] [color={rgb, 255:red, 0; green, 0; blue, 0 }  ,draw opacity=1 ][fill={rgb, 255:red, 0; green, 0; blue, 0 }  ,fill opacity=1 ][line width=0.75]      (0, 0) circle [x radius= 2.01, y radius= 2.01]   ;
%Straight Lines [id:da4522143709511789] 
\draw [color={rgb, 255:red, 0; green, 0; blue, 0 }  ,draw opacity=1 ]   (221,349.5) -- (207,371.5) ;
\draw [shift={(207,371.5)}, rotate = 122.47] [color={rgb, 255:red, 0; green, 0; blue, 0 }  ,draw opacity=1 ][fill={rgb, 255:red, 0; green, 0; blue, 0 }  ,fill opacity=1 ][line width=0.75]      (0, 0) circle [x radius= 2.01, y radius= 2.01]   ;
%Shape: Circle [id:dp22289329246929468] 
\draw  [color={rgb, 255:red, 0; green, 0; blue, 0 }  ,draw opacity=1 ][fill={rgb, 255:red, 255; green, 255; blue, 255 }  ,fill opacity=1 ] (218.5,320.5) .. controls (218.5,319.12) and (219.62,318) .. (221,318) .. controls (222.38,318) and (223.5,319.12) .. (223.5,320.5) .. controls (223.5,321.88) and (222.38,323) .. (221,323) .. controls (219.62,323) and (218.5,321.88) .. (218.5,320.5) -- cycle ;
%Straight Lines [id:da9352074028140094] 
\draw [color={rgb, 255:red, 0; green, 0; blue, 0 }  ,draw opacity=1 ]   (221,349.5) -- (234,371.5) ;
\draw [shift={(234,371.5)}, rotate = 59.42] [color={rgb, 255:red, 0; green, 0; blue, 0 }  ,draw opacity=1 ][fill={rgb, 255:red, 0; green, 0; blue, 0 }  ,fill opacity=1 ][line width=0.75]      (0, 0) circle [x radius= 2.01, y radius= 2.01]   ;
%Straight Lines [id:da8079502736522035] 
\draw [color={rgb, 255:red, 0; green, 0; blue, 0 }  ,draw opacity=1 ]   (381,320.5) -- (381,349.5) ;
\draw [shift={(381,349.5)}, rotate = 90] [color={rgb, 255:red, 0; green, 0; blue, 0 }  ,draw opacity=1 ][fill={rgb, 255:red, 0; green, 0; blue, 0 }  ,fill opacity=1 ][line width=0.75]      (0, 0) circle [x radius= 2.01, y radius= 2.01]   ;
%Shape: Circle [id:dp30255062925064413] 
\draw  [color={rgb, 255:red, 0; green, 0; blue, 0 }  ,draw opacity=1 ][fill={rgb, 255:red, 255; green, 255; blue, 255 }  ,fill opacity=1 ] (378.5,320.5) .. controls (378.5,319.12) and (379.62,318) .. (381,318) .. controls (382.38,318) and (383.5,319.12) .. (383.5,320.5) .. controls (383.5,321.88) and (382.38,323) .. (381,323) .. controls (379.62,323) and (378.5,321.88) .. (378.5,320.5) -- cycle ;
%Straight Lines [id:da7947251438194423] 
\draw [color={rgb, 255:red, 0; green, 0; blue, 0 }  ,draw opacity=1 ]   (381,349.5) -- (394,371.5) ;
\draw [shift={(394,371.5)}, rotate = 59.42] [color={rgb, 255:red, 0; green, 0; blue, 0 }  ,draw opacity=1 ][fill={rgb, 255:red, 0; green, 0; blue, 0 }  ,fill opacity=1 ][line width=0.75]      (0, 0) circle [x radius= 2.01, y radius= 2.01]   ;
%Straight Lines [id:da24419790920349305] 
\draw [color={rgb, 255:red, 0; green, 0; blue, 0 }  ,draw opacity=1 ] [dash pattern={on 1.5pt off 1.5pt}]  (381,249) -- (394,271) ;
\draw [shift={(394,271)}, rotate = 59.42] [color={rgb, 255:red, 0; green, 0; blue, 0 }  ,draw opacity=1 ][line width=0.75]    (-3.35,0) -- (3.35,0)(0,3.35) -- (0,-3.35)   ;
\draw [shift={(381,249)}, rotate = 59.42] [color={rgb, 255:red, 0; green, 0; blue, 0 }  ,draw opacity=1 ][line width=0.75]    (-3.35,0) -- (3.35,0)(0,3.35) -- (0,-3.35)   ;
%Straight Lines [id:da09735483189882166] 
\draw [color={rgb, 255:red, 0; green, 0; blue, 0 }  ,draw opacity=1 ] [dash pattern={on 1.5pt off 1.5pt}]  (381,222.02) -- (381,246.99) ;
\draw [shift={(381,249)}, rotate = 90] [color={rgb, 255:red, 0; green, 0; blue, 0 }  ,draw opacity=1 ][line width=0.75]      (0, 0) circle [x radius= 3.02, y radius= 3.02]   ;
\draw [shift={(381,220)}, rotate = 90] [color={rgb, 255:red, 0; green, 0; blue, 0 }  ,draw opacity=1 ][line width=0.75]      (0, 0) circle [x radius= 3.02, y radius= 3.02]   ;
%Straight Lines [id:da4006828667752579] 
\draw [color={rgb, 255:red, 0; green, 0; blue, 0 }  ,draw opacity=1 ] [dash pattern={on 0.75pt off 22.5pt}]  (381,246.5) -- (381,220) ;
\draw [shift={(381,220)}, rotate = 270] [color={rgb, 255:red, 0; green, 0; blue, 0 }  ,draw opacity=1 ][line width=0.75]    (-3.35,0) -- (3.35,0)(0,3.35) -- (0,-3.35)   ;
%Straight Lines [id:da4861100476943466] 
\draw [color={rgb, 255:red, 74; green, 144; blue, 226 }  ,draw opacity=1 ]   (301,249) -- (314,271) ;
%Straight Lines [id:da015542756030852978] 
\draw [color={rgb, 255:red, 74; green, 144; blue, 226 }  ,draw opacity=1 ]   (301,222.5) -- (301,249) ;
%Curve Lines [id:da9159688717379441] 
\draw [color={rgb, 255:red, 74; green, 144; blue, 226 }  ,draw opacity=1 ]   (314,271) .. controls (314,272) and (321,299) .. (321,299) ;
\draw [shift={(321,299)}, rotate = 0] [color={rgb, 255:red, 74; green, 144; blue, 226 }  ,draw opacity=1 ][fill={rgb, 255:red, 74; green, 144; blue, 226 }  ,fill opacity=1 ][line width=0.75]      (0, 0) circle [x radius= 2.01, y radius= 2.01]   ;
%Straight Lines [id:da9329786342335934] 
\draw [color={rgb, 255:red, 74; green, 144; blue, 226 }  ,draw opacity=1 ]   (301,122.5) -- (301,149) ;
%Straight Lines [id:da8285180706440122] 
\draw [color={rgb, 255:red, 74; green, 144; blue, 226 }  ,draw opacity=1 ]   (301,149) -- (314,171) ;
%Curve Lines [id:da932154492071075] 
\draw [color={rgb, 255:red, 74; green, 144; blue, 226 }  ,draw opacity=1 ]   (314,171) .. controls (314,172) and (321,199) .. (321,199) ;
\draw [shift={(321,199)}, rotate = 0] [color={rgb, 255:red, 74; green, 144; blue, 226 }  ,draw opacity=1 ][fill={rgb, 255:red, 74; green, 144; blue, 226 }  ,fill opacity=1 ][line width=0.75]      (0, 0) circle [x radius= 2.01, y radius= 2.01]   ;
%Straight Lines [id:da8981465589402037] 
\draw    (158,210) -- (473,210) ;
%Straight Lines [id:da13964882577967763] 
\draw    (158,110) -- (473,110) ;
%Straight Lines [id:da7782577256032471] 
\draw    (158,310) -- (473,310) ;
%Straight Lines [id:da8147704820966575] 
\draw    (158,410.5) -- (473,410.5) ;
%Straight Lines [id:da1615469051325389] 
\draw    (207,271) -- (207,299) ;
\draw [shift={(207,299)}, rotate = 90] [color={rgb, 255:red, 0; green, 0; blue, 0 }  ][fill={rgb, 255:red, 0; green, 0; blue, 0 }  ][line width=0.75]      (0, 0) circle [x radius= 2.01, y radius= 2.01]   ;
%Straight Lines [id:da8633923506113577] 
\draw    (287,171) -- (287,199) ;
\draw [shift={(287,199)}, rotate = 90] [color={rgb, 255:red, 0; green, 0; blue, 0 }  ][fill={rgb, 255:red, 0; green, 0; blue, 0 }  ][line width=0.75]      (0, 0) circle [x radius= 2.01, y radius= 2.01]   ;
%Straight Lines [id:da7544124197770699] 
\draw    (367,171) -- (367,199) ;
\draw [shift={(367,199)}, rotate = 90] [color={rgb, 255:red, 0; green, 0; blue, 0 }  ][fill={rgb, 255:red, 0; green, 0; blue, 0 }  ][line width=0.75]      (0, 0) circle [x radius= 2.01, y radius= 2.01]   ;
%Straight Lines [id:da988794067172716] 
\draw [color={rgb, 255:red, 0; green, 0; blue, 0 }  ,draw opacity=1 ]   (234,171) -- (227,199) ;
\draw [shift={(227,199)}, rotate = 104.04] [color={rgb, 255:red, 0; green, 0; blue, 0 }  ,draw opacity=1 ][fill={rgb, 255:red, 0; green, 0; blue, 0 }  ,fill opacity=1 ][line width=0.75]      (0, 0) circle [x radius= 2.01, y radius= 2.01]   ;
%Straight Lines [id:da7568555184752452] 
\draw [color={rgb, 255:red, 74; green, 144; blue, 226 }  ,draw opacity=1 ] [dash pattern={on 3pt off 2.25pt}]  (234,171) -- (241,199) ;
\draw [shift={(241,199)}, rotate = 75.96] [color={rgb, 255:red, 74; green, 144; blue, 226 }  ,draw opacity=1 ][fill={rgb, 255:red, 74; green, 144; blue, 226 }  ,fill opacity=1 ][line width=0.75]      (0, 0) circle [x radius= 2.01, y radius= 2.01]   ;
%Straight Lines [id:da15957232406712618] 
\draw [color={rgb, 255:red, 0; green, 0; blue, 0 }  ,draw opacity=1 ]   (234,271) -- (227,299) ;
\draw [shift={(227,299)}, rotate = 104.04] [color={rgb, 255:red, 0; green, 0; blue, 0 }  ,draw opacity=1 ][fill={rgb, 255:red, 0; green, 0; blue, 0 }  ,fill opacity=1 ][line width=0.75]      (0, 0) circle [x radius= 2.01, y radius= 2.01]   ;
%Straight Lines [id:da0701405917827631] 
\draw [color={rgb, 255:red, 74; green, 144; blue, 226 }  ,draw opacity=1 ] [dash pattern={on 3pt off 2.25pt}]  (234,371.5) -- (241,400.5) ;
\draw [shift={(241,400.5)}, rotate = 76.43] [color={rgb, 255:red, 74; green, 144; blue, 226 }  ,draw opacity=1 ][fill={rgb, 255:red, 74; green, 144; blue, 226 }  ,fill opacity=1 ][line width=0.75]      (0, 0) circle [x radius= 2.01, y radius= 2.01]   ;
%Straight Lines [id:da6686391427254832] 
\draw [color={rgb, 255:red, 74; green, 144; blue, 226 }  ,draw opacity=1 ] [dash pattern={on 3pt off 2.25pt}]  (234,271) -- (241,300) ;
\draw [shift={(241,300)}, rotate = 76.43] [color={rgb, 255:red, 74; green, 144; blue, 226 }  ,draw opacity=1 ][fill={rgb, 255:red, 74; green, 144; blue, 226 }  ,fill opacity=1 ][line width=0.75]      (0, 0) circle [x radius= 2.01, y radius= 2.01]   ;
%Straight Lines [id:da01111012870932937] 
\draw    (287,271) -- (287,299) ;
\draw [shift={(287,299)}, rotate = 90] [color={rgb, 255:red, 0; green, 0; blue, 0 }  ][fill={rgb, 255:red, 0; green, 0; blue, 0 }  ][line width=0.75]      (0, 0) circle [x radius= 2.01, y radius= 2.01]   ;
%Straight Lines [id:da38011698043278286] 
\draw    (367,271) -- (367,299) ;
\draw [shift={(367,299)}, rotate = 90] [color={rgb, 255:red, 0; green, 0; blue, 0 }  ][fill={rgb, 255:red, 0; green, 0; blue, 0 }  ][line width=0.75]      (0, 0) circle [x radius= 2.01, y radius= 2.01]   ;
%Straight Lines [id:da8763634832091234] 
\draw [color={rgb, 255:red, 0; green, 0; blue, 0 }  ,draw opacity=1 ] [dash pattern={on 0.75pt off 22.5pt}]  (394.49,272.95) -- (400.51,297.05) ;
\draw [shift={(401,299)}, rotate = 75.96] [color={rgb, 255:red, 0; green, 0; blue, 0 }  ,draw opacity=1 ][line width=0.75]      (0, 0) circle [x radius= 3.02, y radius= 3.02]   ;
\draw [shift={(394,271)}, rotate = 75.96] [color={rgb, 255:red, 0; green, 0; blue, 0 }  ,draw opacity=1 ][line width=0.75]      (0, 0) circle [x radius= 3.02, y radius= 3.02]   ;
%Straight Lines [id:da995345323216883] 
\draw [color={rgb, 255:red, 0; green, 0; blue, 0 }  ,draw opacity=1 ]   (394,371.5) -- (387,399.5) ;
\draw [shift={(387,399.5)}, rotate = 104.04] [color={rgb, 255:red, 0; green, 0; blue, 0 }  ,draw opacity=1 ][fill={rgb, 255:red, 0; green, 0; blue, 0 }  ,fill opacity=1 ][line width=0.75]      (0, 0) circle [x radius= 2.01, y radius= 2.01]   ;
%Straight Lines [id:da08117620827187522] 
\draw    (367,371.5) -- (367,399.5) ;
\draw [shift={(367,399.5)}, rotate = 90] [color={rgb, 255:red, 0; green, 0; blue, 0 }  ][fill={rgb, 255:red, 0; green, 0; blue, 0 }  ][line width=0.75]      (0, 0) circle [x radius= 2.01, y radius= 2.01]   ;
%Straight Lines [id:da9932090518112906] 
\draw    (207,371.5) -- (207,399.5) ;
\draw [shift={(207,399.5)}, rotate = 90] [color={rgb, 255:red, 0; green, 0; blue, 0 }  ][fill={rgb, 255:red, 0; green, 0; blue, 0 }  ][line width=0.75]      (0, 0) circle [x radius= 2.01, y radius= 2.01]   ;
%Straight Lines [id:da017665069325508487] 
\draw [color={rgb, 255:red, 0; green, 0; blue, 0 }  ,draw opacity=1 ]   (394,371.5) -- (402,399.5) ;
\draw [shift={(402,399.5)}, rotate = 74.05] [color={rgb, 255:red, 0; green, 0; blue, 0 }  ,draw opacity=1 ][fill={rgb, 255:red, 0; green, 0; blue, 0 }  ,fill opacity=1 ][line width=0.75]      (0, 0) circle [x radius= 2.01, y radius= 2.01]   ;
%Straight Lines [id:da8523573477243299] 
\draw [color={rgb, 255:red, 0; green, 0; blue, 0 }  ,draw opacity=1 ]   (234,371.5) -- (227,399.5) ;
\draw [shift={(227,399.5)}, rotate = 104.04] [color={rgb, 255:red, 0; green, 0; blue, 0 }  ,draw opacity=1 ][fill={rgb, 255:red, 0; green, 0; blue, 0 }  ,fill opacity=1 ][line width=0.75]      (0, 0) circle [x radius= 2.01, y radius= 2.01]   ;
%Straight Lines [id:da04468757206517204] 
\draw  [dash pattern={on 1.5pt off 1.5pt}]  (297.11,422) -- (294.56,424.29) ;
\draw [shift={(294.56,424.29)}, rotate = 138.09] [color={rgb, 255:red, 0; green, 0; blue, 0 }  ][line width=0.75]    (-3.35,0) -- (3.35,0)(0,3.35) -- (0,-3.35)   ;
%Straight Lines [id:da15210239950783877] 
\draw  [dash pattern={on 1.5pt off 1.5pt}]  (297.11,422) -- (296.06,422.94) ;
\draw [shift={(294.56,424.29)}, rotate = 138.09] [color={rgb, 255:red, 0; green, 0; blue, 0 }  ][line width=0.75]      (0, 0) circle [x radius= 3.02, y radius= 3.02]   ;
%Curve Lines [id:da4071297256253843] 
\draw  [dash pattern={on 0.75pt off 7.5pt}]  (292.5,422.16) .. controls (298.47,428.39) and (296.33,426.16) .. (294.96,424.72) ;
\draw [shift={(294.56,424.29)}, rotate = 45.79] [color={rgb, 255:red, 0; green, 0; blue, 0 }  ][line width=0.75]      (0, 0) circle [x radius= 2.68, y radius= 2.68]   ;

%Flowchart: Connector [id:dp660807412373076] 
\draw   (231,423.5) .. controls (231,422.12) and (232.12,421) .. (233.5,421) .. controls (234.88,421) and (236,422.12) .. (236,423.5) .. controls (236,424.88) and (234.88,426) .. (233.5,426) .. controls (232.12,426) and (231,424.88) .. (231,423.5) -- cycle ;
%Flowchart: Connector [id:dp9370752766778178] 
\draw  [fill={rgb, 255:red, 0; green, 0; blue, 0 }  ,fill opacity=1 ] (171,423.35) .. controls (171,422.05) and (172.05,421) .. (173.35,421) .. controls (174.65,421) and (175.7,422.05) .. (175.7,423.35) .. controls (175.7,424.65) and (174.65,425.7) .. (173.35,425.7) .. controls (172.05,425.7) and (171,424.65) .. (171,423.35) -- cycle ;
%Flowchart: Connector [id:dp47260233645410354] 
\draw  [color={rgb, 255:red, 208; green, 2; blue, 27 }  ,draw opacity=1 ] (421,423.5) .. controls (421,422.12) and (422.12,421) .. (423.5,421) .. controls (424.88,421) and (426,422.12) .. (426,423.5) .. controls (426,424.88) and (424.88,426) .. (423.5,426) .. controls (422.12,426) and (421,424.88) .. (421,423.5) -- cycle ;
%Flowchart: Connector [id:dp7352686392704819] 
\draw  [color={rgb, 255:red, 208; green, 2; blue, 27 }  ,draw opacity=1 ][fill={rgb, 255:red, 208; green, 2; blue, 27 }  ,fill opacity=1 ] (354,423.35) .. controls (354,422.05) and (355.05,421) .. (356.35,421) .. controls (357.65,421) and (358.7,422.05) .. (358.7,423.35) .. controls (358.7,424.65) and (357.65,425.7) .. (356.35,425.7) .. controls (355.05,425.7) and (354,424.65) .. (354,423.35) -- cycle ;

% Text Node
\draw (162,142.4) node [anchor=north west][inner sep=0.75pt]    {$Z_{t}$};
% Text Node
\draw (429,142.4) node [anchor=north west][inner sep=0.75pt]    {$Z_{t}{}_{+}{}_{1}$};
% Text Node
\draw (162,341.9) node [anchor=north west][inner sep=0.75pt]    {$X_{t}$};
% Text Node
\draw (429,341.9) node [anchor=north west][inner sep=0.75pt]    {$X_{t}{}_{+}{}_{1}$};
% Text Node
\draw (162,241.4) node [anchor=north west][inner sep=0.75pt]    {$Y_{s}$};
% Text Node
\draw (429,241.4) node [anchor=north west][inner sep=0.75pt]    {$Y_{s}{}_{+}{}_{1}$};
% Text Node
\draw (273,344.5) node [anchor=north west][inner sep=0.75pt]  [font=\small] [align=left] {(no check)};
% Text Node
\draw (202,93) node [anchor=north west][inner sep=0.75pt]  [font=\small] [align=left] {Step 1};
% Text Node
\draw (282,93) node [anchor=north west][inner sep=0.75pt]  [font=\small] [align=left] {Step 2};
% Text Node
\draw (362,93) node [anchor=north west][inner sep=0.75pt]  [font=\small] [align=left] {Step 3};
% Text Node
\draw (181,417) node [anchor=north west][inner sep=0.75pt]  [font=\small] [align=left] {\CT};
% Text Node
\draw (241,417) node [anchor=north west][inner sep=0.75pt]  [font=\small] [align=left] {\CF};
% Text Node
\draw (302,417) node [anchor=north west][inner sep=0.75pt]  [font=\small] [align=left] {\PF};
% Text Node
\draw (364,417) node [anchor=north west][inner sep=0.75pt]  [font=\small] [align=left] {\ZCT};
% Text Node
\draw (434,417) node [anchor=north west][inner sep=0.75pt]  [font=\small] [align=left] {\ZCF};

\end{tikzpicture}